\pdfoutput=1
\documentclass[10pt,a4paper]{article}
\usepackage[utf8]{inputenc}
\usepackage{amsfonts, amsmath, amssymb, color, a4wide}
\usepackage[utf8]{inputenc} 
\usepackage[T1]{fontenc}    
\usepackage{url}            
\usepackage{booktabs}       
\usepackage{amsfonts}       
\usepackage{nicefrac}       
\usepackage{microtype}      
\usepackage{fancybox,framed}
\usepackage{lmodern}
\usepackage{amsfonts}
\usepackage{amsmath}
\usepackage{amssymb}
\usepackage{bm}
\usepackage{algorithm}
\usepackage{algorithmic}
\usepackage{mathtools}
\usepackage{xspace}
\usepackage{xcolor}
\usepackage{amsthm}
\usepackage[hidelinks]{hyperref}
\usepackage{cases}
\usepackage{natbib}
\usepackage{dsfont}
\usepackage{bbm}
\RequirePackage[shortlabels]{enumitem}
\usepackage{ stmaryrd }
\usepackage{ulem}
\usepackage{authblk}
\usepackage{subcaption}

\newtheorem{theorem}{Theorem}[section]

\newtheorem{proposition}[theorem]{Proposition}
\newtheorem{lemma}[theorem]{Lemma}
\newtheorem{corollary}[theorem]{Corollary}

\newtheorem*{remark}{Remarks}

\DeclarePairedDelimiter{\ceil}{\lceil}{\rceil}
\DeclarePairedDelimiter\abs{\lvert}{\rvert}

\DeclareMathOperator*{\argmax}{arg\,max}

\DeclareMathOperator{\kl}{kl}

\newtheoremstyle{proofsketchstyle}%
{}{}              
{\itshape}        
{}                
{\bfseries}       
{.}               
{ }               
{\thmname{#1}\thmnote{ (#3)}} 

\theoremstyle{proofsketchstyle}

\newcommand{\vertiii}[1]{{\left\vert\kern-0.25ex\left\vert\kern-0.25ex\left\vert #1 
		\right\vert\kern-0.25ex\right\vert\kern-0.25ex\right\vert}}

\DeclarePairedDelimiter\floor{\lfloor}{\rfloor}
\newcommand{\KL}{\textrm{KL}}

\newcommand{\bDelta}{\bm{\Delta}}
\newcommand{\bSigma}{\boldsymbol{\Sigma}}
\newtheorem*{theorem*}{Theorem}
\newtheorem*{OP}{Open Problem}

\newenvironment{keywords}{\paragraph{Keywords.}}{}
\newenvironment{ack}{\section*{Acknowledgments}}{}

\title{\vspace{-2.2cm}The Sampling Complexity of Condorcet Winner Identification in Dueling Bandits}

\author{%
	\begin{minipage}[t]{0.46\textwidth}\centering
		El Mehdi Saad\footnote{Equal contribution}\\\vspace{-0.35cm}
		UM6P College of Computing\\
		Rabat, Morocco\\
		\texttt{elmehdi.saad@um6p.ma}
	\end{minipage}\hfill
	\begin{minipage}[t]{0.50\textwidth}\centering
		Victor Thuot$^{*}$\\\vspace{-0.35cm}
		INRAE, MISTEA, Institut Agro, Univ. Montpellier\\
		Montpellier, France\\
		\texttt{victor.thuot@inrae.fr}
	\end{minipage}
	\vspace{1.2em}\\
	\centering
	Nicolas Verzelen\\
	INRAE, MISTEA, Institut Agro, Univ Montpellier\\
	Montpellier, France\\
	\texttt{nicolas.verzelen@inrae.fr}%
}
\date{}

\begin{document}
	\maketitle
	\begin{abstract}%
		We study best-arm identification in  large-scale stochastic dueling bandits under the sole assumption that a Condorcet winner exists, i.e., an arm that wins each noisy pairwise comparison with probability at least $1/2$. We introduce a new identification procedure that exploits the full gap matrix $\Delta_{i,j}=q_{i,j}-\tfrac12$ (where $q_{i,j}$ is the probability that arm $i$ beats arm $j$), rather than only the gaps between the Condorcet winner and the other arms.
		We derive high-probability, instance-dependent sample-complexity guarantees that (up to logarithmic factors) improve the best known ones by leveraging informative comparisons beyond those involving the winner.
		We complement these results with matching lower bounds that establish the optimality of our procedures in all regimes.  
		Overall, our results reveal the general form of the sampling complexity, characterized by a trade-off between the cost of locating informative entries and the verification cost required to achieve the desired confidence.
		In particular, this complexity  drastically differs from what is suggested by pure asymptotic results or by procedures that are tailored to Strongly Stochastic Transitive models.
	\end{abstract}
	
	{\small }

		\section{Motivation and High-Level Overview}\label{sec:intro}
	In many modern machine learning applications, obtaining trustworthy absolute feedback can be difficult, expensive, or systematically biased. By contrast, relative judgments are often easier to elicit and can be highly informative. This is especially apparent in information retrieval and recommendation systems, where users naturally compare two alternatives such as rankings or models rather than providing calibrated relevance scores \citep{joachims2007evaluating,Hofmann2016OnlineEF}.
	
	The dueling bandits framework formalizes this paradigm by allowing a learner to adaptively query pairs of arms and observe only a noisy binary outcome indicating which arm is preferred. At each round, the learner selects a pair of two arms $(i,j)\in[K]\times [K]$ and observes the outcome of their duel: the feedback is $1$ if arm $i$ is preferred to arm $j$, and $0$ otherwise. This observation is modeled as a Bernoulli random variable with unknown parameter $q_{i,j}\in[0,1]$. The collection of pairwise preference probabilities is represented by the matrix $\bm{Q}=(q_{i,j})_{i,j\in[K]}$. Since self-comparisons are uninformative and  preferences are anti-symmetric, namely $q_{i,j}=1-q_{j,i}$ for all $i, j \in [K]$, the unknown matrix $\bm{Q}$ satisfies a skew-symmetry condition. Equivalently, we define  the \emph{gap matrix} $\bm{\Delta}=(\Delta_{i,j})_{i,j\in[K]}$ by  $\Delta_{i,j}\coloneqq q_{i,j}-1/2$, which satisfies $\Delta_{i,j}=-\Delta_{j,i}$ and $\Delta_{i,i}=0$.
	
	Stochastic dueling bandits have been studied extensively under a variety of structural assumptions on $\bm{Q}$ and with multiple notions of optimality \citep{bengs2021preference,komiyama2015regret,falahatgar2017maximum,falahatgar2018limits,ren2020sample, jamieson2015sparse,zoghi2015copeland,haddenhorst2021identification}. Unlike in the classical multi-armed bandit setting, defining an ``optimal arm'' is not immediate, which has led to several competing winner definitions; see the survey of \citet{bengs2021preference}. In this work, we focus on instances where a distinguished arm $i^*\in[K]$ defeats, in expectation, every other arm, i.e., $q_{i^*,j}>1/2$ for all $j\in[K]\setminus\{i^*\}$.
	Such an arm is called a \textit{Condorcet winner} (CW) and is unique; we also refer to it as the optimal arm. Most existing work on dueling bandits assumes the existence of a CW \citep{zoghi2014relative,zoghi2015mergerucb,li2020mergedts,komiyama2015regret,chen2017dueling,saha2022versatile,saha2022optimal}, or even imposes the stronger requirement that the arms admit a total order \citep{yue2012k,yue2009interactively,chen2017dueling}. Alternative notions of optimality are discussed in Section~\ref{sec:related_w}.
	\paragraph{Objective: Condorcet winner identification.}
	Given $\delta\in(0,1)$, the learner must output the CW $i^*$ with probability at least $1-\delta$ by adaptively and sequentially choosing pairs $(i,j)$ to be compared and choosing a stopping time. We evaluate an algorithm by its (random) number of duels $N_\delta$, called the \emph{budget}, and we seek instance-dependent guarantees in terms of the centered gaps $\Delta_{i,j}=q_{i,j}-\tfrac12$, which encode both preference direction (e.g., $\Delta_{i,j}>0$ means $i$ beats $j$) and statistical difficulty.
	
	\paragraph{State-of-the-art.}
	Although CW identification has attracted quite a lot of attention \citep{komiyama2015regret,ailon2014reducing,chen2017dueling,saha2022versatile,pekoz2022dueling}, the optimal budget for this task remains poorly understood. 
	\citet{maiti2024near} developed an algorithm that exploits that the CW row is the unique one with only positive gaps. They obtained a high-probability guarantee of the budget of the order of $H_{\mathrm{cw}}(\delta)$ where
	\begin{equation}\label{eq:maiti}
		H_{\mathrm{cw}}(\delta):= \log(1/\delta)\sum_{i\neq i^*}\frac{1}{\Delta_{i^*,i}^2}\enspace .    
	\end{equation}
	In the specific scenario where the \emph{CW is the strongest opponent of every suboptimal arm}, that is,
	\[
	\tag{CW-SO}\label{eq:CW_is_the_strongest}
	i^* = \mathrm{argmin}_{j : \Delta_{i,j} < 0} \Delta_{i,j}\enspace , \forall i\neq i^* \enspace , 
	\]
	the condition~\eqref{eq:maiti} is matched by the lower bound from~\cite{haddenhorst2021identification} which implies that the sample complexity  $H_{\mathrm{cw}}(\delta)$ is optimal. Although~\eqref{eq:CW_is_the_strongest} is satisfied for Strong Stochastic Transitive (SST) models\footnote{SST assumes a total order such that, whenever
		$i$ is ranked above $j$ and $j$ above $k$, we have:
		$q_{i,k}\ge \max\{q_{i,j},q_{j,k}\}$.}, this condition is arguably quite strong and is not even guaranteed when there is a total order on the arms. In particular, the bound~\eqref{eq:maiti} can be overly conservative when some arm $i$ is nearly tied with $i^*$ as it largely ignores potentially informative comparisons among suboptimal arms.

	Prior to that, \citet{karnin2016verification} introduced a verification-based approach.  The expected budget of this procedure asymptotically satisfies
	
	\begin{equation}\label{eq:bound1}
		\mathbb{E}[N_\delta]\le c\log\left(\frac{1}{\delta}\right)\sum_{i\neq i^*}\min_{j: \Delta_{i,j}<0}\frac{1}{\Delta_{i,j}^2} + \tilde{O}\left[\sum_{i\neq i^*} \frac{1}{\Delta^2_{i^*,i}}+ \sum_{i\neq i^*}\sum_{j\neq i}\min\left( \frac{1}{\Delta^2_{i,j}}, \min_{j': \Delta_{i,j'}<0}\frac{1}{\Delta^2_{i,j'}}\right) \right] ~,
	\end{equation}
	where $c$ is a positive numerical constant and $\tilde{O}$ hides possible logarithmic dependencies. The $\log(1/\delta)$ term in the bound~\eqref{eq:bound1} interprets as the sum, over all non-CW arms $i$, of $\log(1/\delta)/ [\min_{j}\Delta_{i,j}]^{2}$ which is the minimal budget required to check whether the row $\Delta_{i,\cdot}$ is non-negative if an oracle provides to the learner the  information on the best opponent of $i$. Indeed, Theorem~5.2 of \cite{haddenhorst2021identification} establishes a lower bound of the order $\log(1/\delta)/ [\min_{j}\Delta_{i,j}]^{2}$. However, in a large-scale problem where $K$ is large  compared to $\log(1/\delta)$, the right-hand side term in~\eqref{eq:bound1} is sizable as it scales at least like $K^2/(\max_{i\neq j}\Delta_{i,j}^2)$ and its optimality is questionable.
	Altogether, apart from the small-scale asymptotic regime ($K$ fixed, $\delta\rightarrow 0$) or under the restrictive condition~\eqref{eq:CW_is_the_strongest}, the sample complexity remains to be understood. 
	
	\begin{OP}
		What is the true sampling complexity of CW identification and how does it depend on gap matrix $\bm{\Delta}$ beyond the rows and columns of the Condorcet winner?
	\end{OP}
	
	This question falls within \textit{structured} pure exploration, where the feedback is noisy but constrained by an underlying latent object (here, a skew-symmetric matrix possessing a positive row), so the goal is to exploit structure rather than estimate all entries. We emphasize that such problems require ideas and techniques that are at the crossroads of active learning and high-dimensional statistics. Related challenges arise in noisy payoff matrix games \cite{maiti2025open}, e.g., in pure Nash equilibrium identification.
	
	\paragraph{Contributions.} 
	Our main contributions are threefold: (i) we introduce new elimination-based algorithms for both the fixed budget and the fixed confidence settings and provide non-asymptotic guarantees, (ii) we establish matching lower bounds. (iii) Overall, this allows us to highlight the \emph{trade-offs and the multiple strategies that underlie CW identification}. For the sake of simplicity, we focus in the main text on the $\delta$-PAC setting; we also treat the fixed-budget framework and provide analogous guarantees in Appendices~\ref{sec:lb_fb} and~\ref{sec:proof_fb}.
	Our guarantees recover, up to logarithmic factors, the known optimal rate
	$H_{\mathrm{cw}}(\delta)$ in the favorable structured regimes such as~\eqref{eq:CW_is_the_strongest}. In contrast, when this restrictive assumption fails, our bounds can be drastically smaller than the existing $H_{\mathrm{cw}}(\delta)$-type guaranties of~\eqref{eq:maiti} or the bound~\eqref{eq:bound1} of~\cite{karnin2016verification} by adapting to the full gaps matrix.

	At a high level, our elimination-based procedure (FC-CWI), described in Algorithms~\eqref{algo:3} and~\eqref{algo:fc},  
	
	iteratively scores the current candidates for CW using subroutines that (i) \textit{search} in the gap matrix $\bm{\Delta}$ for informative comparisons and exploit its skew-symmetric structure, and (ii) \textit{estimate} the signs of the discovered entries with sufficient accuracy. Candidates are then ranked by these scores and a constant fraction of arms is eliminated at each round. Our analysis reveals a delicate dependence on the full gap matrix $\bDelta$. Indeed, providing evidence that $i^*$ is the CW either amounts to showing that all the CW gaps $\{\Delta_{i^*,i}\}_{i\neq i^*}$ are positive or amounts to showing that all arms $i\neq i^*$ are not CW. The evidence of sub-optimality for a given  arm $i\neq i^*$ is governed both by the number of negative entries in its row, $K_{i;<0}:=\abs{\{j:\Delta_{i,j}<0\}}$, and  by the magnitudes of these gaps, denoted by the ordered values $\Delta_{i,(1)}\le \cdots \le \Delta_{i,(K_{i;<0})}<0$. For each $i\neq i^*$, fix an integer $s_i\le K_{i;<0}$ and write $\bm{s}=(s_1,\dots,s_K)$. The following results will involve a trade-off in $\bm{s}$.
	Our analysis decomposes the complexity into $H_{\text{cw}}(\delta)$ --see~\eqref{eq:maiti}--, which corresponds to the cost of separating $i^*$ from  every competitor only relying  on duels with $i^*$, as well as two new components:
	\begin{itemize}
		
		\item \textit{Exploration/Selection cost.} This term quantifies the effort required to select a negative entry whose absolute value is at least $|\Delta_{i,(s_i)}|$ in
		each suboptimal row.
		\begin{equation}\label{eq:definition:Hexplore}
			H_{\text{explore}}(\bm{s},\delta) :=\max_{i \neq i^*} \frac{K\log(1/\delta)}{s_i \Delta_{i,(s_i)}^2}+\sum_{i \neq i^*}\frac{K}{s_i \Delta_{i,(s_i)}^2}~,
		\end{equation}
		Note that the second term of the right-hand-side expression is independent of $\delta$ and accounts for the fact that looking for an entry at least $|\Delta_{i,(s_i)}|$ out of $K$ depends on both the number $s_i$ of such entries and the magnitude $|\Delta_{i,(s_i)}|$. The $\log(1/\delta)$-dependency only arises for a single arm $i\neq i^*$.

		\item \textit{Certification cost.} This term corresponds to the number of samples required to estimate the signs of the selected gaps (at the exploration step) at confidence level $1-\delta$
		\[
		H_{\text{certify}}(\bm{s},\delta) := \sum_{i \neq i^*}\frac{\log(1/\delta)}{\Delta_{i,(s_i)}^2} ~.
		\]
	\end{itemize}
	Our main upper bound shows that, with probability at least $1-\delta$, the budget $N_{\delta}$ of FC-CWI  satisfies
	\begin{equation}\label{eq:sample_comp}
		N_{\delta} \lesssim H_{\text{cw}}(\delta) \wedge \min_{\substack{(s_i)_{i \neq i^*} \\ \forall i,~s_i \le K_{i;<0}\wedge K/8}}
		\left\{ H_{\text{certify}}(\bm{s},\delta)+ H_{\text{explore}}(\bm{s},\delta)\right\}~,
	\end{equation}
	where the notation $\lesssim$ hides logarithmic factors in $K, (\Delta_{i,(1)})_{i\neq i^*}$ and a $\log\log(1/\delta)$ factor. Under the scenario \eqref{eq:CW_is_the_strongest}, our procedure still achieves budget smaller than $H_{\mathrm{cw}}(\delta)$ as in~\cite{maiti2024near} but also achieves better guarantees for other gap matrices $\bDelta$, where the budget is driven by the right-hand side in~\eqref{eq:sample_comp}. In the above infimum in~\eqref{eq:sample_comp}, the smaller the $s_i$'s are, the smaller $H_{\text{certify}}(\bm{s},\delta)$ is, but the exploration cost $H_{\text{explore}}(\bm{s},\delta)$ for localizing a good candidate can increase for small $s_i$'s. In the following, we denote $\bm s^*_{\bDelta}$ as the vector $(s^*_i)_{i\ne i^*}$ achieving the best trade-off in Equation~\eqref{eq:sample_comp}. We interpret $s^*_{\bDelta}$ as an effective sparsity of $\bDelta$, although it also depends on $\delta$. Importantly, our algorithm does not take  $\bm s^*_{\bDelta}$ as input and therefore \textit{automatically} achieves the best balance captured by \eqref{eq:sample_comp}.

	In Section~\ref{sec:lb_fc}, we provide matching lower bounds supporting the optimality of our budget in  all regimes with respect to $\delta$ and $K$, as well as for general forms of the matrix $\bDelta$. Importantly, this showcases that the exploration/certification trade-off unveiled in~\eqref{eq:sample_comp} is  unavoidable and intrinsic to the sample complexity of CW-identification. The benefit of our procedure is also illustrated on numerical experiments.

	\paragraph{Emblematic regimes.} As the sample complexity is quite intricate in the general case, we discuss some specific regimes to emphasize key phenomena. 
	\begin{itemize}
		\item \emph{CW has small gaps}. If some entries of the $i^*$-th row of $\Delta$ are small, then $H_{\text{cw}}(\delta)$ is not the minimum in~\eqref{eq:sample_comp} and our sample complexity is of the order of $H_{\text{certify}}(\bm{s}^*_{\bDelta},\delta)+H_{\text{explore}}(\bm{s}^*_{\bDelta},\delta)$ which  \emph{can be arbitrarily smaller} than the budget $H_{\text{cw}}(\delta)$ of~\cite{maiti2024near}.
		\item \emph{Large-scale/fixed probability regime}.  Let us consider $\delta$ as a constant and $K$ as large. 
		In~\eqref{eq:sample_comp}, when $H_{\text{cw}}(\delta)$ is not the minimum, then the sample complexity is of the order\footnote{To show this, we observe that $\max_{s_i=1,\ldots, K_{i; <0}} s_i \Delta^2_{i,(s_i)}\asymp \|\Delta^-_{i}\|_2^2$, see Lemma~\ref{lem:tech1}.} of
		\[
		\min\left(\sum_{i\ne i^*} \frac{1}{\Delta_{i^*,i}^2},  \sum_{i\neq i^*} \frac{K}{\|\Delta^{-}_{i}\|^2_2} \right) \ ,  
		\]
		where  $\Delta^-_{i,j}:=\min(\Delta_{i,j},0)$. In comparison to the complexity of \cite{karnin2016verification} in~\eqref{eq:bound1}, our sample complexity is always smaller and can be smaller by a factor as larger as $K$.

		\item \emph{Small probability regime}. This last observation is more subtle. Similarly to~\cite{karnin2016verification}, consider the asymptotic regime where $\log(1/\delta)$ goes to infinity, while $K$ and $\bDelta$ are fixed. 
		In this asymptotic, our lower and upper bounds on the $(1-\delta)$-quantile of the budget are of the form
		{\small\[
			\log\left(\frac{1}{\delta}\right)\inf_{\bm s}\left[ \sum_{i\neq i^*}  \frac{1}{\Delta_{i,(s_i)}^2}+  \max_{i\neq i^*}\frac{K}{s_i\Delta_{i,(s_i)}^2}\right]
			\enspace ,
			\]}
		
		whereas the bound in~\cite{karnin2016verification} on the expected budget only involves the smaller quantity $\sum_{i\neq i^*}  \frac{\log(1/\delta)}{\Delta_{i,(1)}^2}$. Since our bound is optimal in quantile, this establishes that there is a significant,  yet intrinsic, gap between guarantees in expectation and in quantile of the budget. Especially, when there is heterogeneity in the $\Delta_{i,(s_1)}$, this gap can be as large as a factor $K$. This phenomenon is central for the analysis of fixed-budget algorithms - see Section~\ref{sec:proof_fb}.
	\end{itemize}
	
	All the way through these two extreme regimes ($\delta$ fixed, $\delta\to 0$),  both~\eqref{eq:sample_comp} and the matching lower bound illustrate a trade-off between exploration and certification. Intuitively, when $\log(1/\delta)$ increases, the effective sparsity $s^*_{\bDelta}$ tends to decrease so the algorithm explores other arms more thoroughly to identify stronger opponents.

	\paragraph{Technical Innovations.} Our Algorithms~\ref{algo:3} and~\ref{algo:fc} are based on a new iterative scoring strategy that builds on the selection, for each 'active'  arm $i$, of a strong opponent as well  as the estimation of some quantile of the estimation $\Delta_{i,.}$. For that purpose, we need to introduce a new active quantile estimation algorithm achieving optimal $\epsilon$-error simultaneously for all $\epsilon$ --see Appendix~\ref{sec:intermediary}. Our main lower bounds use novel approaches and techniques as our aim is to lower bound the $(1-\delta)$ quantile of the budget. For that purpose, we reduce the problem to an active multiple testing problem of the existence of negative entries within a vector of size $K-1$, our lower bound introduce techniques that go beyond classical MAB lower bounds are at the crossroad of high-dimensional statistics and bandit theory.

	\paragraph{Organization.}
	Section~\ref{sec:upper_bounds_fc} presents our algorithms and fixed-confidence upper bounds; the fixed-budget guarantees are deferred to the appendix. Section~\ref{sec:lb_fc} gives lower bounds for the fixed-confidence setting. Section~\ref{sec:numerical_simulations} illustrates our results on numerical experiments.  Section~\ref{sec:discussion} concludes with implications and future directions, while related work is discussed in Section~\ref{sec:related_w}. All proofs and fixed-budget results are provided in the appendix.

	\section{Upper Bounds: Algorithm and Guarantees}\label{sec:upper_bounds_fc}
	
	This section presents our fixed-confidence identification procedure and the upper bound announced in Section~\ref{sec:intro}. The algorithm is built from a budgeted elimination-and-certification subroutine, denoted FB-CWI (Algorithm~\ref{algo:3}). Given a tentative budget scale $T$, a confidence level $\delta$, this subroutine performs one elimination run using $\mathcal{O}(T)$ duels and returns a candidate arm together with a Boolean certificate indicating whether the run can be safely stopped. The budget $T$ should be viewed only as a guessed complexity scale: if it is too small for the instance, the subroutine is allowed to fail to certify rather than output a final answer.
	
	The fixed-confidence algorithm FC-CWI (Algorithm~\ref{algo:fc}) removes the need to know this scale in advance by calling FB-CWI with geometrically increasing values of $T$ and stopping at the first successful certificate. Thus, the main algorithm is a genuine $\delta$-PAC procedure: the stopping time is determined by the certification step, while the budgeted subroutine only provides the candidate-generation and elimination mechanism. We describe this subroutine first, and then state the high-probability sample-complexity guarantee for the resulting doubled procedure. The corresponding fixed-budget framework and guarantees are deferred to the appendix.

	\textsc{FB-CWI} is an elimination procedure initialized with $A_1=[K]$: at each round $k$, it assigns a score $S_k(\alpha)$ to every active arm $\alpha\in A_k$, ranks the arms accordingly, and discards the bottom $1/8$ fraction (so $\abs{A_{k+1}}=\floor{7\abs{A_k}/8}$). Therefore, the number of rounds is $O(\log K)$.
	
	\noindent The core of \textsc{FB-CWI} is the score computation, whose purpose is to keep the CW ranked above the elimination threshold.
	At round $k$, we split a budget of order $T/\log(K)$ across the active set $A_k$. For each $\alpha\in A_k$, we devote one quarter of its share to search for a strong opponent by running Bracketed Sequential Halving (BSH), a variant of SH \citep{karnin2013almost}, with the minimization convention of Section~\ref{sec:bsh_appendix}, on the instance of the duels $\{(\alpha,\beta):\beta\in[K]\setminus\{\alpha\}\}$, yielding an opponent $\alpha^{(s)}$ with a small gap $\Delta_{\alpha,\alpha^{(s)}}$, and another quarter to estimate the gap $\Delta_{\alpha,\alpha^{(s)}}$ via an empirical mean; this estimate defines the strong-opponent component of the score. We call $\alpha^{(s)}$ `strong' because it is selected from all $K$ arms (not only from $A_k$): this tends to penalize sub-optimal arms more sharply, at the price of higher uncertainty due to the larger search space.
	\noindent Relying only on the strong-opponent term can be brittle: if the CW is nearly tied with some arm, the selected opponent may yield a gap estimate close to zero and provide little separation with the elimination threshold. We therefore add a `weak-opponent' term that yields adaptivity to larger gaps with the CW. More specifically, writing $\bm{\Delta}^{(k)}:=\bm{\Delta}_{A_k\times A_k}$, skew-symmetry implies that at least half of the entries of $\bm{\Delta}^{(k)}$ are non-positive, and a simple pigeon-hole type argument implies that at least $\abs{A_k}/4$ rows contain at least $\abs{A_k}/4$ off-diagonal non-positive entries (Lemma~\ref{lem:pure_tech3} in the appendix). Accordingly, for each $\alpha\in A_k$ we estimate a point  whose value lies between the $1/8$- and $1/4$-quantiles of the row $(\Delta_{\alpha,\beta})_{\beta\in A_k\setminus\{\alpha\}}$ (via \textsc{Range-Quantile}). This lower-tail statistic is typically negative for many sub-optimal arms pushing them into the bottom-$1/8$ region, while remaining positive for the CW and leveraging the fact that most of its gaps can still be large.
	\vspace{-0.2cm}
	\paragraph{Subroutines: \textsc{Bracketed Sequential Halving} and \textsc{Range-Quantile}.}
	Our score construction relies on two subroutines. For the strong-opponent search we use Bracketed Sequential Halving (BSH) \citep{zhao2023revisiting}, an anytime and parameter-free variant of SH, \citep{karnin2013almost}, designed for the data-poor regime, chosen for its adaptive guarantees on simple regret, which translate in our context into gap-dependent guarantees --see \cite{zhao2023revisiting} and Section~\ref{sec:bsh_appendix} of the appendix. For the weak-opponent choice, we introduce \hyperref[algo:inter]{Range-Quantile} (Algorithm~\ref{algo:inter}), a general fixed-budget procedure revealing a point in a prescribed quantile range:
	given $N$ arms with means $(\mu_i)_{i\in [N]}$ ordered as $\mu_{(1)}\le\cdots\le\mu_{(N)}$ and indices $d<u$, it returns an estimate $\hat t$ that falls between the $d$-th and $u$-th means (up to an additive error $\varepsilon$) with error probability decaying as $\exp(-\widetilde{\Theta}(\min\{\tfrac{d}{N},1-\tfrac{u}{N}\}\tfrac{(u-d)^2}{N^2}T\varepsilon^2))$ --see Theorem~\ref{thm:conj}). Importantly, \textsc{Range-Quantile} does not require $\varepsilon$ as input, and the same output satisfies a deviation bound for every $\epsilon$; in \textsc{FB-CWI} we instantiate it with $N=\abs{A_k}-1$, $d=\ceil{(N+1)/8}$ and $u=\ceil{(N+1)/4}$ and analyze it at accuracy $\varepsilon = \tfrac12\,\Delta_{i^*,(\lceil \abs{A_k}/8\rceil)}$ to obtain a value between the $1/8$- and $1/4$-quantiles of $(\Delta_{\alpha,\beta})_{\beta\in A_k\setminus\{\alpha\}}$ up to this additive error. When $N\le3$, these ranks satisfy $d=u=1$, and we use the fallback specified in Algorithm~\ref{algo:inter}. Note that \citet{maiti2024near} gives a fixed-confidence routine that, given $(\delta,\varepsilon)$, outputs, in their decreasing-order convention, a value in
	$[\mu_{(N/2)}-\varepsilon,\;\mu_{(N/4+1)}+\varepsilon]$ with probability at least $1-\delta$.
	Here, $\varepsilon$ is instance-dependent and unknown, which motivates our adaptive \textsc{Range-Quantile} subroutine that does not take $\varepsilon$ as input.
	
	\begin{samepage}
		Given a prescribed rank interval and a fixed sampling budget, quantile estimation must balance inspecting enough arms to locate the interval against measuring their means accurately. \textsc{Range-Quantile} adapts to this trade-off without a prescribed accuracy parameter: a Lepski-type comparison across sampling scales produces one estimate whose deviation probability is controlled for every $\varepsilon\in(0,1)$. The construction and proof are given in Appendix~\ref{sec:intermediary}.
		
	\end{samepage}

	\begin{algorithm}
		\caption{ FB-CWI + Certification \label{algo:3} }
		\begin{algorithmic}
			\STATE \textbf{Input}: Fixed budget ($T$), Certification($\delta, T, c$).
			\STATE If $T<8K\log_{8/7}(K)$, return $\texttt{False}$ and an arbitrary arm $I\in[K]$.
			\STATE $k \gets 1$, $A_1 \gets [K]$, $n \gets 1+\ceil*{\log_2\left(\frac{T}{8K\log(K)}\right)}$.
			\STATE $\phi_1, \phi_2 \gets \texttt{True}$
			\WHILE{$\abs{A_k}>1$}
			\STATE Let $B_k \gets \floor*{\frac{T}{\abs{A_k} \log_{8/7}(K)}}$.
			\FOR{$\alpha \in A_k$}
			\STATE \texttt{/* Finding a strong opponent */}
			\STATE $\bullet$ Run Bracketed SH as specified in Section~\ref{sec:bsh_appendix} with a budget $\ceil{B_k/4}$ and where the candidate arms are $\{ ( \alpha, \beta) \text{ for } \beta \in [K]\setminus \{\alpha\}\}$. Let $(\alpha,\alpha^{(s)})$ denote the output.
			\STATE $\bullet$ Query $\ceil{B_k/4}$ samples of $(\alpha; \alpha^{(s)})$ and compute the quantity $Z_k^{(s)}(\alpha)$ (corresponding to the empirical mean of the gaps). 
			\STATE  \texttt{/* Computing a weak score */}
			\STATE $\bullet$ Run  \hyperref[algo:inter]{Range-Quantile} on duels between $\alpha$ and arms in $A_k\setminus \{\alpha\}$, with a budget $\ceil{B_k/2}$ and quantiles $(\ceil{\abs{A_k}/8}, \ceil{\abs{A_k}/4})$ let $Z_k^{(w)}(\alpha)$ denote the output. 
			\ENDFOR
			\STATE Compute the scores $S_k(\alpha) = \min\{Z_k^{(s)}(\alpha), 0 \}+Z_k^{(w)}(\alpha)$ for each $\alpha \in A_k$.
			\STATE Rank the arms in $A_k$ following the scores $S_k(\cdot)$ and put in $A_{k+1}$ the top $\abs{A_k}-\ceil{\abs{A_k}/8}$ arms.
			\STATE \texttt{/* Check fixed confidence */}
			\STATE Let $\bar{\alpha}$ denote the arm ranked $\abs{A_k}-\ceil{\abs{A_k}/8}+1$ according to the scores $S_k(\cdot)$.
			\STATE $\phi_1 \gets \phi_1 \cdot \mathds{1}\left(S_k(\bar{\alpha}) < -\sqrt{\frac{2c\log(T)}{\ceil{B_k/4}}\log\left(8K^2\log_{8/7}(K)\log(T) \cdot \frac{n(n+1)}{\delta}\right)} \right)$ 
			\STATE  $k \gets k+1$. 
			\ENDWHILE
			\STATE Let $I$ denote the unique arm in $A_k$.
			\STATE {\small \texttt{/* Use $T$ queries to test the sign of gaps of $I$ (unique arm in $A_k$) at confidence $\delta$ */}}
			\STATE Run \hyperref[algo:4]{Test-CW} with inputs $(I,\delta,T )$ to check the sign of gaps of arm $I$, let $\phi_2$ denote its output.
			\STATE Return $\phi_1 \lor \phi_2$ and $I$. 
		\end{algorithmic}
	\end{algorithm}

	\begin{algorithm}
		\caption{FC-CWI  \label{algo:fc} }
		\begin{algorithmic}
			\STATE \textbf{Input}: $c$, confidence input $\delta$.
			\STATE $\varphi \gets \texttt{False}$, $T \gets 8K\log(K)$.
			\WHILE{$\neg \varphi$}
			\STATE Run \textsc{FB-CWI} procedure (Algorithm~\ref{algo:3}) with inputs $\delta, T, c$. Let $\varphi, I$ be its output.
			\STATE $T \gets 2 T$.  
			\ENDWHILE
			\STATE Return $I$.
			
		\end{algorithmic}
	\end{algorithm}
	\vspace{-0.2cm}
	\paragraph{From a budgeted subroutine to fixed confidence.}
	Algorithm~\ref{algo:3} is used as a certificate-driven subroutine: for a tentative budget scale $T$ and confidence level $\delta$, it spends $\mathcal{O}(T)$ duels and returns a candidate $I$ together with a success flag $\Psi=\phi_1\vee\phi_2$. The direct certificate $\phi_2$ verifies that $I$ is the CW, leading to the classical $H_{\mathrm{cw}}(\delta)$ regime. The elimination certificate $\phi_1$ instead checks that the elimination frontier remains safely negative, thereby exploiting informative comparisons among suboptimal arms. The analysis shows that, up to logarithmic factors, certification succeeds once
	\begin{equation}\label{eq:bound_fb}
		T \gtrsim H_{\mathrm{cw}}(\delta)
		\wedge \min_{\substack{(s_i)_{i\neq i^*}\\ s_i\in\{1,\dots,K_{i;<0}\}}} \left\{ H_{\text{certify}}(\bm{s}, \delta)+ H_{\text{explore}}(\bm{s}, \delta)\right\}~.
	\end{equation}
	Since these quantities are unknown, we run FB-CWI under a doubling schedule on $T$ and stop as soon as a budgeted verification succeeds. It uses at most $2T$ queries per stage (in Algorithm~\ref{algo:3}, $\phi_1$ is computed from the same samples as the scores, and $\phi_2$ runs $\textsc{Test-CW}$ with at most $T$ extra queries).
	A direct verification is to certify that the returned arm $I$ is Condorcet by testing $\Delta_{I,j}>0$ for all $j\neq I$ at level $1-\delta$, which, when $I=i^*$, costs $\sum_{j\neq I}\log(1/\delta)/\Delta_{I,j}^2$ up to logarithmic factors and matches the $H_{\text{cw}}(\delta)$ regime. When the second term in~\eqref{eq:bound_fb} is dominant, our identification relies instead on certifying sub-optimality: many sub-optimal arms have scores $S_k(\cdot)$ that are typically negative while the CW remains positive. This motivates the second verification $\phi_1$, which checks that the elimination frontier is on the negative side. We stop at the first success of $\phi_1$ or $\phi_2$, which guarantees $\delta$-correctness. The complete proof of Theorem~\ref{thm:fc2} is presented in Section~\ref{sec:proof_fc} of the appendix.
	
	\begin{theorem}\label{thm:fc2}
		There exists a constant $c_0$ such that the following holds for any $\delta \in (0,1/6)$. Under the assumption of the existence of a CW, the output of Algorithm~\ref{algo:fc}, denoted $\psi$, with input $(\delta, c)$ where $c\ge c_0$ satisfies: $\mathbb{P}\left(\psi_{\delta} \neq i^*\right) \le \delta$~.
		
		Moreover, the total number of queries $N_{\delta}$ satisfies, with probability at least $1-\delta$,
		\[
		N_{\delta} \le \tilde{c}\cdot \bigl( H_{\text{cw}}(\delta) \wedge \min_{\substack{(s_i)_{i \neq i^*} \\ \forall i\neq i^*,~s_i\in\{1,\dots,K_{i;<0}\}}}
		\bigl\{  H_{\text{certify}}(\bm{s}, \delta)+ H_{\text{explore}}(\bm{s}, \delta) \bigr\}\bigr)~,
		\]
		where $\tilde{c}$ is proportional to $c$ and hides logarithmic factors in $K$ and the instance-dependent complexity terms, including the $\log^6$ exploration warm-up factor, and a $\log\log(1/\delta)$ factor. The quantities $H_{\text{cw}}, H_{\text{certify}}, H_{\text{explore}}$ are defined in Section~\ref{sec:intro}~.
	\end{theorem}
	\begin{remark}[On the constant $c$ in the stopping rule]
		The stopping condition in Algorithm~\ref{algo:3} involves a numerical constant $c_0$  inherited from the high-probability analysis of Corollary~\ref{cor:quant}. This \emph{absolute} value of $c_0$ can be made explicit by tracking constants in the proof. Since we did not optimize numerical factors, we keep $c_0$ symbolic for readability.
	\end{remark}
	\paragraph{Guarantees: intuition.}
	The proof of Theorem~\ref{thm:fc2} analyzes one call of the budgeted subroutine FB-CWI at a tentative scale $T$. In the outer fixed-confidence algorithm, a call that does not certify is harmless: FC-CWI simply doubles $T$ and tries again. Thus, the key point is to show that once $T$ exceeds the relevant instance-dependent scale, the CW is not eliminated and one of the certificates succeeds with high probability. Conditionally on $i^*\in A_k$, the only way the elimination step can be harmful is if the CW score falls below the bottom-$1/8$ cutoff in some round. The analysis therefore
	controls the separation between the CW score and the elimination frontier across the $O(\log K)$ rounds.
	
	The weak-opponent term gives a baseline separation. At round $k$, the CW benefits from a positive lower-tail gap of order
	$\Delta_{i^*,(\ceil*{\abs{A_k}/8})}$, estimated with $B_k=\Theta(T/(\abs{A_k}\log K))$ samples. Concentration bounds, together with the \textsc{Range-Quantile} guarantee, yield an error of order $\exp(-\widetilde{\Theta}(B_k\Delta_{i^*,(\lceil \abs{A_k}/8\rceil)}^2))$. The worst round is controlled by
	\[
	\max_k \frac{|A_k|/8}{\Delta^2_{i^*,(\lceil |A_k|/8\rceil)}} \le \max_{i\in[K-1]}\frac{i}{\Delta_{i^*,(i)}^2} \le
	\sum_{i\neq i^*}\frac{1}{\Delta_{i,i^*}^2}~,
	\]
	which leads, after the final CW verification, to the classical $H_{\mathrm{cw}}(\delta)$ regime.
	
	The strong-opponent term yields the matrix-adaptive regime by actively finding negative witnesses for suboptimal arms. Fix integers $1\le s_i\le K_{i,<0}$. For each $i\neq i^*$, Bracketed Sequential Halving needs budget of order $K/(s_i\Delta_{i,(s_i)}^2)$ to locate an opponent whose gap is at most $\frac{7}{8}\Delta_{i,(s_i)}$, while certifying the sign of this gap costs $1/\Delta_{i,(s_i)}^2$, up to logarithmic factors. The aggregate cost of locating such witnesses gives to $\sum_{i\neq i^*}\frac{K}{s_i\Delta_{i,(s_i)}^2}$  whereas the high-confidence part is governed by $\sum_{i\neq i^*}\frac{1}{\Delta_{i,(s_i)}^2}+\max_{i\neq i^*}\frac{K}{s_i\Delta_{i,(s_i)}^2}$. 
	Thus,  certification succeeds once $T$ is larger than the bound stated in Theorem~\ref{thm:fc2}.

	\section{Instance Dependent Lower Bounds}\label{sec:lb_fc}
	
	Beyond the instance-dependent lower bound~\eqref{eq:bound1} from~\cite{haddenhorst2021identification} that we complement with a corresponding quantile-bound and restate as Proposition~\ref{thm:LB_exp_instance_first_regime}, we establish in Theorem~\ref{thm:LB_hp_instance} the first lower bound that highlight the intrinsic cost of detecting informative entries for CWI. In Appendix~\ref{sec:lb_fb}, we also derive fixed-budget minimax lower bounds.
	Denote by $\mathbb{D}_{\mathrm{cw}}$ the class of dueling bandit environments that admits a CW\footnote{The  restriction to gaps in away from $-1/2$ and $1/2$ is standard in lower bounds with Bernoulli rewards.} :
	\begin{equation}\label{eq:def_acw}
		\mathbb{D}_{\mathrm{cw}}
		:= \Bigl\{
		\bDelta \in [-\tfrac14,\tfrac{1}{4}]^{K \times K} :
		\bDelta=- \bDelta^T \text{ and }\exists\, i^* \in [K]~\text{such that}~ \forall j\ne i^*,~ \Delta_{i^*,j} > 0
		\Bigr\}\enspace .
	\end{equation}
	We say that an algorithm $A$ is $\delta$-correct for CW identification if, for any $\bDelta\in \mathbb{D}_{\mathrm{cw}}$, it identifies $i^*$ with error probability at most $\delta$, that is, $\mathbb{P}_{\bDelta,A}(\hat{i}\ne i^*(\bDelta))\leqslant \delta$, where $\mathbb{P}_{\bDelta,A}$ denotes the probability\footnote{ denote $\mathbb{E}_{\bDelta,A}$ for the corresponding expectation} induced by the interaction between $A$ and the environment with gap matrix $\bDelta$. 
	
	\begin{proposition}\label{thm:LB_exp_instance_first_regime}
		Let $K \geqslant 2$ and $\delta \in (0,1/6)$ and consider any gap matrix $\bDelta\in\mathbb{D}_{\mathrm{cw}}$. For any algorithm $A$ that is $\delta$-correct on $\mathbb{D}_{\mathrm{cw}}$, the budget $N_{\delta}$ satisfies
		\begin{equation} \label{eq:LB_exp_instance_first_regime}
			\mathbb{E}_{\bDelta,A}[N_{\delta}]
			\;\geqslant\; \frac{1}{4} \sum_{i\neq i^*} \frac{\log\!\left(1/(4\delta)\right)}{\Delta_{i,(1)}^2} \; \text{, and } \quad \mathbb{P}_{\boldsymbol{\Delta},A}\left(N_{\delta}\geqslant\frac{1}{3}\sum_{i\neq i^*}\frac{\log\left(1/(6\delta)\right)}{\Delta_{i,(1)}^2}\right)\geqslant\delta.~.
		\end{equation}
	\end{proposition}
	The lower Bound of the expected budget is a particular case of Theorem 5.2 in~\cite{haddenhorst2021identification}. We complement it here with a quantile-bound. This proposition certifies optimality of the complexity $H_{\mathrm{cw}}(\delta)$ in the CW--SO regime. Indeed, under \eqref{eq:CW_is_the_strongest}, the strongest negative entry of every suboptimal row is the comparison with the Condorcet winner, so $\Delta_{i,(1)}=\Delta_{i,i^\star}=-\Delta_{i^\star,i}$. Hence the lower bound in Proposition~\ref{thm:LB_exp_instance_first_regime} is, up to constants, $\sum_{i\neq i^\star}\log(1/\delta)/\Delta_{i^\star,i}^2 =H_{\mathrm{cw}}(\delta)$. 

	To characterize the optimality of~\eqref{eq:sample_comp}, we establish a lower bound on the $\delta$-quantile of any algorithm under a local class of instances.   The budget condition~\eqref{eq:sample_comp} only depends on the gap matrix $\bDelta$ through three vectors: (i) the row $i^*$ of the CW $\Delta_{i^*,\cdot}$, (ii) the effective sparsity $\bm{s}^*_{\boldsymbol{\Delta}}$, and (iii) the gaps at the sparsity level $\bm{s}^*$: $(\Delta_{i,(s^*_i)})_{i\neq i^*}$.  Given any gap matrix $\bDelta$, we define the collection $\mathbb{D}(\boldsymbol{\Delta})$ of gap matrices $\tilde{\bDelta}$ that leave the Condorcet winner $i^*_{\boldsymbol{\Delta}}$, the effective sparsity $\bm{s}^*_{\boldsymbol{\Delta}}$, and the gaps $(\Delta_{i,(s^*_i)})_{i\neq i^*}$ unchanged 
	\begin{equation}\label{def:minimax_class}
		\mathbb{D}(\boldsymbol{\Delta}):=\{\tilde{\boldsymbol{\Delta}} \text{  s.t. } \;i^*_{\tilde{\boldsymbol{\Delta}}}=i^*_{\boldsymbol{\Delta}},\ \bm{s}^*_{\tilde{\boldsymbol{\Delta}}}=\bm{s}^*_{\boldsymbol{\Delta}},\ (\tilde{\Delta}_{i,(s^*_i)})_{i\neq i^*}=(\Delta_{i,(s^*_i)})_{i\neq i^*}\}\enspace .
	\end{equation}
	
	\begin{theorem}\label{thm:LB_hp_instance}
		Let $A$ be a $\delta$-correct algorithm for CW identification, with $\delta \leqslant 1/12$, and let $\bDelta\in \mathbb{D}_{\mathrm{cw}}$.
		Assume that $\bDelta$ has no ties, that is, $\forall i\ne j$, $\Delta_{i,j}\ne 0$. For this matrix $\bDelta$, one can construct a matrix $\tilde \bDelta$ by permuting the entries of $\bDelta$ in such a way that $\tilde \bDelta\in \mathbb D(\bDelta)$, and such that
		\begin{equation}\label{eq:lower_bound_Hexplore}
			\mathbb{P}_{\tilde{\bDelta},A}\!
			\left(
			N_{\delta}
			\;\geqslant\;
			\frac{1}{3}\,
			\max_{i\ne i^*}
			\frac{K_{i;< 0}}{\|\Delta_i^-\|^2}
			\log\!\left(\frac{1}{6\delta}\right)
			\,\vee\,
			\frac{1}{37 \log(2K)}
			\sum_{i \ne i^*}
			\frac{K_{i;< 0}}{\|\Delta_i^-\|^2}
			\right) 
			\;\geqslant\; \delta \enspace.
		\end{equation}
		Moreover, for all $i\neq i^*$, the rows satisfy $(\tilde{\Delta}_{i,(j)})_{j\le K_{i;<0}}=(\Delta_{i,(j)})_{j\le K_{i;<0}}$, i.e., $\tilde{\Delta}_{i,\cdot}$ and $\Delta_{i,\cdot}$ share the same $K_{i;<0}$ negative entries, up to permutation.
	\end{theorem}
	\begin{remark}
		While the bound~\eqref{eq:LB_exp_instance_first_regime} from Proposition~\ref{thm:LB_exp_instance_first_regime} captures the cost of certification, the lower bound~\eqref{eq:lower_bound_Hexplore} captures the intrinsic need of exploration. Indeed, from the classical bound of Lemma~\ref{lem:tech1}, we have $\min_{\bm{s}} H_{\text{explore}}(\bm{s},\delta)= \max_{i\ne i^*} K\log(1/\delta)/\|\Delta^-_{i}\|^2+\sum_{i\ne i^*}K/\|\Delta^-_{i}\|^2$,  up to a  factor $\log(2K)$. 
		Importantly, the two  lower bounds in Proposition~\ref{thm:LB_exp_instance_first_regime} and  Theorem \ref{thm:LB_hp_instance} imply the following.
	\end{remark}
	
	\begin{corollary}\label{coro:Lower_bound_minimax}
		Let $A$ be a $\delta$-correct algorithm for CW identification, with $\delta \leqslant 1/12$, and let $\bDelta\in \mathbb{D}_{\mathrm{cw}}$. Then, one can construct $\tilde\bDelta\in \mathbb{D}(\bDelta)$, such that 
		\begin{equation}
			\mathbb{P}_{\tilde{\boldsymbol{\Delta}},A}\left(N_{\delta}\gtrsim H_{\text{certify}}(\bm{s}^*,\delta)+H_{\text{explore}}(\bm{s}^*,\delta)\right)\geqslant\delta ~,
		\end{equation}
		where $\gtrsim$ hides a log term in $K$ and a numerical constant. 
	\end{corollary}
	
	This corollary establishes, up to logarithmic factors, the local minimax optimality of $H_{\mathrm{certify}}(\bm{s}^*,\delta)+H_{\mathrm{explore}}(\bm{s}^*,\delta)$ over $\mathbb{D}(\bDelta)$. By definition, the hard instance $\tilde{\bDelta}\in\mathbb{D}(\bDelta)$ has the same Condorcet winner $i^*$, effective sparsities $\bm{s}^*$, and selected gaps $(\Delta_{i,(s_i^*)})_{i\neq i^*}$ as $\bDelta$. Its construction first builds an auxiliary matrix with these parameters and then applies the permutation argument of Theorem~\ref{thm:LB_hp_instance}.

	\paragraph{Proof Sketch Theorem~\ref{thm:LB_hp_instance}.}
	The proof reduces CW identification to a collection of active signal-detection tasks: for each non-winner row, the learner must certify the existence of at least one negative entry under a global error constraint. The key technical novelty is an adversarial permutation construction that hides the locations of negative gaps and reduces the problem to multiple hypothesis testing; this yields the complexity term $\frac{K_{i;<0}}{\|\Delta_i^-\|^2}\log(1/\delta)$ via active signal-detection lower-bound arguments. Finally, the two terms in~\eqref{eq:lower_bound_Hexplore} are obtained through new multiple-hypothesis techniques.

	\section{Numerical simulations}\label{sec:numerical_simulations}
	
	We compare FC-CWI with two standard baselines: MidSearch~\citep{maiti2024near}
	and Explore--Verify~\citep{karnin2016verification}. Our experiments target two
	predictions of the theory. First, we vary the strength $\rho$ of informative
	suboptimal-suboptimal comparisons, while keeping the direct CW gaps fixed, to
	isolate the benefit of exploiting the full gap matrix. Second, we study the
	scaling with the number of arms $K$ in the same sparse informative regime. Each
	method is run with a fixed numerical prefactor, chosen once to compensate for
	conservative theoretical constants and then kept unchanged across all main-text
	experiments; details and robustness checks are given in
	Appendix~\ref{sec:appendix_sim_extra}.
	
	All instances are sparse skew-symmetric gap matrices with arm $0$ as
	the Condorcet winner. The CW beats every other arm with margin at least
	$\gamma$, while the suboptimal block contains an $\alpha$ fraction of nonzero
	entries per row, with gaps equal to $\pm\rho$ and all remaining entries set to
	zero. Thus, $\gamma$ controls the difficulty of direct CW verification, whereas
	$\rho$ controls the usefulness of suboptimal-suboptimal comparisons. The precise
	construction is deferred to Appendix~\ref{sec:appendix_sim_extra}.
	
	We fix $K=64$, $\alpha=0.3$, and $\delta=0.1$,
	and vary $\rho$ in the range $[0.2, 0.45]$ with step $0.05$.
	For the large-scale diagnostic, we fix $\gamma=0.3$, $\rho=0.4$, and
	$\alpha=0.3$, and vary $K\in\{100,200,\ldots,900\}.$
	Figure~\ref{fig:main_experiments} summarizes the two diagnostics. In the left
	panel, FC-CWI remains more sample-efficient than MidSearch for different $\rho$ values. Explore--Verify also improves as $\rho$ increases and can
	become competitive in this small-scale setting, where finding and verifying
	strong opponents is cheap. This behavior is consistent with the
	theory: the cost of Explore--Verify can be favorable at small $K$, but its
	sample complexity carries a stronger dependence on $K$. This is
	visible in the right panel, where, as $K$ grows, Explore--Verify scales much
	faster than FC-CWI, while FC-CWI keeps the most favorable query growth on the
	logarithmic scale.
	
	\begin{figure}[t]
		\centering
		\begin{subfigure}{0.47\linewidth}
			\centering
			\includegraphics[width=\linewidth]{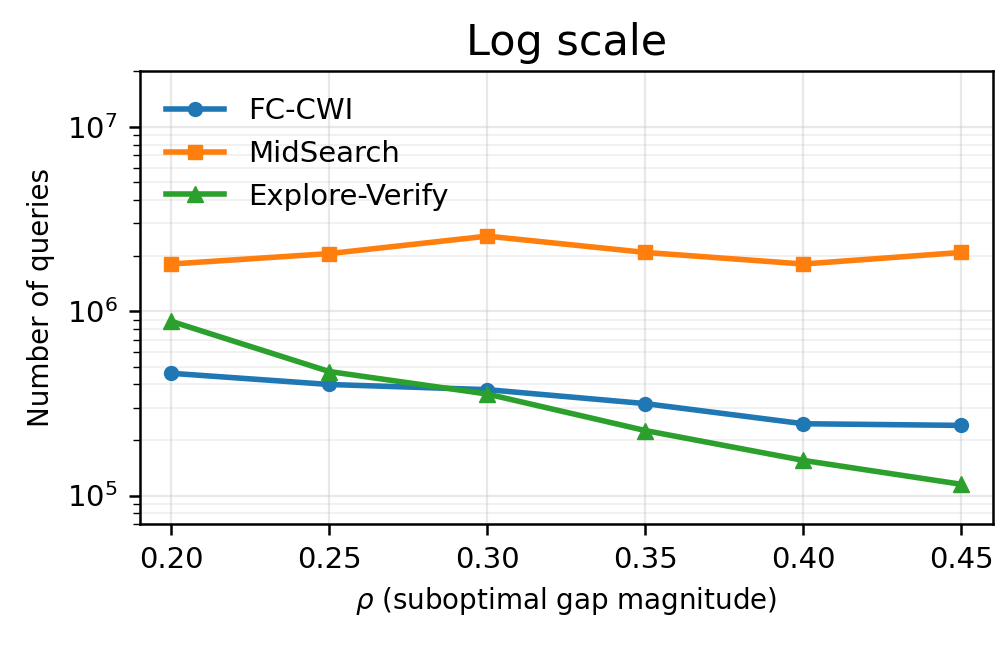}
			\caption{Informativeness sweep.}
		\end{subfigure}
		\hfill
		\begin{subfigure}{0.49\linewidth}
			\centering
			\includegraphics[width=\linewidth]{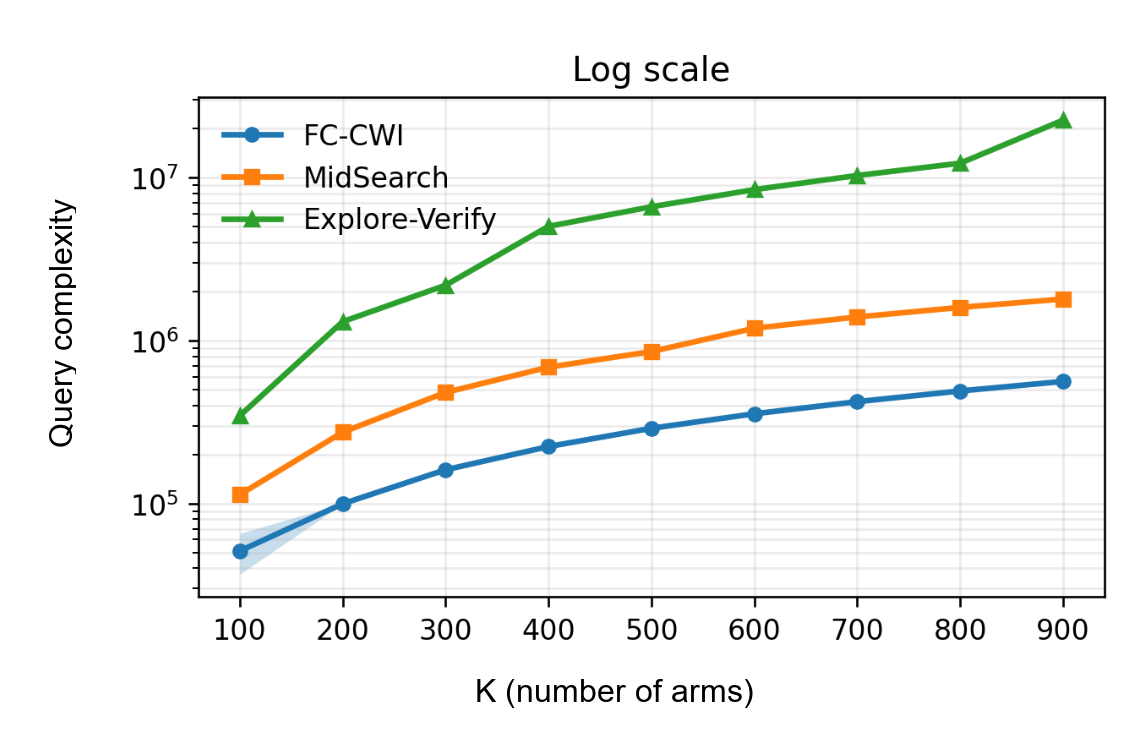}
			\caption{Large-scale behavior.}
		\end{subfigure}
		\caption{
			Left: query complexity in the sparse gaps matrix
			model with $K=64$, $\alpha=0.3$, and $\delta=0.1$, as the suboptimal gap
			magnitude $\rho$ varies. Right: query complexity as a function of $K$ in the
			sparse model with $\gamma=0.3$, $\rho=0.4$, and $\alpha=0.3$,
			shown on a logarithmic scale. 
		}
		\label{fig:main_experiments}
	\end{figure}
	
	\section{Discussion}\label{sec:discussion}
	
	In this manuscript, we consider $\delta$-PAC and fixed-budget Condorcet-winner
	identification in stochastic dueling bandits under the sole assumption that a CW
	exists. We derive instance-dependent, high-probability sample-complexity
	guarantees that exploit the full gap matrix, and complement them with new lower
	bounds covering both large-scale regimes, where $K$ is large, and small
	probability regimes, where $\delta$ is small. We further improve over the
	state-of-the-art under additional structural assumptions such as weak stochastic
	transitivity (WST), both on the upper- and lower-bound sides; see
	Section~\ref{sec:related_w}.
	
	More broadly, CWI is a structured pure-exploration problem over a latent matrix,
	with close connections to noisy payoff matrix games and Nash equilibrium
	identification~\citep{zhao2023revisiting, maiti2024near, maiti2025query,
		ito2025instance}. We believe that the exploration-certification decomposition
	and the tail-oriented lower-bound viewpoint may help clarify analogous
	expectation-versus-high-probability separations in equilibrium learning, whose
	optimal instance-dependent query complexity remains largely open; see
	\citet{maiti2025open}.

	\begin{ack}
		This work has partially been supported by ANR-21-CE23-0035 (ASCAI, ANR).
	\end{ack}
	\bibliographystyle{plainnat}
	\bibliography{bib_db}
	\newpage

	
	\appendix
	
	\section{Comparison to related framework and hypotheses in dueling bandits}\label{sec:related_w}
	
	\paragraph{Total order implies a Condorcet winner.}
	Weak stochastic transitivity (WST) is a standard sufficient condition for a latent total order: for any $i,j,k\in[K]$, $\Delta_{i,j}\ge0$ and $\Delta_{j,k}\ge0$ imply $\Delta_{i,k}\ge0$. Thus (up to tie-breaking) preferences are transitive and admit a total order, whose maximal element is a Condorcet winner (CW). CW/optimal-arm identification in such total-order regimes is studied in, e.g., \cite{falahatgar2017maximum,falahatgar2018limits}, which provide worst-case guarantees under WST (notably an $\mathcal{O}(K\varepsilon^{-2}\log(K/\delta))$ bound for $(\varepsilon,\delta)$-maxing, i.e., returning an arm $I$ such that $\mathbb{P}(\Delta_{I,i^*}\le -\varepsilon)\le \delta$)). Since WST is strictly stronger than merely assuming the existence of a CW, our CW-based bounds directly apply under WST. In particular, our bound for FC-CW1 in Theorem~\ref{thm:fc2} improves over the state-of-the art under WST without relying on that assumption. Furthermore, the proofs of Proposition~\ref{thm:LB_exp_instance_first_regime} and Theorem~\ref{thm:LB_hp_instance} also yield lower bounds for algorithms that are $\delta$-correct over the class of WST instances without ties. Indeed, the specific permutation construction used in Theorem~\ref{thm:LB_hp_instance} preserves every pairwise sign and hence WST. A stronger assumption is strong stochastic transitivity (SST), which adds magnitude constraints (for ordered $i\succ j\succ k$, $\Delta_{i,k}\ge \max\{\Delta_{i,j},\Delta_{j,k}\}$); in particular, SST implies a CW and ensures that each suboptimal arm attains its largest loss against the CW. Instance-dependent sample complexity under SST (often with mild regularity such as STI) is characterized in \cite{falahatgar2017maximum,ren2020sample}, and our upper bounds recover these guarantees (of the order $H_{\text{cw}}(\delta)$) up to logarithmic factors without even relying on the SST assumption. 
	Finally, parametric random-utility models such as Bradley--Terry--Luce, Plackett--Luce, and Thurstone \citep{bradley1952rank,luce1959individual,plackett1975analysis,thurstone2017law} are more restrictive than SST and therefore fall within our scope; in particular, our bounds recover existing guarantees for BTL (see, e.g., \cite{ren2020sample}). See the survey of \citet{bengs2021preference} for a broader overview.
	
	\paragraph{Works on other types of winners.}
	Beyond the CW objective \citep{komiyama2015regret,ailon2014reducing,chen2017dueling,saha2022versatile,pekoz2022dueling}, alternative notions include the \textit{Borda winner} \citep{chen2020combinatorial,jamieson2015sparse} (maximizing average pairwise advantage) and the \textit{Copeland winner} \citep{zoghi2015copeland,komiyama2016copeland} (maximizing the number of beaten opponents). Since the Borda and Condorcet winners can differ, Borda-specific guarantees do not transfer to our setting. Copeland winners always exist (possibly non-unique) and generalize the CW; fixed-confidence identification is studied in \citet{zoghi2015copeland}, but when specialized to CW instances the resulting complexity is $\mathcal{O}\left(\max_{i\in[K]}\sum_{j\neq i}\log(1/\delta)/\Delta_{i,j}^2\right)$, which is looser than our bounds.
	\section{Intermediate results: Adaptive quantile estimation}\label{sec:intermediary}
	
	In this section, we present an algorithm for quantile estimation with a fixed budget of samples, and a high-probability guarantee on the estimation error. This result is of independent interest and will be used as a key building block in the proofs of the main results.
	
	\subsection{Quantile Bracketing}
	
	Consider a classical $K$-armed bandit setting, where we are given $K$ arms with means ${(\mu_i)}_{i\in[K]}$. We assume that the samples from the arms are range-bounded\footnote{The result can be extended to sub-Gaussian variables} by $1$, and without loss of generality that the means are distinct. Denote by $\mu_{(1)} < \mu_{(2)} < \ldots < \mu_{(K)}$ the ordered means of the $K$ arms. For two integers $d \le u$ in $[K]$, our objective is to find a point in the interval $[\mu_{(d)}, \mu_{(u)}]$.
	
	\noindent For this task, a learner is given a fixed budget of $T$ queries, after which the learner outputs a quantity $q_T
	$. In this framework, we are targeting an `adaptive' guarantee in the following sense. Given the budget $T$ as input, we want the output to satisfy 
	\[
	\forall \epsilon >0, \quad \mathbb{P}\left( q_T \notin \left[ \mu_{(d)}-\epsilon, \mu_{(u)}+\epsilon\right]\right) \le \exp\left(-c\cdot r\epsilon^2 \frac{T}{\log(T)}\right)~,
	\]
	where $c$ is a positive universal constant, and $r$ is a nonnegative quantity depending only on $d, u$ and $K$.
	
	\noindent The allocation strategy we develop requires a sufficiently large budget. Specifically, we assume that 
	\[
	T \ge \frac{128K}{u-d}\log_2\left(\frac{128K}{u-d}\right)~. 
	\]
	When $T$ falls below this threshold, the resulting guarantee, including its logarithmic prefactor, becomes vacuous (the stated upper bound on the error probability exceeds~$1$). In this regime, we therefore resort to an arbitrary heuristic. Algorithm~\ref{algo:inter} implements this by explicitly branching between the small-budget case $T < \frac{128K}{u-d}\log_2\!\left(\frac{128K}{u-d}\right)$ and the regime where the budget is large enough for the analysis to be meaningful.
	
	Solving this problem requires balancing the tasks of locating the arms whose means fall within the desired rank range, and estimating these means accurately. The algorithm runs a multi-step scheme indexed by $\ell$.
	At each level $\ell$, it draws a random multi-set $\mathcal{A}_\ell$ of arms large enough to contain, with good probability, representatives of the $[d,u]$ quantile range. Then, it allocates $\Theta(\epsilon_\ell^{-2})$ samples per selected arm so that empirical means are accurate up to $\epsilon_\ell$.
	
	\noindent Then we form three empirical quantiles $\hat{t}^{(1)}_{\ell},\hat{t}^{(2)}_{\ell},\hat{t}^{(3)}_{\ell}$ corresponding to ranks slightly below, near the middle of, and slightly above $[d,u]$, yielding a noisy bracket around the target interval. Finally, a Lepski-type stability rule selects the earliest level $\bar{\ell}$ whose middle estimate $\hat{t}^{(2)}_{\ell}$ remains consistent with the brackets produced at all finer levels (up to the tolerance $2\epsilon_{\ell'}$), and returns $\hat{t}^{(2)}_{\bar{\ell}}$.
	
	\noindent Theorem~\ref{thm:conj} states that with budget $T$, the output lies in $[\mu_{(d)}-\epsilon,\mu_{(u)}+\epsilon]$ with high probability for every $\epsilon\in(0,1)$, and the failure probability decays essentially as $\exp(-\Theta(\epsilon^2T/\log T))$ (up to the multiplicative factor $40\log_2(T)$ and the extra $\log(\tfrac{16K}{u-d})$ term).
	The quantity
	\[
	r = \min \left\lbrace \frac{d}{K}, 1-\frac{u}{K} \right\rbrace \cdot \left(\frac{u-d}{K}\right)^2\, ,
	\]
	captures the difficulty of the target rank range as it decreases when the interval is narrower ($u-d$ small) or when it is close to the extremes (small $d$ or large $u$). When $d$ and $u$ are constant fractions of $K$, we have $r=\Theta(1)$ and the bound simplifies to $\log(T)\exp(-c\epsilon^2T/\log T)$.
	
	\begin{algorithm}
		\caption{Range-Quantile$(K, d, u, T)$\label{algo:inter} }
		\begin{algorithmic}
			\STATE \textbf{Input}: $K$ number of arms and a budget $T$, integers $d \le u$ in $[K]$.
			\STATE \texttt{/* Consider separately the case where we have a small budget */}
			\IF{$u=d$}
			\STATE Set $T_0=+\infty$
			\ELSE
			\STATE Set $T_0=\frac{128K}{u-d}\log_2\left(\frac{128K}{u-d}\right)$
			\ENDIF
			\IF{$T\le T_0$}
			\STATE If $T<K$, return $0$. Otherwise, distribute the budget as evenly as possible among the $K$ arms and return the average of the empirical means with ranks $d,\ldots,u$
			\ENDIF
			\STATE $L = \floor*{\log_{2}(T/\log_{2}(T))}$ and $\ell_{\min} = \ceil*{\log_2\left(\frac{16K}{u-d}\right)}$
			\STATE \texttt{/* Otherwise if we have large budget: */}
			\FOR{$\ell = \ell_{\min},\dots, L-1$}
			\STATE Let $\epsilon_{\ell} =2\cdot\ 2^{-(L-\ell)/2}$.
			\STATE Sample a set $\mathcal{A}_{\ell}$ of $\floor*{\frac{\epsilon_{\ell}^2T}{\log\left(\frac{16K}{u-d}\right)\log_{2}(T)}}$ arms from $[K]$ with replacement (duplicates are treated as different arms).
			\STATE Allocate $\ceil*{\frac{\log\left(\frac{16K}{u-d}\right)}{2\epsilon_{\ell}^2}}$ samples to each arm $a \in \mathcal{A}_{\ell}$, and compute its empirical mean $\hat{\mu}_a$.
			\STATE Rank the empirical means in $\mathcal{A}_{\ell}$ in increasing order: $\hat{\mu}_{(1)} \le \dots \le \hat{\mu}_{(\abs{\mathcal{A}_{\ell}})}$. 
			\STATE Let:
			\begin{equation}\label{eq:def_ht}
				\hat{t}^{(1)}_{\ell} = \hat{\mu}_{(\ceil*{\frac{3d+u}{4K}\abs{\mathcal{A}_{\ell}}})},~\hat{t}^{(2)}_{\ell} = \hat{\mu}_{(\ceil*{\frac{d+u}{2K}\abs{\mathcal{A}_{\ell}}})},~\hat{t}^{(3)}_{\ell} = \hat{\mu}_{(\ceil*{\frac{d+3u}{4K}\abs{\mathcal{A}_{\ell}}})}.
			\end{equation}
			\ENDFOR
			\STATE Let $\bar{\ell} = \min_{\ell \in \llbracket \ell_{\min},L-1\rrbracket} \left\lbrace \forall \ell' \in \{\ell, \dots, L-1\}: \hat{t}^{(2)}_{\ell} \in \left[ \hat{t}^{(1)}_{\ell'} - 2\epsilon_{\ell'},~\hat{t}^{(3)}_{\ell'} + 2\epsilon_{\ell'} \right] \right\rbrace$,
			\STATE Return $\hat{t}^{(2)}_{\bar{\ell}}$.	
		\end{algorithmic}
	\end{algorithm}

	In the early-return case, we also write $\hat{t}^{(2)}_{\bar{\ell}}$ for the returned value.
	\begin{theorem}\label{thm:conj}
		Consider Algorithm~\ref{algo:inter} with inputs $(K, d, u, T)$, such that $u>d$ and $T\ge 2$. Then, the output satisfies for any $\epsilon\in (0,1)$:
		
		\begin{equation}\label{eq:equation:theorem12}
			\mathbb{P}\left( \hat{t}^{(2)}_{\bar{\ell}} \notin [\mu_{(d)}-\epsilon, \mu_{(u)}+\epsilon] \right) \le 40\log_2(T)\exp\left(-c \cdot r \cdot \frac{\epsilon^2T}{\log\!\left(\frac{16K}{u-d}\right)\log_2(T)}\right)~,
		\end{equation}
		where $r=\min\left\lbrace \frac{d}{K},1-\frac{u}{K} \right\rbrace \left(\frac{u-d}{K}\right)^2$ is a nonnegative quantity depending only on $d$, $u$ and $K$, and $c$ is an absolute numerical constant. When $u=K$, we have $r=0$, and the bound is vacuous.
	\end{theorem}

	\noindent Here, we did not try to optimize the constants. Next, we state a corollary that will be used in the proofs of the main theorems. 
	
	\begin{corollary}\label{cor:quant}
		Consider Algorithm~\ref{algo:inter} with inputs $(K, \ceil{K/8} , \ceil{K/4}, T)$, where $T \ge 4$. 
		Then, the output satisfies for any $\epsilon\in (0,1)$:
		\begin{equation*}
			\mathbb{P}\left( \hat{t}^{(2)}_{\bar{\ell}} \notin [\mu_{(\ceil{K/8})}-\epsilon, \mu_{(\ceil{K/4})}+\epsilon] \right) \le \log(T)\exp\left(-c\cdot \frac{\epsilon^2 T}{\log(T)}\right),
		\end{equation*}
		where $c$ is an absolute numerical constant smaller than $1$.
		The same bound holds for Algorithm~\ref{algo:inter} with inputs $(N,\ceil{(N+1)/8},\ceil{(N+1)/4},t)$, for every $N\ge 1$ and $t\ge 4$: writing $d=\ceil{(N+1)/8}$ and $u=\ceil{(N+1)/4}$, for any $\epsilon\in(0,1)$,
		\[
		\mathbb{P}\left(\hat{t}^{(2)}_{\bar{\ell}}\notin[\mu_{(d)}-\epsilon,\mu_{(u)}+\epsilon]\right)\le\log(t)\exp\left(-c_1\frac{\epsilon^2t}{\log(t)}\right)~,
		\]
		where $c_1>0$ is an absolute numerical constant.
	\end{corollary}
	\begin{proof}[of Corollary~\ref{cor:quant}]
		Suppose $K \ge 5$, then $\ceil{K/4}>\ceil{K/8}$. Let $\epsilon>0$ and $I = [\mu_{(\ceil{K/8})}-\epsilon, \mu_{(\ceil{K/4})}+\epsilon]$. We have
		\begin{equation}\label{eq:quant1}
			\min\left\lbrace \frac{\ceil{K/8}}{K}, 1-\frac{\ceil{K/4}}{K} \right\rbrace \ge \min \left\lbrace \frac{1}{8}, 1-\frac{K/4+1}{K}\right\rbrace \ge \frac{1}{8}~.
		\end{equation}
		Moreover, using $K \ge 5$ and $K=8q+r$ where $q \in \mathbb{N}$ and $r \in \{0,\dots, 7\}$, we show that
		\begin{equation}\label{eq:quant2}
			\left(\frac{\ceil{K/4}-\ceil{K/8}}{K}\right)^2 \ge \frac{1}{144}~.
		\end{equation}
		Thus $\log\left(\frac{16K}{u-d}\right)\le\log(192)<6$ for $d=\ceil{K/8}$ and $u=\ceil{K/4}$.
		Applying Theorem~\ref{thm:conj} with $u=\ceil{K/4}$, $d=\ceil{K/8}$ and using the bounds~\eqref{eq:quant1} and~\eqref{eq:quant2} we obtain
		\begin{align*}
			\mathbb{P}\left(\hat{t}^{(2)}_{\bar{\ell}} \notin I \right) &\le \min\left\lbrace 1,~ 40\log_2(T)\exp\left( -c~\frac{1}{8}\cdot \frac{1}{144}\cdot \frac{\epsilon^2T}{6\log_2(T)}\right)\right\rbrace\\
			&\le\log(T)\exp\left( -c' \frac{\epsilon^2T}{\log(T)}\right)~,
		\end{align*} 
		where $c'$ is a numerical constant. The last line follows by absorbing all numerical constants into $c'>0$, using $T\ge 4$.
		
		\medskip 
		\noindent Suppose now that $K \in \{2, 3, 4\}$, then $\ceil{K/4} = \ceil{K/8} = 1$. In this case, Algorithm~\ref{algo:inter} allocates at least $\floor{T/K}\ge T/(2K)\ge T/8$ samples to each arm and outputs the minimal empirical mean. Let $a \in [K]$ denote the index corresponding to the arm with the smallest true mean, we therefore have (let $\hat{t}$ denote the output)
		\begin{align*}
			\mathbb{P}\!\left(\hat{t} \notin [\mu_{(\ceil{K/8})}-\epsilon, \mu_{(\ceil{K/4})}+\epsilon]\right) &= 	\mathbb{P}\!\left(\hat{t} <\mu_{(1)}-\epsilon\right) + \mathbb{P}\!\left(\hat{t} >\mu_{(1)}+\epsilon\right)\\
			&\le \mathbb{P}\!\left( \hat{\mu}_a >\mu_{a}+\epsilon\right) + \sum_{i=1}^K \mathbb{P}\!\left( \hat{\mu}_i <\mu_{i}-\epsilon\right)\, ,
		\end{align*} 
		where $\hat{\mu}_i$ denotes the empirical mean of arm $i$. We conclude using Hoeffding's inequality, with the fact that each arm receives at least $\floor{T/K}\ge T/(2K)\ge T/8$ samples that
		\[
		\mathbb{P}\!\left(\hat{t} \notin [\mu_{(\ceil{K/8})}-\epsilon, \mu_{(\ceil{K/4})}+\epsilon]\right) \le 5\exp\!\left(-\epsilon^2\frac{T}{4} \right)\, ,
		\]
		which corresponds to the result.
		
		\medskip
		\noindent For the shifted ranks, if $N\ge 4$, then $\min\{d/N,1-u/N\}\ge 1/8$ and $(u-d)/N\ge 1/12$, the latter by~\eqref{eq:quant2} with $K=N+1$. Thus $\log(16N/(u-d))\le\log(192)<6$, and Theorem~\ref{thm:conj} gives the same bound after decreasing the numerical constant. If $N\in\{1,2,3\}$, then $d=u=1$; the fallback samples each arm at least $\floor{t/N}\ge t/(2N)\ge t/6$ times and returns the smallest empirical mean. The preceding Hoeffding argument gives
		\[
		\mathbb{P}\left(\hat{t}^{(2)}_{\bar{\ell}}\notin[\mu_{(1)}-\epsilon,\mu_{(1)}+\epsilon]\right)\le 4\exp(-\epsilon^2t/3)~,
		\]
		which has the claimed form by the same constant-absorption argument.
	\end{proof}

	\subsection{Proof of Theorem~\ref{thm:conj}}
	Since the right-hand side of~\eqref{eq:equation:theorem12} does not depend on the values of $\mu_{(1)},\ldots, \mu_{(K)}$, it suffices to treat the strictly ordered case where $\mu_{(1)} < \mu_{(2)} < \dots < \mu_{(K)}$; the case where some values are perhaps identical follows by a continuity argument.

	\begin{proof}
		If $T \le \frac{128K}{u-d}\log_2\left(\frac{128K}{u-d}\right)$ the bound is vacuous, including when the fallback returns $0$: in this range $rT/\log_2(T)\le 128$, so the right-hand side of~\eqref{eq:equation:theorem12} exceeds $1$ for sufficiently small $c$. Assume that $T > \frac{128K}{u-d}\log_2\left(\frac{128K}{u-d}\right)$ so that $\log_2(\log_2 (T)) \ge 1$ and $\ell_{\min} \le L-1$. We introduce the following additional notation, for each $\ell \in \{\ell_{\min}, \dots, L-1\}$ where $\ell_{\min} = \ceil*{\log_2\left(\frac{16K}{u-d}\right)}$, let 
		\begin{equation}\label{eq:def_T}
			N_{\ell} := \abs{\mathcal{A}_{\ell}} = \floor*{\frac{\epsilon_{\ell}^2T}{\log\left(\frac{16K}{u-d}\right)\log_{2}(T)}} \qquad \text{and} \qquad T_{\ell} := \ceil*{\frac{\log\left(\frac{16K}{u-d}\right)}{2\epsilon_{\ell}^2}}~,	
		\end{equation}
		and let 
		\[
		r_0 := \frac{d}{K} \quad \text{,} \quad r_1 := \frac{3d+u}{4K} \quad \text{,} \quad r_2 := \frac{d+u}{2K} \quad \text{,} \quad r_3 := \frac{d+3u}{4K} \quad \text{and} \quad r_4 := \frac{u}{K}~. 
		\]

		\noindent The proof follows the steps below
		\begin{itemize}
			\item We start by a sanity check, verifying that the total number of queries made by Algorithm~\ref{algo:inter} is at most $T$, and that $\bar{\ell}$ exists.
			\item Next, we show an intermediary result about the quantities $\hat{t}_{\ell}^{(i)}$ for $i \in \{1, 2,3\}$ in the form of an upper-bound on
			\[
			\mathbb{P}\left( \hat{t}_{\ell}^{(i)} \notin \left[\mu_{\left(\ceil*{\frac{r_{i-1}+r_i}{2} \cdot K}\right)}-\epsilon_{\ell}, ~\mu_{\left(\ceil*{\frac{r_i+r_{i+1}}{2}\cdot K}\right)}+\epsilon_{\ell}\right] \right)~.
			\]
			\item Finally, we build on the obtained intermediary result to prove that the way $\bar{\ell}$ is defined allows to have the stated guarantees. 
		\end{itemize}
		
		\paragraph{Sanity checks:} Recall the expressions $L = \floor*{\log_{2}(T/\log_{2}(T))}$, $\epsilon_{\ell} =2.\ 2^{-(L-\ell)/2}$ and $\abs{\mathcal{A}_{\ell}} = \floor*{\frac{\epsilon_{\ell}^2T}{\log\left(\frac{16K}{u-d}\right)\log_{2}(T)}}$. Algorithm~\ref{algo:inter} comprises $L-\ell_{\min}$ iterations, for each iteration $\ell \in \{\ell_{\min}, \dots, L-1\}$ it makes $\abs{\mathcal{A}_{\ell}}T_{\ell}$ queries. Thus, the total number of queries is
		\begin{align*}
			\sum_{\ell=\ell_{\min}}^{L-1} \abs{\mathcal{A}_{\ell}} \cdot \ceil*{\frac{\log\left(\frac{16K}{u-d}\right)}{2\epsilon_{\ell}^2}} &= \sum_{\ell=\ell_{\min}}^{L-1} \floor*{\frac{\epsilon_{\ell}^2T}{\log\left(\frac{16K}{u-d}\right)\log_{2}(T)}} \cdot \ceil*{\frac{\log\left(\frac{16K}{u-d}\right)}{2\epsilon_{\ell}^2}}\\ 
			&\le \sum_{\ell=\ell_{\min}}^{L-1} \frac{\epsilon_{\ell}^2T}{\log\left(\frac{16K}{u-d}\right)\log_{2}(T)} \cdot \left(\frac{\log\left(\frac{16K}{u-d}\right)}{2\epsilon_{\ell}^2} + 1 \right)\\
			&= \sum_{\ell=\ell_{\min}}^{L-1} \frac{T}{2\log_2(T)}+ \frac{\epsilon_{\ell}^2T}{\log\left(\frac{16K}{u-d}\right) \log_2(T)}\\
			&< \frac{T}{2} + \frac{T}{\log\left(\frac{16K}{u-d}\right)\log_2(T)} \sum_{\ell=\ell_{\min}}^{L-1} \epsilon_{\ell}^2\\
			&\le \frac{T}{2} + \frac{4T}{\log\left(\frac{16K}{u-d}\right)\log_2(T)} \le T~,	
		\end{align*} 
		where we used in the last line the threshold condition on $T$, giving
		\[
		\log(16K/(u-d))\log_2(T) \ge 8\, .
		\]
		
		\noindent For the definition of the quantity $\bar{\ell}$, note that the set over which the minimum is taken is not empty, since it always contains $L-1$.

		\bigskip
		
		\noindent \textbf{ Step 2: Per-level quantile control.}
		
		\medskip
		
		\noindent In this step we will prove that for every level $\ell\in\{\ell_{\min},\dots,L-1\}$ and every $i\in\{1,2,3\}$
		\begin{equation}\label{eq:i1}
			\mathbb{P}\!\left(\hat{t}^{(i)}_{\ell}\notin C_{\ell,i}\right)\le p_\ell\,,
		\end{equation}
		where we define (for each $\ell$ and $i\in\{1,2,3\}$)
		\begin{align*}
			C_{\ell,i}
			&\coloneq 
			\left[
			\mu_{\left(\left\lceil \frac{r_{i-1}+r_i}{2}K\right\rceil\right)}-\epsilon_\ell,\ 
			\mu_{\left(\left\lceil \frac{r_i+r_{i+1}}{2}K\right\rceil\right)}+\epsilon_\ell
			\right],\\
			\kappa_{d,u}
			&\coloneq 
			\min\Big\{\frac{d}{K},1-\frac{u}{K}\Big\}\Big(\frac{u-d}{60K}\Big)^2,
			\qquad
			p_\ell \coloneq 4\exp\!\left(-\kappa_{d,u}N_\ell\right).
		\end{align*}
		Throughout this step, fix $\ell\in\{\ell_{\min},\dots,L-1\}$ and $i\in\{1,2,3\}$.
		Define the two (random-sample) ranks
		\[
		r_{-}\coloneq \left\lceil \frac{r_{i-1}+2r_i}{3}\,N_\ell\right\rceil, \qquad r_{+}\coloneq \left\lceil \frac{2r_i+r_{i+1}}{3}\,N_\ell\right\rceil,
		\]
		and define the two (population) bracket points
		\[
		m_{-}\coloneq \mu_{\left(\left\lceil \frac{r_{i-1}+r_i}{2}K\right\rceil\right)}, \qquad m_{+}\coloneq \mu_{\left(\left\lceil \frac{r_i+r_{i+1}}{2}K\right\rceil\right)}\,.
		\]
		
		\noindent Let $\gamma_1,\dots,\gamma_{N_\ell}$ be the (random) true means of the sampled multiset
		$\mathcal{A}_\ell$, and let $\gamma_{(1)}\le \dots \le \gamma_{(N_\ell)}$ be their order
		statistics (ties broken arbitrarily).
		
		\noindent Next, introduce the event $E^{(i)}$ given by
		\[
		E^{(i)} \coloneq \big\{\gamma_{(r_-)}<m_-\big\}\ \cup\ \big\{\gamma_{(r_+)}>m_+\big\}\,.
		\]
		Then, by a union bound,
		\begin{equation}\label{eq:decomp1}
			\mathbb{P}\!\left(\hat{t}^{(i)}_\ell\notin C_{\ell,i}\right) \le \underbrace{\mathbb{P}\!\left(E^{(i)}\right)}_{\text{Term 1}} +
			\underbrace{\mathbb{P}\!\left(\hat{t}^{(i)}_\ell\notin C_{\ell,i}\ \text{and}\ \neg E^{(i)}\right)}_{\text{Term 2}}.
		\end{equation}
		
		\noindent Next we will control the probability of $E^{(i)}$ (Term $1$). Define the counts
		\[
		M_1 \coloneq \abs{\{j\in[N_\ell]:\gamma_j<m_-\}},\qquad M_2 \coloneq \abs{\{j\in[N_\ell]:\gamma_j>m_+\}}\,.
		\]
		Since $\mathcal{A}_\ell$ is obtained by sampling arms i.i.d. with replacement from $[K]$,
		these are binomials:
		\[
		M_1 \sim \mathrm{Bin}\!\left(N_\ell,\ \frac{\left\lceil \frac{r_{i-1}+r_i}{2}K\right\rceil-1}{K}\right), \qquad M_2 \sim \mathrm{Bin}\!\left(N_\ell,\ 1-\frac{\left\lceil \frac{r_i+r_{i+1}}{2}K\right\rceil}{K}\right)\,.
		\]
		Moreover, by definition of order statistics we have
		\[
		\{\gamma_{(r_-)}<m_-\}\subseteq \{M_1\ge r_-\}, \qquad \{\gamma_{(r_+)}>m_+\}\subseteq \{M_2\ge N_\ell-r_+ + 1\}\,.
		\]
		Therefore,
		\[
		\mathbb{P}\!\left(E^{(i)}\right) \le \mathbb{P}(M_1\ge r_-)\ +\ \mathbb{P}(M_2\ge N_\ell-r_+ +1)\,.
		\]
		Let us bound the two terms in the upper bound above using binomial tail bounds.
		Using $\lceil x\rceil-1\le x$ and $\lceil x\rceil\ge x$, we have the parameter bounds
		\[
		\frac{\left\lceil \frac{r_{i-1}+r_i}{2}K\right\rceil-1}{K}\ \le\ \frac{r_{i-1}+r_i}{2}, \qquad 1-\frac{\left\lceil \frac{r_i+r_{i+1}}{2}K\right\rceil}{K}\ \le\ 1-\frac{r_i+r_{i+1}}{2}\,.
		\]
		Hence, $M_1$ and $M_2$ are stochastically dominated by
		$\mathrm{Bin}(N_\ell,\frac{r_{i-1}+r_i}{2})$ and $\mathrm{Bin}(N_\ell,1-\frac{r_i+r_{i+1}}{2})$,
		respectively. Also, by construction we have
		\[
		\frac{r_{i-1}+2r_i}{3} - \frac{r_{i-1}+r_i}{2} = \frac{r_i-r_{i-1}}{6}, \qquad \Big(1-\frac{2r_i+r_{i+1}}{3}\Big) - \Big(1-\frac{r_i+r_{i+1}}{2}\Big)=\frac{r_{i+1}-r_i}{6}\,.
		\]
		Applying Hoeffding's inequality to these dominating binomials yields
		\begin{align*}
			\mathbb{P}(M_1\ge r_-)
			&\le \exp\!\left(-2N_\ell\Big(\frac{r_i-r_{i-1}}{6}\Big)^2\right),\\
			\mathbb{P}(M_2\ge N_\ell-r_+ +1)
			&\le \exp\!\left(-2N_\ell\Big(\frac{r_{i+1}-r_i}{6}\Big)^2\right)\,.
		\end{align*}
		Since for each $j\in\{1,2,3,4\}$ one has $r_j-r_{j-1}\ge \frac{u-d}{4K}$, we obtain
		\begin{equation}\label{eq:bound_e1}
			\text{Term 1}=\mathbb{P}\!\left(E^{(i)}\right) \le 2\exp\!\left(-\frac{N_\ell (u-d)^2}{288K^2}\right)\,.
		\end{equation}
		
		\noindent Next, let us upper bound Term $2$ in~\eqref{eq:decomp1}.
		On $\neg E^{(i)}$ we have $\gamma_{(r_-)}\ge m_-$ and $\gamma_{(r_+)}\le m_+$, hence
		\[
		[m_- - \epsilon_\ell,\ m_+ + \epsilon_\ell] \supseteq [\gamma_{(r_-)}-\epsilon_\ell,\ \gamma_{(r_+)}+\epsilon_\ell].
		\]
		Therefore,
		\begin{align*}
			\text{Term 2}
			&= \mathbb{P}\!\left(\hat{t}^{(i)}_\ell\notin [m_--\epsilon_\ell,m_++\epsilon_\ell]\ \text{and}\ \neg E^{(i)}\right)\\
			&\le \mathbb{P}\!\left(\hat{t}^{(i)}_\ell\notin [\gamma_{(r_-)}-\epsilon_\ell,\ \gamma_{(r_+)}+\epsilon_\ell]\right).
		\end{align*}
		By Lemma~\ref{lem:inter}, this implies
		\begin{equation}\label{eq:bound_e2}
			\text{Term 2}\le 2\exp\!\left(-\kappa_{d,u}N_\ell\right).
		\end{equation}
		
		\smallskip
		
		\noindent Combining~\eqref{eq:decomp1},~\eqref{eq:bound_e1} and~\eqref{eq:bound_e2}, and using that
		$\kappa_{d,u}\le (u-d)^2/(288K^2)$, we obtain
		\begin{equation*}
			\mathbb{P}\!\left(\hat{t}^{(i)}_{\ell}\notin C_{\ell,i}\right) \le 4\exp\!\left(-\kappa_{d,u}N_\ell\right)=p_\ell\,,
		\end{equation*}
		which is exactly~\eqref{eq:i1}.
		
		\bigskip
		\noindent \textbf{Step 3: Conclusion.}
		
		\noindent If $\epsilon <3\epsilon_{\ell_{\min}}$, then the upper bound of the theorem is greater than $1$ and the bound is vacuous. Assume that $\epsilon \ge 3\epsilon_{\ell_{\min}}$. Let $\ell^{\star}$ be the largest level such that 
		\[
		3\epsilon_{\ell^{\star}} \le \epsilon~.
		\]
		This implies in particular, since $\ell^{\star}+1$ violates the condition above, that $\epsilon <3\epsilon_{\ell^{\star}+1} = 3\sqrt{2}\epsilon_{\ell^{\star}}$, therefore
		\begin{equation}\label{eq:epsilon_l_star}
			\epsilon_{\ell^{\star}} \ge \frac{\epsilon}{3\sqrt{2}}~. 
		\end{equation}
		
		Next, we will prove that for any $\ell \in \{\ell_{\min}, \dots, L-1\}$, we have 
		\[
		\mathbb{P}(\hat{t}_{\bar{\ell}}^{(2)} \notin [\mu_{(d)}-3\epsilon_{\ell},~\mu_{(u)}+3\epsilon_{\ell}]) \le 2(4L+1)p_{\ell}~.
		\]
		Let $\ell \in \{\ell_{\min}, \dots, L-1\}$, recall that by definition of $\bar \ell$, for $\ell'\geqslant \bar{\ell}$, one has $\hat{t}^{(2)}_{\bar{\ell}} \leqslant~\hat{t}^{(3)}_{\ell'} + 2\epsilon_{\ell'}$. Then, it holds that 
		\[
		\mathbb{P}\left(\hat{t}_{\bar{\ell}}^{(2)} > \mu_{(u)}+3\epsilon_{\ell}\right) \le \mathbb{P}(\bar{\ell}>\ell)+\mathbb{P}(\hat{t}_{\ell}^{(3)}>\mu_{(u)}+\epsilon_{\ell}) \le 4Lp_{\ell}+p_{\ell}~,
		\]
		where we use Lemma~\ref{lem:i2}, which ensures that for every $\ell \in \{\ell_{\min}, \dots, L-1\}$ we have $\mathbb{P}(\bar{\ell}>\ell) \le 4Lp_{\ell}$,
		and we use the Bound~\ref{eq:i1} from step 2 with $i=3$. The second bound 
		\[
		\mathbb{P}\left(\hat{t}_{\bar{\ell}}^{(2)}< \mu_{(d)}-3\epsilon_{\ell}\right) \le (4L+1)p_{\ell}~,
		\]
		is proven using the same arguments (in particular Bound~\ref{eq:i1} with $i=1$).

		\noindent Applying this bound to $\ell^{\star}$, using $3\epsilon_{\ell^{\star}}\leqslant \epsilon$, we have
		\begin{align*}
			\mathbb{P}\left(\hat{t}_{\bar{\ell}}^{(2)} \notin [\mu_{(d)}-\epsilon,~\mu_{(u)}+\epsilon] \right) &\le \mathbb{P}\left( \hat{t}_{\bar{\ell}}^{(2)} \notin [\mu_{(d)}-3\epsilon_{\ell^{\star}},~\mu_{(u)}+3\epsilon_{\ell^{\star}}] \right)\\
			&\le 2(4L+1)p_{\ell^{\star}}~.
		\end{align*}
		Next, in order to upper bound $p_{\ell^{\star}}$ we use the following lower bound on $N_{\ell^{\star}}$
		\begin{align*}
			N_{\ell^{\star}} = \floor*{\frac{\epsilon_{\ell^{\star}}^2T}{\log\left(\frac{16K}{u-d}\right)\log_2(T)}} &\ge \frac{1}{2} \frac{\epsilon_{\ell^{\star}}^2T}{\log\left(\frac{16K}{u-d}\right)\log_2(T)}\\
			&\ge \frac{\epsilon^2T}{36\log\left(\frac{16K}{u-d}\right)\log_2(T)}~,
		\end{align*}
		where we use the fact that from the assumption on the budget $T$, $N_{\ell^{\star}}\geqslant 2$, and  $\epsilon_{\ell^{\star}}\geqslant \epsilon/3\sqrt{2}$ (see~\eqref{eq:epsilon_l_star}). 
		Therefore, using the definition of $p_{\ell}$,
		\begin{align*}
			p_{\ell^{\star}} &= 4\exp(-\kappa_{d,u}\cdot N_{\ell^{\star}})\\
			&\le 4 \exp\left(-\min \left\lbrace\frac{d}{K},~1-\frac{u}{K} \right\rbrace \left(\frac{u-d}{60K}\right)^2\cdot \frac{\epsilon^2 T}{36\log\left(\frac{16K}{u-d}\right) \log_2(T)}\right)\\
			&= 4\exp\left(-\frac{c r \epsilon^2 T}{\log\left(\frac{16K}{u-d}\right)\log_2(T)}\right)~,
		\end{align*}
		where $c>0$ is an absolute constant, and $r=\min \left\lbrace\frac{d}{K},~1-\frac{u}{K} \right\rbrace \left(\frac{u-d}{K}\right)^2$. Finally, given that $L \le \log_2(T)$ and $2(4L+1)\cdot 4 \le 40\log_2(T)$ for $T\ge 2$,
		\[
		\mathbb{P}\left(\hat{t}_{\bar{\ell}}^{(2)} \notin [\mu_{(d)}-\epsilon,~\mu_{(u)}+\epsilon]\right) \le 40 \log_2(T)\exp\left(-\frac{cr\epsilon^2 T}{\log\left(\frac{16K}{u-d}\right)\log_2(T)}\right)~,
		\]
		which is the desired bound.
	\end{proof}
	
	\noindent Below are two technical lemmas deferred here to avoid cluttering the proof above.

	\begin{lemma}\label{lem:i2}
		For every $\ell \in \{\ell_{\min}, \dots, L-1 \}$, we have $\mathbb{P}(\bar{\ell}>\ell) \le 4Lp_{\ell}\,$.
	\end{lemma}
	\begin{proof}
		Suppose that $\bar{\ell}> \ell$, then using the definition of $\bar{\ell}$ we have necessarily that there exists $\ell' \geqslant \ell$ such that $\hat{t}^{(2)}_{\ell} \notin I_{\ell'}$, with $I_{\ell'}=\left[ \hat{t}^{(1)}_{\ell'} - 2\epsilon_{\ell'},~\hat{t}^{(3)}_{\ell'} + 2\epsilon_{\ell'} \right]$. Therefore,
		\[
		\mathbb{P}(\bar{\ell}>\ell) \le \sum_{\ell' \geqslant \ell} \mathbb{P}\left(\hat{t}_{\ell}^{(2)} < \hat{t}_{\ell'}^{(1)}-2\epsilon_{\ell'}\right)+\sum_{\ell'\geqslant\ell} \mathbb{P}\left(\hat{t}_{\ell}^{(2)} > \hat{t}_{\ell'}^{(3)}+2\epsilon_{\ell'}\right)~.
		\]
		Let
		\[
		m_{-} \coloneq \mu_{\left(\ceil*{\frac{r_1+r_2}{2}K}\right)},\quad m_{+} \coloneq \mu_{\left( \ceil*{\frac{r_2+r_3}{2}K}\right)}~.
		\]
		Let $\ell'\geqslant\ell$, the event $\hat{t}_{\ell}^{(2)} < \hat{t}_{\ell'}^{(1)}-2\epsilon_{\ell'}$ implies $\hat{t}_{\ell'}^{(1)} > m_{-}+\epsilon_{\ell'}$ or $\hat{t}_{\ell}^{(2)}<m_{-}-\epsilon_{\ell'}$. Since we have $\epsilon_{\ell'} \ge \epsilon_{\ell}$ for $\ell' \geqslant \ell$, Bound~\ref{eq:i1} yields $\mathbb{P}(\hat{t}_{\ell'}^{(1)} > m_{-}+\epsilon_{\ell'}) \le p_{\ell'}$ and
		\begin{align*}
			\mathbb{P}(\hat{t}_{\ell}^{(2)}<m_{-}-\epsilon_{\ell'}) &\le \mathbb{P}(\hat{t}_{\ell}^{(2)} < m_{-}-\epsilon_{\ell})\\
			&\le p_{\ell}~.
		\end{align*}
		Therefore,
		\[
		\mathbb{P}(\hat{t}_{\ell}^{(2)} < \hat{t}_{\ell'}^{(1)} -2\epsilon_{\ell'}) \le p_{\ell}+p_{\ell'}~.
		\]
		Similarly,
		\[
		\mathbb{P}(\hat{t}_{\ell}^{(2)}> \hat{t}_{\ell'}^{(3)}+2\epsilon_{\ell'}) \le p_{\ell}+p_{\ell'}~.
		\]
		Therefore, for every $\ell'\geqslant\ell$,
		\[
		\mathbb{P}(\hat{t}_{\ell}^{(2)}\notin I_{\ell'}) \le 2p_{\ell}+2p_{\ell'}~.
		\]
		Summing over $\ell'\geqslant\ell$,
		\[
		\mathbb{P}(\bar{\ell}>\ell) \le \sum_{\ell'\geqslant\ell} (2p_{\ell}+2p_{\ell'}) \le 2Lp_{\ell}+2\sum_{\ell'\geqslant\ell} p_{\ell'}~.
		\]
		Since $N_{\ell'}$ increases with $\ell'$ (thus $p_{\ell'}$ is decreasing), we have $\sum_{\ell'\geqslant\ell}p_{\ell'} \le Lp_{\ell}$, which gives $\mathbb{P}(\bar{\ell}>\ell)\le 4Lp_{\ell}$. 
	\end{proof}

	\begin{lemma}\label{lem:inter}
		Let $\ell \in \{\ell_{\min}, \dots, L-1\}$ and consider the notation introduced in the proof of Theorem~\ref{thm:conj}. For each $i \in \{1, 2, 3\}$, define the two indices
		\[
		r_{-} \coloneq \ceil*{\frac{r_{i-1}+2r_i}{3}N_{\ell}}\quad\text{and}\quad r_{+} \coloneq \ceil*{\frac{2r_i+r_{i+1}}{3}N_{\ell}}\, .
		\]
		Then
		\[
		\mathbb{P}\left( \hat{t}_{\ell}^{(i)} \notin \left[ \gamma_{\left(r_{-}\right)}-\epsilon_{\ell}, \gamma_{\left(r_{+}\right)}+\epsilon_{\ell} \right] \mid \mathcal{A}_{\ell}\right) \le 2\exp\left(-\kappa_{d,u} \cdot N_{\ell}\right)\,. 
		\]
	\end{lemma}
	\begin{proof}
		Fix $\ell \in \{\ell_{\min}, \dots, L-1\}$ and $i\in \{1,2,3\}$. Define $q=\ceil{r_iN_{\ell}}$. Given that we have $T_{\ell} = \ceil{\log(16K/(u-d))/(2\epsilon_{\ell}^2)}$, let
		\begin{align}
			\delta_{\ell} &\coloneq \exp\left(-2T_{\ell}\epsilon_{\ell}^2\right)\nonumber\\
			&\le \exp\left(-\log(16K/(u-d))\right) = \frac{u-d}{16K}\,.\label{eq:il1}
		\end{align}
		Moreover, for every arm $j$, given that samples are $1$-range bounded and $\hat{\mu}_j$ is computed with $T_{\ell}$ samples, Hoeffding's inequality gives
		\begin{equation}\label{eq:il2}
			\mathbb{P}(\hat{\mu}_j \le \gamma_j-\epsilon_{\ell} \mid \mathcal{A}_{\ell}) \le \delta_{\ell},\quad \mathbb{P}(\hat{\mu}_j \ge \gamma_j + \epsilon_{\ell} \mid \mathcal{A}_{\ell}) \le \delta_{\ell}\,.
		\end{equation}
		Next, we prove the lower tail of the claimed bound. Recall that $\hat{t}_{\ell}^{(i)} = \hat{\mu}_{(q)}$ and $r_{-} \coloneq \ceil*{\frac{r_{i-1}+2r_i}{3}N_{\ell}}$ and define the event 
		\[
		\mathcal{E}_{-} \coloneq \{ \hat{\mu}_{(q)} < \gamma_{(r_{-})} - \epsilon_{\ell}\}\,.
		\]
		If $\mathcal{E}_{-}$ occurs, then at least $q$ empirical means are smaller than $\gamma_{(r_{-})}-\epsilon_{\ell}$ (with $q\geqslant r_-$). Given that $(\gamma_{(k)})_{k}$ are in non-decreasing order, we conclude that among the set 
		\[
		\mathcal{G}_{-} \coloneq \{j:\, \gamma_j \ge \gamma_{(r_{-})}\}\,,
		\]
		then at least $q-(r_{-}-1)$ elements must satisfy $\hat{\mu}_j < \gamma_{(r_{-})}-\epsilon_{\ell}$. For each element $j \in \mathcal{G}_{-}$ we have $\gamma_j \ge \gamma_{(r_{-})}$, hence $\{\hat{\mu}_j < \gamma_{(r_{-})}-\epsilon_{\ell} \} \subseteq \{ \hat{\mu}_j < \gamma_{j}-\epsilon_{\ell} \}$, using~\eqref{eq:il1} and~\eqref{eq:il2} we have that, conditionally on $\mathcal{A}_{\ell}$, each such downward deviation has probability at most $(u-d)/(16K)$. Therefore, we conclude that 
		\[
		\mathbb{P}(\mathcal{E}_{-}\mid \mathcal{A}_{\ell}) \le \mathbb{P}\left(\text{Bin}(\abs{\mathcal{G}_{-}},(u-d)/(16K)) \ge q-(r_{-}-1)\right)\,.
		\] 
		Using the definitions of $q, r_{-}$ and $r_i$, we have
		\begin{align*}
			q-(r_{-}-1) &= \ceil*{r_i N_{\ell}}-\left(\ceil*{\frac{r_{i-1}+2r_i}{3}N_{\ell}}-1\right)\\
			&\ge r_i N_{\ell} - \frac{r_{i-1}+2r_i}{3}N_{\ell}\\
			&= \frac{r_i-r_{i-1}}{3}N_{\ell} = \frac{u-d}{12K}N_{\ell} \\
			& \geq \frac{u-d}{12K}\abs{\mathcal{G}_{-}} \,. 
		\end{align*} 
		
		Applying Hoeffding's inequality for binomials yields
		\begin{align*}
			\mathbb{P}(\mathcal{E}_{-}\mid \mathcal{A}_{\ell}) &\le \exp\left(-2\abs{\mathcal{G}_{-}}\left(\frac{u-d}{12K}-\frac{u-d}{16K}\right)^2\right)\\
			&\le \exp\left(-2\abs{\mathcal{G}_{-}}\left(\frac{u-d}{48K}\right)^2\right)\,.
		\end{align*}
		Finally, for $i\in \{1,2,3\}$ we have $\frac{r_{i-1}+2r_i}{3}\le \frac{u}{K}$, hence $\abs{\mathcal{G}_{-}}\ge N_{\ell}-r_{-}+1 \ge (1-u/K)N_{\ell}$. Therefore,
		\begin{align}
			\mathbb{P}(\mathcal{E}_{-}\mid \mathcal{A}_{\ell}) &\le \exp\left(-2\left(1-\frac{u}{K}\right)N_{\ell}\left(\frac{u-d}{48K}\right)^2\right)\nonumber\\
			&\le \exp\left(-\kappa_{d,u}\cdot N_{\ell}\right)\,.\label{eq:il_p1}
		\end{align}	
		
		\noindent Let us show the upper tail of the claimed bound. We follow similar steps as in the lower tail proof. Consider $r_{+} \coloneq \ceil*{\frac{2r_i+r_{i+1}}{3}N_{\ell}}$ and define
		\[
		\mathcal{E}_{+} \coloneq \{\hat{\mu}_{(q)} > \gamma_{(r_{+})}+\epsilon_{\ell}\}\,.
		\]
		If $\mathcal{E}_{+}$ occurs, then at least $N_{\ell}-q+1$ empirical means exceed $\gamma_{(r_{+})}+\epsilon_{\ell}$. At most $N_{\ell}-r_{+}$ arms can have true mean larger than $\gamma_{(r_{+})}$, therefore at least $(N_{\ell}-q+1)-(N_{\ell}-r_{+})=r_{+}-q+1$ arms from the set
		\[
		\mathcal{G}_{+} \coloneq \{j: \gamma_j \le \gamma_{(r_{+})}\}\,
		\]
		must satisfy $\hat{\mu}_j > \gamma_{(r_{+})}+\epsilon_{\ell}$. For each $j \in \mathcal{G}_{+}$ we have $\gamma_j \le \gamma_{(r_{+})}$, therefore 
		$\{\hat{\mu}_j > \gamma_{(r_{+})}+\epsilon_{\ell}\} \subseteq \{\hat{\mu}_j \ge \gamma_j +\epsilon_{\ell}\}$, using~\eqref{eq:il1} and~\eqref{eq:il2} we have that, conditionally on $\mathcal{A}_{\ell}$, each such upward deviation has probability at most $(u-d)/(16K)$ conditionally on $\mathcal{A}_{\ell}$. Thus, conditionally on $\mathcal{A}_{\ell}$, we have 
		\[
		\mathbb{P}(\mathcal{E}_{+}\mid \mathcal{A}_{\ell}) \le \mathbb{P}\left(\text{Bin}(\abs{\mathcal{G}_{+}},(u-d)/(16K)) \ge r_{+}-q+1 \right)\,.
		\]
		Also, we have
		\begin{align*}
			r_{+}-q+1 &\ge \frac{2r_i+r_{i+1}}{3}N_{\ell} - r_i N_{\ell}\\
			&= \frac{r_{i+1}-r_i}{3}N_{\ell} = \frac{u-d}{12K}N_{\ell}\,.
		\end{align*}
		Therefore, using the binomial Hoeffding bound we obtain
		\begin{align*}
			\mathbb{P}(\mathcal{E}_{+}\mid \mathcal{A}_{\ell}) &\le \exp\left(-2\abs{\mathcal{G}_{+}}\left(\frac{u-d}{48K}\right)^2\right)\\
			&\le \exp\left(-2r_{+}\left(\frac{u-d}{48K}\right)^2\right)\,.
		\end{align*}
		Moreover, for $i\in \{1,2,3\}$ we have $\frac{2r_i+r_{i+1}}{3} \ge \frac{d}{K}$, so $r_{+} \ge \frac{d}{K}N_{\ell}$. Hence,
		\begin{align}
			\mathbb{P}(\mathcal{E}_{+}\mid \mathcal{A}_{\ell}) &\le \exp\left(-2\frac{d}{K}N_{\ell} \left(\frac{u-d}{48K}\right)^2\right)\nonumber\\
			&\le \exp\left(-\kappa_{d,u}\cdot N_{\ell}\right)\,.\label{eq:il_p2}
		\end{align}
		The conclusion follows by combining~\eqref{eq:il_p1}, and~\eqref{eq:il_p2} which leads to the bound
		\[
		\mathbb{P}(\hat{\mu}_{(q)} \notin [\gamma_{(r_{-})}-\epsilon_{\ell},~\gamma_{(r_{+})}+\epsilon_{\ell}]) \le 2\exp\left(-\kappa_{d,u}N_{\ell}\right)~.
		\]
	\end{proof}

	\subsection{A result on Bracketed Sequential Halving based on~\cite{zhao2023revisiting}}\label{sec:bsh_appendix}
	
	Consider a $K$-armed bandit problem with Bernoulli reward with unknown means $\mu_1,\ldots, \mu_K$. As in the previous subsection, we write $\mu_{(1)}\leq \ldots\leq \mu_{(K)}$ for its ordered values.
	Sequential halving is a classical elimination scheme for pure-exploration problems~\citep{karnin2013almost}. It proceeds in at most $\lceil \log_2 K\rceil$ phases. Starting from the full set of $K$ candidate arms, each phase spends approximately $\lfloor T/\log_2 K\rfloor$ samples by distributing them uniformly across the surviving arms, then ranks arms by their empirical means and discards the top half. Since our goal is to identify arms with the smallest mean, we retain the bottom-ranked half after each phase. Bracketed SH uses the bracket-opening and sampling schedule of~\cite{zhao2023revisiting}, with the following recommendation convention: a bracket is eligible only after its opening phase ends (when the next bracket opens), and only if it has completed an SH run. Among eligible brackets we return the representative with the smallest final-stage empirical mean; if there is none, we return an arbitrary arm. With a single candidate arm, we simply return that arm. The following bound combines the SH guarantee of~\cite{zhao2023revisiting} with the bracketing argument.
	\begin{theorem}\label{thm:main0}
		Consider Bracketed SH as defined above with $K\ge2$ and an integer budget $T\ge1$. The output $I_T$ satisfies for any $\epsilon\in(0,1)$ and $m \in [K]$:
		\[
		\mathbb{P}\!\left(\mu_{I_T}\ge \mu_{(m)}+\epsilon\right)
		\le \min\!\left\{1,\;c_1\log(2eKT)
		\exp\!\left[-\frac{c_2m\epsilon^2T}{K\log(eT)\log^3(2K)}\right]\right\},
		\]
		where $c_1,c_2>0$ are universal constants.
	\end{theorem}
	\begin{proof}
		Write $x=m\epsilon^2T/[K\log(eT)\log^3(2K)]$. By increasing $c_1$ and decreasing $c_2$, the claim is vacuous when $x$ is smaller than a sufficiently large universal constant; this covers small budgets and the case with no eligible bracket. We henceforth assume that $x$ exceeds this constant. Let $a>0$ be a sufficiently small universal constant and set
		\[
		z=\frac{a\epsilon^2T}{\log(eT)\log^2(2m)\log(2K)}~.
		\]
		The $b$-th bracket of Bracketed SH contains $\min\{K,2^b\}$ arms and its opening phase ends when bracket $b+1$ opens. Choose a bracket of size $n=K$ if $z\ge K$, and otherwise a bracket of size $n\in[z/2,z]$. Such a bracket exists since the sizes double until they reach $K$; the choice of $a$ ensures that its opening phase has ended by time $T$. Indeed, the phase of a bracket of size $2^b$ ends by time $b2^b$, while $z\log_2(2z)\le C aT$ for an absolute constant $C$. The full-$K$ bracket is likewise mature when $z\ge K$. If $B\ge3$ brackets are open at time $T$, then $T<B2^B$ and the last completed opening phase gives each of the $B-1$ mature brackets at least $2^{B-2}>T/(4B)\ge cT/\log(eT)$ samples, since $B\le C\log(eT)$. Since $n\log_2 n\le CaT/\log(eT)$, our chosen bracket has completed an SH run when $a$ is small enough. In the doubling schedule of SH runs, the last completed run in every eligible bracket uses at least one quarter of that bracket's samples.
		
		Let $X$ be the number of positions in the chosen bracket occupied by globally best-$m$ arms. If $n<K$, then $X\sim\mathrm{Bin}(n,m/K)$. Writing $v=nm/K$, our choice of $n$ gives $v\ge ax/2$, and the binomial lower-tail bound gives
		\[
		\mathbb{P}(X<v/2)\le \exp(-v/8)\le\exp(-c x)~.
		\]
		On the complementary event, $q=\floor{v/2}\ge v/4$ of the bracket's arms have means at most $\mu_{(m)}$; the initial reduction to nonvacuous $x$ ensures $v\ge4$. If $n=K$, the bracket contains every arm and we instead take $q=m$, with no binomial error.
		
		Conditional on the bracket, apply the SH bound of Theorem~4 in~\cite{zhao2023revisiting} to the last completed SH run, with the reward signs reversed and accuracy $\epsilon/4$. Its budget is at least $cT/\log(eT)$, while $q\le m$ and $n\le z$ in the non-full case. Since the $q$-th mean in the bracket is at most $\mu_{(m)}$, choosing $a$ small enough absorbs the offset in that bound and yields
		\[
		\mathbb{P}\!\left(\mu_{J_*}\ge\mu_{(m)}+\epsilon/2,\ X\ge q\right)
		\le C\log(2K)\exp(-c x)~,
		\]
		where $J_*$ is the chosen bracket's representative. In the full-$K$ case the same conclusion follows directly from its run budget and $q=m$. Here and below $C,c>0$ denote universal constants, which may change from line to line.
		
		It remains to compare the representatives. In every completed SH run, the final-stage candidate set has at most two arms and is determined before the fresh final-stage samples are drawn. The representative's final-stage empirical mean may therefore be controlled by a union bound over these two candidates, without treating the selected arm as fixed. Every eligible completed bracket has a last-run budget at least $cT/\log(eT)$, so each final-stage candidate receives at least $cT/[\log(eT)\log(2K)]$ samples. There are at most $C\log(eT)$ eligible brackets. Hoeffding's inequality and a union bound thus give
		\[
		\mathbb{P}\!\left(\exists\text{ eligible completed bracket }b:\;\left|\hat\mu_{J_b}-\mu_{J_b}\right|\ge\epsilon/4\right)
		\le C\log(eT)\exp\!\left(-\frac{c\epsilon^2T}{\log(eT)\log(2K)}\right)~.
		\]
		Outside these events, the chosen bracket has true mean below $\mu_{(m)}+\epsilon/2$, and its empirical mean is below $\mu_{(m)}+3\epsilon/4$. The recommended arm has no larger empirical mean, so its true mean is below $\mu_{(m)}+\epsilon$. Combining the three bounds, using $m\le K$ and $\log^2(2m)\le\log^2(2K)$, and absorbing their constants into $c_1,c_2$ proves the claim.
	\end{proof}

	\section{Guarantees on \textsc{FB CWI} procedure (Algorithm~\ref{algo:3}) in the fixed budget setting}\label{sec:proof_fb}
	
	In order to structure the proofs, in this section we present guarantees about the output of Algorithm~\ref{algo:3} when fed with input $(\delta, T)$. More precisely the output being $(\phi_1 \vee \phi_2, I)$, here we provide upper bounds on the probability of misidentification error for the arm candidate $I$. This corresponds to the typical kind of guarantees encountered in context of best arm identification in the fixed budget framework. In turn, we apply these results in Section~\ref{sec:proof_fc} to prove the guarantees presented in Theorem~\ref{thm:fc2}.

	\subsection{First Upper Bound}
	\begin{theorem}\label{thm:amain1}
		The output of \textsc{FB CWI} (Algorithm~\ref{algo:3}) with input $T\ge 2$ satisfies:
		\[
		\mathbb{P}\left( \psi_T \neq i^*\right) \le 27K\log(K)\log(T) \cdot\exp\left(-c\cdot\frac{T}{\log(T)\log(K)H_{\text{cw}}} \right),
		\]
		where $c$ is a numerical constant, and we recall that $H_{\text{cw}}$ is defined by
		\[
		H_{\text{cw}} = \sum_{i\neq i^*} \frac{1}{\Delta^2_{i^*,i}}~,
		\]
		if $\Delta_{i^*,i}>0$ for all $i \in [K]\setminus \{i^*\}$ and $H_{\text{cw}} = +\infty$ otherwise.
	\end{theorem}

	\paragraph{Notation:} 
	Let $\bm{\Delta}^{(k)} \in [-1/2, 1/2]^{\abs{A_k}\times \abs{A_k}}$ denote the sub-matrix of $\bm{\Delta}$ restricted to rows and columns in $A_k$. For $\alpha \in A_k$, let $\left(\Delta^{(k)}_{\alpha,(i)}\right)_{i \in \{1, \dots, \abs{A_k}-1\}}$ denote the ordered gaps between $\alpha$ and arms in $A_k\setminus \{\alpha\}$ such that:
	\[
	\Delta^{(k)}_{ \alpha,(1)} \le \dots \le \Delta^{(k)}_{ \alpha,(\abs{A_k}-1)}~.
	\]
	Since the gaps sub-matrix for arms in $A_k$ is skew-symmetric (i.e., $\forall i,j: \bm{\Delta}_{i,j}^{(k)} = -\bm{\Delta}_{j,i}^{(k)}$), the number of arms such that $\Delta^{(k)}_{\alpha, \left(\ceil*{\abs{A_k}/4}\right)} \le 0$ is at least $\ceil*{\abs{A_k}/4}$ (see Lemma~\ref{lem:pure_tech3}). Let $E_k \subset A_k$ denote the last set of arms: 
	\[
	E_k := \left\lbrace \alpha \in A_k: \Delta^{(k)}_{\alpha, \left(\ceil*{\abs{A_k}/4}\right)} \le 0 \right\rbrace~.
	\]
	Finally, we remind the reader that for any $j \in [K]$, the quantities $\Delta_{j,(1)} \le \dots \le \Delta_{j,(K-1)}$ correspond to the ordered gaps between $j$ and all arms in $[K]\setminus \{j\}$.
	
	\paragraph{Conditioning convention:}
	For each round $k$, let $\mathcal{F}_{k-1}$ denote the complete history before the scores at round $k$ are computed. All roundwise estimates below are established conditionally on $\mathcal{F}_{k-1}$, for every history such that $i^*\in A_k$. In particular, $A_k$, $B_k$, $E_k$, $F_k$ and the associated gap thresholds are then fixed. We omit this common conditioning from the notation within the technical lemmas and the roundwise calculations unless it is explicitly displayed. The final bounds follow by averaging these conditional estimates over the history.
	
	\paragraph{Proof of Theorem~\ref{thm:amain1}.}
	Suppose that $T \ge 8K\log_{8/7}(K)$, otherwise the bound is vacuous. Assume that $\Delta_{i^*,i}>0$ for all $i\neq i^*$. Otherwise, if $\Delta_{i^*,j}=0$ for some
	$j\neq i^*$, then $H_{\mathrm{cw}}=+\infty$ and the stated bound is trivial.
	
	\medskip
	
	We start by bounding the probability of the event $\psi_T \neq i^*$ by the probabilities that $i^*$ gets eliminated at some step $k$.
	Since $i^*\in A_1=[K]$, the event $\{\psi_T\neq i^*\}$ implies that there exists a round $k\in\{1,\dots,k_{\max}-1\}$ such that $i^*\in A_k$ but
	$i^*\notin A_{k+1}$. Hence,
	\begin{align}
		\mathbb{P}(\psi_T\neq i^*)
		&= \mathbb{P}\Big(\bigcup_{k=1}^{k_{\max}-1}\{i^*\in A_k,\ i^*\notin A_{k+1}\}\Big)\nonumber\\
		&\le \sum_{k=1}^{k_{\max}-1}\mathbb{P}(i^*\in A_k,\ i^*\notin A_{k+1})\nonumber\\
		&\le k_{\max}\cdot \max_{k\in\{1,\dots,k_{\max}-1\}} \mathbb{P}(i^*\notin A_{k+1}\mid i^*\in A_k)\,.\label{eq:bn1}
	\end{align}
	Recall that $k_{\max}\le \lceil\log_{8/7}(K)\rceil$.
	
	\medskip
	
	\noindent Next, we upper-bound $\mathbb{P}(i^*\notin A_{k+1}\mid i^*\in A_k)$. Fix $k\in\{1,\dots,k_{\max}-1\}$ and condition on a realization of $\mathcal{F}_{k-1}$ such that $i^*\in A_k$.
	If $i^*\notin A_{k+1}$, then by the definition of the next set (keeping only the top fraction),
	the number of arms in $A_k$ with score smaller than $S_k(i^*)$ is at most
	$\ceil{\abs{A_k}/8}$. Equivalently, at least $\abs{A_k}-\ceil{\abs{A_k}/8}$ arms in $A_k$ have score at least $S_k(i^*)$. 
	
	\noindent By Lemma~\ref{lem:39} applied to the skew-symmetric matrix $\bm{\Delta}^{(k)}$, we conclude that if $\abs{A_k}\ge 3$ then the intersection between $E_k$ and $A_{k+1}$ (which have a size of $\abs{A_k}-\ceil{\abs{A_k}/8}$) is non-empty, therefore
	\[
	\exists \alpha\in E_k:\quad S_k(\alpha)\ge S_k(i^*).
	\]
	Otherwise, if $\abs{A_k} = 2$ and $i^* \in A_k$, we necessarily have $E_k = A_{k}\setminus \{i^*\}$. We conclude that
	\begin{align*}
		&\mathbb{P}(i^*\notin A_{k+1}\mid\mathcal{F}_{k-1})
		\le \mathbb{P}\Big(\exists \alpha\in E_k:\ S_k(\alpha)\ge S_k(i^*)\Big)\\
		& \quad \le \underbrace{\mathbb{P}\left(S_k(i^*) \le \frac12\,\Delta^{(k)}_{i^*,(\ceil{\abs{A_k}/8})}\right)}_{\text{Term 1}}
		+\underbrace{\mathbb{P}\left(\exists \alpha\in E_k:\ S_k(\alpha)\ge \frac12\,\Delta^{(k)}_{i^*,(\ceil{\abs{A_k}/8})}\right)}_{\text{Term 2}}~.
	\end{align*}
	Denote $\Delta_k \coloneq \Delta^{(k)}_{i^*,(\ceil{\abs{A_k}/8})}$. By Lemma~\ref{lem:main1}, Terms 1 and 2 satisfy
	\begin{align}
		\mathbb{P}(i^*\notin A_{k+1}\mid\mathcal{F}_{k-1})
		&\le (K+2\log(T))\exp\left(-c\,\frac{\Delta_k^2}{\log(\ceil{B_k/2})}\,B_k\right)\nonumber\\
		&\qquad + K\log(T)\exp\left(-c'\,\frac{\Delta_k^2}{\log(\ceil{B_k/2})}\,B_k\right)\nonumber\\
		&\le 3K\log(T) \exp\left(-c_1\,\frac{\Delta_k^2}{\log(\ceil{B_k/2})}\,B_k\right)~,
		\label{eq:bt1}
	\end{align}
	where $c_1=\min\{c,c'\}$.
	
	\medskip
	
	\noindent
	Next, we develop a bound on $\Delta^{(k)}_{i^*,(\ceil{\abs{A_k}/8})}$ using $H_{\mathrm{cw}} = \sum_{i \neq i^*} \frac{1}{\Delta_{i,i^*}^2}$.
	Recall $B_k=\left\lfloor \frac{T}{\abs{A_k}\log_{8/7}(K)}\right\rfloor$. We have
	\begin{align*}
		\Delta_k^2\,B_k
		&= \Delta_k^2\left\lfloor \frac{T}{\abs{A_k}\log_{8/7}(K)}\right\rfloor\\
		&\ge \Delta_k^2\cdot \frac{T}{2\abs{A_k}\log_{8/7}(K)}\\
		&\ge \frac{\Delta_k^2}{\lceil \abs{A_k}/8\rceil}\cdot \frac{T}{16\log_{8/7}(K)}\\
		&\ge \frac{T}{16\log_{8/7}(K)\cdot \sum_{i\neq i^*}\frac{1}{\Delta_{i^*,i}^2}}
		= \frac{T}{16\log_{8/7}(K)\,H_{\mathrm{cw}}}\,,
	\end{align*}
	where we used in the second line the fact that $T \ge 8K\log_{8/7}(K)$ and Lemma~\ref{lem:T} in the last line (using $\Delta_k \coloneq \Delta^{(k)}_{i^*,(\ceil{\abs{A_k}/8})}$).
	Plugging this into~\eqref{eq:bt1} yields
	\begin{align}
		\mathbb{P}(i^*\notin A_{k+1}\mid\mathcal{F}_{k-1}) &\le 3K\log(T) \cdot \exp\!\left(-c_1\,\frac{T}{16\log_{8/7}(K)\,\log(\lceil B_k/2\rceil)\,H_{\mathrm{cw}}}\right)\nonumber\\
		&\le 3K\log(T) \exp\!\left(-c_1'\,\frac{T}{\log(K)\log(T)\,H_{\mathrm{cw}}}\right)~,
		\label{eq:b_case2}
	\end{align}
	for numerical constants $c_1,c_1'>0$ (using $\log_{8/7}(K)=\Theta(\log K)$ and $\log(\lceil B_k/2\rceil)\le \log(T)$).
	
	\medskip
	
	\noindent The bound above is uniform over all histories satisfying $i^*\in A_k$. Therefore, using
	\[
	\mathbb{P}(i^*\notin A_{k+1}\mid i^*\in A_k)
	=\mathbb{E}\!\left[
	\mathbb{P}(i^*\notin A_{k+1}\mid\mathcal{F}_{k-1})
	\,\middle|\,i^*\in A_k
	\right],
	\]
	the same bound holds conditionally only on $i^*\in A_k$.
	
	\medskip
	
	\noindent Finally, we combine the bounds~\eqref{eq:bn1} and~\eqref{eq:b_case2}, and use $k_{\max}\le \ceil{\log_{8/7}(K)}$, we get
	\[
	\mathbb{P}(\psi_T\neq i^*) \le \lceil\log_{8/7}(K)\rceil\,\cdot 3K\log(T) \exp\!\left(-c_1'\,\frac{T}{\log(K)\log(T)\,H_{\mathrm{cw}}}\right)~,
	\]
	which yields the claimed form
	\[
	\mathbb{P}(\psi_T\neq i^*) \le 27K\log(K)\log(T)\exp\!\left(-c\,\frac{T}{\log(T)\log(K)H_{\mathrm{cw}}}\right)~,
	\]
	for a numerical constant $c>0$.

	\medskip 
	
	\noindent It remains to prove the following technical lemma.

	\begin{lemma}\label{lem:main1}
		Consider step $k$ in Algorithm~\ref{algo:3} with $T\ge 8K\log_{8/7}(K)$ and assume that $i^* \in A_k$. Let $\Delta_k \coloneq \Delta^{(k)}_{i^*, (\lceil \abs{A_k}/8\rceil)}$, then
		\begin{align*}
			\mathbb{P}\left( S_k(i^*) \le \frac{1}{2} \Delta_k\right)
			&\le (K+2\log(T))\exp\left( -c \frac{\Delta_{k}^2}{\log(\lceil B_k/2\rceil)}\,B_k\right),\\
			\mathbb{P}\left(\exists \alpha \in E_k:\ S_k(\alpha) \ge \frac{1}{2} \Delta_{k}\right)
			&\le K\log(T)\exp\left(-c'~\frac{\Delta_{k}^2}{\log(\lceil B_k/2\rceil)}\,B_k\right)~,
		\end{align*}
		for numerical constants $c,c'>0$.
	\end{lemma}
	
	\begin{proof}
		Assume $T \ge 8K\log_{8/7}(K)$, this guarantees $B_k=\left\lfloor\frac{T}{\abs{A_k}\log_{8/7}(K)}\right\rfloor \ge 8$.
		Let $c_1$ denote the constant from Corollary~\ref{cor:quant}.
		
		\medskip
		\noindent{\bfseries Proof of the first bound:}
		Let $i_s^*$ be the strong opponent chosen for $i^*$ at step $k$ of Algorithm~\ref{algo:3}.
		Recall
		\[
		S_k(i^*)=\min\{Z_k^{(s)}(i^*),0\}+Z_k^{(w)}(i^*)\,.
		\]
		Therefore, we have
		\begin{equation}
			\mathbb{P}\!\left(S_k(i^*)\le \frac12\Delta_k\right) \le
			\mathbb{P}\!\left(Z_k^{(s)}(i^*)+Z_k^{(w)}(i^*)\le \frac12\Delta_k\right)
			+\mathbb{P}\!\left(Z_k^{(w)}(i^*)\le \frac12\Delta_k\right)\,.\label{eq:lem-main1-split}
		\end{equation}
		
		\smallskip
		\noindent Recall that the event $\left\lbrace Z_k^{(s)}(i^*)+ Z_k^{(w)}(i^*) \le \frac{1}{2}\Delta_k \right\rbrace$ implies that
		\[
		\left\lbrace Z_k^{(s)}(i^*) \le -\frac{1}{4}\Delta_k~\text{ or }~Z_k^{(w)}(i^*) \le \frac{3}{4}\Delta_k \right\rbrace~,
		\]
		Combining with Inequality~\ref{eq:lem-main1-split} we obtain
		{\small \begin{equation}\label{eq:ub}
				\mathbb{P}\left(S_k(i^*) \le \frac{1}{2}\Delta_k\right) \le \mathbb{P}\left(Z_k^{(s)}(i^*) \le -\frac{1}{4}\Delta_k\right)+2\mathbb{P}\left(Z_k^{(w)}(i^*) \le \frac{3}{4}\Delta_k \right)~.
		\end{equation}}
		\noindent We use Hoeffding's inequality (Lemma~\ref{lem:H}) to bound the first term in the upper bound above. For any fixed $i \in [K]\setminus\{i^*\}$ and $\epsilon>0$, we have
		\begin{equation*}
			\mathbb{P}\left(\hat{\Delta}_{i^*, i}-\Delta_{i^*, i} \le -\epsilon \right) \le \exp\left(-\frac{\epsilon^2B_k}{2} \right)~,	
		\end{equation*}
		where $\hat{\Delta}_{i^*, i}$ is the empirical mean of duels between $(i^*, i)$ computed using $\ceil{B_k/4}$ samples. Therefore, applying the bound above with $\epsilon = \Delta_k/4$ and a union bound over the arms, we have
		\begin{align}
			\mathbb{P}\left(Z_k^{(s)}(i^*) \le -\frac{1}{4}\Delta_k\right) &\le \mathbb{P}\left(Z_k^{(s)}(i^*)-\Delta_{i^*,i^*_s} \le -\frac{1}{4}\Delta_k\right)\nonumber\\
			&\le (K-1)\exp\left( -\frac{\Delta_k^2}{32}B_k\right)~.\label{eq:p1}
		\end{align}
		where we used the fact that $\Delta_{i^*, j} \ge 0$ for all $j \in [K]$. Now, using the shifted-rank version of Corollary~\ref{cor:quant} with $N=\abs{A_k}-1$, which gives a guarantee on the output $Z_k^{(w)}(i^*)$, we have
		\begin{align}
			\mathbb{P}\left( Z_k^{(w)}(i^*) \le \frac{3}{4}\Delta_k \right) &= \mathbb{P}\left( Z_k^{(w)}(i^*) \le \Delta^{(k)}_{i^*,(\ceil{\abs{A_k}/8})}-\frac{1}{4}\Delta_k \right)\nonumber\\
			&\le \log\left(\ceil*{\frac{B_k}{2}}\right)\exp\left( -c_1 \frac{\Delta_k^2}{32\log(\ceil{B_k/2})}B_k\right)~. \label{eq:p2}
		\end{align}
		
		\smallskip
		\noindent We conclude by combining~\eqref{eq:p1},~\eqref{eq:p2} and~\eqref{eq:ub} that 
		\begin{equation*}
			\mathbb{P}\left(S_k(i^*) \le \frac{1}{2}\Delta_{k}\right) \le (K-1+2\log(T))\exp\left( -c\frac{\Delta_{k}^2}{\log(\ceil{B_k/2})}B_k\right)~,
		\end{equation*}
		where $c$ is a numerical constant.

		\medskip
		\noindent{\bfseries Proof of the second bound:}
		Fix $\alpha\in E_k$. By definition of $E_k$, we have
		\[
		\Delta^{(k)}_{\alpha,(\lceil \abs{A_k}/4\rceil)} \le 0~.
		\]
		Moreover, $\min\{Z_k^{(s)}(\alpha),0\}\le 0$, hence
		\begin{align*}
			\mathbb{P}\left(S_k(\alpha)\ge \frac12\Delta_k\right)
			&= \mathbb{P}\left(\min\{Z_k^{(s)}(\alpha),0\}+Z_k^{(w)}(\alpha)\ge \frac12\Delta_k\right)\\
			&\le \mathbb{P}\left(Z_k^{(w)}(\alpha)\ge \frac12\Delta_k\right)\\
			&\le \mathbb{P}\left(Z_k^{(w)}(\alpha)\ge \Delta^{(k)}_{\alpha,(\lceil \abs{A_k}/4\rceil)}+\frac12\Delta_k\right)\,,
		\end{align*}
		where the last step uses $\Delta^{(k)}_{\alpha,(\lceil \abs{A_k}/4\rceil)}\le 0$.
		Applying the shifted-rank version of Corollary~\ref{cor:quant} with $N=\abs{A_k}-1$ to $Z_k^{(w)}(\alpha)$ then yields
		\begin{equation}\label{eq:lem-main1-eachalpha}
			\mathbb{P}\left(S_k(\alpha)\ge \frac12\Delta_k\right) \le \log\left(\left\lceil\frac{B_k}{2}\right\rceil\right)
			\exp\left(-c_1\,\frac{\Delta_k^2}{8\log(\lceil B_k/2\rceil)}\,B_k\right)~.
		\end{equation}
		
		\smallskip
		\noindent Finally, union bound over $\alpha\in E_k$ gives
		\[
		\mathbb{P}\left(\exists \alpha\in E_k:\ S_k(\alpha)\ge \frac12\Delta_k\right)
		\le \abs{E_k} \log\left(\left\lceil\frac{B_k}{2}\right\rceil\right) \exp\left(-c'\,\frac{\Delta_k^2}{\log(\lceil B_k/2\rceil)}B_k\right)~,
		\]
		\noindent We may bound $\abs{E_k}$ by $K$; hence the above yields the stated form
		\[
		\mathbb{P}\left(\exists \alpha \in E_k:\ S_k(\alpha) \ge \frac12 \Delta_k\right)
		\le K\log(T)\exp\left(-c'~\frac{\Delta_k^2}{\log(\lceil B_k/2\rceil)}\,B_k\right)~,
		\]
		for a numerical constant $c'>0$. This proves the second inequality and concludes the lemma.
	\end{proof}
	
	\subsection{Second Upper Bound}
	For each $i \neq i^*$, let $\Delta_{i,(k)}$ denote the ordered gaps $(\Delta_{i,j})_{i\neq j}$ as 
	\[
	\Delta_{i,(1)} \le \dots \le \Delta_{i, (K-1)}~,
	\]
	Denote by $K_{i;<0}$ the number of $j$ such that $\Delta_{i,j} < 0$. For each $i \in [K]$, let $s_i \le K_{i;<0}$, and $\bm{s}=(s_1, \dots, s_K)$. Here, we take the convention $K_{i^*;<0}=0$. We recall the expressions of the quantities $H_{\text{certify}}(\bm{s})$, $H_{\text{explore}}^{(0)}(\bm{s})$ and $H_{\text{explore}}^{(1)}(\bm{s})$
	\[
	H_{\text{certify}}(\bm{s}) = \sum_{i\neq i^*} \frac{1}{\Delta_{i,(s_i)}^2},\quad H_{\text{explore}}^{(1)}(\bm{s}) = \max_{i \neq i^*} \frac{K}{s_i\Delta_{i,(s_i)}^2}~\text{ and }~H^{(0)}_{\text{explore}}(\bm{s}) = \sum_{i \neq i^*} \frac{K}{s_i\Delta_{i,(s_i)}^2}~.
	\]
	
	\begin{theorem}\label{thm:amain2}
		For $T\ge 2$ and any $\bm{s}$ such that, for every $i\neq i^*$, $1 \le s_i \le K_{i;<0}$, it holds that 
		\begin{align*}
			\mathbb{P}\left( \psi_T \neq i^*\right) &\le 47K\log(K)\log(T)\\
			&\quad\times\exp\!\left( -\frac{c_1}{\log^4(K)\log(T)}\frac{T-c_2\log^6(H^{(0)}_{\text{explore}}(\bm{s}))\cdot H^{(0)}_{\text{explore}}(\bm{s})}{H_{\text{explore}}^{(1)}(\bm{s})+H_{\text{certify}}(\bm{s})}\right), 
		\end{align*}
		where $c_1$ and $c_2$ are numerical constants.
	\end{theorem}

	\noindent We restate and extend the notation introduced in the last section.
	\paragraph{Notation:} 
	Let $\bm{\Delta}^{(k)} \in [-1/2, 1/2]^{\abs{A_k}\times \abs{A_k}}$ denote the sub-matrix of $\bm{\Delta}$ restricted to lines and rows in $A_k$. For $\alpha \in A_k$, let $\left(\Delta^{(k)}_{\alpha,(i)}\right)_{i \in \{1, \dots, \abs{A_k}-1\}}$ denote the ordered gaps between $\alpha$ and arms in $A_k$ such that:
	\[
	\Delta^{(k)}_{ \alpha,(1)} \le \dots \le \Delta^{(k)}_{ \alpha,(\abs{A_k}-1)}.
	\]

	\noindent Recall that since the gaps sub-matrix for arms in $A_k$ is skew-symmetric (i.e., $\forall i,j: \bm{\Delta}_{i,j}^{(k)} = -\bm{\Delta}_{j,i}^{(k)}$), the number of arms such that $\Delta^{(k)}_{\alpha, (\ceil*{\abs{A_k}/4})} \le 0$ is at least $\ceil*{\abs{A_k}/4}$ (see Lemma~\ref{lem:pure_tech3}). Let $E_k \subset A_k$ denote the last set of arms
	\[
	E_k := \left\lbrace \alpha \in A_k: \Delta^{(k)}_{\alpha, (\ceil{\abs{A_k}/4})} \le 0 \right\rbrace~.
	\]
	We rank the quantities $(\Delta_{\alpha, (s_{\alpha})})_{\alpha \in E_k}$. We denote the ranked sequence with ties broken arbitrarily $\left(\Delta_{E_k:i}\right)_{i \in [\abs{E_k}]}$
	\[
	\Delta_{E_k:1} \le \dots \le \Delta_{E_k:\abs{E_k}}~.
	\] 
	
	\noindent Define the quantity $\bar{\Delta}_k$ by
	\begin{equation}\label{eq:def_bar}
		\bar{\Delta}_k := \Delta_{E_k:\ceil*{\frac{7}{8}\abs{E_k}}}\le 0~.
	\end{equation}
	Observe that when $\bar{\Delta}_k = 0$, we necessarily have $\Delta_{i,(s_i)} = 0$ for some $i \in [K]\setminus \{i^*\}$, this implies in particular that $H_{\text{certify}} = \infty$ and the bound becomes loose. Therefore, in the remainder of this proof, we assume that $\bar{\Delta}_k <0$.
	
	\smallskip
	\noindent Let $F_k$ denote the subset of arms in $E_k$ such that $\Delta_{\alpha, (s_{\alpha})} \le \bar{\Delta}_k$. 
	\begin{equation}\label{eq:def_F}
		F_k := \left\lbrace \alpha \in E_k: \Delta_{\alpha, (s_{\alpha})} \le \bar{\Delta}_k\right\rbrace~.	
	\end{equation}
	
	\noindent Finally, we denote for each $i\neq i^*:~\Gamma_i := s_i \Delta_{i, (s_i)}^2$, and let $(\Gamma_{(i)})_{i \neq i^*}$ correspond to the ranked quantities $\Gamma_{(2)} \le \dots \le \Gamma_{(K)}$, with ties broken arbitrarily.

	\paragraph{Proof of Theorem~\ref{thm:amain2}.}
	Fix $\bm{s}$ such that $1\le s_i\le K_{i;<0}$ for all $i\neq i^*$ (and $K_{i^*;<0}=0$ by convention). Note that by the assumption of the uniqueness of the Condorcet winner we have $K_{i;<0} \ge 1$ for any $i \neq i^*$.
	Let $c>0$ be a numerical constant (chosen smaller than the constants appearing in Corollary~\ref{cor:quant} and Theorem~\ref{thm:main0}). We assume that $T \ge 8K\log_{8/7}(K)$, otherwise the bound of the theorem is vacuous.
	
	\medskip
	
	\noindent Similar to the proof of Theorem~\ref{thm:amain1}, we start by bounding the probability of the event $\psi_T \neq i^*$ by the probabilities that $i^*$ gets eliminated at some step $k$. We have
	\begin{equation}\label{eq:union_rounds_amain2}
		\mathbb{P}(\psi_T\neq i^*)
		\le \sum_{k=1}^{k_{\max}-1} \mathbb{P}(i^*\notin A_{k+1},\ i^*\in A_k)
		\le k_{\max}\cdot \max_{k\le k_{\max}-1}\mathbb{P}(i^*\notin A_{k+1}\mid i^*\in A_k)\,.
	\end{equation}
	Recall $k_{\max}\le \lceil \log_{8/7}(K)\rceil$. Fix $k \in \{1, \dots, k_{\max}-1\}$ and condition on a realization of $\mathcal{F}_{k-1}$ such that $i^*\in A_k$. We will first consider the case where $\abs{A_k} \ge 3$. The case where $\abs{A_k} = 2$ is simple and is left to the end of this proof.
	
	\medskip
	Next, we build the argument of our proof on the observation that given $i^* \in A_k$, the event $i^* \notin A_{k+1}$ implies in particular that the number of arms $\alpha \in A_k$ with a score $S_k(\alpha)$ at least as large as $S_k(i^*)$ is at least $\abs{A_k}-\ceil{\abs{A_k}/8}$. Therefore, the event $i^* \notin A_{k+1}$ implies that the number of arms in $F_k$ with a score $S_k(\cdot)$ at least as large as $S_k(i^*)$ is at least $\abs{A_{k+1} \cap F_k} \ge \ceil{\abs{F_k}/3}$ as stated in the following lemma
	\begin{lemma}\label{lem:Fk}
		Let $k \in \{1, \dots, k_{\max}-1\}$, recall the definition of $F_k$ given in~\eqref{eq:def_F}. We have, if $\abs{A_k} \ge 3$ then
		\[
		\abs{A_{k+1} \cap F_k} \ge \ceil*{\frac{1}{3}\abs{F_k}}~.
		\]
	\end{lemma}
	This lemma implies that if $i^*$ is eliminated at step $k$ (i.e., $i^* \notin A_{k+1}$), then many ``bad'' arms in $F_k$ beat $i^*$. More precisely, since $A_{k+1}$ consists of the top-scoring arms at step $k$, every $\alpha\in A_{k+1}$ satisfies
	$S_k(\alpha)\ge S_k(i^*)$ whenever $i^*\notin A_{k+1}$. Therefore,
	\[
	\{i^*\notin A_{k+1}\}\ \subseteq\ 
	\left\{ \big|\{\alpha\in F_k:\ S_k(\alpha)\ge S_k(i^*)\}\big|
	\ \ge\ \left\lceil\frac{1}{3}\abs{F_k}\right\rceil \right\}.
	\]
	Introduce the threshold $\frac{1}{2}\bar{\Delta}_k$ defined by~\eqref{eq:def_bar} and split
	\begin{align*}
		\mathbb{P}(i^*\notin A_{k+1}\mid\mathcal{F}_{k-1})
		&\le \mathbb{P}\!\left(\big|\{\alpha\in F_k:\ S_k(\alpha)\ge S_k(i^*)\}\big|
		\ge \left\lceil\frac{1}{3}\abs{F_k}\right\rceil\right)\\
		&\le \underbrace{\mathbb{P}\!\left(S_k(i^*)\le \frac{1}{2}\bar{\Delta}_k\right)}_{\text{Term 1}} + \underbrace{\mathbb{P}\!\left(\big|\{\alpha\in F_k:\ S_k(\alpha)\ge \frac{1}{2}\bar{\Delta}_k\}\big| \ge \left\lceil\frac{1}{3}\abs{F_k}\right\rceil\right)}_{\text{Term 2}}\,.
	\end{align*}
	
	\medskip
	\noindent The following lemma is a key step in the proof, we postponed its proof to the next subsection.
	\begin{lemma}\label{lem:main2}
		Under the assumptions of Theorem~\ref{thm:amain2}, with $T\ge 8K\log_{8/7}(K)$, consider step $k$ in Algorithm~\ref{algo:3} and assume that $i^*\in A_k$. Then, we have
		{\small
			\begin{align*}
				\mathbb{P}\left(S_k(i^*) \le \frac{1}{2}\bar{\Delta}_k \mid \mathcal{F}_{k-1} \right) &\le (K+\log(T))\exp\left(-c\frac{\bar{\Delta}_k^2}{\log(B_k)}B_k\right) \\
				\mathbb{P}\left( \abs{\{ \alpha \in F_k: \frac{1}{2} \bar{\Delta}_k \le S_k(\alpha)\}} \ge \ceil*{\frac{1}{3} \abs{F_k}} \mid \mathcal{F}_{k-1}\right)
				&\le \exp\left(-\frac{c}{\log^4(K)\log(T)}\cdot\frac{T-c'\log^6(H^{(0)}_{\text{explore}}(\bm{s}))\cdot H^{(0)}_{\text{explore}}(\bm{s})}{H_{\text{explore}}^{(1)}(\bm{s})}\right)~,
		\end{align*}}
		where $c$ and $c'$ are positive numerical constants.
	\end{lemma}
	A direct application of the lemma above gives (for numerical constants $c,c'>0$)
	\begin{align}
		\mathbb{P}(i^*\notin A_{k+1}\mid\mathcal{F}_{k-1})
		&\le (K+\log(T))\exp\!\left(-c\,\frac{\bar{\Delta}_k^2}{\log(B_k)}\,B_k\right)\nonumber\\
		&\quad + \exp\!\left(
		-\frac{c}{\log^4(K)\log(T)}\cdot
		\frac{T-c'\log^6(H^{(0)}_{\mathrm{explore}}(\bm{s}))\,H^{(0)}_{\mathrm{explore}}(\bm{s})}
		{H^{(1)}_{\mathrm{explore}}(\bm{s})}
		\right).
		\label{eq:tm2_rewrite}
	\end{align}
	
	\medskip
	
	\noindent Next, we will convert the dependence of the bound on $\bar{\Delta}_k$ into $H_{\mathrm{certify}}(\bm{s})$.
	
	\noindent Recall $\abs{E_k}\ge \lceil \abs{A_k}/4\rceil$. Since $\bar{\Delta}_k=\Delta_{E_k:\lceil \frac{7}{8}\abs{E_k}\rceil}$,
	the gaps with ranks at least $\lceil 7\abs{E_k}/8\rceil$ lie between $\bar{\Delta}_k$ and zero. Hence $\sum_{\alpha\in E_k}1/\Delta_{\alpha,(s_\alpha)}^2\ge(\abs{E_k}-\lceil 7\abs{E_k}/8\rceil+1)/\bar{\Delta}_k^2\ge\lceil\abs{E_k}/8\rceil/\bar{\Delta}_k^2$, which yields
	\begin{align}
		\abs{A_k}\cdot \frac{1}{\bar{\Delta}_k^2} &\le 4\abs{E_k}\cdot \frac{1}{\bar{\Delta}_k^2}
		\le 32\cdot \left\lceil \frac{\abs{E_k}}{8}\right\rceil \cdot \frac{1}{\Delta_{E_k:\lceil\frac{7}{8}\abs{E_k}\rceil}^2}
		\le 32\sum_{\alpha\in E_k}\frac{1}{\Delta_{\alpha,(s_\alpha)}^2}
		\le 32 H_{\mathrm{certify}}(\bm{s}).
		\label{eq:fc_bound_0_rewrite}
	\end{align}
	Using $B_k=\left\lfloor \frac{T}{\abs{A_k}\log_{8/7}(K)}\right\rfloor\ge \frac{T}{2\abs{A_k}\log_{8/7}(K)}$
	(and $\log(B_k)\le \log(T)$), we obtain
	\begin{align}
		(K+\log(T))\exp\!\left(-c\,\frac{\bar{\Delta}_k^2}{\log(B_k)}\,B_k\right)
		&\le 2(K+\log(T)) \exp\!\left( -c'\,\frac{T}{H_{\mathrm{certify}}(\bm{s})\,\log(T)\,\log(K)} \right)\,,
		\label{eq:tm3_rewrite}
	\end{align}
	for a numerical constant $c'>0$ (absorbing $\log_{8/7}(K)=\Theta(\log K)$ into constants).
	
	\medskip

	\noindent Next combine~\eqref{eq:tm2_rewrite} and~\eqref{eq:tm3_rewrite}, and using
	$\exp(-a)+\exp(-b)\le 2\exp(-\min\{a,b\})$, we get that if $\abs{A_k} \ge 3$, we have
	\begin{align*}
		\mathbb{P}(i^*\notin A_{k+1}\mid\mathcal{F}_{k-1}) &\le 3(K+\log(T))\\
		&\quad\times\exp\!\left(-\frac{c}{\log^4(K)\log(T)}\cdot
		\frac{T-c_2\log^6(H^{(0)}_{\mathrm{explore}}(\bm{s}))\,H^{(0)}_{\mathrm{explore}}(\bm{s})} {H^{(1)}_{\mathrm{explore}}(\bm{s})+H_{\mathrm{certify}}(\bm{s})}\right)\,,
	\end{align*}
	for a numerical constant $c_2>0$ (renaming constants). To conclude we need to consider the edge case where $\abs{A_k} = 2$ (last iteration). In this case we have $E_k=F_k=\{\alpha\}$. Therefore,
	\begin{align*}
		\mathbb{P}(i^*\notin A_{k+1}\mid\mathcal{F}_{k-1})
		&\le \mathbb{P}\!\left( S_k(\alpha)\ge S_k(i^*)\right)\\
		&\le \mathbb{P}\!\left(S_k(i^*)\le \frac{1}{2}\bar{\Delta}_k\right) + \mathbb{P}\!\left( S_k(\alpha)\ge \frac{1}{2}\bar{\Delta}_k \right)\,.
	\end{align*}
	The first term in the upper bound can be bounded using~\eqref{lem:main2}, the second term can be bounded using Lemma~\ref{lem:prob_f}. The resulting bound is smaller than the one obtained when $\abs{A_k}\ge 3$.
	
	\smallskip
	\noindent The bound above is uniform over all histories satisfying $i^*\in A_k$. Therefore, using
	\[
	\mathbb{P}(i^*\notin A_{k+1}\mid i^*\in A_k)
	=\mathbb{E}\!\left[
	\mathbb{P}(i^*\notin A_{k+1}\mid\mathcal{F}_{k-1})
	\,\middle|\,i^*\in A_k
	\right],
	\]
	the same bound holds conditionally only on $i^*\in A_k$.
	
	\medskip
	
	\noindent Finally, using~\eqref{eq:union_rounds_amain2} and $k_{\max}\le \lceil\log_{8/7}(K)\rceil$,
	and absorbing $\lceil\log_{8/7}(K)\rceil$ and additive logarithms into the prefactor, we obtain
	\begin{align*}
		\mathbb{P}(\psi_T\neq i^*) &\le 47K\log(K)\log(T)\\
		&\quad\times\exp\!\left(-\frac{c_1}{\log^4(K)\log(T)}\cdot \frac{T-c_2\log^6(H^{(0)}_{\mathrm{explore}}(\bm{s}))\,H^{(0)}_{\mathrm{explore}}(\bm{s})}
		{H^{(1)}_{\mathrm{explore}}(\bm{s})+H_{\mathrm{certify}}(\bm{s})}\right)\,,
	\end{align*}
	which is the claim of Theorem~\ref{thm:amain2}.
	
	\subsection{Proofs of Technical Lemmas}
	\subsubsection{Proof of Lemma~\ref{lem:Fk}}
	\begin{proof}
		Recall $\bar{\Delta}_k=\Delta_{E_k:\lceil \frac78 \abs{E_k}\rceil}$ and
		$F_k=\{\alpha\in E_k:\Delta_{\alpha,(s_\alpha)}\le \bar{\Delta}_k\}$, hence
		\begin{equation}\label{eq:Fk_quant}
			\abs{F_k}\ \ge\ \Big\lceil \frac78 \abs{E_k}\Big\rceil.
		\end{equation}
		Algorithm~\ref{algo:3} keeps $\abs{A_{k+1}}=\abs{A_k}-\ceil*{\abs{A_k}/8}$ arms, so for any
		$A_{k+1}\subseteq A_k$ and $F_k\subseteq A_k$,
		\begin{equation}\label{eq:inter_det}
			\abs{A_{k+1}\cap F_k}\ \ge \abs{A_{k+1}}+\abs{F_k}-\abs{A_k} = \abs{F_k}-\ceil*{\frac{\abs{A_k}}{8}}~.
		\end{equation}
		
		\medskip
		\noindent\textbf{Case 1: $\abs{A_k}\ge 5$.}
		By Lemma~\ref{lem:pure_tech3}, $\abs{E_k}\ge \lceil \abs{A_k}/4\rceil\ge 2$, hence
		$\lceil \abs{A_k}/8\rceil\le \lceil \abs{E_k}/2\rceil$.
		Moreover, for every integer $m\ge 2$,
		\begin{equation}\label{eq:keyceil}
			2\Big\lceil \frac{7m}{8}\Big\rceil \ge 3\Big\lceil \frac{m}{2}\Big\rceil~,
		\end{equation}
		(which follows by writing $m=8q+r$ and checking $r\in\{0,\dots,7\}$; the only delicate
		residue $r=1$ is harmless since then $q\ge 1$).
		Applying~\eqref{eq:keyceil} with $m=\abs{E_k}$ and using~\eqref{eq:Fk_quant} gives
		$\lfloor \tfrac23 \abs{F_k}\rfloor \ge \lceil \abs{E_k}/2\rceil \ge \lceil \abs{A_k}/8\rceil$.
		Plugging into~\eqref{eq:inter_det} yields
		\[
		\abs{A_{k+1}\cap F_k} \ge \abs{F_k}-\Big\lfloor \frac23 \abs{F_k}\Big\rfloor = \Big\lceil \frac13\abs{F_k}\Big\rceil~.
		\]
		
		\medskip
		\noindent\textbf{Case 2: $\abs{A_k}\in\{3,4\}$.}
		Here $\lceil \abs{A_k}/8\rceil=1$. By skew-symmetry of $\Delta^{(k)}$, at most one row can have
		all off-diagonal entries $>0$, hence at least $\abs{A_k}-1$ rows have $\Delta^{(k)}_{\alpha,(1)}\le 0$,
		so $\abs{E_k}\ge \abs{A_k}-1\in\{2,3\}$. Then $\lceil \tfrac78\abs{E_k}\rceil=\abs{E_k}$, and since $F_k\subseteq E_k$,
		\eqref{eq:Fk_quant} implies $\abs{F_k}=\abs{E_k}\ge \abs{A_k}-1$. Using~\eqref{eq:inter_det},
		\[
		\abs{A_{k+1}\cap F_k}\ \ge\ \abs{F_k}-1~,
		\]
		and for $\abs{F_k}\in\{2,3\}$ this satisfies $\abs{F_k}-1\ge \lceil \abs{F_k}/3\rceil$.
		
		\medskip
		\noindent Combining both cases proves that for every $\abs{A_k}\ge 3$,
		\[
		\abs{A_{k+1}\cap F_k}\ \ge\ \Big\lceil \frac13 \abs{F_k}\Big\rceil~. 
		\]
	\end{proof}

	\subsubsection{Proof of Lemma~\ref{lem:main2}}
	\begin{proof}
		Fix a round $k\in\{1,\dots,k_{\max}-1\}$ and recall $\bar{\Delta}_k\le 0$ by construction.
		
		\bigskip
		
		\noindent\textit{Proof of the first bound }\;
		\[
		\mathbb{P}\!\left(S_k(i^*)\le \frac12\bar{\Delta}_k\right) \le (K+\log(T))\exp\!\left(-c\frac{\bar{\Delta}_k^2}{\log(B_k)}B_k\right)\, .
		\]
		If $\bar{\Delta}_k=0$ the bound is immediate. Assume $\bar{\Delta}_k<0$.
		Recall $S_k(i^*)=\min\{Z_k^{(s)}(i^*),0\}+Z_k^{(w)}(i^*)$. Then
		\begin{align*}
			\mathbb{P}\!\left(S_k(i^*)\le \frac12\bar{\Delta}_k\right) &\le \mathbb{P}\!\left(Z_k^{(s)}(i^*)+Z_k^{(w)}(i^*)\le \frac12\bar{\Delta}_k\right) +\mathbb{P}\!\left(Z_k^{(w)}(i^*)\le \frac12\bar{\Delta}_k\right)\\
			&\le \mathbb{P}\!\left(Z_k^{(s)}(i^*)\le \frac14\bar{\Delta}_k\right) +2\,\mathbb{P}\!\left(Z_k^{(w)}(i^*)\le \frac14\bar{\Delta}_k\right),
		\end{align*}
		where the last step uses $\bar{\Delta}_k<0$, hence
		$\{Z_k^{(w)}(i^*)\le \frac12\bar{\Delta}_k\}\subseteq \{Z_k^{(w)}(i^*)\le \frac14\bar{\Delta}_k\}$.
		
		\medskip
		
		\noindent The first term in the last upper-bound is bounded using Hoeffding and a union bound over the opponent choice,
		\[
		\mathbb{P}\!\left(Z_k^{(s)}(i^*)\le \frac14\bar{\Delta}_k\right) \le (K-1)\exp\!\left(-\frac{\bar{\Delta}_k^2}{32}B_k\right)\,.
		\]
		The second term is bounded by the shifted-rank version of Corollary~\ref{cor:quant} with $N=\abs{A_k}-1$,
		\begin{align*}
			\mathbb{P}\!\left(Z_k^{(w)}(i^*)\le \frac14\bar{\Delta}_k\right) &\le \mathbb{P}\!\left(Z_k^{(w)}(i^*)\le \Delta^{(k)}_{i^*, (\ceil{\abs{A_k}/8})}+\frac14\bar{\Delta}_k\right)\\
			&\le \log(T)\exp\!\left(-c\,\frac{\bar{\Delta}_k^2}{\log(\lceil B_k/2\rceil)}\,B_k\right)\,.	
		\end{align*}
		Finally, absorbing $\log(\lceil B_k/2\rceil)$ into $\log(B_k)$ and constants yields
		\[
		\mathbb{P}\!\left(S_k(i^*)\le \frac12\bar{\Delta}_k\right)
		\le (K+\log(T))\exp\!\left(-c\,\frac{\bar{\Delta}_k^2}{\log(B_k)}\,B_k\right).
		\]
		
		\bigskip
		
		\noindent\textit{Proof of the second bound }\;
		{\small \begin{equation*}
				\mathbb{P}\!\left(|\{\alpha\in F_k:\ S_k(\alpha)\ge \frac12\bar{\Delta}_k\}|
				\ge \lceil \abs{F_k}/3\rceil\right) \le
				\exp\left(-\frac{c}{\log^4(K)\log(T)}\\[-0.3em]\cdot\frac{T-c'\log^6(H^{(0)}_{\text{explore}}(\bm{s}))\cdot H^{(0)}_{\text{explore}}(\bm{s})}{H_{\text{explore}}^{(1)}(\bm{s})}\right)\,.
		\end{equation*}}
		We start from the indicator-sum form
		\begin{equation}\label{eq:exp_t2}
			\mathbb{P}\!\left(|\{\alpha\in F_k:\ S_k(\alpha)\ge \tfrac12\bar{\Delta}_k\}| \ge \left\lceil\frac{\abs{F_k}}{3}\right\rceil\right) =
			\mathbb{P}\!\left(\sum_{\alpha\in F_k}\mathds{1}\!\left(S_k(\alpha)\ge \tfrac12\bar{\Delta}_k\right) \ge \left\lceil\frac{\abs{F_k}}{3}\right\rceil\right)\,.
		\end{equation}
		
		\medskip
		
		\noindent Next, we keep only the hardest $3/4$ of $F_k$. More formally, we rank $\Gamma_\alpha=s_\alpha\Delta_{\alpha,(s_\alpha)}^2$ over $\alpha\in F_k$ as
		$\Gamma_{F_k:1}\le \cdots \le \Gamma_{F_k:\abs{F_k}}$, and let $F_k^{(3/4)}$ be the subset
		containing the top $\lceil 3\abs{F_k}/4\rceil$ arms with largest $\Gamma_\alpha$.
		Then $\abs{F_k\setminus F_k^{(3/4)}} = \floor{\abs{F_k}/4}$, so
		\[
		\left\{\sum_{\alpha\in F_k}\mathds{1}(S_k(\alpha)\ge \tfrac12\bar{\Delta}_k) \ge \left\lceil\frac{\abs{F_k}}{3}\right\rceil\right\}
		\subseteq \left\{\sum_{\alpha\in F_k^{(3/4)}}\mathds{1}(S_k(\alpha)\ge \tfrac12\bar{\Delta}_k) \ge \left\lceil\frac{\abs{F_k}}{12}\right\rceil\right\}\,.
		\]
		
		\medskip
		
		\noindent Let us develop a uniform per-arm bound on $F_k^{(3/4)}$. Lemma~\ref{lem:prob_f} below gives such a bound
		\begin{lemma}\label{lem:prob_f}
			Let $\alpha \in F_k^{(3/4)}$ and $T\ge 8K\log_{8/7}(K)$. We have
			\begin{equation}\label{eq:prob}
				\mathbb{P}\left( S_k(\alpha) \ge \frac{1}{2}\bar{\Delta}_k\right) \le (\log(T)+K)\exp\!\left(-c"\,\frac{T}{\log^4(K)\log(T)\,\sum_{i\in E_k}\frac{K}{s_i\Delta_{i,(s_i)}^2}}\right)~,
			\end{equation}
			where $c"$ is a positive numerical constant. Moreover, there exists a sufficiently large universal numerical constant $c_2>0$, depending only on $c"$, such that if $T \ge c_2\,H^{(0)}_{\mathrm{explore}}(\bm{s})\,\log^6\!\big(H^{(0)}_{\mathrm{explore}}(\bm{s})\big)$, we have 
			\[
			\mathbb{P}\left( S_k(\alpha) \ge \frac{1}{2}\bar{\Delta}_k\right) \le \frac{1}{18}~.
			\]
		\end{lemma}
		
		\smallskip
		\noindent In the remainder of this proof we assume that the condition $T \ge c_2\,H^{(0)}_{\mathrm{explore}}(\bm{s})\,\log^6\!\big(H^{(0)}_{\mathrm{explore}}(\bm{s})\big)$ is satisfied, otherwise the upper bound stated by the lemma is greater than $1$ and is thus vacuous. Denote by $p_k$ the bound given by the lemma above
		\begin{equation*}
			p_k \coloneq \frac{1}{18} \wedge ~(\log(T)+K)\exp\!\left(-c"\,\frac{T}{\log^4(K)\log(T)\,\sum_{i\in E_k}\frac{K}{s_i\Delta_{i,(s_i)}^2}}\right)~.
		\end{equation*}
		We use Lemma~\ref{lem:absorb}, which is purely technical and deferred to Section~\ref{sec:tech}, to obtain the following upper bound
		\begin{equation}\label{eq:p_k}
			p_k \le \exp\left(-\bar{c}_1\cdot \frac{T}{\log^4(K)\log(T)\,\sum_{i\in E_k}\frac{K}{s_i\Delta_{i,(s_i)}^2}}\right)~,
		\end{equation}
		where $\bar{c}_1$ is a numerical constant depending only on $c"$.
		Conditionally on $\mathcal{F}_{k-1}$, the scores $(S_k(\alpha))_{\alpha\in F_k^{(3/4)}}$ are independent, since the algorithm uses independent fresh samples and randomization for each arm. Therefore, the following stochastic domination holds conditionally on $\mathcal{F}_{k-1}$:
		\[
		\sum_{\alpha\in F_k^{(3/4)}}\mathds{1}\!\left(S_k(\alpha)\ge \tfrac12\bar{\Delta}_k\right)
		\ \ \text{is stochastically dominated by}\ \ 
		M_k\sim \mathrm{Bin}\!\left(\left\lceil\frac{3}{4}\abs{F_k}\right\rceil,\ p_k\right)\,.
		\]
		Consequently, using the fact that $p_k \le \frac{1}{18}$ implies $\left\lceil\frac{\abs{F_k}}{12}\right\rceil - p_k\left\lceil\frac{3}{4}\abs{F_k}\right\rceil \ge \frac{\abs{F_k}}{24}$ we have
		\begin{align}
			\mathbb{P}\!\left(\sum_{\alpha\in F_k^{(3/4)}}\mathds{1}\!\left(S_k(\alpha)\ge \tfrac12\bar{\Delta}_k\right) \ge \left\lceil\frac{\abs{F_k}}{12}\right\rceil\right) &\le \mathbb{P}\!\left(M_k\ge \left\lceil\frac{\abs{F_k}}{12}\right\rceil\right)\nonumber\\
			&\le \mathbb{P}\!\left(M_k-\mathbb{E}[M_k]\ge \frac{\abs{F_k}}{24}\right)~.\label{eq:e5}
		\end{align}
		Next, we use Lemma~\ref{thm:tech2} which provides a deviation bound for binomial variables in regimes where the parameters can be small. Recall that $M_k$ is a binomial distribution with parameters $(p_k, \ceil{3\abs{F_k}/4})$. We have
		\begin{equation}\label{eq:e7}
			\mathbb{P}\!\left(M_k-\mathbb{E}[M_k]\ge \frac{\abs{F_k}}{24}\right) \le \exp\!\left(-\frac{\abs{F_k}^2}{1152\ceil{3\abs{F_k}/4}\,\phi(p_k)}\right) \le \exp\!\left(-\frac{\abs{F_k}}{1152\,\phi(p_k)}\right)~,
		\end{equation}
		where $\ceil{3\abs{F_k}/4}\le\abs{F_k}$ and $\phi$ is the function defined in Lemma~\ref{thm:tech1}. Since we have proved that $p_k \le \frac{1}{18}$, the expression of $\phi(p_k)$ is therefore given by
		\[
		\phi(p_k) = \frac{\frac{1}{2}-p_k}{\log(1-p_k) - \log(p_k)}~.
		\]
		Since the function $\phi$ is increasing on $(0,1/2)$ and Lemma~\ref{lem:pure_tech} gives, for any $y>0$,
		\[
		0< \frac{\frac{1}{2}-\exp(-y)}{\log(1-\exp(-y))-\log(\exp(-y))} \le \frac{1}{2y}~,
		\]
		with the middle expression defined by continuity to equal $1/4$ at $y=\log 2$, we conclude using the bound~\eqref{eq:p_k} that
		\[
		\frac{1}{\phi(p_k)} \ge 2 \frac{\bar{c}_1}{320 \log(T) \log^4(K)}\cdot \frac{T}{\sum_{i \in E_k} \frac{K}{s_i \Delta_{i,(s_i)}^2}}~.
		\]
		Using the bound above with~\eqref{eq:e7} and $\abs{F_k}\ge \frac{7}{8}\abs{E_k}$, we have
		\begin{align*}
			\mathbb{P}\left( M_k-\mathbb{E}[M_k] \ge \frac{\abs{F_k}}{24}\right) &\le \exp\left(-\frac{\abs{F_k} }{1152}\cdot \frac{\bar{c}_1}{320\log(T)\log^4(K)}\cdot \frac{T}{\sum_{i \in E_k} \frac{K}{s_i\Delta_{i,(s_i)}^2}}\right)\\
			&\le \exp\left( -\bar{c}_2\cdot\frac{T \abs{E_k}}{\log^4(K)\log(T)\sum_{i \in E_k} \frac{K}{s_i\Delta_{i,(s_i)}^2}}\right)~,
		\end{align*}
		where $\bar{c}_2$ is a numerical constant. We then use the fact that
		\[
		\frac{\abs{E_k}}{\sum_{i\in E_k}\frac{K}{s_i\Delta_{i,(s_i)}^2}}
		\ \ge\ \frac{1}{\max_{i\neq i^*}\frac{K}{s_i\Delta_{i,(s_i)}^2}}
		\ =\ \frac{1}{H^{(1)}_{\mathrm{explore}}(\bm{s})}~.
		\]
		Plugging these two relations into~\eqref{eq:e5} then \eqref{eq:exp_t2} yields for a numerical constant $\bar{c}_3$
		\[
		\mathbb{P}\!\left(|\{\alpha\in F_k:\ S_k(\alpha)\ge \tfrac12\bar{\Delta}_k\}| \ge \left\lceil\frac{\abs{F_k}}{3}\right\rceil\right)
		\le \exp\!\left(-\,\bar{c}_3\cdot \frac{T}{\log^4(K)\log(T)\,H^{(1)}_{\mathrm{explore}}(\bm{s})}\right),
		\]
		as soon as $T\ge c_2 H^{(0)}_{\mathrm{explore}}(\bm{s})\log^6(H^{(0)}_{\mathrm{explore}}(\bm{s}))$.
		Reintroducing the shift (to cover smaller $T$) gives the stated bound in Lemma~\ref{lem:main2}.
	\end{proof}
	\begin{proof}{\textbf{of Lemma~\ref{lem:prob_f}}}.
		Fix $\alpha \in F_k^{(3/4)}$, let $\alpha^{(s)}$ denote the strong opponent chosen for $\alpha$. We have
		\begin{align}
			\mathbb{P}\left( S_k(\alpha) \ge \frac{1}{2} \bar{\Delta}_k\right) &= \mathbb{P}\left( \min\{Z^{(s)}_k(\alpha),~0 \}+ Z^{(w)}_k(\alpha)\ge \frac{1}{2} \bar{\Delta}_k\right)\nonumber\\
			&\le \mathbb{P}\left( Z^{(s)}_k(\alpha)+ Z^{(w)}_k(\alpha)\ge \frac{1}{2} \bar{\Delta}_k\right)\nonumber\\
			&\le \mathbb{P}\left( Z_k^{(s)}(\alpha)-\Delta_{\alpha, \alpha^{(s)}} \ge-\frac{1}{4} \bar{\Delta}_k\right)\nonumber\\
			&\qquad + \mathbb{P}\left( \Delta_{\alpha, \alpha^{(s)}} \ge \frac{7}{8} \bar{\Delta}_k\right)+\mathbb{P}\left( Z^{(w)}_k(\alpha) \ge -\frac{1}{8} \bar{\Delta}_k\right)~.\label{eq:qq-1}
		\end{align}
		Using Hoeffding's concentration inequality with a union bound over the possible choices of $\alpha^{(s)}$, we have
		\begin{equation}\label{eq:qq1}
			\mathbb{P}\left( Z_k^{(s)}(\alpha)-\Delta_{\alpha, \alpha^{(s)}} \ge -\frac{1}{4} \bar{\Delta}_k\right) \le (K-1)\exp\left(- \frac{\bar{\Delta}_k^2}{32} B_k\right)~.
		\end{equation}
		Since $\alpha \in F_k \subset E_k$, we have $\Delta^{(k)}_{\alpha, (\ceil{\abs{A_k}/4})} \le 0$. Therefore, by the shifted-rank version of Corollary~\ref{cor:quant} with $N=\abs{A_k}-1$, we get
		\begin{align}
			\mathbb{P}\left(Z_k^{(w)}(\alpha) \ge -\frac{1}{8} \bar{\Delta}_k\right) &\le \mathbb{P}\left( Z_k^{(w)}(\alpha) \ge \Delta^{(k)}_{\alpha, (\ceil{\abs{A_k}/4})}- \frac{1}{8}\bar{\Delta}_k \right)\nonumber\\
			&\le \log\left(T\right)\exp\left(-c\cdot \frac{\bar{\Delta}_k^2}{64\log(\ceil{B_k/2})}\ceil*{\frac{B_k}{2}}\right)\nonumber\\
			&\le \log\left(T\right)\exp\left(-c\cdot \frac{\bar{\Delta}_k^2}{128\log(B_k)}B_k\right)~.\label{eq:qq3}
		\end{align}
		Since $\alpha \in F_k$, by definition of $F_k$ given in~\eqref{eq:def_F}, we have $\Delta_{\alpha, (s_{\alpha})} \le \bar{\Delta}_k \le 0$. Therefore,
		\[
		\mathbb{P}\left(\Delta_{\alpha, \alpha^{(s)}} \ge \frac{7}{8} \bar{\Delta}_k\right) \le \mathbb{P}\left(\Delta_{\alpha, \alpha^{(s)}} \ge \frac{7}{8}\Delta_{\alpha, (s_{\alpha})} \right)~.
		\]
		If $K=2$, the sole candidate opponent is returned and the following probability is zero. Otherwise, using Theorem~\ref{thm:main0} and absorbing numerical constants into $C,c>0$, we have
		\begin{align}
			\mathbb{P}\left(\Delta_{\alpha, \alpha^{(s)}} \ge \frac{7}{8} \bar{\Delta}_k\right) &\le \mathbb{P}\left(\Delta_{\alpha, \alpha^{(s)}} \ge \frac{7}{8}\Delta_{\alpha, (s_{\alpha})} \right)\nonumber\\
			&=\mathbb{P}\left(\Delta_{\alpha, \alpha^{(s)}} \ge \Delta_{\alpha, (s_{\alpha})} - \frac{1}{8} \Delta_{\alpha, (s_{\alpha})}\right)\nonumber\\ 
			&\le C\log(2eKT)\exp\left(-c\cdot \frac{s_{\alpha}\Delta^2_{\alpha, (s_{\alpha})}}{K\log(eT)\log^3(2K)}B_k\right)~.\label{eq:qq2}
		\end{align}
		
		\noindent We conclude by plugging the bounds~\eqref{eq:qq1},~\eqref{eq:qq3} and~\eqref{eq:qq2} into~\eqref{eq:qq-1}
		\begin{equation*}
			\mathbb{P}\left( S_k(\alpha) \ge \frac{1}{2}\bar{\Delta}_k\right) \le C(K+\log(T))\exp\left(-c'~\min\left\lbrace \frac{\Gamma_{\alpha}}{K\log(eT)\log^3(2K)}, \frac{\bar{\Delta}_k^2}{\log(B_k)}\right\rbrace B_k\right)~,
		\end{equation*}
		where $c' := \frac{c}{1024}$. Now it remains to prove that 
		\[
		\min\left\{\frac{\Gamma_{F_k:\lceil |F_k|/4\rceil}}{K\log^3(K)\log(T)},\ \frac{\bar{\Delta}_k^2}{\log(B_k)}\right\}B_k
		\ \ge c' \frac{T}{\log^4(K)\log(T)\,\sum_{i \in E_k} \frac{K}{\Gamma_i}}~.
		\]
		Recall Lemma~\ref{lem:T} gives
		\[
		\frac{\ceil*{\frac{\abs{F_k}}{4}}}{\Gamma_{F_k:\ceil*{\frac{1}{4}\abs{F_k}}}} \le \sum_{i \in F_k} \frac{1}{\Gamma_i}\le \sum_{i \in E_k} \frac{1}{\Gamma_i}~.
		\]
		Therefore,
		\begin{equation*}
			\frac{\abs{F_k}}{4\sum_{i\in E_k} \frac{1}{\Gamma_i}} \le \Gamma_{F_k:\ceil*{\frac{1}{4}\abs{F_k}}}~.
		\end{equation*}
		Hence, using the bound above and the definition of $B_k$ we obtain
		\begin{align*}
			\frac{\Gamma_{F_k:\ceil*{\frac{1}{4}\abs{F_k}}}}{K\log^3(K)\log(T)}B_k &\ge \frac{\abs{F_k}}{4\sum_{i\in E_k} \frac{1}{\Gamma_i}}\cdot \frac{1}{K\log^3(K)\log(T)}\cdot \frac{T}{2\abs{A_k}\log_{8/7}(K)}\\
			&= \frac{T}{8\log^3(K)\log(T)\log_{8/7}(K)}\cdot \frac{1}{\sum_{i\in E_k} \frac{K}{\Gamma_i}}\cdot \frac{\abs{F_k}}{\abs{A_k}}\\
			&\ge \frac{T}{138\log^4(K)\log(T)\sum_{i\in E_k} \frac{K}{\Gamma_i}}\cdot \frac{\abs{F_k}}{\abs{A_k}}~,
		\end{align*}
		Recall that $\abs{F_k} \ge \ceil*{\frac{7}{8}\abs{E_k}} \ge \ceil*{\frac{3}{16}\abs{A_k}}$. Therefore, the bound above gives
		\begin{equation}\label{eq:pk1}
			\frac{\Gamma_{F_k:\ceil*{\frac{1}{4}\abs{F_k}}}}{K\log^3(K)\log(T)}B_k \ge \frac{T}{736\log^4(K)\log(T)\sum_{i\in E_k} \frac{K}{\Gamma_i}}~.
		\end{equation}
		Moreover, we have 
		\begin{align}
			\abs{A_k} \frac{1}{\bar{\Delta}_k^2}&\le 4\abs{E_k}\cdot \frac{1}{\bar{\Delta}_k^2}\le 32\cdot\ceil*{\frac{1}{8} \abs{E_k}} \frac{1}{\Delta^2_{E_k:\ceil*{\frac{7}{8}\abs{E_k}}}}\nonumber\\
			&\le 32\cdot\sum_{\alpha \in E_k} \frac{1}{\Delta^2_{\alpha, (s_{\alpha})}}~,\label{eq:pk0}
		\end{align}
		where the last inequality follows directly because the gaps with ranks at least $\lceil 7\abs{E_k}/8\rceil$ lie between $\bar{\Delta}_k$ and zero, so $\sum_{\alpha\in E_k}1/\Delta_{\alpha,(s_\alpha)}^2\ge(\abs{E_k}-\lceil7\abs{E_k}/8\rceil+1)/\bar{\Delta}_k^2\ge\lceil\abs{E_k}/8\rceil/\bar{\Delta}_k^2$. Therefore, we have
		\begin{align}
			\frac{\bar{\Delta}_k^2}{\log(B_k)} B_k &\ge \frac{\bar{\Delta}_k^2}{\log(B_k)} \frac{T}{2\abs{A_k}\log_{8/7}(K)}\nonumber\\
			&\ge \frac{1}{\sum_{\alpha \in E_k} \frac{1}{\Delta^2_{\alpha, (s_{\alpha})}}} \cdot \frac{ T}{64\log(B_k)\log_{8/7}(K)}~.\label{eq:pk2}
		\end{align}
		Therefore, combining~\eqref{eq:pk1} and~\eqref{eq:pk2}, we get, 
		{ \begin{align}
				\mathbb{P}\left(S_k(\alpha) \ge \frac{1}{2}\bar{\Delta}_k\right) &\nonumber\\
				& \leqslant \min\left\{1, C (\log(T)+K)\exp\left(-c'~\min\left\lbrace \frac{\Gamma_{F_k:\ceil*{\frac{1}{4}\abs{F_k}}}}{K\log^3(K)\log(T)},~\frac{\bar{\Delta}_k^2}{\log(B_k)}\right\rbrace B_k\right)\right\}\nonumber\\
				&\le (\log(T)+K) \exp\left(-c_1~\frac{T}{\log^4(K)\log(T)\sum_{i \in E_k} \frac{K}{\Gamma_i}} \right)~, \label{eq:pk3}
		\end{align}}
		where we used $\log(T) \ge \log(K)$ and $\sum_{i\in E_k}K/\Gamma_i\ge\sum_{\alpha\in E_k}1/\Delta_{\alpha,(s_\alpha)}^2$ in the last line. The numerical factor $C$ is absorbed by decreasing the exponent constant: for $(\log(T)+K) \ge 2$, $\min\{1,C(\log(T)+K) e^{-bx}\}\le l e^{-b'x}$ with a sufficiently small numerical $b'>0$.
		
		\smallskip
		\noindent For the remainder of this proof, let $H\coloneq H^{(0)}_{\text{explore}}(\bm{s})$ and $T_0\coloneq c_2H\log^6(H)$. We have $H\ge 4K\ge 8$ and
		\[
		\log(T_0)\le\left(7+\frac{\log(c_2)}{\log 8}\right)\log(H)~.
		\]
		Choose $c_2$ large enough that $T_0\ge e^6$ and $T_0/(\log T_0)^6\ge 2H/c_1$; such a universal choice exists since $c_2/(7+\log(c_2)/\log 8)^6\to\infty$. Since $t/(\log t)^6$ is non-decreasing for $t\ge e^6$, for every $T\ge T_0$ we have $T/(\log T)^6\ge 2H/c_1$. Using $\sum_{i\in E_k}K/\Gamma_i\le H$ and $K\le T$, it follows that
		\[
		\frac{c_1T}{\log^4(K)\log(T)\sum_{i\in E_k}K/\Gamma_i}\ge\frac{c_1T}{H(\log T)^5}\ge 2\log T~.
		\]
		For sufficiently large $c_2$, we also have $K+\log T\le T$ and $T\ge 18$. Hence~\eqref{eq:pk3} gives
		\[
		\mathbb{P}\left(S_k(\alpha)\ge\frac12\bar{\Delta}_k\right)\le T e^{-2\log T}=\frac1T\le\frac1{18}~.
		\]
	\end{proof}

	\section{Proof of Theorem~\ref{thm:fc2}  }\label{sec:proof_fc}
	This routine, presented in Algorithm~\ref{algo:4}, serves as one of the two certification sub-procedures in the fixed-confidence algorithm. Given a confidence level~$\delta$, a query budget~$T$, and a candidate Condorcet winner~$I$, it sequentially tests whether one can certify—using at most~$T$ comparisons—that all pairwise gaps $(\Delta_{I,i})_{i\neq I}$ are positive with probability at least~$1-\delta$. The budget is allocated uniformly across these gaps, and the procedure terminates as soon as either (i) a negative gap is detected, (ii) all gaps are certified positive, or (iii) the budget~$T$ is exhausted.
	
	\begin{algorithm} 
		\caption{ Test-CW \label{algo:4} }
		\begin{algorithmic}
			\STATE \textbf{Input}: $I \in [K], \delta, T$.
			\STATE \textbf{Initialize}: $C=[K]\setminus \{I\}$, empirical means $\tilde{\Delta}_{I,j} = 0$ for $j \in C$, count variable $t \gets 1$.
			\IF{$T<8K\log_{8/7}(K)$}
			\STATE Return $\textbf{False}$.
			\ENDIF
			\STATE Let $n \gets 1+\ceil*{\log_2\left(\frac{T}{8K\log(K)}\right)}$.
			\STATE $N_{I,j} \gets 0$ for all $j\in [K]$  //Query count
			\WHILE{$C \neq \emptyset$ \text{ and } $ t\le T$}
			\STATE Sample duel $(I,j)$ for $j\in \text{argmin}_{j \in C} N_{I,j}$ and update corresponding empirical means.
			\STATE $N_{I,j} \gets N_{I,j}+1$, $t \gets t+1$.
			\STATE \texttt{/* Check the sign of the gaps using concentration */}
			\FOR{$j \in C$}
			\IF{$N_{I,j}>0$}
			\IF{$\tilde{\Delta}_{I,j} \ge \sqrt{\frac{\log\left(KN^2_{I,j}\frac{n(n+1)}{\delta}\right)}{N_{I,j}}}$}
			\STATE $C \gets C \setminus \{j\}$.
			\ELSIF{$\tilde{\Delta}_{I, j} \le -\sqrt{\frac{\log\left(KN^2_{I,j}\frac{n(n+1)}{\delta}\right)}{N_{I,j}}}$}
			\STATE Return $\textbf{False}$.
			\ENDIF
			\ENDIF
			\ENDFOR
			\ENDWHILE
			\IF{$ C=\emptyset$}
			\STATE Return $\textbf{True}$
			\ELSE
			\STATE Return $\textbf{False}$
			\ENDIF
		\end{algorithmic}
	\end{algorithm}
	
	Both Algorithms~\ref{algo:3} and~\ref{algo:4} use the same counter
	\[
	n=1+\ceil*{\log_2\left(\frac{T}{8K\log(K)}\right)}~.
	\]
	On the doubling schedule $T^{(n)}=8K\log(K)\,2^{n-1}$ of Algorithm~\ref{algo:fc}, this is exactly the call number.
	Both sub-procedures return a false certification flag when $T<8K\log_{8/7}(K)$.
	All samples used by a new call, by each score computation, and by Test-CW are fresh and independent of the preceding history.
	As in Section~\ref{sec:proof_fb}, roundwise concentration bounds are proved conditionally on the complete pre-round history and then averaged; this conditioning is omitted from the notation.
	
	\subsection{Proof of $\delta$-correctness}
	
	Let $c_0$ denote the absolute numerical constant corresponding to the one appearing in the upper bound of Corollary~\ref{cor:quant}.
	Theorem~\ref{thm:fc2} states that Algorithm~\ref{algo:fc}, with input $\delta \in (0,1/6)$ and $c \ge 2/c_0$, outputs an arm different from the CW with probability at most $\delta$. Let $\psi_{\delta}$ denote the output of Algorithm~\ref{algo:fc} when the input is $\delta$. We will prove that
	\[
	\mathbb{P}(\psi_{\delta} \neq i^*) \le \delta~.
	\]
	To prove this claim, we introduce the following notation. In the $n$-th iteration (i.e., the $n$-th call to Algorithm~\ref{algo:3}), denote by
	$\bar{\alpha}^{(n)}, \phi_1^{(n)}, \phi_2^{(n)}, I^{(n)}$ and $T^{(n)}$ the corresponding values of
	$\bar{\alpha},\phi_1,\phi_2,I$ and $T$, and let $\varphi^{(n)}\coloneqq \phi_1^{(n)}\vee \phi_2^{(n)}$.
	For convenience define, for all $n\ge 1$,
	\[
	\delta_{n}\coloneqq \frac{\delta}{8K^2\log_{8/7}(K)\,\log(T^{(n)})\,n(n+1)}~.
	\]
	Let $S_k^{(n)}(\cdot)$ denote the score used at round $k$ within the $n$-th call to Algorithm~\ref{algo:3},
	and let $Z_{k}^{(w,n)}(\cdot)$ denote its weak component. Recall that each call to Algorithm~\ref{algo:3}
	has at most $k_{\max}\le \lceil\log_{8/7}(K)\rceil$ rounds. Finally, in the $n$-th call to \textsc{Test-CW}
	(Algorithm~\ref{algo:4}) with inputs $(I^{(n)},\delta,T^{(n)})$, let $\tilde{\Delta}_{i,j}(N)$ denote the empirical estimate based on the first $N$ fresh samples of pair $(i,j)$ in that call.
	For the concentration arguments, these independent sample sequences may be generated in advance, including their unused prefixes.
	Calls with $T^{(n)}<8K\log_{8/7}(K)$ return $\texttt{False}$ and do not contribute to the error event; the following bounds concern the remaining calls.
	
	\smallskip
	\noindent If $\psi_{\delta}\neq i^*$, then for some $n\ge 1$ the algorithm must have certified an incorrect candidate,
	namely $\{I^{(n)}\neq i^*\}$ and $\{\varphi^{(n)}=\texttt{True}\}$. Hence, by a union bound,
	\begin{align}
		\mathbb{P}(\psi_{\delta}\neq i^*)
		&\le \mathbb{P}\Big(\exists n\ge 1:\ I^{(n)}\neq i^*,\ \varphi^{(n)}=\texttt{True}\Big)\nonumber\\
		&\le \sum_{n=1}^{\infty}\mathbb{P}\Big(I^{(n)}\neq i^*,\ \phi_1^{(n)}=\texttt{True}\Big)
		+\sum_{n=1}^{\infty}\mathbb{P}\Big(I^{(n)}\neq i^*,\ \phi_2^{(n)}=\texttt{True}\Big)~.
		\label{eq:fc_split_terms}
	\end{align}
	
	\medskip
	\noindent We first bound the contribution of $\phi_1$. On the event $\{I^{(n)}\neq i^*,\ \phi_1^{(n)}=\texttt{True}\}$,
	during the $n$-th run of Algorithm~\ref{algo:3} the true Condorcet winner $i^*$ must have been eliminated at
	some round $k<k_{\max}$ (otherwise the procedure would return $I^{(n)}=i^*$). By the definition of the
	selection $\bar{\alpha}^{(n)}$ and the condition $\phi_1^{(n)}=\texttt{True}$, this implies that for some
	$k<k_{\max}$,
	\[
	S_k^{(n)}(i^*)<-\sqrt{\frac{2c\log(T^{(n)})\log(1/\delta_n)}{\lceil B_k^{(n)}/4\rceil}}.
	\]
	Using $S_k^{(n)}(i^*)=\min\{\hat{\Delta}^{(k,n)}_{i^*,u},0\}+Z^{(w,n)}_{k}(i^*)$ (for the opponent $u$ queried
	at that round) and the fact that $\min\{x,0\}+y<-\eta$ implies $(x<-\eta/2)\ \text{or}\ (y<-\eta/2)$, we get
	{\small \begin{align}
			\mathbb{P}\Big(I^{(n)}\neq i^*,\ \phi_1^{(n)}=\texttt{True}\Big)
			&\le \sum_{k<k_{\max}}
			\mathbb{P}\!\left(i^*\in A_k^{(n)},\ S_k^{(n)}(i^*)<-\sqrt{\frac{2c\log(T^{(n)})\log(1/\delta_n)}{\lceil B_k^{(n)}/4\rceil}}\right)\nonumber\\
			&\le \sum_{k<k_{\max}}\sum_{u\neq i^*}
			\mathbb{P}\!\left(i^*\in A_k^{(n)},\ \hat{\Delta}^{(k,n)}_{i^*,u}<-\sqrt{\frac{c\log(T^{(n)})\log(1/\delta_n)}{2\lceil B_k^{(n)}/4\rceil}}\right)\nonumber\\
			&\quad+\sum_{k<k_{\max}}
			\mathbb{P}\!\left(i^*\in A_k^{(n)},\ Z^{(w,n)}_{k}(i^*)<-\sqrt{\frac{c\log(T^{(n)})\log(1/\delta_n)}{2\lceil B_k^{(n)}/4\rceil}}\right)~.
			\label{eq:phi1_decomp}
	\end{align}}
	For both terms, condition on a pre-round history such that $i^*\in A_k^{(n)}$. For the first term, since $\Delta_{i^*,u}>0$ for all $u\neq i^*$, we can center and apply Hoeffding's inequality:
	for $N=\lceil B_k^{(n)}/4\rceil$,
	\begin{align}
		\mathbb{P}\!\left(\hat{\Delta}^{(k,n)}_{i^*,u}<-\sqrt{\frac{c\log(T^{(n)})\log(1/\delta_n)}{2N}}\right)
		&\le \mathbb{P}\!\left(\hat{\Delta}^{(k,n)}_{i^*,u}-\Delta_{i^*,u}<-\sqrt{\frac{\log(1/\delta_n)}{2N}}\right)\nonumber\\
		&\le \exp\!\left(-2N\cdot\frac{\log(1/\delta_n)}{2N}\right)\le \delta_n~.
		\label{eq:hoeffding_fc}
	\end{align}
	For the second term in \eqref{eq:phi1_decomp}, the shifted-rank version of Corollary~\ref{cor:quant} with $N=\abs{A_k}-1$ (with constant $c_0$) gives
	{\small \begin{align}\nonumber
			\mathbb{P}\!\left(Z^{(w,n)}_{k}(i^*)<-\sqrt{\frac{c\log(T^{(n)})\log(1/\delta_n)}{2\lceil B_k^{(n)}/4\rceil}}\right)
			&\le \mathbb{P}\!\left(Z^{(w,n)}_{k}(i^*)<\Delta^{(k)}_{i^*, (\ceil{\abs{A_k}/8})}-\sqrt{\frac{c\log(T^{(n)})\log(1/\delta_n)}{2\lceil B_k^{(n)}/4\rceil}}\right)\\
			&\le \log\!\left(\ceil*{\frac{B_k^{(n)}}{2}}\right)
			\exp\!\left(-c_0\cdot\frac{c\log(T^{(n)})\log(1/\delta_n)}{2\log(\ceil{B_k^{(n)}/2})}\right)\nonumber\\
			&\le \log\!\left(\ceil*{\frac{B_k^{(n)}}{2}}\right)\delta_n,
			\label{eq:quant_fc}
	\end{align}}
	where, without loss of generality, $c_0\le 1$, and we used $c\ge 2/c_0$ and $B_k^{(n)}\leqslant T^{(n)}$. Averaging over the pre-round history and plugging \eqref{eq:hoeffding_fc} and \eqref{eq:quant_fc} into
	\eqref{eq:phi1_decomp}, summing over $k<k_{\max}$ and using $\log(\lceil B_k^{(n)}/2\rceil)\le \log(T^{(n)})$
	and $k_{\max}\le \lceil\log_{8/7}(K)\rceil$, we obtain
	\begin{align}
		\sum_{n=1}^{\infty}\mathbb{P}\Big(I^{(n)}\neq i^*,\ \phi_1^{(n)}=\texttt{True}\Big)
		&\le \sum_{n=1}^{\infty}\left((k_{\max}(K-1))\delta_n + k_{\max}\log(T^{(n)})\delta_n\right)\nonumber\\
		&\le \sum_{n=1}^{\infty}\frac{\delta}{2n(n+1)}\le \frac{\delta}{2}.
		\label{eq:term1_fc}
	\end{align}
	
	\medskip
	\noindent We now bound the contribution of $\phi_2$ in \eqref{eq:fc_split_terms}. On the event $\{I^{(n)}\neq i^*,\ \phi_2^{(n)}=\texttt{True}\}$,
	the $n$-th call to \textsc{Test-CW} returns \texttt{True} although $I^{(n)}$ is not the Condorcet winner.
	In particular, for some $N\ge 1$ the test must have accepted the comparison against $i^*$, meaning that
	\[
	\tilde{\Delta}_{I^{(n)},i^*}(N)\ge\sqrt{\frac{\log\!\left(KN^2\frac{n(n+1)}{\delta}\right)}{N}}.
	\]
	Hence, by a union bound over $n\ge 1$, $N\ge 1$ and all $i\neq i^*$, and since $\Delta_{i,i^*}\le 0$ when $i\neq i^*$,
	\begin{align}
		\sum_{n=1}^{\infty}\mathbb{P}\Big(I^{(n)}\neq i^*,\ \phi_2^{(n)}=\texttt{True}\Big)
		&\le \sum_{n=1}^{\infty}\sum_{N=1}^{\infty}\sum_{i\neq i^*}
		\mathbb{P}\!\left(\tilde{\Delta}_{i,i^*}(N)\ge\sqrt{\frac{\log\!\left(KN^2\frac{n(n+1)}{\delta}\right)}{N}}\right)\nonumber\\
		&\le \sum_{n=1}^{\infty}\sum_{N=1}^{\infty}\sum_{i\neq i^*}
		\mathbb{P}\!\left(\tilde{\Delta}_{i,i^*}(N)-\Delta_{i,i^*}\ge\sqrt{\frac{\log\!\left(KN^2\frac{n(n+1)}{\delta}\right)}{N}}\right)\nonumber\\
		&\le \sum_{n=1}^{\infty}\sum_{N=1}^{\infty}\sum_{i\neq i^*}\frac{\delta}{4K\,N^2\,n(n+1)}
		\ \le\ \frac{\delta}{2}~,
		\label{eq:term2_fc}
	\end{align}
	where the last inequality follows from Hoeffding's inequality and $\sum_{N\ge 1}1/N^2\le 2$.
	
	\smallskip
	\noindent Finally, combining \eqref{eq:fc_split_terms} with \eqref{eq:term1_fc} and \eqref{eq:term2_fc} yields
	\[
	\mathbb{P}(\psi_{\delta}\neq i^*)\le \frac{\delta}{2}+\frac{\delta}{2}=\delta~,
	\]
	which concludes the proof.

	\subsection{Proof of Theorem~\ref{thm:fc2} (sample complexity statement)}\label{sec:proof_fc_sc}
	
	We build on the guarantees established for Algorithm~\ref{algo:3} (fixed-budget elimination with certification)
	to prove the second statement of Theorem~\ref{thm:fc2}. We use the notation: for each $i\neq i^*$,
	$\Delta_{i,(1)}\le \cdots \le \Delta_{i,(K-1)}$ denotes the ordered list of gaps $(\Delta_{i,j})_{j\neq i}$,
	and $K_{i;<0}:=\big|\{j:\Delta_{i,j}<0\}\big|$. Fix any vector $\bm{s}=(s_1,\dots,s_K)$ such that
	$1\le s_i\le K_{i;<0}$, with $s_i$ integer, for all $i\neq i^*$ (and $K_{i^*;<0}=0$ by convention), and recall
	\[
	H_{\text{certify}}(\bm{s})=\sum_{i\neq i^*}\frac{1}{\Delta_{i,(s_i)}^2},\qquad
	H^{(1)}_{\text{explore}}(\bm{s})=\max_{i\neq i^*}\frac{K}{s_i\Delta_{i,(s_i)}^2},\qquad
	H^{(0)}_{\text{explore}}(\bm{s})=\sum_{i\neq i^*}\frac{K}{s_i\Delta_{i,(s_i)}^2}~.
	\]
	
	\noindent Let $c_1$ be the numerical constant in Theorem~\ref{thm:amain1} and let $c_2,c_3$ be the numerical constants
	in Lemma~\ref{lem:main2}. Let $c_0$ be the numerical constant in Corollary~\ref{cor:quant}, and assume
	$c\ge 2/c_0$ as in the statement of Theorem~\ref{thm:fc2}.
	Without loss of generality, take $c_0\le 1$, $0<c_1\le 1/16$,
	$0<c_2\le \log(8/7)(\log 2)^2/32$, and $c_3\ge 1$; decreasing the constants in the exponents and increasing the warm-up constant preserves the preceding bounds.
	For concision, define
	\begin{align*}
		G_{1,\delta}
		&:= \frac{32}{c_1}H_{\text{cw}} \log(K)\log\!\Big(\frac{32}{c_1}\frac{K H_{\text{cw}}}{\delta}\Big)
		\log\!\Big(c_1^{-1} H_{\text{cw}}\log(K/\delta)\Big)~,\\
		G_{2,\delta}
		&:= \frac{512c}{c_2} H_{\text{certify}}(\bm{s})\,\log^3(K)\,
		\log\!\Big(\frac{32c}{c_2}\frac{K H^{(0)}_{\text{explore}}(\bm{s})}{\delta}\Big)\nonumber\\
		&\qquad\times\log\!\Bigg(\frac{512c}{c_2} H_{\text{certify}}(\bm{s})\,\log^3(K)\,
		\log\!\Big(\frac{32c}{c_2}\frac{K H^{(0)}_{\text{explore}}(\bm{s})}{\delta}\Big)\Bigg)~,\\
		G_{3,\delta}
		&:= \frac{32c}{c_2} H^{(1)}_{\text{explore}}(\bm{s})\,\log^4(K)\,\log\left(\frac{2Kk_{\max}}{\delta}\right)
		\log\left(\frac{32c}{c_2}H^{(1)}_{\text{explore}}(\bm{s}) \log^4(K)\log\left(\frac{2Kk_{\max}}{\delta}\right)\right)~,\\
		G_{0}
		&:= \frac{2c_3}{c_2}\,H^{(0)}_{\text{explore}}(\bm{s})\,
		\log^6\!\Big(\frac{2c_3}{c_2}H^{(0)}_{\text{explore}}(\bm{s})\Big)~.
	\end{align*}
	Algorithm~\ref{algo:fc} doubles the budget parameter $T$ at each unsuccessful iteration; therefore, if one can show that whenever
	\begin{equation}\label{eq:threshold_T_fc_sc}
		T\in\Big[M_\delta,\ 2M_\delta\Big],\qquad \text{ where } M_\delta:=\min\Big\{G_{1,\delta},\ G_{2,\delta}+G_{3,\delta}+G_0\Big\}~,
	\end{equation}
	the call to Algorithm~\ref{algo:3} with inputs $(\delta,T,c)$ returns $\phi_1\vee \phi_2=\texttt{True}$
	with probability at least $1-\delta$, then it follows that the total number of queries $N_\delta$
	used by Algorithm~\ref{algo:fc} is at most a universal constant multiple of $M_\delta$ with probability at least
	$1-\delta$ (since the sum of a doubling schedule up to the first successful budget is at most $2$ times that budget).
	We now verify this success probability for any $T$ satisfying \eqref{eq:threshold_T_fc_sc}, distinguishing two regimes
	depending on which term attains the minimum.
	
	\medskip\noindent
	\textbf{Regime 1: $G_{1,\delta}\le G_{2,\delta}+G_{3,\delta}+G_0$.}
	Then, we focus on the regime $T\in[G_{1,\delta},2G_{1,\delta}]$. In this regime we certify correctness through the fixed-budget guarantee
	of Theorem~\ref{thm:amain1}. Indeed, for $T\in [G_{1,\delta},~2G_{1,\delta}]$, we have
	\begin{align}
		\mathbb{P}(I\neq i^*)
		&\le 27K\log(K)\log(T)\exp\!\Big(-c_1\,\frac{T}{\log(T)\log(K)H_{\text{cw}}}\Big)\nonumber\\
		&\le 27K\log(K)\log(2G_{1,\delta})
		\exp\!\Big(-c_1\,\frac{G_{1,\delta}}{\log(2G_{1,\delta})\log(K)H_{\text{cw}}}\Big)~.\label{eq:fcl3}
	\end{align}
	We now use the explicit definition of $G_{1,\delta}$ and the crude upper bound
	\begin{equation}\label{eq:logG1_bound_fc_sc}
		\log(2G_{1,\delta}) \le 16 \log\Big(c_1^{-1}H_{\text{cw}}\log(K/\delta)\Big)~,
	\end{equation}
	which results from the expression of $G_{1, \delta}$, $K,H_{\text{cw}}\ge 2$ and $\delta\in(0,1/6)$.
	Using the bound \eqref{eq:logG1_bound_fc_sc} and plugging it back in \eqref{eq:fcl3}, we obtain
	\begin{align*}
		\mathbb{P}\left(I \neq i^*\right)
		&\le 432K\log(K)\log\!\Big(c_1^{-1}H_{\text{cw}}\log(K/\delta)\Big)
		\exp\!\Big(-2\log(32c_1^{-1}KH_{\text{cw}}/\delta)\Big)\\
		&\le \frac{27c_1\delta^2\log(K)\log(K/\delta)}{64K H_{\text{cw}}}
		\le \frac{\delta}{4}~.
	\end{align*}
	Here we used $\log x\le x$, $H_{\text{cw}}\ge 4(K-1)$, $\log K\le K-1$, and
	$\delta\log(K/\delta)\le K/2$ for $K\ge 2$ and $\delta\in(0,1/6)$.
	\smallskip
	\noindent It remains to argue that, conditional on $I=i^*$, the auxiliary certification Test-CW (Algorithm~\ref{algo:4})
	returns $\texttt{True}$ with probability at least $1-\delta/2$, hence overall
	$\mathbb{P}(\phi_1\vee \phi_2=\texttt{True})\ge 1-3\delta/4\ge 1-\delta$ in this regime.
	Run Test-CW with inputs $(i^*,\delta,T)$ and let $T_i$ be the number of comparisons allocated to pair $(i^*,i)$.
	By construction, $\sum_{i\neq i^*}T_i\le T$. If Test-CW returns \texttt{False}, then either it exhausted the budget
	without eliminating all opponents, or it triggered a negative-deviation stopping rule. Formally, define
	\[
	\mathcal{E}_1:=\Big\{\sum_{i\neq i^*}T_i=\floor*{T}\Big\},\qquad
	\mathcal{E}_2:=\Big\{\exists j\neq i^*,\ \exists N\ge 1:\ 
	\tilde{\Delta}_{i^*,j}(N)\le -\sqrt{\tfrac{\log(KN^2\,\frac{n(n+1)}{\delta})}{N}}\Big\},
	\]
	so that $\{\phi_2=\texttt{False}\}\subseteq \mathcal{E}_1\cup\mathcal{E}_2$. Since $\Delta_{i^*,j}\ge 0$ for all $j\neq i^*$,
	Hoeffding's inequality and a union bound give
	\begin{align}
		\mathbb{P}(\mathcal{E}_2) &\le \sum_{j\neq i^*}\sum_{N\ge 1} \mathbb{P}\!\left(\tilde{\Delta}_{i^*,j}(N)-\Delta_{i^*,j}\le -\sqrt{\tfrac{\log(KN^2\,\frac{n(n+1)}{\delta})}{N}}\right)\nonumber\\
		&\le \sum_{j\neq i^*}\sum_{N\ge 1}\frac{\delta}{K\,n(n+1)\,N^2}
		\le \frac{2\delta}{n(n+1)}\le \frac{\delta}{4}~,\label{eq:E2_bound_fc_sc}
	\end{align}
	where we used $\sum_{N\ge 1}N^{-2}\le 2$ and $n\ge 4$, which follows from $T\ge 8K\log_{8/7}(K)$ and the common definition of $n$.
	
	\noindent Next, for each $i\neq i^*$ define
	\[
	\bar{T}_i:=\ceil*{\frac{16}{\Delta_{i^*,i}^2} \log\left(\frac{32Kn(n+1)}{\delta \Delta_{i^*,i}^2}\right)}~.
	\]
	Lemma~\ref{lem:tech_fc} below ensures that $\sum_{i\neq i^*}\bar{T}_i<\floor*{G_{1,\delta}}\le\floor*{T}$, hence
	\begin{align}
		\mathbb{P}(\mathcal{E}_1)
		&=\mathbb{P}\!\left(\sum_{i\neq i^*}T_i=\floor*{T}\right)
		\le \mathbb{P}\!\left(\sum_{i\neq i^*}T_i\ge \floor*{G_{1,\delta}}\right)
		\le \mathbb{P}\!\left(\exists i\neq i^*: T_i>\bar{T}_i\right)\nonumber\\
		&\le \sum_{i\neq i^*}\mathbb{P}(T_i>\bar{T}_i)~.\label{eq:E1_reduction_fc_sc}
	\end{align}
	If $T_i>\bar{T}_i$, then at time $N=\bar{T}_i$ the arm $i$ was not eliminated, meaning
	\[
	\tilde{\Delta}_{i^*,i}(\bar{T}_i)\;<\;\sqrt{\frac{\log\!\big(K\bar{T}_i^2\,\frac{n(n+1)}{\delta}\big)}{\bar{T}_i}}~.
	\]
	By Lemma~\ref{lem:tech_fc}, the RHS is at most
	$\Delta_{i^*,i}-\sqrt{\frac{\log(K\bar{T}_i^2\,\frac{n(n+1)}{\delta})}{2\bar{T}_i}}$, hence
	\begin{align}
		\mathbb{P}(T_i>\bar{T}_i)
		&\le \mathbb{P}\!\left(\tilde{\Delta}_{i^*,i}(\bar{T}_i)-\Delta_{i^*,i}
		< -\sqrt{\frac{\log(K\bar{T}_i^2\,\frac{n(n+1)}{\delta})}{2\bar{T}_i}}\right)\nonumber\\
		&\le \sum_{N\ge 1}\frac{\delta}{K\,n(n+1)\,N^2}
		\le \frac{2\delta}{K\,n(n+1)}\le \frac{\delta}{4K}~,\label{eq:Ti_tail_fc_sc}
	\end{align}
	where the last line uses Hoeffding and a union bound over $N\ge 1$.
	Combining \eqref{eq:E1_reduction_fc_sc} and \eqref{eq:Ti_tail_fc_sc} yields $\mathbb{P}(\mathcal{E}_1)\le \delta/4$.
	Together with \eqref{eq:E2_bound_fc_sc}, we obtain $\mathbb{P}(\phi_2=\texttt{False})\le \delta/2$, hence
	$\mathbb{P}(\phi_2=\texttt{True})\ge 1-\delta/2$ when $I=i^*$. This completes Regime~1.
	
	\medskip\noindent
	\textbf{Regime 2: $G_{1,\delta}> G_{2,\delta}+G_{3,\delta}+G_0$.}
	Then $T\in[G_{2,\delta}+G_{3,\delta}+G_0,\ 2(G_{2,\delta}+G_{3,\delta}+G_0)]$.
	We show that, for such $T$, the certification variable $\phi_1$ in Algorithm~\ref{algo:3} remains $\texttt{True}$ with probability at least $1-\delta$.
	In this regime, we condition on an arbitrary pre-round history, including histories for which $i^*\notin A_k$.
	The second bound in Lemma~\ref{lem:main2} remains valid under this conditioning: its proof uses only $E_k\subseteq [K]\setminus\{i^*\}$, skew-symmetry, and the fresh independent samples for each arm, and does not require $i^*\in A_k$.
	Likewise, Lemma~\ref{lem:Fk} remains valid; in its case $\abs{A_k}\in\{3,4\}$, one simply allows $\abs{F_k}\in\{2,3,4\}$, for which $\abs{F_k}-1\ge\ceil*{\abs{F_k}/3}$ still holds.
	Let $\bar{\alpha}_k$ be the arm ranked $\abs{A_k}-\ceil{\abs{A_k}/8}+1$ at round $k$ according to scores $S_k(\cdot)$.
	Define
	\[
	L_{k,\delta}:=\sqrt{\frac{2c\log(T)}{\lceil B_k/4\rceil}\log\!\Big(\frac{1}{\delta_{n,K}}\Big)},\qquad
	\delta_{n,K}:=\frac{\delta}{8K^2\log_{8/7}(K)\log(T)\,n(n+1)}~.
	\]
	By the update rule for $\phi_1$, the event $\{\phi_1=\texttt{False}\}$ implies that for some $k<k_{\max}$,
	$S_k(\bar{\alpha}_k)\ge -L_{k,\delta}$. The definition of $\bar{\alpha}_k$ entails that at most $\ceil*{\abs{A_k}/8}$ arms have score strictly smaller than $-L_{k,\delta}$, i.e.
	\begin{equation}\label{eq:score_count_fc_sc}
		\{S_k(\bar{\alpha}_k)\ge -L_{k,\delta}\}\ \subseteq\
		\left\{\sum_{\alpha\in A_k}\mathds{1}\big(S_k(\alpha)<-L_{k,\delta}\big)\le \ceil*{\frac{\abs{A_k}}{8}}\right\}~.
	\end{equation}
	
	\smallskip
	\noindent We now relate $-L_{k,\delta}$ to the threshold $\frac12\bar{\Delta}_k$ used in Lemma~\ref{lem:main2}. 
	Recall the definitions (as in the proof of Theorem~\ref{thm:amain2}): let
	\[
	E_k:=\Big\{\alpha\in A_k:\ \Delta^{(k)}_{\alpha,(\ceil{\abs{A_k}/4})}\le 0\Big\},\qquad \abs{E_k}\ge \ceil*{\abs{A_k}/4}~,
	\]
	and define the $7/8$-quantile $\bar{\Delta}_k:=\Delta_{E_k:\ceil{(7/8)\abs{E_k}}}\le 0$ and the subset
	\[
	F_k:=\Big\{\alpha\in E_k:\ \Delta_{\alpha,(s_\alpha)}\le \bar{\Delta}_k\Big\}~.
	\]
	Since $1\le s_\alpha\le K_{\alpha;<0}$ for every $\alpha\in E_k$, we have $\bar{\Delta}_k<0$.
	Lemma~\ref{lem:tech_fc_bis} ensures that for all $k<k_{\max}$ and all $T$ in the present regime,
	\begin{equation}\label{eq:L_vs_Deltabar_fc_sc}
		-L_{k,\delta}\ \ge\ \frac12\,\bar{\Delta}_k~.
	\end{equation}
	We assume that $\abs{A_k} \ge 3$, the case $\abs{A_k} = 2$ is treated in the end. Using \eqref{eq:L_vs_Deltabar_fc_sc} inside \eqref{eq:score_count_fc_sc} gives
	\begin{align}
		\mathbb{P}\!\left(S_k(\bar{\alpha}_k)\ge -L_{k,\delta}\right) &\le \mathbb{P}\!\left(\sum_{\alpha\in A_k}\mathds{1}\Big(S_k(\alpha)<\tfrac12\bar{\Delta}_k\Big) \le \ceil*{\frac{\abs{A_k}}{8}}\right)\nonumber\\
		&= \mathbb{P}\!\left(\sum_{\alpha\in A_k}\mathds{1}\Big(S_k(\alpha)\ge\tfrac12\bar{\Delta}_k\Big)
		\ge \abs{A_k}-\ceil*{\frac{\abs{A_k}}{8}}\right)~.\label{eq:Ak_to_upper_tail_fc_sc}
	\end{align}
	Since $\abs{A_k}-\ceil*{\abs{A_k}/8}=\abs{A_{k+1}}$ and Lemma~\ref{lem:Fk} gives $\abs{A_{k+1}\cap F_k}\ge \ceil*{\abs{F_k}/3}$, the RHS of \eqref{eq:Ak_to_upper_tail_fc_sc} is upper bounded by
	\[
	\mathbb{P}\!\left(\sum_{\alpha\in F_k}\mathds{1}\Big(S_k(\alpha)\ge\tfrac12\bar{\Delta}_k\Big)\ge \ceil*{\frac{\abs{F_k}}{3}}\right)~.
	\]
	Lemma~\ref{lem:main2} then yields, for a numerical constant $c_2>0$,
	\begin{equation}\label{eq:lemma_main2_apply_fc_sc}
		\mathbb{P}\!\left(S_k(\bar{\alpha}_k)\ge -L_{k,\delta}\right) \le \exp\!\left( -c_2\frac{T-c_3 H^{(0)}_{\text{explore}}(\bm{s})\log^6(H^{(0)}_{\text{explore}}(\bm{s}))}{\log^4(K)\log(T)}\cdot\frac{1}{H^{(1)}_{\text{explore}}(\bm{s})}\right)~.
	\end{equation}
	
	\smallskip
	\noindent Next, we use $T \ge G_{3,\delta}+G_0$ with Lemma~\ref{lem:inverse}, which turns the last inequality into a bound on the exponent term of \eqref{eq:lemma_main2_apply_fc_sc} leading to
	\begin{equation}\label{eq:per_round_bound}
		\mathbb{P}\!\left(S_k(\bar{\alpha}_k)\ge -L_{k,\delta}\right) \le \frac{\delta}{2Kk_{\max}}~.
	\end{equation}
	which is the desired per-round bound.
	
	\smallskip 
	\noindent Suppose that $\abs{A_k}=2$. Then $\bar{\alpha}_k$ is an arm with the smaller score.
	Thus $S_k(\bar{\alpha}_k)\ge -L_{k,\delta}$ implies that every arm in $F_k$ has score at least $\tfrac12\bar{\Delta}_k$. Consequently,
	\begin{align*}
		\mathbb{P}\!\left(S_k(\bar{\alpha}_k)\ge -L_{k,\delta}\right)
		&\le \mathbb{P}\!\left(\sum_{\alpha\in F_k}\mathds{1}\!\left(S_k(\alpha)\ge \tfrac12\bar{\Delta}_k\right)
		\ge \ceil*{\frac{\abs{F_k}}{3}}\right)\\
		&\le \exp\!\left(-c_2\frac{T-c_3 H^{(0)}_{\text{explore}}(\bm{s})\log^6(H^{(0)}_{\text{explore}}(\bm{s}))}
		{\log^4(K)\log(T)H^{(1)}_{\text{explore}}(\bm{s})}\right)\\
		&\le \frac{\delta}{2Kk_{\max}}~,
	\end{align*}
	where we used the second bound in Lemma~\ref{lem:main2}, whose proof also applies when $\abs{A_k}=2$, and Lemma~\ref{lem:inverse}.
	
	\medskip
	\noindent Finally, since $\{\phi_1=\texttt{False}\}\subseteq \bigcup_{k<k_{\max}}\{S_k(\bar{\alpha}_k)\ge -L_{k,\delta}\}$, a union bound and \eqref{eq:per_round_bound} yield
	\[
	\mathbb{P}(\phi_1=\texttt{False})\ \le\ \sum_{k=1}^{k_{\max}-1}\frac{\delta}{2K\,k_{\max}}\ \le\ \frac{\delta}{2K}\le\frac{\delta}{4}~,
	\]
	which completes Regime~2.
	
	\medskip\noindent
	\textbf{Conclusion.}
	In either regime, for any $T$ satisfying \eqref{eq:threshold_T_fc_sc} the call to Algorithm~\ref{algo:3}
	returns $\phi_1\vee \phi_2=\texttt{True}$ with probability at least $1-\delta$.
	Since Algorithm~\ref{algo:fc} doubles $T$ until this event occurs, its total number of queries $N_\delta$
	is at most a universal constant multiple of $M_\delta=\min\{G_{1,\delta},G_{2,\delta}+G_{3,\delta}+G_0\}$
	with probability at least $1-\delta$. Absorbing numerical constants into $\bar{c}_1,\bar{c}_2$ yields the stated
	bounds of Theorem~\ref{thm:fc2}.
	The Regime~2 argument also applies to every $T\ge G_{2,\delta}+G_{3,\delta}+G_0$, and its failure bound tends to zero as $T\to\infty$.
	Hence the doubling procedure terminates almost surely.
	
	\bigskip
	\noindent The lemmas below are technical. 
	
	\begin{lemma}\label{lem:tech_fc}
		Consider the notation introduced in Regime~1 of the proof of Theorem~\ref{thm:fc2}, with $T\le 2G_{1,\delta}$.
		Then we have
		\[
		\sum_{i\neq i^*} \bar{T}_i < \floor*{G_{1, \delta}}~.
		\]
		Moreover, for all $i\neq i^*$
		\[
		2\sqrt{\frac{\log\left(K\bar{T}^2_{i}\frac{n(n+1)}{\delta}\right)}{\bar{T}_i}} < \Delta_{i^*,i}~.
		\] 
	\end{lemma}
	\begin{proof}
		Using $\ceil{x}\le x+1$ and $\Delta_{i^*,i}^{-2}\le H_{\text{cw}}$, we have
		\begin{align}
			\sum_{i\neq i^*}\bar{T}_i
			&\le K-1+16H_{\text{cw}}\Big(\log(32KH_{\text{cw}}/\delta)+\log(n(n+1))\Big)~.\label{eq:a1}
		\end{align}
		The common counter definition gives $n+1\le\log_2 T$ for $T\ge 8K\log_{8/7}(K)$.
		Since $T\le 2G_{1,\delta}$, \eqref{eq:logG1_bound_fc_sc} therefore implies
		\begin{equation}\label{eq:target}
			\log(n(n+1))\le 2\log\!\Big(\log_2(2G_{1,\delta})\Big)
			\le 4\log\!\Big(c_1^{-1}H_{\text{cw}}\log(K/\delta)\Big)~,
		\end{equation}
		where we used $c_1\le 1/16$, $H_{\text{cw}}\ge 4(K-1)$ and $\delta<1/6$ in the last inequality.
		Since $K-1\le H_{\text{cw}}/4$ and
		$\log(c_1^{-1}H_{\text{cw}}\log(K/\delta))\le\log(32c_1^{-1}KH_{\text{cw}}/\delta)$, we obtain
		\begin{align*}
			\sum_{i\neq i^*}\bar{T}_i
			&\le \frac{H_{\text{cw}}}{4}+80H_{\text{cw}}\log\!\Big(32c_1^{-1}KH_{\text{cw}}/\delta\Big)\\
			&\le 81H_{\text{cw}}\log\!\Big(32c_1^{-1}KH_{\text{cw}}/\delta\Big)
			< G_{1,\delta}-1\le\floor*{G_{1,\delta}}~.
		\end{align*}
		This proves the first claim.
		
		\noindent For the second claim, the definition of $\bar{T}_i$ and $\abs{\Delta_{i^*,i}}\le 1/2$ give
		\[
		\bar{T}_i\le \frac{17}{\Delta_{i^*,i}^2}\log\!\left(\frac{32Kn(n+1)}{\delta\Delta_{i^*,i}^2}\right).
		\]
		Consequently,
		\[
		\log\!\left(K\bar{T}_i^2\frac{n(n+1)}{\delta}\right)
		<4\log\!\left(\frac{32Kn(n+1)}{\delta\Delta_{i^*,i}^2}\right)
		\le \frac{\Delta_{i^*,i}^2\bar{T}_i}{4}~.
		\]
		Indeed, the first strict inequality follows by substituting the preceding upper bound on $\bar{T}_i$ and using
		$\log^2 x<x$ for $x\ge 256$; here $32Kn(n+1)/(\delta\Delta_{i^*,i}^2)>256$.
		This is equivalent to the second claim.
	\end{proof}
	
	\begin{lemma}\label{lem:tech_fc_bis}
		Consider the notation introduced in the proof of Theorem~\ref{thm:fc2}.
		If $\bar{\Delta}_k<0$ and
		\[
		T\ge G_{2,\delta}+G_0~,
		\]
		then
		\[
		-L_{k,\delta} \ge \frac{1}{2}\bar{\Delta}_k~.
		\]
	\end{lemma}
	
	\begin{proof}
		Let 
		\[
		\delta_{n,K}=\frac{\delta}{8K^2\log_{8/7}(K)\log(T)\,n(n+1)}.
		\]
		Assume $\bar{\Delta}_k<0$. Since $L_{k,\delta}\ge 0$, the inequality
		$-L_{k,\delta}\ge \tfrac12\bar{\Delta}_k$ is equivalent to
		$L_{k,\delta}\le -\tfrac12\bar{\Delta}_k$, i.e.
		\begin{equation}\label{eq:tar}
			L_{k,\delta}^2 \ \le\ \frac{\bar{\Delta}_k^2}{4}.	
		\end{equation}
		Recalling
		\[
		L_{k,\delta}^2
		=\frac{2c\log(T)}{\lceil B_k/4\rceil}\,\log\!\Big(\frac{1}{\delta_{n,K}}\Big), \quad B_k=\floor*{\frac{T}{\abs{A_k}\log_{8/7}(K)}}~,
		\]
		and using $\ceil*{B_k/4}\ge B_k/4\ge T/(8\abs{A_k}\log_{8/7}(K))$, we see that \eqref{eq:tar} is implied by 
		\begin{equation}\label{eq:goal_fc_bis}
			\log(T)\,\log\!\Big(\frac{1}{\delta_{n,K}}\Big)
			\ \le\ \frac{\bar{\Delta}_k^2}{64c}\,\frac{T}{\abs{A_k}\log_{8/7}(K)}~.
		\end{equation}
		Next, using the inequality \eqref{eq:pk0}
		\[
		\frac{\abs{A_k}}{\bar{\Delta}_k^2}\ \le\ 32\sum_{\alpha\in E_k}\frac{1}{\Delta_{\alpha,(s_\alpha)}^2}~,
		\]
		and the fact that $i^*\notin E_k$ (if $i^*\in A_k$, its restricted gaps are all positive, and otherwise it is absent from $E_k\subseteq A_k$),
		we have $\sum_{\alpha\in E_k}\frac{1}{\Delta_{\alpha,(s_\alpha)}^2}\le H_{\mathrm{certify}}(\bm{s})$. Hence
		\begin{equation}\label{eq:sim1}
			\frac{\bar{\Delta}_k^2}{|A_k|}\ \ge\ \frac{1}{32\,H_{\mathrm{certify}}(\bm{s})}~.	
		\end{equation}
		Moreover, the common counter definition gives $n+1\le\log_2 T$. Hence
		\begin{align}
			\log\!\Big(\frac{1}{\delta_{n,K}}\Big)
			&=\log\!\Big(\frac{8K^2\log_{8/7}(K)}{\delta}\Big)+\log\log(T)+\log(n(n+1))\nonumber\\
			&\le\log\!\Big(\frac{8K^2\log_{8/7}(K)}{\delta}\Big)+3\log\log(T)+4~.\label{eq:sim2}
		\end{align}
		Combining \eqref{eq:sim1}, \eqref{eq:sim2} and \eqref{eq:goal_fc_bis}, it suffices to show
		{\small\begin{equation}\label{eq:last_tar}
				\log(T)\left[\log\!\Big(\frac{8K^2\log_{8/7}(K)}{\delta}\Big)+3\log\log(T)+4\right]
				\le \frac{T}{2048c\log_{8/7}(K)H_{\text{certify}}(\bm{s})}~.
		\end{equation}}
		Since $T\ge G_{2,\delta}+G_0$, this follows from Lemma~\ref{lem:inverse2}, which concludes the proof.
	\end{proof}
	
	\begin{lemma}\label{lem:inverse}
		Suppose that $T \ge G_0 +G_{3,\delta}$.
		Then we have
		\[
		\exp\left( -c_2 \frac{T-c_3H^{(0)}_{\text{explore}}(\bm{s}) \log^6(H^{(0)}_{\text{explore}}(\bm{s}))}{\log^4(K)\log(T)} \frac{1}{H_{\text{explore}}^{(1)}(\bm{s})}\right) \le \frac{\delta}{2Kk_{\max}}~.
		\]
	\end{lemma}
	\begin{proof}
		The desired inequality
		is equivalent to
		\begin{equation}\label{eq:goal_inverse_new}
			\frac{T-c_3H^{(0)}_{\text{explore}}(\bm{s}) \log^6(H^{(0)}_{\text{explore}}(\bm{s}))}{\log(T)}\ \ge\ \frac{1}{c_2}\,H_{\text{explore}}^{(1)}(\bm{s})\log^4(K) \log\left(\frac{2Kk_{\max}}{\delta}\right)~.
		\end{equation}
		Define
		\[
		A \coloneq 32\frac{c}{c_2}H_{\text{explore}}^{(1)}(\bm{s})\log^4(K)\log\left(\frac{2Kk_{\max}}{\delta}\right)~.
		\]
		By assumption,
		{\small 
			\[
			T \ \ge\ 2\frac{c_3}{c_2}H^{(0)}_{\text{explore}}(\bm{s})\log^6(H^{(0)}_{\text{explore}}(\bm{s})) \;+\; A\log(A)~.
			\]
		}
		Since $c_2\le 1$ (as is the case for the numerical constant $c_2$ coming from the preceding bounds),
		the first term implies $T\ge 2c_3 H^{(0)}_{\text{explore}}(\bm{s}) \log^6(H^{(0)}_{\text{explore}}(\bm{s}))$, hence
		\begin{equation}\label{eq:TminusB_lower}
			T-c_3 H^{(0)}_{\text{explore}}(\bm{s}) \log^6(H^{(0)}_{\text{explore}}(\bm{s})) \ge\ \frac{T}{2}.
		\end{equation}
		Next, the function $f(x):=x/\log x$ for $x>e$ is increasing.
		Moreover, we have $\log(2Kk_{\max}/\delta)>1$ and $H_{\text{explore}}^{(1)}(\bm{s})\ge 4$ (since $\abs{\Delta}_{i,(s_i)}\le \tfrac12$ and $s_i\le K$), and with $c\ge 1$ and $c_2\le 1$ this yields $A\ge 32\cdot 4\cdot (\log 2)^4\cdot \log 4>e$. Therefore, $A\log A>e$ and in particular $\log(A\log A)>0$.
		
		\smallskip
		\noindent Since $T\ge A\log A$ and $f$ is increasing on $(e,\infty)$, we obtain
		\begin{equation}\label{eq:f_monotone}
			\frac{T}{\log T}\ =\ f(T)\ \ge\ f(A\log A)\ =\ \frac{A\log A}{\log(A\log A)}.
		\end{equation}
		Finally, because $A>e$, we have $\log(A\log A)=\log A+\log\log A\le \log A+\log A=2\log A$, and therefore
		\begin{equation}\label{eq:fA_lower}
			\frac{A\log A}{\log(A\log A)}\ \ge\ \frac{A\log A}{2\log A}\ =\ \frac{A}{2}.
		\end{equation}
		Combining \eqref{eq:TminusB_lower}, \eqref{eq:f_monotone}, and \eqref{eq:fA_lower} gives
		\[
		\frac{T-c_3 H^{(0)}_{\text{explore}}(\bm{s}) \log^6(H^{(0)}_{\text{explore}}(\bm{s}))}{\log T}\ \ge\ \frac{1}{2}\cdot \frac{T}{\log T}\ \ge\ \frac{1}{2}\cdot \frac{A}{2}\ =\ \frac{A}{4}.
		\]
		By the definition of $A$,
		\[
		\frac{A}{4}\ =\ 8\,\frac{c}{c_2}\,H_{\text{explore}}^{(1)}(\bm{s})\log^4(K)\,\log(2Kk_{\max}/\delta)\ \ge\ \frac{1}{c_2}\,H_{\text{explore}}^{(1)}(\bm{s})\log^4(K)\,\log(2Kk_{\max}/\delta),
		\]
		where we used $c\ge 1$.
		This proves \eqref{eq:goal_inverse_new}, and hence
		\[
		\exp\left( -c_2\,\frac{T-c_3 H^{(0)}_{\text{explore}}(\bm{s}) \log^6(H^{(0)}_{\text{explore}}(\bm{s}))}{\log^4(K)\log(T)\,H_{\text{explore}}^{(1)}(\bm{s})}\right)
		\le e^{-\log(2Kk_{\max}/\delta)} = \frac{\delta}{2Kk_{\max}}~.
		\]
	\end{proof}
	\begin{lemma}\label{lem:inverse2}
		Let $K\ge 2$, $\delta\in(0,1)$ and $T\ge G_0+G_{2,\delta}$. Then
		{\small\[
			\log(T)\left[\log\!\Big(\frac{8K^2\log_{8/7}(K)}{\delta}\Big)+3\log\log(T)+4\right]
			\le \frac{T}{2048c\log_{8/7}(K)H_{\text{certify}}(\bm{s})}~.
			\]}
	\end{lemma}
	\begin{proof}
		Let $L:=\log_{8/7}(K)$, $H:=H_{\mathrm{certify}}(\bm{s})$, and set
		\[
		M:=2048cLH,\qquad B:=\log\!\Big(\frac{32c}{c_2}\frac{K H^{(0)}_{\mathrm{explore}}(\bm{s})}{\delta}\Big)~.
		\]
		Since $H\le H^{(0)}_{\mathrm{explore}}(\bm{s})$, $H^{(0)}_{\mathrm{explore}}(\bm{s})\ge 4K$,
		$c\ge 1$ and $c_2\le\log(8/7)(\log 2)^2/32$, we have
		\[
		B\ge 8,\qquad \log M\le B,\qquad \log\!\Big(\frac{8K^2L}{\delta}\Big)\le 2B~.
		\]
		Moreover,
		\[
		\frac{512c}{c_2}H\log^3(K)\ge 8M~,
		\]
		so the definition of $G_{2,\delta}$ gives
		\[
		T\ge G_{2,\delta}\ge T_0:=8MB\log(8MB)\ge e^2~.
		\]
		
		\medskip
		\noindent Set
		\[
		u(T):=\log(T)\left[\log\!\Big(\frac{8K^2L}{\delta}\Big)+3\log\log T+4\right]~.
		\]
		A direct differentiation shows that $u(T)/T$ is decreasing on $[e^2,\infty)$. Hence, for $T\ge T_0$,
		\[
		u(T)\le \frac{T}{T_0}u(T_0)~.
		\]
		Since $B\ge 8$ and $\log M\le B$, we have
		\[
		\log(8MB)\le 2B,\qquad
		\log T_0\le 2\log(8MB)\le 4B,\qquad
		\log\log T_0\le\log(4B)\le B/2~.
		\]
		Therefore,
		\begin{align*}
			u(T_0)
			&\le 2\log(8MB)\big(2B+3B/2+4\big)\\
			&\le 8B\log(8MB)=\frac{T_0}{M}~.
		\end{align*}
		Combining these inequalities yields $u(T)\le T/M$, which is the desired bound.
	\end{proof}

	\section{Minimax Fixed-Budget Lower Bounds}\label{sec:lb_fb}
	
	In this section, we establish lower bounds for the fixed-budget CW identification. In some way, these are the counterparts of the results of Section~\ref{sec:lb_fc}, only that we only derive minimax bounds, similar to Corollary~\ref{coro:Lower_bound_minimax}.  
	
	We specify a class of distributions, parametrized by the row of the Condorcet Winner. Let $\underline{\Delta} = (\Delta_1,\Delta_2, \dots, \Delta_K )$, with $\Delta_{1}=0$ and $\Delta_i\in(0,1/4)$ for $i\ne 1$. Define $\mathbb{D}^{(1)}(\underline{\Delta})$ as the set of dueling feedback distributions whose gap matrix $\bDelta \in \mathbb{D}_{\mathrm{cw}}$ is such that the row of the Condorcet winner $i^*$, namely $\Delta_{i^*,\cdot}$, is equal to $\underline{\Delta}$ up to a permutation $\sigma$. Formally,
	\begin{equation}\label{def:D1}
		\mathbb{D}^{(1)}(\underline{\Delta})
		\coloneqq \Bigl\{
		\bDelta\in \mathbb{D}_{\mathrm{cw}}
		: \exists \,\sigma \in \mathfrak{S}_{K}
		\text{ s.t. } i^*_{\bDelta}=\sigma^{-
        1}(1),  \text{ and } 
		\Delta_{i^*,\cdot} = \sigma(\underline{\Delta})
		\Bigr\},
	\end{equation}
	where $\mathfrak{S}_{K}$ is the set of permutations on $[K]$, and for any $x\in\mathbb{R}^K$, $\sigma(x)=(x_{\sigma(i)})_{i\in [K]}$. The following result is the counterpart of Theorem~\ref{thm:LB_exp_instance_first_regime}. 
	
	\begin{theorem}\label{thm:LB_exp_minimax_first_regime}
		Let $K \ge 2$ and $T\in \mathbb{N}^*$ and  consider any vector $\underline{\Delta}$ with $\Delta_1=0$, and $\Delta_i\in(0,1/4)$ for $i\ne 1$. For any algorithm $A$ with a fixed budget $T$, one has
		\[
		\max_{\bDelta \in \mathbb{D}^{(1)}(\underline{\Delta})}
		\mathbb{P}_{\bDelta, A}\left( \hat{i}_T \neq i^* \right)
		\ge \frac{1}{4}\exp\!\left(-22 \frac{T}{H_{\mathrm{cw}}}\right)~,
		\] 
		where $H_{\mathrm{cw}}=\sum_{i=2}^{K} \frac{1}{\Delta_i^2}$.
	\end{theorem}
	
	\begin{remark}
		This theorem reveals that one can exhibit a matrix $\Delta$ for which the the exponential error decay scaling as $\exp(-T/H_{\mathrm{cw}})$ in Theorem~\ref{thm:amain1} is tight. The proof can be found in Appendix~\ref{sec:proof_LB_exp_FB}.
	\end{remark}
	
	We now turn to the counterpart of Theorem~\ref{thm:LB_hp_instance}. Consider the following class of environments, for which the quantities $(\bm{s}_{\bDelta}^*,\Delta_{(\bm s^*)})$, defined in Section~\ref{sec:lb_fc} are equal to a given couple $(\underline{\Delta},\underline{s})$ up to a permutation of the arms. By convention, we extend $\bm{s}_{\bDelta}^*=(s^*_i)_{i\ne i^*}$ into a $K$-dimensional vector by fixing as $0$ the $i^*$-th entry, that is $s^*_{i^*}\coloneqq 0$. We proceed similarly for $\Delta_{(s^*)}$. 
	\begin{equation}\label{eq:definition:D2}
		\mathbb{D}^{(2)}(\underline{\Delta},\underline{s}) := 
		\left\{\bDelta\in\mathbb{D}_{\mathrm{cw}}:\exists\sigma\in\mathfrak{S}_K\text{ s.t.}\bm{s}_{\bDelta}^{*}=\sigma(\underline{s})\text{ and }\bDelta_{(s^*)}=-\sigma(\underline{\Delta})\right\}.
	\end{equation}
	
	\begin{theorem}\label{thm:LB_minimax_H2}
		Let $T\in\mathbb{N}^*$, $K$ a multiple of $8$, $\underline{\Delta}=(\Delta_1,\dots,\Delta_K)$ with $\Delta_1=0$ and $\Delta_i\in(0,1/4)$ for $i\neq1$, and $\underline{s}=(s_1,\dots,s_K)$ with $s_1=0$, $1\leqslant s_i\leqslant K/8$ for $i=2,\dots,K$. Then, any fixed-budget algorithm $A$ satisfies
		\begin{equation}\label{eq:LB_FB_H2thm}
			\sup_{\boldsymbol{\Delta}\in\mathbb{D}^{(2)}(\underline{\Delta},\underline{s})}\mathbb{P}_{A,\boldsymbol{\Delta}}(\hat{i}_T\neq i^*(\boldsymbol{\Delta}))
			\geqslant\frac{1}{4}\exp\!\left(-5\frac{T}{\max_{i=2}^{K/2}\frac{K}{s_i\Delta_i^2}}\right).
		\end{equation}
		
		If, additionally $\Delta_i=\Delta>0$ and $s_i=s\in[K/8]$ for all $i\in\{2,\dots,K/2\}$, then
		\begin{align}
			\sup_{\boldsymbol{\Delta}\in\mathbb{D}^{(2)}(\underline{\Delta},\underline{s})}\mathbb{P}_{\boldsymbol{\Delta},A}(\hat{i}_T\neq i^*(\boldsymbol{\Delta}))
			&\geqslant\frac{1}{4}\exp\!\left(-10\frac{T\Delta^2}{K}\right),\label{eq:LB_FB_H3thm}\\
			\sup_{\boldsymbol{\Delta}\in\mathbb{D}^{(2)}(\underline{\Delta},\underline{s})}\mathbb{P}_{\boldsymbol{\Delta},A}(\hat{i}_T\neq i^*(\boldsymbol{\Delta}))
			&\geqslant\frac{1}{2}-\sqrt{37\frac{s\Delta^2}{K^2}T}.\label{eq:LB_FB_constant}
		\end{align}
	\end{theorem}
	
	\begin{remark}
		Equation~\eqref{eq:LB_FB_constant} shows that for matrices where half the rows equal (up to permutation) a vector with $s$ entries at $-\Delta$ and the rest near zero, any algorithm must pay $\frac{K^2}{s\Delta^2}$ to reach a constant success probability. This is aligned with the probability-independent cost $H_{\mathrm{explore}}^{(0)}(\underline{s}))$ that suffers Algorithm~\ref{algo:3}  --see Theorem~\ref{thm:amain1}. 
		
		Equation~\eqref{eq:LB_FB_H2thm} establishes probability of error is no smaller than $\exp(-T/H_{\mathrm{explore}}^{(1)}(\underline{s}))$, while \eqref{eq:LB_FB_H3thm} implies that the quantity $\exp(-T/H_{\mathrm{certify}}(\bm{s}))$ is required, at least for some specific highly structured matrices.
	\end{remark}
	
	\paragraph{Comparison with Fixed Confidence lower bounds.}
	
	The proofs for fixed-confidence and fixed-budget settings are remarkably similar. This reflects the strong connection between fixed-budget algorithms and fixed-confidence algorithms with high-probability budget bounds. However, fixed-budget minimax lower bounds cannot be directly deduced from the instance-dependent fixed confidence lower bounds derived earlier. We conjecture that obtaining instance-dependent lower bounds in the fixed-budget setting is considerably more challenging---if not outright impossible. 
	
	The fundamental reason lies in the nature of the algorithms themselves. Any $\delta$-correct fixed-confidence algorithm must incorporate an internal stopping criterion that checks that the identified arm $\hat i$ is indeed the CW, under the assumption that such a unique CW exists in $\mathbb{D}_{\mathrm{cw}}$. Consider now what happens when applying this algorithm to a modified environment with two ``weak'' Condorcet winner candidates $i^*_1$ and $i^*_2$, where both rows satisfy $\Delta_{i^*_1,\cdot}\geqslant\mathbf{0}$ and $\Delta_{i^*_2,\cdot}\geqslant\mathbf{0}$. The algorithm cannot distinguish between these candidates with probability $1-\delta$, forcing an infinite expected stopping time as the verification step never confidently resolves the ambiguity.
	
	Fixed-budget algorithms fundamentally lack such stopping rules, rendering this infinite-budget argument inapplicable. Instead, our lower bounds rely on symmetry arguments: for matrices where multiple suboptimal arms have identical ``difficulty profiles'' (i.e., identical multisets of negative gaps), any reasonable algorithm must select uniformly at random among the ambiguous best arms. Deriving such lower bounds requires constructing highly symmetric matrices that exploit this randomization, a significantly more restrictive condition than the instance-dependent constructions used in fixed confidence.
	
	In Sections~\ref{sec:proof_LB_FC} and~\ref{sec:proof_LB_FB}, we propose two complete constructions respectively for each setting. The proofs for fixed-budget results are similar to the fixed-confidence ones, except that they are applied to specific highly symmetric matrices for the reasons described above.
	
	\section{Proofs of Section~\ref{sec:lb_fc}}\label{sec:proof_LB_FC}
	
	In this section, we provide all proofs for the instance-dependent fixed-confidence lower bounds. 
	
	We begin with a short roadmap in Subsection~\ref{sec:roadmap_LB} describing the classical change-of-measure arguments underlying all our constructions (along the way, we fix some notation). In Subsection~\ref{sec:proof_LB_exp_fc}, we prove Proposition~\ref{thm:LB_exp_instance_first_regime}, separating the expected budget bound (Subsection~\ref{sec:proof_expectation_Hkarnin}) from the high-probability quantile bound (Subsection~\ref{sec:proof_HP_karnin}). 
	
	In Subsection~\ref{sec:proof_LB_HP_fc}, we explain the construction leading to Theorem~\ref{thm:LB_hp_instance} and state a more precise formulation in Theorem~\ref{thm:precise_LB_high_proba}, proved in Subsection~\ref{sec:proof_precise_LB}. Corollary~\ref{coro:Lower_bound_minimax} follows in Subsection~\ref{sec:proof_coro_minimax}. 
	
	Finally, Subsection~\ref{sec:additional_LB} discusses lower bounds preserving CW row structure.
	
	\subsection{Roadmap on change-of-measure lower bounds}\label{sec:roadmap_LB}
	
	All our proofs follow a common three-step structure.
	
	\paragraph{First step: reference and alternative instances.}
	We fix a \emph{reference instance} $\bDelta \in \mathbb{D}_{\mathrm{cw}}$, that is, a gap matrix admitting 
	a (unique) Condorcet winner $i^*(\bDelta)$. In the fixed-confidence regime, our lower bounds are 
	instance-dependent, so all constructions are built directly from the given $K\times K$ matrix
	\[
	\bDelta = (\Delta_{i,j})_{i,j \in [K]}.
	\]
	For some results (e.g., Theorem~\ref{thm:LB_hp_instance}), we also consider a local class of 
	instances obtained from $\bDelta$ by permuting the negative entries of each row, while preserving 
	prescribed structural features (CW, sign structure, multiset of negative entries, effective 
	sparsity, and so on). This leads to families $\{\bDelta^{(\pi)}\}_\pi$ indexed by permutations $\pi$, 
	but the reference object remains $\bDelta$.
	
	For each suboptimal arm $k \neq i^*$, we construct an \emph{alternative instance}
	\[
	\bDelta^{(k)}
	\]
	in such a way that $i^*$ is no longer the (strong) Condorcet winner, while $k$ becomes the CW or at 
	least a \emph{weak} CW, in the sense that the $k$-th row of $\bDelta^{(k)}$ contains only 
	non-negative entries. A typical construction consists in modifying only the $k$-th row of $\bDelta$, 
	for example by setting all its negative entries to a small constant $\epsilon \geq 0$, and updating 
	the $k$-th column so as to preserve symmetry. In more refined arguments, we first permute the 
	negative entries within each row.
	
	\paragraph{Second step: total variation.}
	We then exploit the properties of a given algorithm $A$ in order to exhibit, for each $k \neq i^*$, a separating event $B_k$ on which the two laws assign very different probabilities, 
	\[
	\mathbb{P}_{\bDelta,A}(B_k) \text{ is large (typically } \geq 1-\delta\text{),}
	\qquad
	\mathbb{P}_{\bDelta^{(k)},A}(B_k) \text{ is small (typically } \leq \delta\text{).}
	\]
	Here $\mathbb{P}_{\bDelta,A}$ denotes the law of all observations (and internal randomness) when 
	algorithm $A$ interacts with environment $\bDelta$. A natural choice, when $A$ is 
	$\delta$-correct for CW identification, is the event $\{\hat{i} = k\}$, where $\hat{i}$ is the 
	recommendation output by $A$. For quantile (high-probability) lower bounds, we additionally introduce 
	the $(1-\delta)$-quantile $\chi$ of the budget $N_\delta$ under $\mathbb{P}_{\bDelta,A}$, and we 
	consider events such as $\{N_\delta \leq \chi\}$ or their intersections with identification events. 
	The precise choice of $B_k$ varies from theorem to theorem, but the goal is always to produce a 
	set on which the two laws $\mathbb{P}_{\bDelta,A}$ and $\mathbb{P}_{\bDelta^{(k)},A}$ have very 
	different probabilities, thereby enforcing a large total-variation distance:
	\[
	\mathrm{TV}\!\big(\mathbb{P}_{\bDelta,A}, \mathbb{P}_{\bDelta^{(k)},A}\big)
	\;\geq\; 
	\big| \mathbb{P}_{\bDelta,A}(B_k) - \mathbb{P}_{\bDelta^{(k)},A}(B_k) \big|.
	\]
	
	The total-variation distance is then controlled from above through a standard data-processing 
	inequality: we use either Pinsker's inequality, the Bretagnolle--Huber inequality, or a 
	Fano-type inequality (see Lemma~\ref{lem:fano-type}) to relate $\mathrm{TV}$ to the 
	Kullback--Leibler divergence. For example,
	\[
	\mathrm{TV}(P,Q) 
	\;\leq\; \sqrt{\tfrac{1}{2}\,\mathrm{KL}(P,Q)}, 
	\qquad\text{or}\qquad
	1 - \mathrm{TV}(P,Q) \;\geq\; \tfrac{1}{2}\exp\big(-\mathrm{KL}(P,Q)\big).
	\]
	The choice depends on the error-probability regime: Bretagnolle--Huber is convenient in the very 
	small-$\delta$ regime, while Pinsker is often sharper in moderate-error regimes.
	
	\paragraph{Third step: decomposition and control of the KL divergence.}
	The last step is to decompose the Kullback--Leibler divergence between the laws induced by $A$ 
	under the two environments. Let $N_{i,j}$ denote the total number of observed duels of the ordered 
	pair $(i,j)$ between time $1$ and the stopping time $N_\delta$, and let 
	$N_{\{i,j\}} = N_{i,j} + N_{j,i}$ be the total number of observations of the unordered pair 
	$\{i,j\}$. A standard KL-decomposition for adaptive bandit algorithms (see, e.g., 
	Lattimore and Szepesvári, 2020, Lemma~15.1) yields
	\[
	\mathrm{KL}\big(\mathbb{P}_{\bDelta,A}, \mathbb{P}_{\bDelta^{(k)},A}\big)
	\;=\; 
	\sum_{1 \leq i < j \leq K} 
	\mathbb{E}_{\bDelta,A}\big[N_{\{i,j\}}\big] \,
	\mathrm{KL}\big( \nu_{\Delta_{i,j}}, \nu_{\Delta^{(k)}_{i,j}} \big),
	\]
	where $\nu_{\Delta_{i,j}}$ denotes the Bernoulli distribution with parameter $1/2 + \Delta_{i,j}$. 
	In all our constructions, $\bDelta$ and $\bDelta^{(k)}$ differ only on a small set of pairs 
	(typically those involving arm $k$), so the sum reduces to those indices. The Bernoulli 
	KL-divergence admits the classical upper bound, for $p,q \in (0,1)$,
	\[
	\mathrm{kl}(p,q) 
	\;\leq\; \frac{(p-q)^2}{q(1-q)},
	\]
	which, in our setting, leads to bounds of the form
	\[
	\mathrm{KL}\big(\mathbb{P}_{\bDelta,A}, \mathbb{P}_{\bDelta^{(k)},A}\big)
	\;\lesssim\; 
	\sum_{\{i,j\} \,\text{modified}} 
	\mathbb{E}_{\bDelta,A}[N_{\{i,j\}}] \,\big(\Delta_{i,j} - \Delta^{(k)}_{i,j}\big)^2.
	\]
	
	Combining this upper bound with the lower bound on total variation from the third step yields a 
	constraint that any $\delta$-correct algorithm $A$ must satisfy. Rearranging these inequalities 
	then produces a lower bound on the sample complexity, typically expressed in terms of the expected 
	budget $\mathbb{E}_{\bDelta,A}[N_\delta]$ or on the $(1-\delta)$-quantile of $N_\delta$, and 
	involving the instance-dependent hardness parameters (such as $\Delta_{i,(1)}$ or 
	$\|\Delta^-_i\|_2^2/K_{i;<0}$).

	\subsection{Proof of Proposition~\ref{thm:LB_exp_instance_first_regime}}\label{sec:proof_LB_exp_fc}
	
	We prove separately the bound in expectation (paragraph~\ref{sec:proof_expectation_Hkarnin}) and the bound on the $(1-\delta)$-quantile of $N_\delta$ (paragraph~\ref{sec:proof_HP_karnin}). The two arguments share the same change-of-measure structure (see Section~\ref{sec:roadmap_LB}) and differ only in the way we exploit either the identification rule or the stopping rule to construct a separating event. The expectation bound~\ref{eq:LB_exp_karnin} already appears as Theorem 5.2 in \cite{haddenhorst2021identification}, but we provide here a proof for completeness. 
	
	\subsubsection{Bound in expectation from Proposition~\ref{thm:LB_exp_instance_first_regime}}\label{sec:proof_expectation_Hkarnin}

	Fix $K \geqslant 2$ and $\delta \in (0,1)$. Consider an algorithm $A$ that is $\delta$-correct over the entire class $\mathbb{D}_{\mathrm{cw}}$. We fix any matrix $\bDelta \in \mathbb{D}_{\mathrm{cw}}$ such that $i^*=1$ and prove the following expected lower bound that holds for this specific instance:
	\begin{equation}\label{eq:LB_exp_karnin}
		\mathbb{E}_{\bDelta,A}[N_{\delta}]
		\;\geqslant\; \frac{1}{4} \sum_{i\neq i^*} \frac{\log\!\left(\frac{1}{4\delta}\right)}{\Delta_{i,(1)}^2}~.
	\end{equation}
	
	\paragraph{Sketch of proof.}
	We follow the three-step roadmap of Section~\ref{sec:roadmap_LB}. 
	\emph{(i) Reference and alternative instances.} Fix any $\bDelta\in\mathbb{D}_{\mathrm{cw}}$ with CW $i^*=1$; this is our reference instance. For each suboptimal arm $k\neq 1$, we construct an alternative 
	instance $\bDelta^{(k)}$ by lifting all non-positive entries in row $k$ to a small constant 
	$\epsilon>0$ (and adjusting the $k$-th column to preserve symmetry), so that $k$ becomes the 
	CW in $\bDelta^{(k)}$. 
	\emph{(ii) Separating event and total variation.} Since the algorithm is $\delta$-correct for CW 
	identification, it must distinguish $\bDelta$ from each $\bDelta^{(k)}$ with error at most 
	$\delta$ when deciding between $i^*$ and $k$ as CW. Using the test event $B_k=\{\hat i = k\}$, we 
	obtain that the total-variation distance between $\mathbb{P}_{\bDelta,A}$ and $\mathbb{P}_{\bDelta^{(k)},A}$ is at 
	least $1-2\delta$, which, via the Bretagnolle--Huber inequality, yields a lower bound on 
	$\mathrm{KL}(\mathbb{P}_{\bDelta,A},\mathbb{P}_{\bDelta^{(k)},A})$. 
	\emph{(iii) KL decomposition.} We then decompose this KL along unordered pairs. The two instances 
	$\bDelta$ and $\bDelta^{(k)}$ differ only on duels involving arm $k$, and for each such pair 
	the Bernoulli parameters differ by at most a constant of order $|\Delta_{k,(1)}|$. A Bernoulli 
	KL upper bound, combined with the decomposition, forces 
	$\mathbb{E}_{\bDelta,A}[N_k]\gtrsim \Delta_{k,(1)}^{-2}\log(1/\delta)$, where $N_k$ counts 
	duels involving $k$. Summing this constraint over all $k\neq 1$ yields the desired lower 
	bound~\eqref{eq:LB_exp_karnin}.
	\medskip 
	
	\begin{proof}
		\noindent\underline{Step 1: reference and perturbed instances.}
		
		The argument is fully instance-dependent: we fix an arbitrary gap matrix 
		$\bDelta \in \mathbb{D}_{\mathrm{cw}}$ with Condorcet winner $i^*=1$, and work throughout with 
		this specific instance. We denote $\mathbb{P}$ the probability induced by the interaction between $\boldsymbol{\Delta}$ and $A$.
		
		Let $\epsilon>0$ be a constant, arbitrary small. 
		Let $k \neq 1$ be an arm that is not the CW under $\bDelta$. A simple way to modify $\bDelta$ so that $k$ becomes the CW, is to make all non-positive entries in the $k$-th row of $\bDelta$ equal to $\epsilon$. 
		
		Construct the gap matrix $\bDelta^{(k)}$ as follows. For all $i,j \notin \{1,k\}$, set $\Delta_{i,j}^{(k)} = \Delta_{i,j}$.
		Set $\Delta^{(k)}_{k,1} = \epsilon$ and $\Delta^{(k)}_{1,k} = -\epsilon$.
		Finally, for each $j \notin \{1,k\}$, define
		\begin{equation}\label{def:Delta_k_epsilon}
			\Delta^{(k)}_{k,j} =
			\begin{cases}
				\Delta_{k,j}, & \text{if } \Delta_{k,j} > 0,\\[2pt]
				\epsilon, & \text{if } \Delta_{k,j} \leqslant 0,
			\end{cases}
			\qquad
			\Delta^{(k)}_{j,k} = - \Delta^{(k)}_{k,j}.
		\end{equation}
		
		For $\epsilon$ small enough, the modified matrix $\bDelta^{(k)}$ can be represented as
		\[
        \Delta^{(k)} =
        \begin{pmatrix}
        	0 & \Delta_{1,2} & \cdots & \Delta_{1,k-1} & \textcolor{blue}{-\epsilon} & \Delta_{1,k+1} & \cdots & \Delta_{1,K} \\[2mm]
        	\Delta_{2,1} & 0 & \cdots & \Delta_{2,k-1} & \textcolor{blue}{-(\epsilon \vee \Delta_{k,2})} & \Delta_{2,k+1} & \cdots & \Delta_{2,K} \\[2mm]
        	\vdots & \vdots & \ddots & \vdots & \textcolor{blue}{\vdots} & \vdots & & \vdots \\[2mm]
        	\Delta_{k-1,1} & \Delta_{k-1,2} & \cdots & 0 & \textcolor{blue}{-(\epsilon \vee \Delta_{k,k-1})} & \Delta_{k-1,k+1} & \cdots & \Delta_{k-1,K} \\[2mm]
        	\textcolor{blue}{\epsilon} &
        	\textcolor{blue}{\epsilon \vee \Delta_{k,2}} &
        	\textcolor{blue}{\cdots} &
        	\textcolor{blue}{\epsilon \vee \Delta_{k,k-1}} &
        	0 &
        	\textcolor{blue}{\epsilon \vee \Delta_{k,k+1}} &
        	\textcolor{blue}{\cdots} &
        	\textcolor{blue}{\epsilon \vee \Delta_{k,K}} \\[2mm]
        	\Delta_{k+1,1} & \Delta_{k+1,2} & \cdots & \Delta_{k+1,k-1} & \textcolor{blue}{-(\epsilon \vee \Delta_{k,k+1})} & 0 & \cdots & \Delta_{k+1,K} \\[2mm]
        	\vdots & \vdots & & \vdots & \textcolor{blue}{\vdots} & \vdots & \ddots & \vdots \\[2mm]
        	\Delta_{K,1} & \Delta_{K,2} & \cdots & \Delta_{K,k-1} & \textcolor{blue}{-(\epsilon \vee \Delta_{k,K})} & \Delta_{K,k+1} & \cdots & 0
        \end{pmatrix}.
        \]
		where the blue entries indicate the differences with respect to the reference $\bDelta$. In fact, only non-positive entries of row $k$  are modified between $\bDelta$ and $\bDelta^{(k)}$.
		
		Since $\epsilon> 0$, the $k$-th row $\Delta^{(k)}_{k,\cdot}$ is positive (aside from $\Delta^{(k)}_{k,k}$). Hence, the CW of $\bDelta^{(k)}$ is $k$. As standard in this line of work, our construction is motivated by the fact that the instance $\bDelta^{(k)}$ is hard to distinguish from the reference gap matrix $\bDelta$.
		
		\medskip
		
		For $k \ge 2$, denote by $\mathbb{P}^{(k)}$ the distribution of the data when the underlying gap matrix is $\bDelta^{(k)}$. Formally, when the algorithm queries a pair $(i,j)$ with $i<j$, it receives a sample $X_{i,j} \sim \mathcal{B}\bigl(\Delta_{i,j}^{(k)} + \tfrac{1}{2}\bigr)$, where $\mathcal{B}(p)$ denotes the Bernoulli distribution with parameter $p$.   
		
		\medskip 
		
		\noindent \underline{Step 2: information-theoretic arguments.}
		
		\noindent Let $\hat{i}$ denote the output of algorithm $A$, which is assumed to be $\delta$-correct over $\mathbb{D}_{\mathrm{cw}}$. When the gap matrix is $\bDelta^{(k)}$ the CW is $k$, so $\delta$-correctness implies
		\[
		\forall k\ne i^*,\;\mathbb{P}^{(k)}(\hat{i} = k) \ge 1-\delta
		\quad\text{and}\quad
		\mathbb{P}(\hat{i} = k) \leqslant \delta~.
		\]
		By the definition of the total variation distance,
		\begin{align}
			\mathrm{TV}\bigl(\mathbb{P}, \mathbb{P}^{(k)}\bigr)
			&\ge \bigl|\mathbb{P}(\hat{i} = k) - \mathbb{P}^{(k)}(\hat{i} = k)\bigr| \ge 1 - 2\delta~. \label{eq:tv1}
		\end{align}
		From \eqref{eq:tv1}, the Bretagnolle–Huber inequality (Theorem 14.2 in \citealp{lattimore2020bandit}) yields
		\begin{equation}\label{eq:BH}
			1 - 2\delta
			\leqslant \mathrm{TV}\bigl(\mathbb{P}, \mathbb{P}^{(k)}\bigr)
			\leqslant 1 - \frac{1}{2}
			\exp\!\left\{-\mathrm{KL}\bigl(\mathbb{P}, \mathbb{P}^{(k)}\bigr)\right\}.
		\end{equation}
		
		\noindent\underline{Step 3: computing the KL divergence and concluding on the budget.}
		
		For any unordered pair $\{i,j\}$ with $i \ne j$, denote by $N_{\{i,j\}}$ the number of duels involving either $(i,j)$ or $(j,i)$, that is,
		\[
		N_{\{i,j\}} \coloneqq N_{i,j} + N_{j,i},
		\]
		where
		\begin{equation*}
			N_{i,j} \coloneqq |\bigl\{t \in [N_{\delta}] : (I_t,J_t) = (i,j)\bigr\}|,
		\end{equation*}
		and $N_{\delta}$ is the stopping time of the algorithm.
		
		Using the divergence decomposition lemma (Lemma 15.1 in \citealp{lattimore2020bandit}) and the fact that the two instances $\bDelta$ and $\bDelta^{(k)}$ differ only on pairs involving arm $k$, we obtain
		\begin{align}
			\mathrm{KL}\bigl(\mathbb{P}, \mathbb{P}^{(k)}\bigr)
			&= \sum_{1 \leqslant i < j \leqslant K}
			\mathbb{E}\!\left[N_{\{i,j\}}\right]
			\,\mathrm{KL}\bigl(\mathbb{P}_{i,j}, \mathbb{P}_{i,j}^{(k)}\bigr)
			\nonumber \\
			&= \sum_{\substack{i=1 \\ i \neq k}}^{K}
			\mathbb{E}\!\left[N_{\{k,i\}}\right]
			\,\mathrm{KL}\bigl(\mathbb{P}_{k,i},\mathbb{P}_{k,i}^{(k)}\bigr).
			\label{eq:tv3}
		\end{align}
		
		We now upper bound $\mathrm{KL}\bigl(\mathbb{P}_{k,i}, \mathbb{P}_{k,i}^{(k)}\bigr)$. Let $i\in[K]$ with $i\ne k$. If $\Delta_{k,i} > 0$, then by construction $\mathbb{P}_{k,i} = \mathbb{P}_{k,i}^{(k)}$. Otherwise, if $\Delta_{k,i} < 0$, the corresponding Bernoulli feedback distributions satisfy
		\[
		\mathbb{P}_{k,i} = \mathcal{B}\!\left(\tfrac{1}{2} + \Delta_{k,i}\right),
		\qquad
		\mathbb{P}_{k,i}^{(k)} = \mathcal{B}\!\left(\tfrac{1}{2}+\epsilon\right),
		\]
		so that
		\begin{align}
			\mathrm{KL}\bigl(\mathbb{P}_{k,i}, \mathbb{P}_{k,i}^{(k)}\bigr)
			&= \mathrm{kl}\!\left(\Delta_{k,i} + \tfrac12,\; \epsilon + \tfrac12\right)  \nonumber \\
			&\leqslant \frac{\bigl(\Delta_{k,i} + \tfrac12 - (\epsilon + \tfrac12)\bigr)^2}
			{\bigl(-\epsilon + \tfrac12\bigr)\bigl(\epsilon + \tfrac12\bigr)}
			\label{eq:kl1}\\
			&\leqslant \frac{(\epsilon-\Delta_{k,(1)})^2}{\bigl(-\epsilon + \tfrac12\bigr)\bigl(\epsilon + \tfrac12\bigr)} \label{eq:kl2}~.
		\end{align}
		Here, \eqref{eq:kl1} follows from the standard upper bound
		\(\mathrm{kl}(p,q) \leqslant (p-q)^2/[q(1-q)]\) for \(p,q \in (0,1)\), while \eqref{eq:kl2} uses that
		\(\Delta_{k,i}^2 \leqslant \Delta_{k,(1)}^2\) by definition of \(\Delta_{k,(1)}\) as the smallest negative entry in row $k$.
		Similarly, if $\Delta_{k,j}=0$, 
		\[
		\mathbb{P}_{k,j} = \mathcal{B}\!\left(\tfrac{1}{2}\right),
		\qquad
		\mathbb{P}_{k,j}^{(k)} = \mathcal{B}\!\left(\tfrac{1}{2} + \epsilon\right),
		\]
		so that the same computation gives
		\[\mathrm{KL}\bigl(\mathbb{P}_{k,j}, \mathbb{P}_{k,j}^{(k)}\bigr) \leqslant \frac{\epsilon^2}{\bigl(-\epsilon + \tfrac12\bigr)\bigl(\epsilon + \tfrac12\bigr)} \enspace, \]
		which vanishes to $0$ with $\epsilon\to 0$. 
		
		Combining these bounds with \eqref{eq:tv3}, and taking the limit $\epsilon\to 0$, we obtain
		\[
		\mathrm{KL}\bigl(\mathbb{P}, \mathbb{P}^{(k)}\bigr)
		\leqslant 4 \left(
		\sum_{\substack{i=1 \\ i \ne k}}^{K}
		\mathds{1}_{\{\Delta_{k,i} < 0\}}
		\,\mathbb{E}\!\left[N_{\{k,i\}}\right]
		\right) \Delta_{k,(1)}^2.
		\]
		Using the Bretagnolle–Huber inequality \eqref{eq:BH} from Step~3 , we then get
		\[
		\sum_{\substack{i=1 \\ i \ne k}}^{K}
		\mathds{1}_{\{\Delta_{k,i} < 0\}}
		\,\mathbb{E}\!\left[N_{\{k,i\}}\right]
		\ge \frac{1}{4\,\Delta^2_{k,(1)}} \log\!\frac{1}{4\delta}~.
		\]
		Summing over $k \ge 2$, we conclude that the total number of queries $N_{\delta}$ satisfies
		\begin{align*}
			\mathbb{E} \left[ N_{\delta}\right]
			&\ge \sum_{k \ne 1}^K \left(
			\sum_{\substack{i=1 \\ i \ne k}}^{K}
			\mathds{1}_{\{\Delta_{k,i} < 0\}}
			\,\mathbb{E}\!\left[N_{\{k,i\}}\right]
			\right) \\
			&\ge \frac{1}{4}\sum_{k = 2}^K \frac{1}{\Delta_{k,(1)}^2}
			\log\!\frac{1}{4\delta}~,
		\end{align*}
		which is exactly the claimed lower bound~\eqref{eq:LB_exp_karnin}.
	\end{proof}
	
	\subsubsection{Bound in quantile in  Proposition~\ref{thm:LB_exp_instance_first_regime}}\label{sec:proof_HP_karnin}
	
	In this paragraph, we prove the quantile bound from Proposition~\ref{thm:LB_exp_instance_first_regime}, namely
	\begin{equation}\label{eq:LB_hp_karnin}
		\mathbb{P}_{\boldsymbol{\Delta},A}\left(N_{\delta}\geqslant\frac{1}{3}\sum_{i\neq i^*}\frac{\left(\frac{1}{6\delta}\right)}{\Delta_{i,(1)}^2}\right)\geqslant\delta~.
	\end{equation}

	\paragraph{Sketch of proof.}
	Even though expectation lower bounds are naturally weaker than quantile bounds, deducing a 
	quantile lower bound from its expectation counterpart is nontrivial. Still, the arguments largely 
	mirror the expectation proof (paragraph~\ref{sec:proof_expectation_Hkarnin}), we follow the same 
	three-step roadmap. \emph{(i) Reference instance and alternative 
		instances}: we fix $\bDelta\in\mathbb{D}_{\mathrm{cw}}$ with CW $i^*$ and, for each $k\neq i^*$, 
	construct $\bDelta^{(k)}$ by zeroing the negative entries in row/column $k$. In particular, 
	$\bDelta^{(k)}$ has two nonnegative rows (those of $1$ and $k$) and therefore does not belong to 
	$\mathbb{D}_{\mathrm{cw}}$. \emph{(ii) Separating event and total variation}: we now use the stopping 
	rule instead of the recommendation and consider the event $B=\{N_\delta > Q\}$, where $Q$ is the 
	$(1-\delta)$-quantile of $N_\delta$ under $\bDelta$. This event has small probability under 
	$\mathbb{P}$, but (after a continuity argument using perturbations of $\bDelta^{(k)}$) it 
	has large probability under $\mathbb{P}^{(k)}$, since $A$ cannot quickly decide between 
	$\hat{i}=1$ and $\hat{i}=k$. \emph{(iii) KL decomposition}: as before, we decompose the KL along 
	pairs and use that $\bDelta$ and $\bDelta^{(k)}$ differ only on duels involving $k$. The only extra 
	ingredient is that, since $B$ depends only on the first $Q$ observations, we introduce a truncated 
	algorithm $\tilde A$ that stops at time $Q$, apply the KL decomposition to $\tilde A$, and then 
	reinterpret the resulting inequality as a lower bound on $Q$, i.e., on the $(1-\delta)$-quantile of 
	$N_\delta$. 
	
	\medskip
\begin{proof}
	Let $A$ be any $\delta$-correct algorithm on $\mathbb{D}_{\mathrm{cw}}$.
	
	\noindent\underline{Step 1: Reference and alternative instances.}
	
	As for the expectation bound, we fix any $\bDelta$ as the reference instance, with $i^*=1$. Denote $\mathbb{P}_A$ for the probability induced by the interaction between $A$ and $\bDelta$. Then, for each suboptimal arm $k\in\{2,\dots,K\}$, construct $\boldsymbol{\Delta}^{(k)}$ by setting to zero all entries $(k,j)$ with $\Delta_{k,j}<0$, that is, $\boldsymbol{\Delta}^{(k)}$ is defined by Equation~\eqref{def:Delta_k_epsilon} with $\epsilon=0$.
	
	The matrices $\boldsymbol{\Delta}$ and $\boldsymbol{\Delta}^{(k)}$ differ only in row/column $k$. By construction, rows $1$ and $k$ of $\boldsymbol{\Delta}^{(k)}$ contain only nonnegative entries, so $\boldsymbol{\Delta}^{(k)}$ contains two weak CW, and $1$ and $k$ are tied. In particular, $\boldsymbol{\Delta}^{(k)}\notin\mathbb{D}_{\mathrm{cw}}$.
	Denote $\mathbb{P}_A^{(k)}$ the distribution induced by $A$ interacting with $\boldsymbol{\Delta}^{(k)}$.
	
	\medskip
	
	\noindent\underline{Step 2: bound in total variation and reduction to fixed budget.}
	
	Consider the recommendation rule $\hat{i}$ and the budget $N_\delta$ of $A$. Define $Q$ as the $(1-\delta)$-quantile of the budget under $\mathbb{P}_A$:
	\begin{equation}\label{def:chi_karnin}
		Q=\inf\{x>0\text{ s.t. }\mathbb{P}_A(N_{\delta}\geqslant x)\leqslant\delta\}.
	\end{equation}
	
	Consider the event $B:=\{N_\delta\leqslant Q\}$ where the budget is at most $Q$.
	
	We define a truncated version $\tilde{A}$ of $A$ with budget at most $Q$ as follows: run $A$ until it stops or $Q$ queries have been made; if $A$ stops by time $Q$, return $\tilde{i}=\hat{i}$; otherwise, return $\tilde{i}=0$. Let $\tilde{N}_\delta$, $\tilde{i}$ be $\tilde{A}$'s budget/recommendation, $\mathbb{P}_{\tilde{A}}$ (resp. $\mathbb{P}_{\tilde{A}}^{(k)}$) its law under $\boldsymbol{\Delta}$ (resp. $\boldsymbol{\Delta}^{(k)}$). By construction,
	\[ B=\{\tilde{i}\ne 0\}~, \]
	so this event is measurable with respect to the observations of algorithm $\tilde{A}$. Moreover, it has the same probability under $A$ and $\tilde{A}$.
	
	\noindent We now lower bound the total variation distance between $\mathbb{P}_{\tilde{A}}$ and $\mathbb{P}_{\tilde{A}}^{(k)}$.
	
	By the definition of $Q$ in~\eqref{def:chi_karnin}, we have
	\begin{equation}\label{eq:H2_tv1_karnin}
		\mathbb{P}_{\tilde{A}}(\tilde{i}=0)=\mathbb{P}_{A}(N_{\delta}>Q)\;\leqslant\;\delta.
	\end{equation}
	
	Under $\bDelta^{(k)}$, the instance does not belong to $\mathbb{D}_{\mathrm{cw}}$, so we cannot
	directly invoke $\delta$-correctness. We therefore approximate $\bDelta^{(k)}$ by nearby instances in
	$\mathbb{D}_{\mathrm{cw}}$. More precisely, define $\bDelta^{(k,\epsilon)}$ as~\eqref{def:Delta_k_epsilon}, i.e., by lifting all off-diagonal zero entries in row $k$ of $\bDelta^{(k)}$ to
	$\epsilon>0$ (and adjusting the $k$-th column to preserve antisymmetry). For $0<\epsilon\leqslant 1/4$, the matrix $\bDelta^{(k,\epsilon)}$ lies in $\mathbb{D}_{\mathrm{cw}}$ and admits $k$ as its CW. Similarly, define $\bDelta^{(k,-\epsilon)}$ by subtracting $\epsilon$ from all off-diagonal zero entries in the $k$-th row of $\bDelta^{(k)}$ (and adjusting the $k$-th column to preserve antisymmetry), so that $\bDelta^{(k,-\epsilon)}\in\mathbb{D}_{\mathrm{cw}}$ and admits $1$ as its CW.
	
	Let $\mathbb{P}^{(k,\epsilon)}_A$ and $\mathbb{P}^{(k,-\epsilon)}_A$ denote the laws of $A$ under
	$\bDelta^{(k,\epsilon)}$ and $\bDelta^{(k,-\epsilon)}$, respectively. Since $A$ is $\delta$-correct
	on $\mathbb{D}_{\mathrm{cw}}$, we have
	\begin{equation}\label{eq:approximate_tv_karnin}
		\mathbb{P}^{(k,\epsilon)}_A(\hat{i}\neq k)\leqslant\delta,\qquad \mathbb{P}^{(k,-\epsilon)}_A(\hat{i}\neq 1)\leqslant\delta~.
	\end{equation}
	
	Moreover, for either sign, the law of the transcript of $\tilde{A}$ under $\bDelta^{(k,\pm\epsilon)}$ converges in total variation to $\mathbb{P}^{(k)}_{\tilde{A}}$ as $\epsilon\to0$, since $\tilde{A}$ uses at most $Q$ observations. In particular, this convergence applies to the events $\{\hat{i}=1,N_\delta\leqslant Q\}$ and $\{\hat{i}\neq1,N_\delta\leqslant Q\}$, which are measurable with respect to this truncated transcript.
	Using these facts and letting $\epsilon\to0$, we obtain
	\begin{align}
		\mathbb{P}_{A}^{(k)}(N_{\delta}\leqslant Q)
		&=\mathbb{P}_{A}^{(k)}(\hat{i}=1,N_{\delta}\leqslant Q)+\mathbb{P}_{A}^{(k)}(\hat{i}\neq1,N_{\delta}\leqslant Q) \nonumber\\
		&=\lim_{\epsilon\to0}\Big[\mathbb{P}_{A}^{(k,\epsilon)}(\hat{i}=1,N_{\delta}\leqslant Q)+\mathbb{P}_{A}^{(k,-\epsilon)}(\hat{i}\neq1,N_{\delta}\leqslant Q)\Big] \nonumber\\
		&\leqslant 2\delta \label{eq:H2_tv2_karnin}\enspace,
	\end{align}
	where the last inequality follows from~\eqref{eq:approximate_tv_karnin}.
	
	Since $\tilde{A}$ and $A$ coincide up to time $Q$, we also have
	$\mathbb{P}^{(k)}_{\tilde{A}}(\tilde{i}\neq0)=\mathbb{P}^{(k)}_{A}(N_{\delta}\leqslant Q)$, so
	\eqref{eq:H2_tv2_karnin} implies
	\[ \mathbb{P}^{(k)}_{\tilde{A}}(\tilde{i}\neq0)\leqslant2\delta. \]
	
	Combining \eqref{eq:H2_tv1_karnin} with this inequality, and using the event $\{\tilde{i}=0\}$, we obtain
	\begin{equation}\label{eq:H2_tv_karnin}
		\mathrm{TV}\bigl(\mathbb{P}_{\tilde{A}},\mathbb{P}_{\tilde{A}}^{(k)}\bigr)\geqslant\mathbb{P}^{(k)}_{\tilde{A}}(\tilde{i}=0)-\mathbb{P}_{\tilde{A}}(\tilde{i}=0)\geqslant1-3\delta\enspace.
	\end{equation}
	
	\begin{remark}
		Intuitively, one may think of $\tilde{A}$ as a fixed-budget algorithm with budget $Q$. This can be
		viewed as a reduction: from any $\delta$-correct algorithm $A$ that enjoys a high-probability
		control on its budget (namely, $\mathbb{P}_A(N_\delta\leqslant Q)\geqslant1-\delta$), we construct a
		fixed-budget algorithm $\tilde{A}$ with budget $Q$ that inherits the same distinguishing power
		between the reference instance and its perturbations.
	\end{remark}
	
	\noindent\underline{Step 3: computing the KL divergence.}
	
	By the Bretagnolle–Huber inequality (see, e.g., \citealp{lattimore2020bandit}), we have
	\begin{equation*}
		1-3\delta\leqslant\mathrm{TV}\bigl(\mathbb{P}_{\tilde{A}},\mathbb{P}^{(k)}_{\tilde{A}}\bigr)\leqslant1-\frac{1}{2}\exp\Bigl\{-\mathrm{KL}\bigl(\mathbb{P}_{\tilde{A}},\mathbb{P}^{(k)}_{\tilde{A}}\bigr)\Bigr\}\enspace.
	\end{equation*}
	In particular,
	\begin{equation}\label{eq:H2_BH_karnin}
		\mathrm{KL}\bigl(\mathbb{P}_{\tilde{A}},\mathbb{P}^{(k)}_{\tilde{A}}\bigr)\;\geqslant\;\log\!\left(\frac{1}{6\delta}\right).
	\end{equation}
	
	We now decompose $\mathrm{KL}\bigl(\mathbb{P}_{\tilde{A}},\mathbb{P}^{(k)}_{\tilde{A}}\bigr)$.
	For $i\neq j$ in $[K]$, let $N_{i,j}$ denote the number of duels $(i,j)$ made by $\tilde{A}$, while $N_{\{i,j\}}=N_{i,j}+N_{j,i}$ is the number of duels with unordered pair $\{i,j\}$. By the standard KL decomposition for adaptive procedures~\citep{lattimore2020bandit},
	\begin{align}
		\mathrm{KL}\bigl(\mathbb{P}_{\tilde{A}},\mathbb{P}^{(k)}_{\tilde{A}}\bigr)
		&=\sum_{i\neq j}\mathbb{E}_{\tilde{A}}\bigl[N_{i,j}\bigr]\,\mathrm{KL}\bigl(\mathbb{P}_{i,j},\mathbb{P}_{i,j}^{(k)}\bigr) \nonumber\\
		&=\sum_{j:\Delta_{k,j}<0}\mathbb{E}_{\tilde{A}}\bigl[N_{\{k,j\}}\bigr]\,\mathrm{KL}\bigl(\mathbb{P}_{k,j},\mathbb{P}^{(k)}_{k,j}\bigr)\label{eq:H2_KL1_karnin}\enspace,
	\end{align}
	since $\bDelta$ and $\bDelta^{(k)}$ differ only on duels $\{k,j\}$ with $\Delta_{k,j}<0$.
	
	For $j$ with $\Delta_{k,j}<0$, $\mathbb{P}_{k,j}^{(k)}=\mathcal{B}(1/2)$, $\mathbb{P}_{k,j}=\mathcal{B}(1/2+\Delta_{k,j})$ with $\Delta_{k,j}\in[-1/4,0]$. Thus,
	\begin{align}
		\mathrm{KL}\bigl(\mathbb{P}_{k,j},\mathbb{P}_{k,j}^{(k)}\bigr)
		&\leqslant\frac{1}{2}\log\!\left(\frac{\tfrac{1}{2}}{\tfrac{1}{2}+\Delta_{k,j}}\right)+\frac{1}{2}\log\!\left(\frac{\tfrac{1}{2}}{\tfrac{1}{2}-\Delta_{k,j}}\right) \nonumber\\
		&=-\frac{1}{2}\log\!\bigl(1-4\Delta_{k,j}^2\bigr)\;\leqslant\;8\log\!\left(\frac{4}{3}\right)\Delta_{k,j}^2\label{eq:bound_kl}~,
	\end{align}
	where \eqref{eq:bound_kl} follows from $\sup_{x\in[0,1/4]}\frac{-\log(1-x)}{x}\leqslant4\log\!\left(\frac{4}{3}\right)$,
	applied with $x=4\Delta_{k,j}^2\in[0,1/4]$.
	
	Plugging these bounds into \eqref{eq:H2_KL1_karnin}, we obtain
	\begin{align}
		\mathrm{KL}\bigl(\mathbb{P}_{\tilde{A}},\mathbb{P}_{\tilde{A}}^{(k)}\bigr)
		&\leqslant8\log\!\left(\frac{4}{3}\right)\sum_{j:\Delta_{k,j}<0}\mathbb{E}_{\tilde{A}}[N_{\{k,j\}}]\,\Delta_{k,j}^2 \nonumber\\
		&\leqslant8\log\!\left(\frac{4}{3}\right)\left(\sum_{j:\Delta_{k,j}<0}\mathbb{E}_{\tilde{A}}[N_{\{k,j\}}]\right)\Delta_{k,(1)}^2~,\label{eq:KL_karnin}
	\end{align}
	since $\Delta_{k,j}^2\leqslant\Delta^2_{k,(1)}$ whenever $\Delta_{k,j}<0$.
	Combining~\eqref{eq:H2_BH_karnin}--\eqref{eq:KL_karnin}, we obtain
	\begin{align}
		\frac{1}{8\log(4/3)}\frac{1}{\Delta^2_{k,(1)}}\log\left(\frac{1}{6\delta}\right)
		&\leqslant\sum_{j:\Delta_{k,j}<0}\mathbb{E}_{\tilde{A}}[N_{\{k,j\}}]~.\label{eq:final_karnin_1}
	\end{align}
	
	Since $\tilde{A}$ uses at most $Q$ duels and each duel is counted at most once, we have almost surely, under $\mathbb{P}_{\tilde{A}}$,
	\[ \sum_{k\ne i^*}\sum_{j:\Delta_{k,j}<0}N_{\{k,j\}}\;\leqslant\;\tilde{N}_{\delta}\leqslant Q~, \]
	in particular, the same bound also holds in expectation $\mathbb{E}_{\tilde{A}}$.
	
	Summing in~\eqref{eq:final_karnin_1} over $k\ne i^*$ yields
	\[ \frac{1}{8\log(4/3)}\sum_{k\ne i^*}\frac{1}{\Delta^2_{k,(1)}}\log\left(\frac{1}{6\delta}\right)\leqslant\sum_{k\ne i^*}\sum_{j:\Delta_{k,j}<0}\mathbb{E}_{\tilde{A}}[N_{\{k,j\}}]\leqslant Q~. \]
	
	Using the numerical bound $8\log(4/3)<3$ and the definition of $Q$ as the $(1-\delta)$-quantile of
	$N_\delta$ then gives the high-probability lower bound~\eqref{eq:LB_hp_karnin}.
	
\end{proof}
	
	\subsection{Proof of Theorem~\ref{thm:LB_hp_instance}}\label{sec:proof_LB_HP_fc}
	
	In this subsection, we prove the high-probability lower bound from Theorem~\ref{thm:LB_hp_instance}. We start with a sketch of proof.

	\paragraph{Proof Sketch of~\eqref{eq:lower_bound_Hexplore}.}
	The key idea is to reduce CW identification to multiple active signal detection problems~\citep{castro2014adaptive}: for each $i\neq i^*$, we need to certify that the row $\Delta_{k,\cdot}$ has at least a negative entry, this  with a probability of error at most $\delta$. Along the way, we have to improve state-of-the-art lower bounds  for such detection problems. 
	
	Consider any $\delta$-correct algorithm $A$. We introduce a collection $\boldsymbol{\Delta}^{(\pi)}$ of gap matrices that differ from $\bDelta$ as we permute, on each row $i\neq i^*$, the position of the negative entries by some collection $\pi=(\pi_i)_{i\neq i^*}$ of permutations while preserving the skew symmetry. Then, we fix a specific arm $i\neq i^*$  and construct the gap matrices $\boldsymbol{\Delta}^{(\pi,i)}$, by setting to $0$ the negative entries of the $i$-th row of $\boldsymbol{\Delta}^{(\pi)}$ while preserving skew symmetry. As $\boldsymbol{\Delta}^{(\pi,i)}$ contains two rows with non-negative entries, one easily deduces that the budget of $A$ under $\boldsymbol{\Delta}^{(\pi,i)}$ is arbitrarily large with probability $1-\delta$. In contrast, under $\boldsymbol{\Delta}$, $A$ finishes before $\chi$---the $(1-\delta)$-quantile of $N_\delta$ under $\mathbb{P}_{\boldsymbol{\Delta},A}$. Hence, we reduce $A$ to an active testing problem with budget $\chi$ for any unknown gap matrix $\tilde{\bDelta}$ of the hypotheses
	\[
	H_0^{(i)}:\tilde{\bDelta}=\bDelta^{(\pi,i)} \text{ for some perm. $\pi$} \quad\text{vs.}\quad H_1^{(i)}:\tilde{\bDelta}=\bDelta^{(\pi)}\text{ for some  perm. $\pi$}\enspace .
	\]
	Since the permutation $\pi$ is unknown to the learner, this allows to improve over~\eqref{eq:LB_exp_instance_first_regime} by accounting for the fact that the algorithm must explore all possible positions of the negative entries in row $i$. 
	
	Then, by a convexity argument, we deduce  $\chi\gtrsim \frac{K_{i;<0}}{\|\Delta_i^-\|^2}\log(1/\delta)$. Optimizing over $i\neq i^*$ yields the first part of the bound  $\chi\gtrsim\max_{i\neq i^*} \frac{K_{i;<0}}{\|\Delta_i^-\|^2}\log(1/\delta)$.
	
	The second term $\sum_{i\neq i^*}\tfrac{K_{i;<0}}{\|\Delta_i^-\|^2}$ interprets as the total cost for testing all hypotheses $H_0^{(i)}$ against $H_1^{(i)}$ with a constant error probability. We develop new arguments for this multiple-hypotheses problem. Unlike the involved technique in~\cite{simchowitz2017simulator}, we reduce w.l.o.g. to the case where all tests have similar complexity $\frac{\|\Delta_i^-\|^2}{K_{i;<0}}$ and we write   $\beta^{2}$ for this common value. Then, we build upon the symmetry of our problem to reduce to the case where each row receives the same sampling effort $\chi/K$. Relying again on lower bounds on active signal detection, we get $\chi/K\gtrsim\beta^{-2}$, which leads to the desired result. We believe that these arguments can generalize to other multiple active testing problems. The full proof is  given in Appendix~\ref{sec:proof_LB_HP_fc}.
	
	We first introduce additional notation and state a more precise version of the result, in particular, we consider here the case where $\bDelta$ might contains some ties.
	
	\paragraph{Notation and precise formulation}
	
	Consider a matrix $\bDelta$ with entries in $\bigl[-\tfrac{1}{4},\tfrac{1}{4}\bigr]$ that admits a unique CW $i^*(\bDelta)$, that is, $\bDelta \in \mathbb{D}_{\mathrm{cw}}$. Without loss of generality, we assume that $i^*(\bDelta) = 1$, and we keep this matrix fixed throughout this section. We now define a family of environments obtained from $\bDelta$ by permuting its entries in a specific way. This will lead to a more precise formulation of Theorem~\ref{thm:LB_hp_instance}.
	
	Fix an antisymmetric matrix $\bSigma = (\sigma_{i,j})_{1 \leqslant i,j \leqslant K}$ defined, for any $1 \leqslant i \ne j \leqslant K$, by
	\begin{equation}\label{def:sign_delta}
		\sigma_{i,j} =
		\begin{cases}
			\operatorname{sign}(\Delta_{i,j})~, & \text{if } \Delta_{i,j} \neq 0~,\\[2pt]
			f_{\mathrm{tb}}(i,j)~, & \text{if } \Delta_{i,j} = 0~,
		\end{cases}
	\end{equation}
	where $f_{\mathrm{tb}}:[K]^2 \to \{-1,0,1\}$ is an antisymmetric function used as a tie-breaking convention. In other words, $\bSigma$ records the sign pattern of the gap matrix $\bDelta$, with a fixed rule when $\Delta_{i,j} = 0$. Note that $\bSigma$ is antisymmetric, since $\bDelta$ is antisymmetric and $f_{\mathrm{tb}}$ is antisymmetric by assumption. For now, we do not specify how we fix this convention, as the precise choice will arise naturally at the very end of our proofs. 
	
	For a suboptimal arm $i \in \{2,\dots,K\}$, denote by $\Sigma_i^-$ the set of arms that beat $i$ (with ties broken according to $f_{\mathrm{tb}}$),
	\[
	\Sigma_i^- \coloneqq \bigl\{ j \in [K]\setminus\{i\} : \sigma_{i,j} = -1 \bigr\} \enspace.
	\]
	
	For any $i \in \{2,\dots,K\}$, define $\Pi_i(\bSigma)$ as the set of permutations of $\Sigma_i^-$:
	\begin{equation*}
		\Pi_i(\bSigma) \coloneqq \bigl\{ \pi_i : \Sigma_i^- \to \Sigma_i^- \,\big|\, \pi_i \text{ is a bijection} \bigr\} \enspace,
	\end{equation*}
	and set
	\begin{equation}\label{def:Pi}
		\Pi(\bSigma) \coloneqq \Pi_2(\bSigma) \times \cdots \times \Pi_K(\bSigma) \enspace.
	\end{equation}
	
	Fix a permutation $\pi \coloneqq (\pi_i)_{i=2}^K \in \Pi(\bSigma)$. For each row $i \neq 1$, we permute the entries of $\bDelta$ indexed by $\Sigma_i^-$ according to $\pi_i$. To preserve antisymmetry, we apply the same permutation $\pi_i$ to the corresponding entries in column $i$. We thus define the matrix $\bDelta^{(\pi)}$ as follows: for all $(i,j) \in [K]^2$,
	\begin{equation}\label{def:Delta_pi}
		\Delta^{(\pi)}_{i,j} \coloneqq
		\begin{cases}
			\Delta_{i,\pi_i(j)}, & \text{if } \sigma_{i,j}=-1 \;, \\[2pt]
			-\Delta_{j,\pi_j(i)}, & \text{if } \sigma_{i,j}=1 \;, \\[2pt]
			\Delta_{i,j}, & \text{else}  \;,
		\end{cases}
	\end{equation}
	the last condition might happen when $\sigma_{i,j}=0$, which only happens for ties. 
	
	The permutations in $\Pi(\bSigma)$ destroy any exploitable ordering structure between arms while preserving, in each row, the multiset of non-positive entries and hence the row-wise hardness parameters $(K_{i;< 0},\|\Delta^-_i\|_2^2)$.
	
	We now list properties that are preserved by the permutation $\pi$. In the following lemma, a row 
	$j$ is said to be a \emph{weak CW} for a given gap matrix if all entries in its row are 
	nonnegative.
	
	\begin{lemma}\label{lemma:Delta_pi}
		For any $\bSigma$ satisfying \eqref{def:sign_delta} and any $\pi \in \Pi(\bSigma)$, we have
		\begin{enumerate}
			\item  \centering $\bDelta^{(\pi)}=-(\bDelta^{(\pi)})^T$ \quad (antisymmetry)
			\item  $\forall j\ne 1, \Delta^{(\pi)}_{1,j}\geqslant 0$ \quad (row $1$ is a weak CW)
			\item $\forall i\ne 1$, $K_{i,< 0}(\bDelta)=K_{i,< 0}(\bDelta^{(\pi)})$ 
			\item $\forall i\ne 1, \forall  s\leqslant K_{i,< 0},\; \Delta^{(\pi)}_{i,(s)}=\Delta_{i,(s)}$ \quad (same ordered nonpositive entries) 
		\end{enumerate}     
		
		In particular, if $\bDelta$ has no ties (that is, $\forall i \ne j$, $\Delta_{i,j}\ne 0$), then $\bDelta^{(\pi)}\in \mathbb{D}_{\mathrm{cw}}$. Moreover, for any $\bm{s}$, $H_{\text{explore}}(\bm{s},\delta)$ and $H_{\text{certify}}(\bm{s},\delta)$ remains unchained under any permutation $\pi \in \Pi(\bSigma)$. 
	\end{lemma}
	
	\begin{proof}[Proof of Lemma~\ref{lemma:Delta_pi}]
		\noindent 1.~\underline{Anti-symmetry} 
		Let $(i,j)\in [K]^2$ with $i\ne j$. \\
		Assume that $(\sigma_{i,j}=-1)$ which is equivalent to $ j\in \Sigma_{i}^-$. By the first line in Equation~\eqref{def:Delta_pi}, one has $\Delta_{i,j}^{(\pi)}=\Delta^{(\pi)}_{i,\pi_i(j)}$ . Then, by the second line applied with $(j,i)$ ($\sigma_{j,i}=-1$), one has $\Delta^{(\pi)}_{j,i}=-\Delta_{i,\pi_i(j)}=-\Delta^{(\pi)}_{i,j}$. \\
		The case $\sigma_{i,j}=1$ is treated similarly. \\
		For $(\sigma_{i,j}=0)$, one also have $\sigma_{j,i}=0$ and $\Delta_{i,j}^{\pi}=\Delta_{i,j}=-\Delta_{j,i}=-\Delta_{j,i}^{\pi}$.  It proves the anti-symmetry of $\bDelta^{(\pi)}$. 
		
		\medskip\noindent
		2.~\underline{CW row} 
		By definition, $i^*\in \Sigma_i^-$ for any $i\ne i^*$ (and $i^*=1$ by assumption). In particular, $\Delta^{(\pi)}_{1,i}=\Delta_{\pi_i(1),i} \geqslant 0$, the non-negativity comes from the fact that $\pi_i(1)\in \Sigma_i^-$, that is $\pi_i(1)$ also beats $i$ (or is even with $i$ in the case where $\Delta$ contains ties). 
		
		\medskip\noindent
		3.~\underline{Negative entries.}
		For every non-CW arm $i\neq 1$, the set of indices $\Sigma_i^-$ contains all indices $j$ such that $\Delta_{i,j}<0$, and only indices $j$ such that $\Delta_{i,j}\leqslant 0$.
		By \eqref{def:Delta_pi}, the entries indexed by $\Sigma_i^-$ are permuted, while all entries outside this set remain nonnegative.
		Hence the multiset of strictly negative entries in row $i$ is preserved, and so is the count $K_{i,<0}$.
		
		\medskip\noindent
		4.~\underline{Order statistics.}
		Since the multiset of strictly negative entries in row $i$ is preserved up to permutation, the ordered sequence $(\Delta_{i,(s)})_{1\leqslant s\leqslant K_{i,<0}}$ is unchanged, which gives $\Delta^{(\pi)}_{i,(s)}=\Delta_{i,(s)}$ for all $1\leqslant s\leqslant K_{i,<0}$. In particular, for any vector $\bm{s}$ with $1\leqslant s_i\leqslant K_{i;<0}$ (a condition independent of $\pi$), the gaps $\Delta^{(\pi)}_{i,(s_i)}$ coincide with $\Delta_{i,(s_i)}$, so that $H_{\mathrm{explore}}(\bm{s},\delta)$ and $H_{\mathrm{certify}}(\bm{s},\delta)$ remain unchanged under any $\pi\in\Pi(\bSigma)$.
	\end{proof}
	
	Now we are ready to state Theorem~\ref{thm:precise_LB_high_proba}, which directly implies Theorem~\ref{thm:LB_hp_instance} and provides a more constructive formulation.
	
	\begin{theorem}\label{thm:precise_LB_high_proba}
		Let $A$ be a $\delta$-correct algorithm over the class $\mathbb{D}_{\mathrm{cw}}$, with $\delta \leqslant 1/12$. Let $\bDelta \in \mathbb{D}_{\mathrm{cw}}$ be such that $i^*(\bDelta)=1$.
		
		Define $\chi$ as the smallest positive number such that, for any sign convention $\bSigma$ (satisfying~\eqref{def:sign_delta}) and any $\pi \in \Pi(\bSigma)$, one has $\mathbb{P}_{\bDelta^{(\pi)},A}(N_{\delta} > \chi) \leqslant \delta$. Equivalently,
		\begin{equation}\label{def:chi}
			\chi \coloneqq \inf \left\{ x > 0 : \sup_{\bSigma} \sup_{\pi \in \Pi(\bSigma)} \mathbb{P}_{\bDelta^{(\pi)},A}(N_{\delta} > x) \leqslant \delta \right\} \enspace.
		\end{equation}
		In other words, $\chi$ is a uniform $(1-\delta)$-quantile of the budget, taken in the worst case 
		over all admissible sign conventions and permutations.
		
		Then, $\chi$ satisfies
		\begin{align}
			\chi &\geqslant \frac{1}{8\log(4/3)}  \max_{i=2}^K \frac{K_{i;\leqslant 0}}{\|\Delta_i^-\|^2}  \log\!\left(\frac{1}{6\delta}\right), \label{eq:lb_H2} \\[2pt]
			\chi &\geqslant \frac{1}{128\log(4/3)}\cdot \frac{1}{\log(2K)}\sum_{i = 2}^K \frac{K_{i;\leqslant 0}}{\|\Delta_i^-\|^2} \enspace. \label{eq:lb_H0} 
		\end{align}
	\end{theorem}
	
	We postpone the proof of Theorem~\ref{thm:precise_LB_high_proba} to the following Subsection~\ref{sec:proof_precise_LB}, which is divided between the proof of Equation~\eqref{eq:lb_H2} and Equation~\eqref{eq:lb_H0}. We first explain how this result directly implies Theorem~\ref{thm:LB_hp_instance}. 
	
	\medskip
	\begin{proof}[Proof of Theorem~\ref{thm:LB_hp_instance}]
		
		Let $A$ be a $\delta$-correct algorithm over the class $\mathbb{D}_{\mathrm{cw}}$, with $\delta \leqslant 1/12$. Fix any matrix $\bDelta \in \mathbb{D}_{\mathrm{cw}}$ with CW $i^* = 1$. Consider $\chi$ as defined in Equation~\eqref{def:chi}. 
		
		Combining the numerical bounds $8\log(4/3)<3$ and $128\log(4/3)<37$ with the definition of $\chi$ and the lower bounds~\eqref{eq:lb_H2} and~\eqref{eq:lb_H0}, we obtain that there exist a sign convention $\bSigma$ and a permutation $\pi \in \Pi(\bSigma)$ such that
		\[
		\mathbb{P}_{\bDelta^{(\pi)},A}\!
		\left(
		N_{\delta} \geqslant
		\frac{1}{3} \cdot
		\max_{i\ne i^*}\frac{K_{i;\leqslant 0}}{\|\Delta_i^-\|^2}
		\log\!\left(\frac{1}{6\delta}\right)
		\,\vee\,
		\frac{1}{37 \log(2K)} \sum_{i \ne i^*} \frac{K_{i;\leqslant 0}}{\|\Delta_i^-\|^2}
		\right) \;\geqslant\; \delta~.
		\]
		
		By Lemma~\ref{lemma:Delta_pi}, if $\bDelta$ contains no ties, then the matrix $\bDelta^{\pi}$ admits a Condorcet winner and verifies all properties required of the matrix $\tilde\bDelta$ from Theorem~\ref{thm:LB_hp_instance}. This proves Theorem~\ref{thm:LB_hp_instance} in the 
		no-ties case.
	\end{proof}
	
	\begin{remark}
		Observe that, by the second point of Lemma~\ref{lemma:Delta_pi}, it is possible that $\bDelta^{(\pi)}$ admits no CW, when $\bDelta$ admit some ties. Yet it still admits a weak Condorcet winner, in the same that $\Delta^{\pi}_{1,\cdot}$ admits only non-negative entries, while all the other rows admit at least one negative entry. Then, one can construct a matrix $\bDelta^{\pi,\epsilon}$ as close as we need to $\bDelta^{(\pi)}$ and such that $\bDelta^{(\pi,\epsilon)}$ admit $1$ as CW. We can for instance add a small constant $\epsilon>0$ to the first row of $\bDelta^{(\pi,\epsilon)}$ [and $-\epsilon$ to the first column]. For $\epsilon$ small enough, the given matrix $\bDelta^{(\pi,\epsilon)}$ shares most properties of $\Delta$.
	\end{remark}
	
	\subsection{Proof of Theorem~\ref{thm:precise_LB_high_proba}}\label{sec:proof_precise_LB}
	
	\subsubsection{Proof of Equation~\eqref{eq:lb_H2} in Theorem~\ref{thm:precise_LB_high_proba}}
	
	Let $\delta\leqslant 1/12$. Let $A$ be any $\delta$-correct algorithm over the entire class $\mathbb{D}_{\mathrm{cw}}$. Fix $\bDelta \in \mathbb{D}_{\mathrm{cw}}$ with CW $i^*(\bDelta) = 1$.   Recall the definition 
	\[\chi \coloneqq \inf \left\{ x > 0 : \sup_{\Sigma} \sup_{\pi \in \Pi(\bSigma)} \mathbb{P}_{\bDelta^{(\pi)},A}(N_{\delta} > x) \leqslant \delta \right\}~.
	\]
	We prove Equation~\eqref{eq:lb_H2}, namely
	\[
	\chi \geqslant \frac{1}{8\log(4/3)}  \max_{i=2}^K \frac{K_{i;\leqslant 0}}{\|\Delta_i^-\|^2}  \log\!\left(\frac{1}{6\delta}\right). 
	\]
	
	\paragraph{Sketch of proof.}
	We follow the three-step roadmap of Section~\ref{sec:roadmap_LB}. The arguments are very similar to the proof of the quantile proof of Proposition~\ref{thm:LB_exp_instance_first_regime} in paragraph~\ref{sec:proof_HP_karnin}, except that we work with a local class of permuted instances 
	and average over all permutations of $\bDelta$. 
	
	\emph{(i) Reference and alternative instances.} We fix $\bDelta\in\mathbb{D}_{\mathrm{cw}}$ with $i^*(\bDelta)=1$ and consider the \emph{local class of reference} $(\Delta^{(\pi)})_{\pi\in \Pi}$ obtained 
	by permuting, within each non-CW row, the positions of its negative entries according to $\pi$.   
	For each suboptimal arm $k$, we construct the \emph{alternative instance} $\bDelta^{(\pi,k)}$ by lifting to $0$ all negative entries in row $k$, so that rows $1$ and $k$ become nonnegative and the instance has two weak CWs (hence lies outside $\mathbb{D}_{\mathrm{cw}}$).
	\emph{(ii) Separating event and total variation.} As in Subsection~\ref{sec:proof_HP_karnin}, we 
	truncate $A$ at the worst-case $(1-\delta)$-quantile $\chi$ to obtain a fixed-budget algorithm 
	$\tilde A$, and we use the event $\{N_\delta>\chi\}$ as a separating event. This yields 
	a total-variation lower bound $\mathrm{TV}(\mathbb{P}_{\tilde A}^{(\pi)},\mathbb{P}_{\tilde A}^{(\pi,k)}) 
	\ge 1-3\delta$, hence a KL lower bound of order $\log(1/(6\delta))$ for each $k$. 
	\emph{(iii) KL decomposition and extraction of $K_{k;\leqslant 0}/\|\Delta_k^-\|_2^2$.} We decompose the 
	KL along pairs and use that $\bDelta^{(\pi)}$ and $\bDelta^{(\pi,k)}$ differ only on duels involving 
	$k$. Averaging over permutations $\pi_k\in\Pi_k(\bSigma)$ symmetrizes the contribution of all negative 
	entries in row $k$, and yields a lower bound $\chi \gtrsim K_{k;\leqslant 0}\|\Delta_k^-\|_2^{-2}
	\log(1/\delta)$. Finally, choosing the row $k$ that maximizes $K_{k;\leqslant 0}/\|\Delta_k^-\|_2^2$ gives 
	\eqref{eq:lb_H2}. \emph{(iv) Tie-breaking convention.} In the last step, we a tie-breaking convention to conclude. 
	
	\medskip
	\begin{proof}
		\noindent\underline{Step 1: Reference and alternative instances.} 
		
		\noindent For now, fix a tie-breaking convention $f_{\mathrm{tb}}$ and a sign matrix $\bSigma$ as in~\eqref{def:sign_delta}. The specific choice of $\bSigma$ will be made in the final step of the proof.
		
		Fix a permutation $\pi \in \Pi(\bSigma)$ (see Equation~\eqref{def:Pi}), and consider the corresponding matrix $\bDelta^{(\pi)}$ defined in Equation~\eqref{def:Delta_pi}. We denote by $\mathbb{P}_{A}^{(\pi)}$ the distribution of the observations induced by the interaction between algorithm $A$ and the environment with gap matrix $\bDelta^{(\pi)}$.
		
		\medskip 
		\noindent Fix a suboptimal arm $k \in \{2,\dots,K\}$. The precise choice of $k$ will be made explicit in the last step of the proof. We construct the gap matrix $\bDelta^{(\pi,k)}$ by setting to zero all entries $(k,j)$ with $j \in \Sigma_k^- \coloneqq \{ j \in [K] : \sigma_{k,j} = -1 \}$. Recall that, by the definition of $\bSigma$ (see~\eqref{def:sign_delta}), if $j \in \Sigma_k^-$ then $\Delta^{(\pi)}_{k,j} \leq 0$, that is $\Sigma_k^-$ contains all arms that beat strictly $k$, together with some arms that are tied with $k$. We define 
		\begin{equation}\label{def:Delta_pik}
			\bDelta^{(\pi,k)}_{i,j} \coloneqq
			\begin{cases}
				\Delta^{(\pi)}_{i,j}, & \text{if } i \neq k \text{ and } j \neq k, \\[4pt]
				\Delta^{(\pi)}_{i,j}, & \text{if } (i = k \text{ and } \sigma_{k,j}=1) \text{ or } (j = k \text{ and } \sigma_{k,i}=1), \\[4pt]
				0, & \text{if } (i = k \text{ and } \sigma_{k,j}=-1) \text{ or } (j = k \text{ and } \sigma_{k,i}=-1).
			\end{cases}
		\end{equation}
		The matrices $\bDelta^{(\pi)}$ and $\bDelta^{(\pi,k)}$ differ only in row/column $k$. Moreover, by construction, rows $1$ and $k$ of $\bDelta^{(\pi,k)}$  contain only nonnegative entries, so that $\bDelta^{(\pi,k)} \notin \mathbb{D}_{\mathrm{cw}}$. We denote by $\mathbb{P}_{A}^{(\pi,k)}$ the distribution of the observations induced by the interaction between algorithm $A$ and the environment with gap matrix $\bDelta^{(\pi,k)}$.

		\medskip
		
		\noindent\underline{Step 2: information-theoretic arguments.}
		
		This step is identical to Step~2 in the proof of the  quantile bound~\eqref{eq:LB_hp_karnin}
		(Section~\ref{sec:proof_HP_karnin}). Consider the event $B:=\{N_{\delta} \leqslant \chi\}$, and 
		define the truncated algorithm $\tilde A$ as in that proof: run $A$ up to time $\chi$, returning 
		$\tilde i = \hat i$ if $A$ stops before $\chi$, and $\tilde i = 0$ otherwise. 
		
		By the definition of $\chi$ in~\eqref{def:chi}---a uniform upper bound on the $(1-\delta)$-quantile of $N_{\delta}$ under any $\mathbb P^{(\pi)}_A$---we have 
		\[
		\mathbb{P}_{\tilde A}^{(\pi)}(\tilde i = 0) = \mathbb{P}_{A}^{(\pi)}(N_{\delta} > \chi) \leqslant \delta~.
		\]
		
		Since $\bDelta^{(\pi,k)}$ can be approximated by CW instances admitting $1$ or $k$ as CW (as in 
		\eqref{eq:approximate_tv_karnin}--\eqref{eq:H2_tv2_karnin}), we obtain 
		\[
		\mathbb{P}_{\tilde A}^{(\pi,k)}(\tilde i \neq 0) \leqslant 2\delta~.
		\]
		Writing $B:=\{\tilde i = 0\}$, this 
		yields
		\begin{equation}\label{eq:H2_tv2}
			\mathrm{TV}\bigl(\mathbb{P}_{\tilde A}^{(\pi)}, \mathbb{P}_{\tilde A}^{(\pi,k)}\bigr)
			\;\geqslant\; (1-2\delta) - \delta
			\;=\; 1-3\delta.
		\end{equation}

		\begin{remark}
			The result from Step~2 can be interpreted as a reduction scheme. Consider the signal detection problem of testing
			\(H_0: \mu = \boldsymbol{0}\) versus \(H_1(\pi_k): \mu = \bigl(\Delta_{k,\pi_k(\ell)}\bigr)_{\ell\in \Sigma_k^-}\)
			in a bandit setting. Equation~\eqref{eq:H2_tv2} shows that, for any permutation \(\pi_k\), \(\tilde{A}\) is \(2\delta\)-correct for this signal detection problem, with a budget bounded by \(\chi\) (independently of the permutation). This reduction is the main novelty of our proof technique. The remaining arguments in Step~3 build upon the literature on active signal detection \citep{castro2014adaptive,saad2023active,graf2025clustering}.
		\end{remark}
		
		\noindent\underline{Step 3: computing the KL divergence.}
		
		By the Bretagnolle–Huber inequality (see, e.g., \citealp{lattimore2020bandit}), the conclusion of Step 2~\eqref{eq:H2_tv2} implies
		\begin{equation}\label{eq:H2_BH}
			\mathrm{KL}\bigl(\mathbb{P}_{\tilde{A}}^{(\pi,k)}, \mathbb{P}_{\tilde{A}}^{(\pi)}\bigr)
			\;\geqslant\; \log\!\left(\frac{1}{6\delta}\right).
		\end{equation}
		
		We now compute $\mathrm{KL}\bigl(\mathbb{P}_{\tilde{A}}^{(\pi,k)}, \mathbb{P}_{\tilde{A}}^{(\pi)}\bigr)$. Observe that we take the law under the alternative instance $\bDelta^{(\pi,k)}$ in the left side of the divergence: 
		this ensures that the expectation $\mathbb{E}_{\tilde{A}}^{(\pi,k)}[N_{\{k,j\}}]$ appearing in the 
		KL decomposition (below) does not depend on $\pi_k$, which will allow us to average over 
		permutations $\pi_k\in\Pi_k(\bSigma)$. 
		
		By the standard decomposition of the Kullback–Leibler divergence for adaptive procedures (see, e.g., \citealp{lattimore2020bandit}, Lemma~15.1),
		\begin{align}
			\mathrm{KL}\bigl(\mathbb{P}_{\tilde{A}}^{(\pi,k)}, \mathbb{P}_{\tilde{A}}^{(\pi)}\bigr)
			&= \sum_{i \neq j} \mathbb{E}_{\tilde{A}}^{(\pi,k)}\bigl[{N}_{i,j}\bigr]\,
			\mathrm{KL}\bigl(\mathbb{P}_{i,j}^{(\pi,k)}, \mathbb{P}_{i,j}^{(\pi)}\bigr) \nonumber \\
			&= \sum_{j \in \Sigma_k^-} \mathbb{E}_{\tilde{A}}^{(\pi,k)}\bigl[N_{\{k,j\}}\bigr]\,
			\mathrm{KL}\bigl(\mathbb{P}_{k,j}^{(\pi,k)}, \mathbb{P}_{k,j}^{(\pi)}\bigr) \label{eq:H2_KL1}\enspace,
		\end{align}
		since the two instances differ only on pairs $(k,j)$ or $(j,k)$ with $j \in \Sigma_k^-$, by construction of $\bDelta^{(\pi,k)}$.
		
		Now fix $j \in \Sigma_k^-$. By the definitions of $\bDelta^{(\pi)}$ and $\bDelta^{(\pi,k)}$ (Equations~\eqref{def:Delta_pi} and \eqref{def:Delta_pik}), we have
		\[
		\mathbb{P}_{k,j}^{(\pi,k)} = \mathcal{B}\!\left(\tfrac{1}{2}\right),
		\qquad
		\mathbb{P}_{k,j}^{(\pi)} = \mathcal{B}\!\left(\tfrac{1}{2} + \Delta_{k,\pi_k(j)}\right),
		\]
		where $\Delta_{k,\pi_k(j)} \in [-1/4,0]$. Hence, using the bound on $\kl$ from~\eqref{eq:bound_kl},
		\begin{align*}
			\mathrm{KL}\bigl(\mathbb{P}_{k,j}^{(\pi,k)}, \mathbb{P}_{k,j}^{(\pi)}\bigr)
			\leqslant\; 8 \log\!\left(\frac{4}{3}\right)\Delta_{k,\pi_k(j)}^2 ~,
		\end{align*}

		Plugging these bounds into \eqref{eq:H2_KL1}, we obtain
		\begin{align*}
			\mathrm{KL}\bigl(\mathbb{P}_{\tilde{A}}^{(\pi,k)}, \mathbb{P}_{\tilde{A}}^{(\pi)}\bigr)
			&\leqslant 8\log\!\left(\frac{4}{3}\right)
			\sum_{j \in \Sigma_k^-}
			\mathbb{E}_{\tilde{A}}^{(\pi,k)}[N_{\{k,j\}}]\,\Delta_{k,\pi_k(j)}^2~.
		\end{align*}
		
		A key property of our construction is that $\bDelta^{(\pi,k)}$ \emph{does not actually depend on $\pi_k$}: 
		all entries permuted by $\pi_k$ in row $k$ of $\bDelta^{(\pi)}$ are set to $0$ under 
		$\bDelta^{(\pi,k)}$.  Consequently, $\mathbb{E}_{\tilde{A}}^{(\pi,k)}$ does not depend on  $\pi_k$. Given $\pi'_k\in \Pi_k(\bSigma)$, we denote $\pi'=(\pi_2,\ldots, \pi_{k-1},\pi'_k,\pi_{k+1},\ldots)$ where we only change $\pi_k$. Averaging the previous inequality over $\pi'_k \in \Pi_k(\bSigma)$ while keeping $(\pi_l)_{l\neq k}$ fixed, we get
		\begin{align}
			\frac{1}{|\Pi_k(\bSigma)|}\sum_{\pi'_k \in \Pi_k(\bSigma)} 
			\mathrm{KL}\bigl(\mathbb{P}_{\tilde{A}}^{(\pi,k)}, \mathbb{P}_{\tilde{A}}^{(\pi')}\bigr)
			&\leqslant 8\log\!\left(\frac{4}{3}\right)
			\frac{1}{|\Pi_k(\bSigma)|}
			\sum_{\pi'_k \in \Pi_k(\bSigma)}
			\sum_{j \in \Sigma_k^-}
			\mathbb{E}_{\tilde{A}}^{(\pi,k)}[N_{\{k,j\}}]\,\Delta_{k,\pi'_k(j)}^2 \nonumber\\
			&= 8\log\!\left(\frac{4}{3}\right)
			\frac{1}{|\Sigma_k^-|}
			\left(\sum_{j \in \Sigma_k^-} \mathbb{E}_{\tilde{A}}^{(\pi,k)}[N_{\{k,j\}}]\right)
			\left(\sum_{j \in \Sigma_k^-} \Delta_{k,j}^2\right) \label{eq:KL_verage_pi} ~,
		\end{align}
		where we used Lemma~\ref{lem:sum_tech} to symmetrize over all permutations $\pi'_k$ of $\Sigma_k^-$.
		
		By definition of $\Sigma_k^-$, it holds that $\Sigma_k^- \subset \{j\in [K]\setminus\{k\} : \Delta_{k,j}\leqslant 0 \}$, so that
		\[
		\sum_{j \in \Sigma_k^-} \Delta_{k,j}^2
		= \|\Delta_k^-\|^2  ~.
		\]
		Moreover, by construction, the modified algorithm $\tilde{A}$ has a budget upper bounded by $\chi$, and 
		\[
		\sum_{j \in \Sigma_k^- }
		\mathbb{E}_{\tilde{A}}^{(\pi,k)}[N_{\{k,j\}}]
		\;\leqslant\; \chi ~.
		\]
		From there, we get
		\begin{equation}\label{eq:H2_KL2}
			\frac{1}{|\Pi_k(\bSigma)|}\sum_{\pi'_k \in \Pi_k(\bSigma)}
			\mathrm{KL}\bigl(\mathbb{P}_{\tilde{A}}^{(\pi,k)}, \mathbb{P}_{\tilde{A}}^{(\pi')}\bigr)
			\leqslant 8\log\!\left(\frac{4}{3}\right) \frac{\|\Delta_k^-\|^2}{|\Sigma_k^-|}\, \chi  \enspace.
		\end{equation}
		Finally, combining \eqref{eq:H2_KL2} with the Bretagnolle–Huber bound \eqref{eq:H2_BH}, we obtain
		\begin{equation}\label{eq:H2_KLfinal}
			\chi \geqslant \frac{1}{8\log(4/3)} \frac{|\Sigma_k^-|}{\|\Delta_k^-\|^2} \log\!\left(\frac{1}{6\delta}\right)
			~.
		\end{equation}
		
		\medskip
		
		\noindent\underline{Step 4: choice of convention $\bSigma$ and conclusion.}
		
		It remains to choose an appropriate arm $k$, and a convention $f_{\mathrm{tb}}$ in the definition of the sign matrix $\bSigma$ (see Equation~\eqref{def:sign_delta}). Consider
		\[
		k^* \in \argmax_{k=2}^K \frac{K_{k;\leqslant 0}}{\|\Delta_k^-\|^2}~,
		\]
		as a suboptimal arm for which detecting a negative entry in its row is the most costly. 
		
		Fix a tie-breaking convention $f_{\mathrm{tb}} : [K]^2\to \{-1,0,1\}$ such that $f_{\mathrm{tb}}(k^*,i)=-1$ for any $i\ne k^*$. For this choice, we have
		\[
		\frac{|\Sigma_{k^*}^-|}{\|\Delta_{k^*}^-\|^2}
		= \frac{K_{k^*;\leqslant 0}}{\|\Delta_{k^*}^-\|^2}
		= \max_{i\ne i^*}\frac{K_{i;\leqslant 0}}{\|\Delta_i^-\|^2}~.
		\]
		Plugging this into \eqref{eq:H2_KLfinal} yields
		\[
		\chi \geqslant  \frac{1}{8\log(4/3)}
		\max_{i\ne i^*}\frac{K_{i;\leqslant 0}}{\|\Delta_i^-\|^2}
		\log\!\left(\frac{1}{6\delta}\right)~,
		\]
		which establishes the first inequality~\eqref{eq:lb_H2} on $\chi$.
	\end{proof}
	\subsubsection{Proof of Equation~\eqref{eq:lb_H0} in Theorem~\ref{thm:precise_LB_high_proba}}
	
	Let $\delta\leqslant 1/12$. Let $A$ be any $\delta$-correct algorithm over the entire class $\mathbb{D}_{\mathrm{cw}}$, and fix $\bDelta \in \mathbb{D}_{\mathrm{cw}}$ with CW $i^*(\bDelta) = 1$.  
	In this subsection, we prove Equation~\eqref{eq:lb_H0} from Theorem~\ref{thm:precise_LB_high_proba}, using the same instance construction as in the proof of bound~\eqref{eq:lb_H2}. Recall this bound: 
	\[
	\chi \geqslant \frac{1}{64\log(4/3)}\frac{1}{\log(2K)}\sum_{i = 2}^K \frac{K_{i;\leqslant 0}}{\|\Delta_i^-\|^2} . 
	\]
	
	\paragraph{Sketch of proof.}
	We follow the three-step roadmap of Section~\ref{sec:roadmap_LB}, reusing the instance construction 
	from the proof of \eqref{eq:lb_H2} but with a more refined separating event. The key differences are 
	(ii) a refined event $B_k$ that also controls the number of duels involving arm $k$, and (iii) the 
	use of Pinsker's inequality and Jensen to average over multiple arms simultaneously.
	
	\emph{(i) Reference and alternative instances.} As before, we consider the local class 
	$(\bDelta^{(\pi)})_{\pi\in\Pi(\bSigma)}$ and, for each $k\in\{2,\dots,K\}$, the alternative 
	$\bDelta^{(\pi,k)}$. We define the row-wise 
	hardness $\beta_k^2 = \|\Delta_k^-\|^2/|\Sigma_k^-|$, assumed ordered increasingly. Let $I$ be the 
	index that maximizes $(k-1)/\beta_k^2$, corresponding to the worst-case average hardness over the 
	first $k$ rows.
	\emph{(ii) Separating event and total variation.} Instead of $\{N_\delta\le\chi\}$, we use the 
	event $B_k=\{N_\delta\leqslant \chi,\, N_{k,\cdot}\leqslant 4\chi/(I-1)\}$, where $N_{k,\cdot}$ counts 
	duels between $k$ and opponents in $\Sigma_k^-$. Define a  truncation $\tilde{A}_k$ that 
	outputs $\mathds{1}_{B_k}$, using budget at most $4\chi/(I-1)$ on $k$. Under $\bDelta^{(\pi,k)}$, 
	$\mathbb{P}(\mathds{1}_{B_k}=1)\leqslant 2\delta$. Under $\bDelta^{(\pi)}$, we use a pigeonhole argument and the average probability is larger than $3/4-\delta$. This yields averaged TV larger than $1/2$---the key novelty of this proof technique. \emph{(iii) KL decomposition and extraction of the sum.} By Pinsker's inequality, the averaged TV 
	lower bound implies an averaged KL lower bound. Each KL is upper bounded as previously, and now we 
	average over $k=2,\dots,I$ and use the truncation constraint 
	$N_{k,\cdot}\leqslant 4\chi/(I-1)$ to get $\chi\gtrsim (I-1)/\beta_I^2$. By the definition of $I$, this 
	gives $\chi\gtrsim \sum_{i=2}^K K_{i;\leqslant 0}/\|\Delta_i^-\|_2^2$, up to logarithmic factors. 
	
	\medskip
	\begin{proof}
		\noindent\underline{Step 1: construction of instances.}
		Fix a tie-breaking convention $f_{\mathrm{tb}}$ and a sign matrix $\bSigma$ (see~\eqref{def:sign_delta}). Again, they will be chosen in the last step of the proof. Fix a permutation $\pi \in \Pi(\bSigma)$ and use $\bDelta^{(\pi)}$ as the reference matrix (see~\eqref{def:Delta_pi}).
		Define, for each $k \in \{2,\dots,K\}$, the row-wise hardness 
		\begin{equation}\label{def:beta_k}
			\beta_k^2(\Sigma) \coloneqq
			\frac{\|\Delta_k^-\|^2}{|\Sigma_k^-|}
			\enspace .
		\end{equation}
		Without loss of generality, assume that the arms are ordered so that  $\beta_2^2  \leqslant \beta_3^2  \leqslant \dots \leqslant \beta_K^2$.
		Define
		\begin{equation}\label{eq:def_I}
			I \coloneqq \argmax_{k = 2,\dots,K} \frac{k-1}{\beta_k^2} \enspace,
		\end{equation}
		the index that captures the worst-case average hardness over the first $k$ rows.
		
		As alternative instances, we consider the family 
		$\{\bDelta^{(\pi,k)}\}_{k=2,\dots,I}$. 
		
		\medskip 
		
		\noindent\underline{Step 2: information-theoretic arguments.}
		In this step, we construct an event under which algorithm $A$ should behave differently depending 
		on whether it interacts with $\bDelta^{(\pi)}$ or $\bDelta^{(\pi,k)}$. To capture the sum lower 
		bound~\eqref{eq:lb_H0}, we need a more refined event than in the proof of \eqref{eq:lb_H2}. For $k \in \{2,\dots,I\}$, define
		\begin{equation}\label{def:Sigma_k}
			B_k \coloneqq \{N_{\delta} \leqslant \chi\} \cap \left\{N_{k,\cdot} \leqslant \frac{4\chi}{I-1}\right\},
		\end{equation}
		where
		\[
		N_{k,\cdot} \coloneqq \sum_{i \in \Sigma_k^- } N_{\{k,i\}}
		\]
		denotes the total number of duels involving arm $k$ against opponents in $\Sigma_k^-$.
		
		\begin{remark}
			The event $B_k$ is designed as follows. The bound~\eqref{eq:lb_H2} shows that the quantity $\beta_k^{-2}$ characterizes the budget needed to find a negative entry in row $k$ of $\bDelta^{(\pi)}$, uniformly over all permutations $\pi$. . Identifying the CW $i^*=1$ amounts to solving simultaneously $K-1$ such 
			signal detection problems, one per suboptimal row. By the definition of $I$ in~\eqref{eq:def_I} and Lemma~\ref{lem:tech1}, it is natural to think of the simplified regime where $(\beta_2^{-2},\dots,\beta_I^{-2})$ are of the same order, so that arms $2,\dots,I$ are equally hard to eliminate. In that case, any reasonable algorithm should allocate its samples roughly uniformly across rows $2,\dots,I$, and the event $B_k$ describes this expected behavior for a $\delta$-correct algorithm $A$.
		\end{remark}
		
		We can compute the event $B_k$ with a truncated procedure $\tilde{A}_k$ that uses a total budget at most $\chi$, and at most $4\chi/(I-1)$ comparisons involving arm $k$ against an opponent in $\Sigma_k^-$.
		
		Define the following procedure $\tilde{A}_k$. For $t = 1,\dots,\chi$, run algorithm $A$. During the process, stop as soon as the algorithm tries to sample more than $4\chi/(I-1)$ samples on the row $k$, or more than $\chi$ samples in total and return $0$. 
		Otherwise, the procedure returns $1$ if the sampling stops by $\chi$ without exceeding the cap. By construction, the binary decision $\psi_k$ computed by $\tilde{A}_k$ satisfies $\psi_k = \mathds{1}_{B_k}$.
		
		We write $\mathbb{P}_{\tilde{A}_k}^{(\pi)}$ (resp. $\mathbb{P}_{\tilde{A}_k}^{(\pi,k)}$) for the distribution induced by the interaction between algorithm $\tilde{A}_k$ and the environment with gap matrix $\bDelta^{(\pi)}$ (resp. $\bDelta^{(\pi,k)}$). We now lower bound the total variation distance between $\mathbb{P}_{\tilde{A}_k}^{(\pi)}$ and $\mathbb{P}_{\tilde{A}_k}^{(\pi,k)}$ using the event $B_k$.
		
		First, $B_k \subset \{N_{\delta} \leqslant \chi\}$. Since $A$ is $\delta$-correct over $\mathbb{D}_{\mathrm{cw}}$ and $\bDelta^{(\pi,k)}$ can be approximated by instances in $\mathbb{D}_{\mathrm{cw}}$ as in the proof of~\eqref{eq:H2_tv2_karnin}, we obtain, for any $k \ne 1$,
		\begin{equation}\label{eq:H0_tv1}
			\mathbb{P}_{\tilde{A}_k}^{(\pi,k)}(B_k)
			\leqslant \mathbb{P}_{\tilde{A}_k}^{(\pi,k)}(N_{\delta} \leqslant \chi)
			\leqslant 2\delta \enspace.
		\end{equation}
		
		Next, consider $\mathbb{P}_{\tilde{A}_k}^{(\pi)}$ and the complement $B_k^c$. We have
		\[
		B_k^c
		= \{N_{\delta} > \chi\} \;\cup\; \{N_{\delta} \leqslant \chi,\; N_{k,\cdot} > 4\chi/(I-1)\}.
		\]
		Since $A$ is $\delta$-correct and $\bDelta^{(\pi)} \in \mathbb{D}_{\mathrm{cw}}$, the definition of $\chi$ implies
		\begin{equation}\label{eq:H0_tv2}
			\mathbb{P}_{\tilde{A}_k}^{(\pi)}(N_{\delta} > \chi)
			= \mathbb{P}_{A}^{(\pi)}(N_{\delta} > \chi)
			\leqslant \delta \enspace.
		\end{equation}
		
		We now average the second term of $B_k^c$ over $k \in \{2,\dots,I\}$:
		\begin{align*}
			\frac{1}{I-1} \sum_{k=2}^{I}
			\mathbb{P}_{\tilde{A}_k}^{(\pi)}\bigl(N_{\delta} \leqslant \chi,\; N_{k,\cdot} > 4\chi/(I-1)\bigr) 
			&= \mathbb{E}_{A}^{(\pi)}\!\left[
			\frac{1}{I-1}
			\sum_{k=2}^{I}
			\mathds{1}_{\{N_{\delta} \leqslant \chi,\; N_{k,\cdot} > 4\chi/(I-1)\}}
			\right]\!.
		\end{align*}
		
		Since $\sum_{k=2}^{K} N_{k,\cdot} \leqslant N_{\delta}$, on the event $\{N_{\delta} \leqslant \chi\}$ at most $(I-1)/4$ indices $k \in \{2,\dots,I\}$ can satisfy $N_{k,\cdot} > 4\chi/(I-1)$. Hence, 
		\[
		\frac{1}{I-1}
		\sum_{k=2}^{I}
		\mathds{1}_{\{N_{\delta} \leqslant \chi,\; N_{k,\cdot} > 4\chi/(I-1)\}}
		\leqslant \frac{1}{4},
		\]
		which yields
		\begin{equation}\label{eq:H0_tv3}
			\frac{1}{I-1} \sum_{k=2}^{I}
			\mathbb{P}_{\tilde{A}_k}^{(\pi)}\bigl(N_{\delta} \leqslant \chi,\; N_{k,\cdot} > 4\chi/(I-1)\bigr)
			\leqslant \frac{1}{4} \enspace.
		\end{equation}
		
		Combining \eqref{eq:H0_tv1}, \eqref{eq:H0_tv2}, and \eqref{eq:H0_tv3}, we obtain
		\begin{align}
			\frac{1}{I-1} \sum_{k=2}^{I}
			\mathrm{TV}\bigl(\mathbb{P}_{\tilde{A}_k}^{(\pi,k)}, \mathbb{P}_{\tilde{A}_k}^{(\pi)}\bigr)
			&\geqslant \frac{1}{I-1} \sum_{k=2}^{I}
			\Bigl(
			\mathbb{P}_{\tilde{A}_k}^{(\pi,k)}(B_k^c)
			- \mathbb{P}_{\tilde{A}_k}^{(\pi)}(B_k^c)
			\Bigr) \nonumber \\
			&\geqslant 1 - 2\delta - \frac{1}{I-1} \sum_{k=2}^{I} \mathbb{P}_{\tilde{A}_k}^{(\pi)}(B_k^c) \nonumber \\
			&\geqslant (1 - 2\delta) - \left(\delta + \frac{1}{4}\right) \nonumber \\
			&\geqslant \frac{1}{2} \nonumber ~,
		\end{align}
		where the last inequality uses the assumption $\delta \leqslant 1/12$. Finally, we apply a data-processing inequality. In this regime which does not depend on $\delta$, we use Pinsker's inequality which implies that
		\begin{equation}\label{eq:H0_pinsker}
			\frac{1}{2}
			\leqslant \frac{1}{I-1}
			\sum_{k=2}^{I}
			\mathrm{TV}\bigl(\mathbb{P}_{\tilde{A}_k}^{(\pi,k)}, \mathbb{P}_{\tilde{A}_k}^{(\pi)}\bigr)
			\leqslant \frac{1}{I-1}
			\sum_{k=2}^{I}
			\sqrt{\frac{1}{2}\,\mathrm{KL}\bigl(\mathbb{P}_{\tilde{A}_k}^{(\pi,k)}, \mathbb{P}_{\tilde{A}_k}^{(\pi)}\bigr)} \enspace.
		\end{equation}
		
		\begin{remark}
			We can again interpret this result as a reduction argument. Averaging \eqref{eq:H0_tv1}, \eqref{eq:H0_tv2}, and \eqref{eq:H0_tv3} over \(\pi \sim \mathcal{U}\mathrm{nif}(\Pi(\bSigma))\), we obtain that there exists some \(k \in \{2,\dots,I\}\) (independent of \(\pi\)) such that \(\tilde{A}_k\) is \(1/2\)-correct, with budget at most \(4\chi/(I-1)\), for the active signal detection problem
			\[
			H_0: \mu = \boldsymbol{0}
			\quad \text{vs} \quad
			H_1: \mu = \bigl(\Delta_{k,\pi_k(\ell)}\bigr)_{\ell\in\Sigma_k^-}, \quad \pi_k \sim \mathcal{U}\mathrm{nif}(\Pi_k(\bSigma))\,.
			\]
		\end{remark}
		
		\noindent\underline{Step 3: computing the KL divergence.}
		We now conclude by computing the KL divergence above. Fix $k \in \{2,\dots,I\}$. Given $\pi'_k\in \Pi_k(\bSigma)$, we write $\pi'=(\pi_2,\ldots,\pi_{k-1}, \pi'_k,\pi_{k+1},\ldots, )$.
		Using the same computation as in Equation~\eqref{eq:H2_KL2}, we obtain
		\begin{align*}
			\frac{1}{|\Pi_k(\bSigma)|} \sum_{\pi_k \in \Pi_k(\bSigma)}
			\mathrm{KL}\bigl(\mathbb{P}_{\tilde{A}_k}^{(\pi,k)}, \mathbb{P}_{\tilde{A}_k}^{(\pi')}\bigr)
			&\leqslant
			8\log\!\left(\frac{4}{3}\right)
			\sum_{j \in \Sigma_k^- }
			\mathbb{E}_{\tilde{A}_k}^{(\pi,k)}[N_{\{k,j\}}] \,\beta_k^2  \\
			&\leqslant
			8\log\!\left(\frac{4}{3}\right)
			\cdot \frac{4\chi}{I-1} \cdot \beta_I^2  \enspace,
		\end{align*}
		where the last inequality uses the facts that
		\(\sum_{j \in \Sigma_k^- } \mathbb{E}_{\tilde{A}_k}^{(\pi,k)}[N_{\{k,j\}}]
		\leqslant 4\chi/(I-1)\) by construction of $\tilde{A}_k$, and that, by our ordering assumption, \(\beta_k^2  \leqslant \beta_I^2\) for all \(k \in \{2,\dots,I\}\).
		Averaging additionally over all permutations $\pi = (\pi_2,\dots,\pi_K) \in \Pi(\bSigma)$, we obtain
		\begin{equation}\label{eq:H0_KL2}
			\frac{1}{|\Pi(\bSigma)|} \sum_{\pi \in \Pi(\bSigma)}
			\frac{1}{I-1} \sum_{k=2}^{I}
			\mathrm{KL}\bigl(\mathbb{P}_{\tilde{A}_k}^{(\pi,k)}, \mathbb{P}_{\tilde{A}_k}^{(\pi)}\bigr)
			\leqslant
			8\log\!\left(\frac{4}{3}\right)
			\cdot \frac{4\chi}{I-1} \cdot \beta_I^2  \enspace.
		\end{equation}
		Finally, combining Pinsker's inequality~\eqref{eq:H0_pinsker} with Jensen's inequality, we get
		\begin{align*}
			\frac{1}{2}
			&\leqslant
			\frac{1}{|\Pi(\bSigma)|} \sum_{\pi \in \Pi(\bSigma)}
			\frac{1}{I-1} \sum_{k=2}^{I}
			\sqrt{\frac{1}{2}\,
				\mathrm{KL}\bigl(\mathbb{P}_{\tilde{A}_k}^{(\pi,k)}, \mathbb{P}_{\tilde{A}_k}^{(\pi)}\bigr)} \\
			&\leqslant
			\sqrt{
				\frac{1}{2}
				\cdot
				\frac{1}{|\Pi(\bSigma)|} \sum_{\pi \in \Pi(\bSigma)}
				\frac{1}{I-1} \sum_{k=2}^{I}
				\mathrm{KL}\bigl(\mathbb{P}_{\tilde{A}_k}^{(\pi,k)}, \mathbb{P}_{\tilde{A}_k}^{(\pi)}\bigr)
			} \\
			&\leqslant
			\sqrt{
				\frac{1}{2}
				\cdot
				8\log\!\left(\frac{4}{3}\right)
				\cdot \frac{4\chi}{I-1}
				\cdot \beta_I^2 
			} \enspace,
		\end{align*}
		where the last inequality follows from \eqref{eq:H0_KL2}. Rearranging yields
		\begin{equation}\label{eq:H0_KLfinal}
			\chi \;\geqslant\; \frac{1}{64\log(4/3)}  \frac{I-1}{\beta_I^2}
			=  \frac{1}{64\log(4/3)}\max_{i = 2,\dots,K} \frac{i-1}{\beta_i^2} 
			\enspace,
		\end{equation}
		where the second equality follows from the definition of $I$ (see~\eqref{eq:def_I}).
		From Lemma~\ref{lem:tech1}, we have
		\[
		\max_{i = 2,\dots,K} \frac{i-1}{\beta_i^2 }
		\geqslant \frac{1}{\log(2K)} \sum_{i = 2}^{K} \frac{1}{\beta_i^2 }
		= \frac{1}{\log(2K)} \sum_{i = 2}^{K} \frac{|\Sigma_i^-|}{\|\Delta_i^-\|^2}  ~.
		\]
		
		\medskip 
		
		\noindent\underline{Step 4: choice of convention $\bSigma$ and conclusion.}
		
		We claim that there exists a tie-breaking convention $f_{\mathrm{tb}}$ that satisfies
		\begin{equation}\label{eq:tie_breaking}
			\sum_{i= 2}^K \frac{|\Sigma_i^-|}{\|\Delta_i^-\|^2}
			\geqslant \frac{1}{2}\sum_{i = 2}^K \frac{K_{i;\leqslant 0}}{\|\Delta_i^-\|^2} \enspace.
		\end{equation}
		Then, the conclusion~\eqref{eq:lb_H0}  directly follows from  the~\eqref{eq:H0_KLfinal} together with~\eqref{eq:tie_breaking}. 
		
		We finally finish with a technical construction of a tie-breaking that satisfies~\eqref{eq:tie_breaking}.
		Consider any suboptimal arm $i \neq 1$. For any $j\ne i$, it holds that
		\[
		\bigl(j\in \Sigma_i^-\bigr)
		\iff
		\bigl(\Delta_{i,j} < 0\bigr)
		\text{ or }
		\bigl(\Delta_{i,j} = 0,\ f_{\mathrm{tb}}(i,j) = -1\bigr),
		\]
		so that
		\[
		|\Sigma_i^-|
		= \sum_{\substack{j=1 \\ j\neq i}}^K
		\mathds{1}_{\{\Delta_{i,j}<0\}}
		+ \mathds{1}_{\{\Delta_{i,j}=0\}}\mathds{1}_{\{f_{\mathrm{tb}}(i,j)=-1\}} ~.
		\]
		
		Summing over $i=2,\dots,K$ gives
		\begin{align}
			\sum_{i=2}^K \frac{|\Sigma_i^-|}{\|\Delta_i^-\|^2}
			&= \sum_{i=2}^K\sum_{\substack{j=1 \\ j\neq i}}^K
			\frac{1}{\|\Delta_i^-\|^2}
			\left( \mathds{1}_{\{\Delta_{i,j}<0\}}
			+ \mathds{1}_{\{\Delta_{i,j}=0\}}\mathds{1}_{\{f_{\mathrm{tb}}(i,j)=-1\}}
			\right) \nonumber\\
			&= \sum_{i=2}^K\sum_{\substack{j=1 \\ j\neq i}}^K
			\frac{1}{\|\Delta_i^-\|^2} \mathds{1}_{\{\Delta_{i,j}<0\}}
			+ \sum_{i=2}^K\frac{1}{\|\Delta_i^-\|^2}\mathds{1}_{\{\Delta_{i,1}=0\}}\mathds{1}_{\{f_{\mathrm{tb}}(i,1)=-1\}} \nonumber\\
			&\quad
			+ \sum_{2\leqslant i < j \leqslant K}
			\mathds{1}_{\{\Delta_{i,j}=0\}}
			\left(
			\frac{\mathds{1}_{\{f_{\mathrm{tb}}(i,j)=-1\}}}{\|\Delta_i^-\|^2}
			+ \frac{\mathds{1}_{\{f_{\mathrm{tb}}(i,j)=1\}}}{\|\Delta_j^-\|^2}
			\right) \label{eq:choice_f_1}\enspace,
		\end{align}
		where the last equality follows from a simple reorganization of the sum. 
		
		We now choose $f_{\mathrm{tb}}$ to make the second and third terms as large as possible. Let $f_{\mathrm{tb}}$ be the antisymmetric function $f_{\mathrm{tb}}:[K]^2 \to \{-1,0,1\}$, defined for any $i<j$ by
		\[
		f_{\mathrm{tb}}(i,j) \coloneqq
		\begin{cases}
			1, & \text{if } i=1,\ j>1, \\[2pt]
			\mathds{1}_{\{\|\Delta_{i}^-\|>\|\Delta_j^-\|\}}
			-\mathds{1}_{\{\|\Delta_{i}^-\|<\|\Delta_j^-\|\}}, & \text{if } 2\leq i<j \text{ and } \|\Delta_{i}^-\|\neq\|\Delta_j^-\|,\\[2pt]
			1 , & \text{if } 2\leq i<j \text{ and } \|\Delta_{i}^-\|=\|\Delta_j^-\| .
		\end{cases}
		\]
		The convention $f_{\mathrm{tb}}$ for the row of the CW implies that 
		\begin{equation}\label{eq:choice_f_2}
			\sum_{i=2}^K\frac{1}{\|\Delta_i^-\|^2}
			\mathds{1}_{\{\Delta_{i,1}=0\}}\mathds{1}_{\{f_{\mathrm{tb}}(i,1)=-1\}}
			= \sum_{i=2}^K\frac{1}{\|\Delta_i^-\|^2}\mathds{1}_{\{\Delta_{i,1}=0\}} \enspace. 
		\end{equation}
		Moreover, the expression of $f_{\mathrm{tb}}(i,j)$ for $2\leqslant i<j\leqslant K$ is chosen so that 
		\begin{align}
			\sum_{2\leqslant i <j \leqslant K}
			\mathds{1}_{\{\Delta_{i,j}=0\}}
			\left(
			\frac{\mathds{1}_{\{f_{\mathrm{tb}}(i,j)=-1\}}}{\|\Delta_i^-\|^2}
			+ \frac{\mathds{1}_{\{f_{\mathrm{tb}}(i,j)=1\}}}{\|\Delta_j^-\|^2}
			\right)
			&= \sum_{2\leqslant i <j \leqslant K}
			\mathds{1}_{\{\Delta_{i,j}=0\}}
			\left(
			\frac{1}{\|\Delta_i^-\|^2} \vee \frac{1}{\|\Delta_j^-\|^2}
			\right) \nonumber \\
			&\geqslant \frac{1}{2} \sum_{2\leqslant i <j \leqslant K}
			\mathds{1}_{\{\Delta_{i,j}=0\}}
			\left(
			\frac{1}{\|\Delta_i^-\|^2}
			+ \frac{1}{\|\Delta_j^-\|^2}
			\right) \label{eq:choice_f_3} ~.
		\end{align} 
		Finally, gathering Equations~\eqref{eq:choice_f_1}, \eqref{eq:choice_f_2}, and \eqref{eq:choice_f_3}, we get 
		\begin{align*}
			\sum_{i=2}^K \frac{|\Sigma_i^-|}{\|\Delta_i^-\|^2}
			&\geqslant 
			\sum_{i=2}^K\sum_{\substack{j=1 \\ j\neq i}}^K
			\frac{1}{\|\Delta_i^-\|^2}
			\left(
			\mathds{1}_{\{\Delta_{i,j}<0\}}
			+ \mathds{1}_{\{j=1,\ \Delta_{i,1}=0\}}
			+ \frac{1}{2}\mathds{1}_{\{j\ne 1,\ \Delta_{i,j}=0\}}
			\right) \\
			&\geqslant  \frac{1}{2}
			\sum_{i=2}^K \frac{K_{i;\leqslant 0}}{\|\Delta_i^-\|^2}
			\enspace,
		\end{align*}
		which proves Equation~\eqref{eq:tie_breaking}, hence finishing the proof.
	\end{proof}
	
	\subsection{Proof of Corollary~\ref{coro:Lower_bound_minimax}}\label{sec:proof_coro_minimax}
	
	\begin{proof}
		Without loss of generality, assume $i^*=1$. Let $\underline{s}=(s_i)_{i\neq1}$ with $s_i\in[K/8]$ for all $i\neq i^*$, and let $\underline{\Delta}=(\Delta_i)_{i\neq i^*}$ with $\Delta_i\in(0,1/4)$. Let $\epsilon>0$ be small.
		
		For simplicity, assume $K$ is a multiple of $8$, and set $d=K/2$. We construct the $K\times K$ antisymmetric matrix $M^\epsilon=M^\epsilon(\underline{s},\underline{\Delta})$:
		\begin{equation}\label{eq:def_minimax_matrix}
			M^\epsilon=
			\begin{pmatrix}
				A & -D \\
				D^\top & \Lambda
			\end{pmatrix},
		\end{equation}
		where $A$, $D$, and $\Lambda$ are $d\times d$ matrices defined below.
		
		The matrix $A$ is the $d\times d$ antisymmetric matrix with first row $A_{1,\cdot}=(0,\epsilon,\dots,\epsilon)\in\mathbb{R}^d$, and for $i=2,\dots,d$,
		\[
		A_{i,\cdot}=(-\epsilon,\dots,-\epsilon,\underbrace{0}_{\text{$i$-th}},\epsilon,\dots,\epsilon)\in\mathbb{R}^d.
		\]
		The matrix $D$ is the $d\times d$ matrix with nonnegative entries where $D_{1,\cdot}=(-\epsilon,\dots,-\epsilon)\in\mathbb{R}^d$, and for $i=2,\dots,d$,
		\[
		D_{i,\cdot}=(\underbrace{\Delta_i,\dots,\Delta_i}_{s_i\text{ times}},\epsilon,\dots,\epsilon)\in\mathbb{R}^d,
		\]
		which is possible since $s_i\leqslant d$.

		To construct $\Lambda$, recall $d$ is a multiple of $4$ and $s_i\in\{1,\dots,d/4\}$. Define $\Lambda$ as the block matrix:
		\[
		\Lambda=
		\begin{pmatrix}
			J_\epsilon & -\Lambda^{(0)} & \epsilon\mathbf{1} & (\Lambda^{(3)})^\top \\
			(\Lambda^{(0)})^\top & J_\epsilon & -\Lambda^{(1)} & \epsilon\mathbf{1} \\
			-\epsilon\mathbf{1} & (\Lambda^{(1)})^\top & J_\epsilon & -\Lambda^{(2)} \\
			-\Lambda^{(3)} & -\epsilon\mathbf{1} & (\Lambda^{(2)})^\top & J_\epsilon
		\end{pmatrix},
		\]
		where $\mathbf{1}$ is the $(d/4)\times(d/4)$ all-ones matrix, $J_\epsilon$ is the $(d/4)\times(d/4)$ antisymmetric matrix with $\epsilon$ above the diagonal, and for $l\in\{0,\dots,3\}$, $\Lambda^{(l)}$ is the $(d/4)\times(d/4)$ matrix where the $i$-th row is
		\[
		\Lambda^{(l)}_{i,\cdot}=(\underbrace{\Delta_j,\dots,\Delta_j}_{s_j\text{ times}},\epsilon,\dots,\epsilon)\in\mathbb{R}^{d/4}, \quad j=d+\tfrac{d}{4}l+i.
		\]

		The matrix $M^\epsilon=M(\underline{s},\underline{\Delta},\epsilon)$ is clearly antisymmetric. Its first row is $(0,\epsilon,\dots,\epsilon)$, so $i^*=1$ and $M^\epsilon\in\mathbb{D}_{\mathrm{cw}}$.
		
		For each arm $i=2,\dots,K$, row $i$ of $M^\epsilon$ contains exactly $s_i$ entries of magnitude $-\Delta_i$, with all other negative entries equal to $-\epsilon$. For sufficiently small $\epsilon>0$, the optimal sparsity $\bm{s}^*_{M^\epsilon}$ achieving the minimum in~\eqref{eq:sample_comp} equals $\underline{s}$, with associated gaps $(M^\epsilon_{i,(s_i)})_{i\neq i^*}=(-\Delta_i)_{i\neq i^*}$. Moreover, $M^\epsilon$ has no ties since $\Delta_i\neq0$ and $\epsilon>0$.
		
		Consider Corollary~\ref{coro:Lower_bound_minimax}. Let $\delta\leqslant 1/12$. For $\bDelta\in\mathbb{D}_{\mathrm{cw}}$, construction yields $\epsilon>0$ small such that $M^\epsilon(\bm{s}^*_{\bDelta},-\bDelta_{(s^*)})\in\mathbb{D}(\bDelta)$ with no ties.
		
		Theorem~\ref{thm:LB_hp_instance} applied to $M^\epsilon$ gives $\tilde{\bDelta}\in\mathbb{D}(M^\epsilon)=\mathbb{D}(\bDelta)$ satisfying
		\[
		\mathbb{P}_{\tilde{\bDelta},A}\!\left(
		N_{\delta}\;\geqslant\;\frac{1}{3}\,
		\max_{i\ne i^*}\frac{K^\epsilon_{i;<0}}{\|(M^\epsilon_i)^-\|^2}\log\!\left(\frac{1}{6\delta}\right)
		\vee\frac{1}{37\log(2K)}\sum_{i\ne i^*}\frac{K^\epsilon_{i;<0}}{\|(M^\epsilon_i)^-\|^2}
		\right)\;\geqslant\; \delta,
		\]
		where $K^\epsilon_{i;<0}$ counts negative entries of row $i$. By construction, $K^\epsilon_{i;<0}\geqslant K/8$ for $i=2,\dots,K$. For small $\epsilon$,
		\[
		s^*_i\Delta^2_{i,(s^*_i)}\leqslant\|(M^\epsilon_i)^-\|^2\leqslant s^*_i\Delta^2_{i,(s^*_i)}+(K-s^*_i)\epsilon^2,
		\leqslant 2s^*_i\Delta^2_{i,(s^*_i)}\]
		yielding
		\[
		\mathbb{P}_{\tilde{\bDelta},A}\!\left(
		N_{\delta}\;\geqslant\;\frac{1}{48}\,
		\max_{i\ne i^*}\frac{K}{s^*_i\Delta^2_{i,(s^*_i)}}\log\!\left(\frac{1}{6\delta}\right)
		\vee\frac{1}{592\log(2K)}\sum_{i\ne i^*}\frac{K}{s^*_i\Delta^2_{i,(s^*_i)}}
		\right)\;\geqslant\; \delta.
		\]
		This scales as $H_{\mathrm{explore}}(\bm{s}^*)$ up to $\log K$ factors, proving the first part of 
		Corollary~\ref{coro:Lower_bound_minimax}.
		
		The $H_{\mathrm{certify}}(\bm{s}^*,\delta)$ term follows from the quantile bound in 
		Proposition~\ref{thm:LB_exp_instance_first_regime}: by construction of $M^\epsilon$, we have 
		$\tilde{\Delta}_{i,(1)}=\tilde{\Delta}_{i,(s_i^*)}$ for all $i\neq i^*$, so
		\[
		H_{\mathrm{certify}}(\bm{s}^*,\delta)
		= \sum_{i\neq i^*}\frac{\log(1/\delta)}{\Delta^2_{i,(s_i^*)}}
		= \sum_{i\neq i^*}\frac{\log(1/\delta)}{\tilde{\Delta}^2_{i,(1)}},
		\]
		and the lower bound applies directly.
	\end{proof}
	
	\subsection{Lower Bounds Preserving CW Row Structure}\label{sec:additional_LB}
	Consider $\boldsymbol{\Delta} \in \mathbb{D}_{\mathrm{cw}}$. For simplicity, assume that $\boldsymbol{\Delta}$ has no ties.\footnote{If $\boldsymbol{\Delta}$ contains ties, we fix the convention $f_{\mathrm{tb}}\equiv 0$, i.e., we never permute zero entries, hence keeping them uninformative.} Let $\boldsymbol{\Sigma}$ be its sign matrix as in Equation~\eqref{def:sign_delta}, and let $\pi\in\Pi(\boldsymbol{\Sigma})$ (see~\eqref{def:Pi}) with associated matrix $\boldsymbol{\Delta}^{\pi}$ defined in Equation~\eqref{def:Delta_pi}. 
	
	By Lemma~\ref{lemma:Delta_pi}, $\boldsymbol{\Delta}^{\pi}$ has the same gap structure as $\boldsymbol{\Delta}$: it preserves all signs (hence all pairwise preferences) and gap magnitudes up to reordering. However, $\boldsymbol{\Delta}^{\pi}$ may alter the Condorcet winner row, so in general $H_{\mathrm{cw}}(\boldsymbol{\Delta}^{\pi})\neq H_{\mathrm{cw}}(\boldsymbol{\Delta})$, and the construction can even drastically increase it: $H_{\mathrm{cw}}(\boldsymbol{\Delta}^{\pi})\gg H_{\mathrm{cw}}(\boldsymbol{\Delta})$.
	
	In this section, we explain how the lower bound techniques from the proof of Theorem~\ref{thm:LB_hp_instance} can be adapted to also preserve the CW row.
	
	To this end, define $\tilde{\Pi}(\boldsymbol{\Sigma})\subset\Pi(\boldsymbol{\Sigma})$ as the subset of permutations preserving the CW row. For each $i\in\{2,\dots,K\}$, let $\tilde{\Pi}_i(\boldsymbol{\Sigma})$ be permutations of $\Sigma_i^-$ that fix $i^*=1$:
	\begin{equation}\label{def:Pitilde}
		\tilde{\Pi}_i(\boldsymbol{\Sigma})\coloneqq\bigl\{\pi_i:\Sigma_i^-\to\Sigma_i^- \,\big|\, \pi_i\text{ is a bijection and }\pi_i(1)=1\bigr\}\enspace,
	\end{equation}
	and set $\tilde{\Pi}(\boldsymbol{\Sigma})\coloneqq\tilde{\Pi}_2(\boldsymbol{\Sigma})\times\cdots\times\tilde{\Pi}_K(\boldsymbol{\Sigma})$. For $\pi\in\tilde{\Pi}(\boldsymbol{\Sigma})$, construct $\boldsymbol{\Delta}^{(\pi)}$ via Equation~\eqref{def:Delta_pi}. In addition to properties 1, 3, and 4 of Lemma~\ref{lemma:Delta_pi}, we have:
	
	\begin{lemma}
		For any $\pi\in\tilde{\Pi}(\boldsymbol{\Sigma})$, $\boldsymbol{\Delta}^{(\pi)}$ satisfies:
		\begin{enumerate}
			\item[2'] $\Delta^{(\pi)}_{1,\cdot}=\Delta_{1,\cdot}$ (CW row preservation).
		\end{enumerate}
	\end{lemma}
	
	We can then derive the following theorem, analogous to Theorem~\ref{thm:LB_hp_instance}:
	
	\begin{theorem}\label{thm:precise_LB_high_proba_tilde}
		Let $A$ be a $\delta$-correct algorithm over $\mathbb{D}_{\mathrm{cw}}$, with $\delta\leqslant1/12$. Let $\boldsymbol{\Delta}\in\mathbb{D}_{\mathrm{cw}}$ have no ties. Define
		\begin{equation}\label{def:chitilde}
			\tilde{\chi}\coloneqq\inf\left\{x>0:\sup_{\pi\in\tilde{\Pi}(\boldsymbol{\Sigma})}\mathbb{P}_{\boldsymbol{\Delta}^{(\pi)},A}(N_{\delta}>x)\leqslant\delta\right\}.
		\end{equation}
		Then,
		\begin{align}
			\tilde{\chi}&\geqslant\frac{1}{16\log(4/3)}\max_{i\neq i^*}\left(\frac{1}{\Delta_{i^*,i}^2}\wedge\frac{K_{i;<0}}{\|\Delta_i^-\|^2}\right)\log\!\left(\frac{1}{6\delta}\right),\label{eq:lb_H2_tilde}\\[2pt]
			\tilde{\chi}&\geqslant\frac{1}{128\log(4/3)}\frac{1}{\log(2K)}\sum_{i\neq i^*}\left(\frac{1}{\Delta_{i^*,i}^2}\wedge\frac{K_{i;<0}}{\|\Delta_i^-\|^2}\right).\label{eq:lb_H0_tilde}
		\end{align}
	\end{theorem}

	Similarly to Section~\ref{sec:lb_fc}, we define a subclass of $\mathbb{D}(\boldsymbol{\Delta})$ (see~\eqref{def:minimax_class}) that additionally preserves the CW row $\Delta_{i^*,\cdot}$:
	\begin{equation}\label{def:minimax_class_cw}
		\mathbb{D}_0(\boldsymbol{\Delta})=\{\tilde{\boldsymbol{\Delta}}\in\mathbb{D}(\boldsymbol{\Delta})\text{ s.t. }(\tilde{\Delta}_{i^*,i})_{i\neq i^*}=(\Delta_{i^*,i})_{i\neq i^*}\}\enspace.
	\end{equation}
	
	\begin{corollary}\label{coro:minimaxtilde}
		Let $A$ be a $\delta$-correct algorithm over $\mathbb{D}_{\mathrm{cw}}$, with $\delta\leqslant1/12$. Let $\boldsymbol{\Delta}\in\mathbb{D}_{\mathrm{cw}}$. Then there exists $\tilde{\boldsymbol{\Delta}}\in\mathbb{D}_0(\boldsymbol{\Delta})$ such that, with $\mathbb{P}_{\tilde{\boldsymbol{\Delta}},A}$-probability at least $\delta$, the budget $N_{\delta}$ satisfies
		\begin{equation}
			N_{\delta}\gtrsim\sum_{i\neq i^*}\frac{\log(1/\delta)}{\Delta_{i^*,i}^2\vee\Delta_{i,(s_i^*)}^2}
			+\max_{i\neq i^*}\frac{\log(1/\delta)}{\Delta_{i^*,i}^2\vee\frac{\|\Delta_{i}^{-}\|^2}{K_{i;<0}}}
			+\sum_{i\neq i^*}\frac{1}{\Delta_{i^*,i}^2\vee\frac{\|\Delta_{i}^{-}\|^2}{K_{i;<0}}},
		\end{equation}
		where $\gtrsim$ hides logarithmic $K$ factors and numerical constants.
	\end{corollary}
	
	\begin{proof}
		Assume without loss of generality that $i^*=1$. Construct $M$ by keeping the CW row and column unchanged and, for every $i,j\neq1$ with $\Delta_{i,j}<0$, setting
		\[ M_{i,j}=\max\{\Delta_{i,j},\Delta_{i,(s_i^*)}\},\qquad M_{j,i}=-M_{i,j}, \]
		while keeping zero entries equal to zero. For every $i\neq1$ and $1\leqslant r\leqslant K_{i;<0}$, 
		\[ 
		|M_{i,(r)}|\leqslant|\Delta_{i,(r)}|,\qquad M_{i,(s_i^*)}=\Delta_{i,(s_i^*)}. 
		\]
		Thus every admissible sparsity cost can only increase, with equality at $\bm{s}^*$. Hence $\bm{s}^*$ remains optimal and, with consistent tie-breaking, $M\in\mathbb{D}_0(\bDelta)$.
		
		Moreover, for every $i\neq1$,
		\[ 
		K_{i;<0}(M)=K_{i;<0},\qquad \|M_i^-\|^2\leqslant\|\Delta_i^-\|^2,\qquad M_{i,(1)}^2=\Delta_{1,i}^2\vee\Delta_{i,(s_i^*)}^2. 
		\]
		Apply the proof of Theorem~\ref{thm:precise_LB_high_proba_tilde} to $M$, leaving zero entries fixed when ties occur. The first two relations yield a hard permutation $\tilde{\bDelta}\in\mathbb{D}_0(\bDelta)$ satisfying the two exploration lower bounds. Since this permutation preserves the negative-entry multisets, the last relation and the quantile bound in Proposition~\ref{thm:LB_exp_instance_first_regime} give the certification lower bound on the same instance. Taking the maximum of these thresholds yields their sum up to a numerical constant.
	\end{proof}
	
	\begin{remark}
		This reveals the fundamental trade-off between eliminating suboptimal arms against the CW versus finding better competitors among them. 
		
		
		
	\end{remark}

	\begin{proof}[Proof of Theorem~\ref{thm:precise_LB_high_proba_tilde}]
		Assume without loss of generality that $i^*=1$. We start with the proof of Equation~\eqref{eq:lb_H2_tilde}, which follows the proof of Equation~\eqref{eq:lb_H2}. From careful inspection, Steps 1, 2, and 3 apply verbatim, replacing $\chi$ by $\tilde{\chi}$ and $\Pi$ by $\tilde{\Pi}$. 
		
		Fix $k\neq i^*$. With the same notation and construction, one constructs $\tilde{A}$ with budget upper bounded by $\tilde{\chi}$ such that
		\begin{align}
			\log\!\left(\frac{1}{6\delta}\right) &\leqslant 8\log\!\left(\frac{4}{3}\right)\frac{1}{|\tilde{\Pi}_k(\boldsymbol{\Sigma})|}\sum_{\pi_k\in\tilde{\Pi}_k(\boldsymbol{\Sigma})}\sum_{j\in\Sigma_k^-}\mathbb{E}_{\tilde{A}}^{(\pi,k)}[N_{\{k,j\}}]\Delta_{k,\pi_k(j)}^2\label{eq:KL_verage_pi_tilde},
		\end{align}
		with the only difference in the subsequent computation. 

        If $K_{k;<0}=1$, then $\Sigma_k^-=\{1\}$, so the preceding inequality gives
        \[ 
            \log\!\left(\frac{1}{6\delta}\right) \leqslant 8\log(4/3)\Delta_{k,1}^2 \widetilde\chi. 
        \]
        Since $\|\Delta_k^-\|^2=\Delta_{k,1}^2$, this yields a stronger bound than required for this arm. For the remainder of this calculation, assume $K_{k;<0}\geqslant2$.
        
		For $\pi\in\tilde{\Pi}_k$, we have $\pi_k(1)=1$ and $\pi_k|_{\Sigma_k^-\setminus\{1\}}$ is a bijection, so $|\tilde{\Pi}_k(\boldsymbol{\Sigma})|\simeq\mathfrak{S}_{k_i}$ where $k_i:=|\Sigma_k^-\setminus\{1\}|=K_{k;<0}-1$. Separating the role of $1$ and the rest of $\Sigma_k^-$ in~\eqref{eq:KL_verage_pi_tilde},
		\begin{align*}
			&\frac{1}{|\tilde{\Pi}_k(\boldsymbol{\Sigma})|}\sum_{\pi_k\in\tilde{\Pi}_k(\boldsymbol{\Sigma})}\sum_{j\in\Sigma_k^-}\mathbb{E}_{\tilde{A}}^{(\pi,k)}[N_{\{k,j\}}]\Delta_{k,\pi_k(j)}^2\\
			&=\mathbb{E}_{\tilde{A}}^{(\pi,k)}[N_{\{k,1\}}]\Delta_{k,1}^2+\frac{1}{|\mathfrak{S}_{k_i}|}\sum_{\pi_k\in\tilde{\Pi}_k(\boldsymbol{\Sigma})}\sum_{j\in\Sigma_k^-\setminus\{1\}}\mathbb{E}_{\tilde{A}}^{(\pi,k)}[N_{\{k,j\}}]\Delta_{k,\pi_k(j)}^2,
		\end{align*}
		where $\mathbb{E}_{\tilde{A}}^{(\pi,k)}$ is independent of $\pi_k$. By Lemma~\ref{lem:sum_tech},
		\begin{align*}
			\frac{1}{8\log(4/3)}\log\!\left(\frac{1}{6\delta}\right)
			&\leqslant\mathbb{E}_{\tilde{A}}^{(\pi,k)}[N_{\{k,1\}}]\Delta_{k,1}^2+\frac{1}{K_{k;<0}-1}\sum_{j\in\Sigma_k^-\setminus\{1\}}\mathbb{E}_{\tilde{A}}^{(\pi,k)}[N_{\{k,j\}}]\sum_{j\in\Sigma_k^-\setminus\{1\}}\Delta_{k,j}^2\\
			&\leqslant\Delta_{k,1}^2\tilde{\chi}+\frac{\|\Delta_k^-\|^2-\Delta_{k,1}^2}{K_{k;<0}-1}\tilde{\chi},
		\end{align*}
		using $\sum_{j\in\Sigma_k^-\setminus\{1\}}\mathbb{E}_{\tilde{A}}^{(\pi,k)}[N_{\{k,j\}}]\leqslant\tilde{\chi}$ and $\sum_{j\in\Sigma_k^-\setminus\{1\}}\Delta_{k,j}^2=\|\Delta_k^-\|^2-\Delta_{k,1}^2$. Finally,
		\[
		\Delta_{k,1}^2+\frac{\|\Delta_k^-\|^2-\Delta_{k,1}^2}{K_{k;<0}-1}\leqslant2\left(\Delta_{k,1}^2\vee\frac{\|\Delta_k^-\|^2}{K_{k;<0}}\right).
		\]
		Taking the maximum over $k\neq i^*$ yields
		\[
		\tilde{\chi}\geqslant\frac{1}{16\log(4/3)}\max_{k\neq i^*}\frac{1}{\Delta_{k,1}^2\vee\frac{\|\Delta_k^-\|^2}{K_{k;<0}}}\log\!\left(\frac{1}{6\delta}\right),
		\]
		which is Equation~\eqref{eq:lb_H2_tilde}.
		
		The proof of~\eqref{eq:lb_H0_tilde} follows the proof of~\eqref{eq:lb_H0} step-by-step, highlighting differences below.
		
		\underline{Step 1}: Define $\tilde{\beta}_k=\Delta_{k,1}^2\vee\frac{\|\Delta_k^-\|^2}{K_{k;<0}}$ and $\tilde{I}:=\argmax_{k=2}^K\frac{k-1}{\tilde{\beta}_k}$.
		
		\underline{Step 2}: The same event (using $\tilde{I}$) bounds the probabilities, so~\eqref{eq:H0_pinsker} holds.
		
		\underline{Step 3}: For $K_{k;<0} \geqslant 2$, the KL upper bound computation adapts as above. If $K_{k;<0}=1$, then $\widetilde\beta_k = \Delta_{k,1}^2$, and the direct KL calculation gives
        \[ 
            \mathrm{KL}\bigl(\mathbb{P}_{\widetilde A_k}^{(\pi,k)},\mathbb{P}_{\widetilde A_k}^{(\pi)}\bigr) \leqslant 8\log(4/3)\Delta_{k,1}^2\mathbb{E}_{\widetilde A_k}^{(\pi,k)}[N_{\{k,1\}}]\leqslant 8\log(4/3)\frac{4\widetilde\chi}{\widetilde I-1}\widetilde\beta_k. 
        \]
        Combining both cases and averaging yields
		\begin{equation}
			\frac{1}{|\tilde{\Pi}|}\sum_{\pi\in\tilde{\Pi}}\frac{1}{I-1}\sum_{k=2}^{I}\mathrm{KL}\bigl(\mathbb{P}_{\tilde{A}_k}^{(\pi,k)},\mathbb{P}_{\tilde{A}_k}^{(\pi)}\bigr)
			\leqslant 16\log\!\Bigl(\frac{4}{3}\Bigr)\cdot\frac{4\tilde{\chi}}{I-1}\cdot\tilde{\beta}_I.
		\end{equation}
		which conclude from rearranging. 
		
		Step 4 is unnecessary since $\bSigma$ avoids ties by convention.
		
	\end{proof}

	\section{Proofs of Section~\ref{sec:lb_fb}}\label{sec:proof_LB_FB}
	
	In this section, we provide all lower bound proofs for the fixed-budget setting. As discussed at the end of Appendix~\ref{sec:lb_fb}, these proofs closely parallel those in Appendix~\ref{sec:proof_LB_FC}, but the fixed-budget nature requires new arguments. 
	
	For completeness, we revisit all results with minimax-style formulations and provide a nearly self-contained presentation. In Subsection~\ref{sec:proof_LB_exp_FB}, we prove Theorem~\ref{thm:LB_exp_minimax_first_regime}. Theorem~\ref{thm:LB_minimax_H2} follows in Subsection~\ref{sec:proof_LB_HP_FB}.
	
	\paragraph{Roadmap for Fixed-Budget Lower Bounds}
	
	The proofs in this section follow the same three-step change-of-measure pattern introduced for the 
	fixed-confidence case in Subsection~\ref{sec:roadmap_LB}. However, the fixed-budget setting requires 
	three important adaptations, which we now describe in detail. In the fixed-budget setting, we aim to 
	lower bound the worst-case error 
	\[
	\inf_A \sup_{M\in\mathbb{D}} \mathbb{P}_{A,M}(\hat{i}_T \neq i^*),
	\]
	where the infimum is over all algorithms $A$ with fixed budget $T$, and $\mathbb{D}$ is some class of instances.
	
	\textbf{Step 1: Reference and alternative instances.} In the fixed-confidence setting, the reference 
	instance can be any arbitrary gap matrix $\bDelta\in\mathbb{D}_{\mathrm{cw}}$. Here, we instead construct 
	a highly symmetric reference matrix $M\in\mathbb{D}$ that serves as a ``hard instance'' for the minimax 
	bound. For each suboptimal arm $k\neq i^*$, we construct an alternative instance $M^{(k)}$ by setting 
	all negative entries in row $k$ (and corresponding column entries to preserve antisymmetry) to zero. 
	This ensures $M^{(k)}\notin\mathbb{D}_{\mathrm{cw}}$.
	
	\textbf{Step 2: Separating event and total variation bound.} Unlike fixed-confidence algorithms, which 
	use a stopping time $N_{\delta}$ to construct events on which the two distributions disagree, fixed-budget 
	algorithms have no stopping rule. Instead, we exploit the symmetry of the reference matrix $M$, together with the recommendation rule.
	
	\textbf{Step 3: KL decomposition.} This step is conceptually identical to the fixed-confidence proofs. 
	The KL divergence decomposes as 
	\[
	\mathrm{KL}(\mathbb{P}_M,\mathbb{P}_{M^{(k)}}) = \sum_{ k< j}\mathbb{E}_M[N_{\{k,j\}}]\,\mathrm{kl}(\Delta_{k,j},\Delta_{k,j}^{(k)}),
	\]
	where $N_{\{k,j\}}$ counts unordered duels between arms $k$ and $j$. The key difference is that the 
	fixed budget $T$ directly bounds $\sum_{k<j}\mathbb{E}_M[N_{\{k,j\}}]\leqslant T$, whereas fixed-confidence 
	proofs require truncation at the $(1-\delta)$-quantile $\chi$ and reinterpretation. Combining the total 
	variation lower bound from Step 2 with this KL upper bound from Step 3 yields the desired lower bound 
	on $\epsilon_T$.

	\subsection{Proof of Theorem~\ref{thm:LB_exp_minimax_first_regime} }\label{sec:proof_LB_exp_FB}
	
	Let $\underline{\Delta}=(\Delta_i)_{i\in[K]}$ be a $K$-dimensional vector such that $\Delta_1=0$ and $\Delta_i\in (0, 1/4)$ for $i\ne 1$, and assume without loss of generality that $\Delta_2 \leqslant \Delta_3\leqslant \dots \leqslant \Delta_K$. Recall that $\mathbb{D}^{(1)}(\underline{\Delta})$ (see Definition~\ref{def:D1})  denotes the class of gap matrices that admit a Condorcet winner whose row is equal to $\underline{\Delta}$ up to permutation.
	
	Fix an algorithm $A$ with a fixed budget $T\in \mathbb{N}^*$. Define the worst-case error probability of $A$ over the class $\mathbb{D}^{(1)}(\underline{\Delta})$ as
	\begin{equation}\label{def:epsilon_T_H1}
		\epsilon_T \coloneqq \max_{M \in \mathbb{D}^{(1)}(\underline{\Delta})}
		\mathbb{P}_{M, A}\!\left( \hat{i}_T \neq i^*(M) \right),
	\end{equation}
	where $\hat{i}_T$ denotes the recommendation of algorithm $A$ after $T$ queries.
	
	We aim to prove Theorem~\ref{thm:LB_exp_minimax_first_regime}, namely that
	\[
	\epsilon_T \;\ge\; \frac{1}{4}\exp\!\left(
	-\frac{T}{\tfrac{1}{22}\sum_{i=2}^{K} \tfrac{1}{\Delta_i^2}}
	\right).
	\]
	
	\paragraph{Sketch of proof.}
	
	The proof follows the three-step roadmap of Subsection~\ref{sec:roadmap_LB}.
	\emph{(i) Reference and perturbed instances $M^{(1)},M^{(k)}$.} Since we seek a minimax bound, we 
	can construct a specific reference matrix $M^{(1)}\in\mathbb{D}^{(1)}(\underline{\Delta})$ with Condorcet 
	winner $i^*=1$, where arm $1$ is the strongest opponent of every suboptimal arm. For each 
	suboptimal $k\neq 1$, $M^{(k)}$ modifies row $k$'s negative entries of $M^{(1)}$ to make $k$ the CW 
	while preserving $M^{(k)}\in\mathbb{D}^{(1)}(\underline{\Delta})$. \emph{(ii) Separating event and TV bound.} Algorithm $A$ is $\epsilon_T$-correct over 
	$\mathbb{D}^{(1)}(\underline{\Delta})$, so $\mathbb{P}_{M^{(k)}}(\hat{i}_T=k)\ge 1-\epsilon_T$. The 
	event $B_k=\{\hat{i}_T\neq k\}$ thus satisfies $\mathbb{P}_{M^{(1)}}(B_k)\ge 1-\epsilon_T$ and 
	$\mathbb{P}_{M^{(k)}}(B_k)\le\epsilon_T$, yielding 
	$\mathrm{TV}(\mathbb{P}_{M^{(1)}},\mathbb{P}_{M^{(k)}})\ge 1-2\epsilon_T$.
	\emph{(iii) KL decomposition and computation.}  Similar computation as in the proof of Theorem~\ref{thm:LB_exp_instance_first_regime}.

	\begin{proof}[Proof of Theorem~\ref{thm:LB_exp_minimax_first_regime}]
		\noindent\underline{Step 1: construction of the reference instance $M^{(1)}$.}
		
		\noindent We construct within the class $\mathbb{D}^{(1)}(\underline{\Delta})$ a highly structured instance. The key structural property of this matrix is that the Condorcet winner $k^* = 1$ is the strongest opponent of every other arm.  
		
		Define a gap matrix $M^{(1)} \in \mathbb{D}_{\mathrm{cw}}$ by setting $M^{(1)}_{1,\cdot} =  \underline{\Delta}$, and, for $i, j \neq 1$,
		\[
		M^{(1)}_{i,j} \coloneqq
		\begin{cases}
			\Delta_j, & \text{if } i < j,\\[2pt]
			-\Delta_i, & \text{if } i > j,\\[2pt]
			0, & \text{if } i=j.
		\end{cases}
		\]
		Thus $M^{(1)}$ has the form
		\[
		M^{(1)} = 
		\begin{pmatrix}
			0        & \Delta_2 & \Delta_3 & \Delta_4 & \cdots   & \Delta_{K-1} & \Delta_K \\[2mm]
			-\Delta_2 & 0       & \Delta_3 & \Delta_4 & \cdots   & \Delta_{K-1} & \Delta_K \\[2mm]
			-\Delta_3 & -\Delta_3 & 0      & \Delta_4 & \cdots   & \Delta_{K-1} & \Delta_K \\[2mm]
			-\Delta_4 & -\Delta_4 & -\Delta_4 & 0     & \cdots   & \Delta_{K-1} & \Delta_K \\[2mm]
			\vdots    & \vdots    & \vdots    & \vdots & \ddots   & \vdots       & \vdots   \\[2mm]
			-\Delta_{K-1} & -\Delta_{K-1} & -\Delta_{K-1} & -\Delta_{K-1} & \cdots & 0      & \Delta_K \\[2mm]
			-\Delta_K & -\Delta_K & -\Delta_K & -\Delta_K & \cdots & -\Delta_K   & 0
		\end{pmatrix}.
		\]
		
		By construction, we have $M^{(1)} \in \mathbb{D}^{(1)}(\underline{\Delta})$, and its Condorcet winner is $i^*(M^{(1)})=1$.
		
		For each $k \geqslant 2$, define the matrix $M^{(k)}$ as follows:
		\begin{itemize}
			\item For $i, j \neq k$, set $M_{i,j}^{(k)} = M^{(1)}_{i,j}$.
			\item For $j < k$, set $M^{(k)}_{k,j} = \Delta_{j+1}$ and $M_{j,k}^{(k)} = -\Delta_{j+1}$.
			\item For $j \geqslant k$, set $M^{(k)}_{k,j} = M^{(1)}_{k,j}$ and $M^{(k)}_{j,k} = M^{(1)}_{j,k}$.
		\end{itemize}
		
		The matrix $M^{(k)}$ can be written as
		\[
		M^{(k)} =
		\begin{pmatrix}
			0 & \Delta_2 & \cdots & \Delta_{k-1} & \textcolor{blue}{-\Delta_2} & \Delta_{k+1} & \cdots & \Delta_K \\[2mm]
			-\Delta_2 & 0 & \cdots & \Delta_{k-1} & \textcolor{blue}{-\Delta_3} & \Delta_{k+1} & \cdots & \Delta_K \\[2mm]
			\vdots & \vdots & \ddots & \vdots & \textcolor{blue}{\vdots} & \vdots & & \vdots \\[2mm]
			-\Delta_{k-1} & -\Delta_{k-1} & \cdots & 0 & \textcolor{blue}{-\Delta_{k}} & \Delta_{k+1} & \cdots & \Delta_K \\[2mm]
			\textcolor{blue}{\Delta_2} & \textcolor{blue}{\Delta_3} & \cdots & \textcolor{blue}{\Delta_{k}} & 0 & \Delta_{k+1} & \cdots & \Delta_K \\[2mm]
			-\Delta_{k+1} & -\Delta_{k+1} & \cdots & -\Delta_{k+1} & -\Delta_{k+1} & 0 & \cdots & \Delta_K \\[2mm]
			\vdots & \vdots & & \vdots & \vdots & \vdots & \ddots & \vdots \\[2mm]
			-\Delta_K & -\Delta_K & \cdots & -\Delta_K & -\Delta_{K} & -\Delta_K & \cdots & 0
		\end{pmatrix} ~,
		\]
		where the blue entries indicate the differences with respect to $M^{(1)}$.
		
		\noindent It is straightforward to check that, for each $k$, $M^{(k)} \in \mathbb{D}_{\mathrm{cw}}$. These matrices have three key properties: 
		(i) $M^{(k)}$ does not have the same Condorcet winner as $M^{(1)}$, indeed $i^*(M^{(k)})=k$; 
		(ii) we have $M^{(k)}\in \mathbb{D}^{(1)}(\underline{\Delta})$, indeed the $k$-th row $M^{(k)}_{k,\cdot}$ is equal to $\underline{\Delta}$ up to a permutation; and 
		(iii) the environment with gap matrix $M^{(k)}$ is difficult to distinguish from the one defined by $M^{(1)}$ in terms of KL divergence.
		
		\noindent For $k \geqslant 2$, denote by $\mathbb{P}^{(k)}$ the distribution of the data when the underlying gap matrix is $M^{(k)}$.
		
		\noindent\underline{Step 2: TV bound.}
		
		\noindent Let $A$ be a $\delta$-correct algorithm over $\mathbb{D}^{(1)}(\underline{\Delta})$, and let $\hat{i}$ denote its output. For any $k \geqslant 1$, when the true gap matrix is $M^{(k)}$ the Condorcet winner is $k$, and $M^{(k)}\in \mathbb{D}^{(1)}(\underline{\Delta})$. Then, the definition of $\epsilon_T$ (see~\eqref{def:epsilon_T_H1}) implies
		\[
		\forall k \in [K],\quad
		\mathbb{P}^{(k)}(\hat{i}\neq k) \leqslant \epsilon_T~.
		\]
		In particular, we have 
		\[
		1 - 2\epsilon_T
		\leqslant \mathbb{P}^{(1)}(\hat{i}\ne k)- \mathbb{P}^{(k)}(\hat{i} \ne k) \leqslant\mathrm{TV}\bigl(\mathbb{P}^{(1)},\mathbb{P}^{(k)}\bigr)~.
		\]
		Then, with Bretagnolle–Huber inequality, we have
		\[
		1 - 2\epsilon_T
		\leqslant \mathrm{TV}\bigl(\mathbb{P}^{(1)},\mathbb{P}^{(k)}\bigr)
		\leqslant 1 - \tfrac{1}{2}\exp\bigl\{-\mathrm{KL}(\mathbb{P}^{(1)},\mathbb{P}^{(k)})\bigr\}.
		\]
		
		\noindent\underline{Step 3: computing the KL divergence and concluding.}
		
		For $i<j$ in $[K]$, let $N_{\{i,j\}}$ denote the total number of observed duels between $i$ and $j$ under algorithm $A$. Using the divergence decomposition lemma (Lemma 15.1 in \citealp{lattimore2020bandit}), we have
		\begin{align}
			\mathrm{KL}\bigl(\mathbb{P}^{(1)}, \mathbb{P}^{(k)}\bigr)
			&= \sum_{1 \leqslant i<j \leqslant K} \mathbb{E}^{(1)}\!\left[N_{\{i,j\}}\right]\,
			\mathrm{KL}\bigl(\mathbb{P}_{i,j}^{(1)}, \mathbb{P}_{i,j}^{(k)}\bigr)  \nonumber\\
			&= \sum_{i=1}^{k-1} \mathbb{E}^{(1)}[N_{\{k,i\}}]\,
			\mathrm{KL}\bigl(\mathbb{P}_{k,i}^{(1)},\mathbb{P}_{k,i}^{(k)}\bigr)
			~, \label{eq:tv2}
		\end{align}
		since the two instances differ only on pairs involving arm $k$.
		
		For all $i<k$, the corresponding Bernoulli feedback distributions satisfy
		\[
		\mathbb{P}_{k,i}^{(1)} = \mathcal{B}\!\left(\tfrac{1}{2} -\Delta_k\right),
		\qquad
		\mathbb{P}_{k,i}^{(k)} = \mathcal{B}\!\left(\tfrac{1}{2}  +\Delta_{i+1}\right),
		\]
		so that, with the same sequence of inequalities as in \eqref{eq:kl1}, \eqref{eq:kl2}, we obtain
		\begin{align*}
			\mathrm{KL}\bigl(\mathbb{P}_{k,i}^{(1)}, \mathbb{P}_{k,i}^{(k)}\bigr)
			&= \mathrm{kl}\bigl(-\Delta_k+\tfrac12,\ \Delta_{i+1}+\tfrac12\bigr) \\
			&\leqslant \frac{\bigl(-\Delta_k+\tfrac12 - (\Delta_{i+1}+\tfrac12)\bigr)^2}
			{(-\Delta_{i+1}+\tfrac12)(\Delta_{i+1}+\tfrac12)} \\
			&\leqslant \frac{(2\Delta_k)^2}{3/16}
			\;\le\; 22\,\Delta_k^2~,
		\end{align*}
		where we used that $(\Delta_j)_{j \in \{2, \dots, K\}}$ is nondecreasing and $k \geqslant i+1$, which yields $\Delta_k \geqslant \Delta_{i+1}$. Also, we use that $\Delta_{i+1}\in(0,1/4)$ and thus the denominator is bounded below by $3/16$.
		
		Now, from \eqref{eq:tv2} and the last inequality, we obtain
		\[
		\mathrm{KL}\bigl(\mathbb{P}^{(1)}, \mathbb{P}^{(k)}\bigr)
		\leqslant 22\sum_{i=1}^{k-1} \mathbb{E}^{(1)}\!\left[N_{\{k,i\}}\right] \Delta_k^2~.
		\]
		From Step~3, this bound implies
		\[
		\sum_{i=1}^{k-1} \mathbb{E}^{(1)}\!\left[N_{\{k,i\}}\right]
		\geqslant \frac{1}{22 \Delta_k^2}\,\log\!\frac{1}{4\epsilon_T}~.
		\]
		Summing over $k \geqslant 2$, and using that $A$ has a budget $T$, we conclude that 
		\begin{align*}
			T
			&\geqslant \sum_{k = 2}^K \sum_{i =1}^{k-1} \mathbb{E}^{(1)} \left[N_{\{k,i\}}\right] \geqslant \frac{1}{22}\sum_{k = 2}^K \frac{1}{\Delta_k^2} \log\!\frac{1}{4\epsilon_T}~, 
		\end{align*}	
		which, after rearranging, is exactly the claimed lower bound.
	\end{proof}
	
	\subsection{Proof of Theorem~\ref{thm:LB_minimax_H2}}\label{sec:proof_LB_HP_FB}
	
	Before proving Theorem~\ref{thm:LB_minimax_H2}, we extend the notion of Condorcet winner to a broader class of matrices. 
	Let $\bDelta$ be an antisymmetric matrix. We denote as a  weak Condorcet winner, an arm  $i^*\in[K]$ such that
	\begin{equation}\label{def:general_cw}
		\forall{i\ne i^*}, \Delta_{i^*,i} \geqslant 0 \text{ and, } \forall{i\ne i^*}, \min_{j\ne i} \Delta_{i,j} < 0 \enspace.
	\end{equation}
	Note that such arm is unique if it exists. 
	Consider $\mathbb{D}_{\mathrm{wcw}}$ as the class of dueling bandit environments for which there exists such weak Condorcet winner
	\begin{equation}\label{eq:def_acw_exists}
		\mathbb{D}_{\mathrm{wcw}}
		:= \Bigl\{
		\bDelta \in [-\tfrac{1}{4},\tfrac{1}{4}]^{K \times K} :
		\exists!\, i^* \in [K]~\text{such that}~ \forall j,~ \Delta_{i^*,j} \geqslant 0
		\Bigr\}.
	\end{equation}
	
	Observe that $\mathbb{D}_{\mathrm{cw}}\subset\mathbb{D}_{\mathrm{wcw}}$, and that naturally, any Condorcet winner is a weak Condorcet winner. Moreover, for a matrix $\bDelta\in\mathbb{D}_{\mathrm{wcw}}$, the quantities $(\bm{s}_{\bDelta}^*, \Delta_{(\bm s^*)})$ defined in Section~\ref{sec:lb_fc} are still well defined. Indeed, the minimum in Equation~\eqref{eq:sample_comp} for which $\bm{s}_{\bDelta}^*$ is defined as the argmin is still well defined under weak Condorcet winner assumption.
	
	As a variant of the class $\mathbb{D}^{(3)}(\underline{\Delta},\underline{s})$ introduced in~\eqref{eq:definition:D2}, we introduce the class:
	\begin{equation}\label{def:D2}
		\mathbb{D}^{(3)}(\underline{\Delta},\underline{s}) := \left\{\bDelta\in\mathbb{D}_{\mathrm{wcw}}:\exists\sigma\in\mathfrak{S}_K\text{ s.t. }\bm{s}_{\bDelta}^{*}=\sigma(\underline{s})\text{ and }\bDelta_{(s^*)}=-\sigma(\underline{\Delta})\right\}.
	\end{equation}
	
	The introduction of this definition is motivated by the fact that we want to consider in the lower bounds, matrices where the Condorcet winner may have some ties. The following lemma implies that this is feasible. 
	
	\begin{lemma}\label{lem:reduction_weack_cw}
		Let $A$ be an algorithm with a fixed budget $T$. Then, we have 
		\[\sup_{\bDelta \in \mathbb{D}^{(3)}(\underline{\Delta}, \underline{s})}
		\mathbb{P}_{A,\bDelta}(\hat{i}_T\ne i^*(\bDelta))= \sup_{\bDelta \in \mathbb{D}^{(2)}(\underline{\Delta}, \underline{s})}
		\mathbb{P}_{A,\bDelta}(\hat{i}_T\ne i^*(\bDelta)) \]
	\end{lemma}
	
	\begin{proof}[Proof of Lemma~\ref{lem:reduction_weack_cw}]
		Let $\bDelta\in \mathbb{D}^{(3)}(\underline{\Delta}, \underline{s})$, with a weak Condorcet winner $i^*$. By definition, for any $i\ne i^*$, $\Delta_{i,\cdot}$ contains at least one negative entry. Let $\epsilon>0$. Consider the modified matrix $\bDelta^{\epsilon}$ obtained from $\bDelta$ by lifting to $\epsilon>0$ all off-diagonal null entries of on CW row $i^*$ [and $-\epsilon$ to $(j,i^*)$], so that, for $\epsilon$ small enough, $\bDelta^{\epsilon}$ admits $i^*$ as a (strong) Condorcet winner, and $\bDelta^{\epsilon}\in \mathbb{D}_{\mathrm{cw}}$. Now, for $\epsilon$ small enough, it holds that $\bm{s}_{\bDelta^{\epsilon}}^*=\bm{s}_{\bDelta}^*$ and $\Delta^{\epsilon}_{\bm (s^*)}=\Delta_{\bm (s^*)}$.  Then, 
		\[ 
		\mathbb{P}_{A,\bDelta^{\epsilon}}(\hat{i}_T\ne i^*(\bDelta)) \leqslant  \sup_{\bDelta \in \mathbb{D}^{(2)}(\underline{\Delta}, \underline{s})}
		\mathbb{P}_{A,\bDelta}(\hat{i}_T\ne i^*(\bDelta)).
		\]
		Since $A$ has a fixed budget $T$,  one can take the limit $\epsilon\to 0$ in the inequality above. Taking a maximum over $\bDelta\in \mathbb{D}^{(3)}(\underline{\Delta}, \underline{s})$, one therefore obtains
		\[\sup_{\bDelta \in \mathbb{D}^{(3)}(\underline{\Delta}, \underline{s})}
		\mathbb{P}_{A,\bDelta}(\hat{i}_T\ne i^*(\bDelta)) \leqslant \sup_{\bDelta \in \mathbb{D}^{(2)}(\underline{\Delta}, \underline{s})}
		\mathbb{P}_{A,\bDelta}(\hat{i}_T\ne i^*(\bDelta))\ . \]
		
		We have $ \mathbb{D}^{(2)}(\underline{\Delta}, \underline{s})\subset \mathbb{D}^{(3)}(\underline{\Delta}, \underline{s})$, so the other side of the inequality is clear. 
	\end{proof}
	
	Now, we are ready to prove Theorem~\ref{thm:LB_minimax_H2}. Assume that $K$ is a multiple of $8$, and denote $d=K/2$.
	
	Let $(\underline{\Delta}, \underline{s})$ be such that $\underline{\Delta}=(\Delta_i)_{i\in[K]}$ with $\Delta_1=0$ and $(\Delta_2, \dots, \Delta_K) \in (0, 1/4)^{K-1}$. Let $\underline{s}=(s_1,s_2, \dots, s_K)$ with $s_1=0$, $1\leqslant s_i \leqslant K/8$ for $i=2,\dots,K$.

	Consider the class  $\mathbb{D}^{(3)}(\underline{\Delta},\underline{s})$ as defined in Equation~\eqref{def:D2}. 
	Fix an algorithm $A$ with a fixed budget $T$, and define the maximum error of $A$ across $\mathbb{D}^{(3)}(\underline{\Delta}, \underline{s})$ as 
	\begin{equation}\label{def:epsilon_T_D2}
		\epsilon_T \coloneqq \sup_{\bDelta \in \mathbb{D}^{(3)}(\underline{\Delta}, \underline{s})}
		\mathbb{P}_{A,\bDelta}(\hat{i}_T\ne i^*(\bDelta))\enspace. 
	\end{equation}
	
	The proof of Theorem~\ref{thm:LB_minimax_H2} is divided in three lemmas, corresponding to the three terms in the lower bound. 
	
	\begin{lemma}\label{lem:LB_FB_H2}
		We have
		\begin{equation}\label{eq:LB_FB_H2}
			\epsilon_T
			\geqslant \frac{1}{4}\exp\!\left(-16\log(4/3)\,\frac{T}{\max_{i=2}^{d} \frac{K}{s_i\Delta_i^2} }\right) \enspace. 
		\end{equation}
	\end{lemma}
	
	\begin{lemma}\label{lem:LB_FB_H0}
		If $(\underline{\Delta}, \underline{s})$ are constants on the indices $i\in\{2,\dots,d\}$, that is, $\exists (\mu,s)$ such that, for any $i\in\{2,\dots,d\}$, $\Delta_i=\mu$, and $s_i=s$, then 
		\begin{equation}\label{eq:LB_FB_H0}
			\epsilon_T
			\geqslant \frac{1}{2}-\sqrt{128\log(4/3) \frac{T}{\frac{K^2}{s\mu^2}}} \enspace.   
		\end{equation} 
	\end{lemma}
	
	\begin{lemma}\label{lem:LB_FB_H3}
		With the same assumption as Lemma~\ref{lem:LB_FB_H0}, then 
		\begin{equation}\label{eq:LB_FB_H3}
			\epsilon_T
			\geqslant \frac{1}{4}\exp\!\left(-32\log(4/3)\,\frac{T}{\tfrac{K}{\mu^2} }\right) \enspace. 
		\end{equation} 
	\end{lemma}

	\begin{proof}[Proof of Theorem~\ref{thm:LB_minimax_H2}]
		Recall that $\epsilon_T$ (see~\eqref{def:epsilon_T_D2}) is the worst-case error over $\mathbb{D}^{(3)}(\underline{\Delta}, \underline{s})$. From Lemma~\ref{lem:reduction_weack_cw}, it is also equal to the worst-case over $\mathbb{D}^{(2)}(\underline{\Delta}, \underline{s})$. Together,  Lemmas~\ref{lem:LB_FB_H2},~\ref{lem:LB_FB_H0}, and~\ref{lem:LB_FB_H3} directly imply Theorem~\ref{thm:LB_minimax_H2}.
	\end{proof}
	
	\begin{proof}[Proof of Lemma~\ref{lem:LB_FB_H2}]
		
		Recalling the definition of $\epsilon_T$ from \eqref{def:epsilon_T_D2}, $\epsilon_T = \sup_{\bDelta \in \mathbb{D}^{(3)}(\underline{\Delta}, \underline{s})}
		\mathbb{P}_{A,\bDelta}(\hat{i}_T\ne i^*(\bDelta))$, we want to prove the following bound, equivalent to~\eqref{eq:LB_FB_H2}
		\[
		T \geqslant \frac{1}{16\log(4/3)} \max_{i=2}^{d} \frac{K}{s_i\Delta_i^2} \log\!\left(\frac{1}{4\epsilon_T}\right) \enspace.
		\]
		We may assume that $\epsilon_T<1/4$, since otherwise the bound is immediate.
		
		\noindent\underline{Step 1: reference matrix $M^{\pi}$.}
		For reference, we consider the same matrix as in Subsection~\ref{sec:proof_coro_minimax}, taking $\epsilon=0$. For completeness, we recall this construction here.
		
		We assumed for simplicity that $K$ is a multiple of $8$, and denote $d=K/2$.
		Consider the $K\times K$ antisymmetric matrix $M$ defined by
		\begin{equation}\label{eq:def_M}
			M
			=
			\begin{pmatrix}
				\mathbf{0} & -D \\
				D^{\top} & \Lambda
			\end{pmatrix}~,
		\end{equation}
		where $D$ and $\Lambda$ are two $d\times d$ matrices specified below.
		
		The matrix $D$ is the $d\times d$ matrix with nonnegative entries such that the first row is $D_{1,\cdot}=(0,\dots,0)\in \mathbb{R}^d$, and for any $i=2,\dots,d$,
		\[
		D_{i,\cdot} = (\underbrace{\Delta_i, \dots, \Delta_i}_{s_i \text{ times}}, 0, \dots, 0) \in \mathbb{R}^d~,
		\]
		which is possible since $s_i\leqslant d$ for $i=2,\dots,d$.
		
		To construct $\Lambda$, recall that $d$ is assumed to be a multiple of $4$ and that we assumed $s_i\in\{1,\dots,d/4\}$ for all $i\in\{d+1,\dots,K\}$. Define $\Lambda$ as the following block matrix:
		\[
		\Lambda =
		\begin{pmatrix}
			0 & -\Lambda^{(0)} & 0 & (\Lambda^{(3)})^{\top} \\
			(\Lambda^{(0)})^{\top} & 0 & -\Lambda^{(1)} & 0 \\
			0 & (\Lambda^{(1)})^{\top} & 0 & -\Lambda^{(2)} \\
			-\Lambda^{(3)} & 0 & (\Lambda^{(2)})^{\top} & 0
		\end{pmatrix},
		\]
		where, for $l\in\{0,\dots,3\}$, the sub-matrix $\Lambda^{(l)}$ is the $d/4\times d/4$ matrix such that, for $i\in \{1,\dots,d/4\}$, the $i$-th row of $\Lambda^{(l)}$ is
		\begin{equation}\label{def:Lambda}
			\Lambda^{(l)}_{i,\cdot} = (\underbrace{\Delta_j, \dots, \Delta_j}_{s_j \text{ times}}, 0, \dots, 0) \in \mathbb{R}^{d/4}
			\quad \text{with } j=d+\frac{d}{4}l+i ~.
		\end{equation}
		
		Overall, $M$ is clearly antisymmetric by construction. Moreover, for each arm $i=2,\dots,K$, the $i$-th row of $M$ contains exactly $s_i$ negative entries of magnitude $\Delta_i$. The first row is equal to $0$, so that $M \in \mathbb{D}_{\text{wcw}}$ (see \eqref{def:general_cw}) and the (weak) Condorcet winner is $i^* = 1$. Finally, since in each row the negative entries are constant, we have $s_{M}^*=(s_1,\dots,s_K)$ and $M\in\mathbb{D}^{(3)}(\underline{\Delta}, \underline{s})$.
		
		We use the same permutation construction as in the proof of Theorem~\ref{thm:LB_hp_instance}. We recall this construction here. Let $\Pi$ be the set of permutations, where $\pi=(\pi_1,\dots,\pi_d)\in \Pi$ if $\pi_i$ is a permutation of $\{1,\dots,d\}$ for any $i\in[d]$.
		
		From any $\pi\in \Pi$, define $M^{\pi}$ as the matrix obtained by permuting the $d$ first columns of $D$ according to $\pi$ in the following way:
		\begin{equation}\label{eq:def_M_pi}
			M^{\pi}
			=
			\begin{pmatrix}
				\mathbf{0} & -D^{\pi} \\
				(D^{\pi})^{\top} & \Lambda
			\end{pmatrix}~,
		\end{equation}
		where, for any $(i,j)\in [d]^2$,
		\[
		D^{\pi}_{i,j} = D_{i,\pi_i(j)} ~.
		\]
		
		By construction, for any $\pi\in \Pi$, we still have $M^{\pi}\in\mathbb{D}^{(3)}(\underline{\Delta}, \underline{s})$.
		We denote by $\mathbb{P}_{A}^{\pi}$ the distribution of the observations induced by the interaction between algorithm $A$ and the environment with gap matrix $M^{\pi}$.
		
		\noindent\underline{Alternative instance $M^{(\pi,k)}$.}
		Fix a suboptimal arm $k \in \{2,\dots,d\}$.
		
		Construct the gap matrix $M^{(\pi,k)}$ by setting to zero all entries in the $k$-th row and the $k$-th column of $M^{\pi}$. By construction, rows $1$ and $k$ of $M^{(\pi,k)}$ only contain zero entries, so that $M^{(\pi,k)}$ does not admit a unique weak Condorcet winner. We denote by $\mathbb{P}_{A}^{(\pi,k)}$ the distribution of the observations induced by the interaction between algorithm $A$ and the environment with gap matrix $M^{(\pi,k)}$.
		
		Also define $\widetilde{M}^{\pi}$ as the matrix obtained by simultaneously exchanging rows $1$ and $k$ and columns $1$ and $k$ of $M^{\pi}$. Then $\widetilde{M}^{\pi}\in\mathbb{D}^{(3)}(\underline{\Delta}, \underline{s})$ and $i^*(\widetilde{M}^{\pi})=k$. We denote its induced distribution by $\widetilde{\mathbb{P}}_{A}^{\pi}$.
		Since rows and columns $1$ and $k$ of $M^{(\pi,k)}$ are zero, $M^{(\pi,k)}$ is unchanged by this exchange.
		
		\noindent\underline{Step 2: information-theoretic arguments.}
		
		Consider the recommendation rule $\hat{i}$ and the budget $T$ of algorithm $A$. By definition of $\epsilon_T$, $A$ can be considered as $\epsilon_T$-correct over $\mathbb{D}^{(3)}(\underline{\Delta}, \underline{s})$.
		
		Consider the event
		\[
		B \coloneqq \{\hat{i}\ne 1\}~.
		\]
		We have
		\begin{equation}\label{eq:H2_minimax_tv1}
			\frac{1}{|\Pi|}\sum_{\pi\in\Pi}
			\left[
			\mathbb{P}_{A}^{(\pi,k)}(B)
			+
			\mathbb{P}_{A}^{(\pi,k)}(B^c)
			\right]
			=1~.
		\end{equation}
		
		For any $\pi$, we have $i^*(M^{\pi})=1$ and $i^*(\widetilde{M}^{\pi})=k$. We can then use the fact that $A$ is $\epsilon_T$-correct over this class, by definition of $\epsilon_T$ (\eqref{def:epsilon_T_D2}), to get
		\begin{equation}\label{eq:H2_minimax_tv3}
			\mathbb{P}_{A}^{\pi}(B)\leqslant\epsilon_T
			\quad\text{and}\quad
			\widetilde{\mathbb{P}}_{A}^{\pi}(B^c)
			=
			\widetilde{\mathbb{P}}_{A}^{\pi}(\hat{i}=1)
			\leqslant
			\widetilde{\mathbb{P}}_{A}^{\pi}(\hat{i}\ne k)
			\leqslant\epsilon_T
			\enspace.
		\end{equation}
		
		Now we use Lemma~\ref{lem:fano-type}, a Fano-type inequality presented as Proposition~4 in \cite{gerchinovitz2020fano}, to obtain
		\[
		\mathbb{P}_{A}^{(\pi,k)}(B)
		\leqslant
		\frac{\mathrm{KL}\bigl(\mathbb{P}_{A}^{(\pi,k)},\mathbb{P}_{A}^{\pi}\bigr)+\log(2)}
		{\log(1/\epsilon_T)}
		\quad\text{and}\quad
		\mathbb{P}_{A}^{(\pi,k)}(B^c)
		\leqslant
		\frac{\mathrm{KL}\bigl(\mathbb{P}_{A}^{(\pi,k)},\widetilde{\mathbb{P}}_{A}^{\pi}\bigr)+\log(2)}
		{\log(1/\epsilon_T)}
		~.
		\]
		Adding these inequalities, averaging over $\pi\in\Pi$, and using \eqref{eq:H2_minimax_tv1}, we get
		\[
		1
		\leqslant
		\frac{
			\frac{1}{|\Pi|}\sum_{\pi\in\Pi}
			\left[
			\mathrm{KL}\bigl(\mathbb{P}_{A}^{(\pi,k)},\mathbb{P}_{A}^{\pi}\bigr)
			+
			\mathrm{KL}\bigl(\mathbb{P}_{A}^{(\pi,k)},\widetilde{\mathbb{P}}_{A}^{\pi}\bigr)
			\right]
			+\log(4)
		}
		{\log(1/\epsilon_T)}
		~.
		\]
		
		Observe that $M^{(\pi,k)}$ does not depend on the permutation $\pi_k$. Denote as $\pi^{(-k)}$ the vector of permutations obtained from $\pi=(\pi_1,\dots,\pi_d)\in\Pi$ by removing the $k$-th component, that is,
		$\pi^{(-k)}=(\pi_1,\dots,\pi_{k-1},\pi_{k+1},\dots,\pi_d)$.
		Denote as $\Pi^{(-k)}$ the family $\{\pi^{(-k)}\}_{\pi\in\Pi}$.
		For a fixed $\pi^{(-k)}\in\Pi^{(-k)}$, we have $M^{(\pi,k)}=M^{(\pi^{(-k)},k)}$. Then, we write the inequality above as
		\begin{equation}\label{eq:H2_minimax_fano}
			\frac{1}{|\Pi^{(-k)}|}\sum_{\pi^{(-k)}\in\Pi^{(-k)}}\frac{1}{|\mathfrak{S}_d|}
			\sum_{\pi'_k\in\mathfrak{S}_d}
			\left[\mathrm{KL}\bigl(\mathbb{P}_{A}^{(\pi^{(-k)},k)},\mathbb{P}_{A}^{\pi'}\bigr)+
			\mathrm{KL}\bigl(\mathbb{P}_{A}^{(\pi^{(-k)},k)},\widetilde{\mathbb{P}}_{A}^{\pi'}\bigr)\right]+\log(4) \geqslant \log(1/\epsilon_T)~,
		\end{equation}
		where inside the sum, $\pi'=(\pi_1,\dots,\pi_{k-1},\pi'_k,\pi_{k+1},\dots,\pi_d)$.
		
		\noindent\underline{Step 3: computing the KL divergence.}
		We now bound, for a fixed $\pi^{(-k)}\in\Pi^{(-k)}$,
		\[
		\frac{1}{|\mathfrak{S}_d|}
		\sum_{\pi'_k\in\mathfrak{S}_d}
		\left[
		\mathrm{KL}\bigl(\mathbb{P}_{A}^{(\pi^{(-k)},k)},\mathbb{P}_{A}^{\pi'}\bigr)
		+
		\mathrm{KL}\bigl(\mathbb{P}_{A}^{(\pi^{(-k)},k)},\widetilde{\mathbb{P}}_{A}^{\pi'}\bigr)
		\right].
		\]
		
		The matrices $M^{\pi'}$ and $M^{(\pi^{(-k)},k)}$ differ only on pairs $\{k,j\}$ with $j\in\{d+1,\dots,K\}$, whereas $\widetilde{M}^{\pi'}$ and $M^{(\pi^{(-k)},k)}$ differ only on pairs $\{1,j\}$ with $j\in\{d+1,\dots,K\}$.
		Using the same computation as in Equation~\eqref{eq:H2_KL2}, and the fact that the law under $M^{(\pi^{(-k)},k)}$ does not depend on $\pi'_k$, we obtain
		\begin{align}
			&\frac{1}{|\mathfrak{S}_d|}
			\sum_{\pi'_k\in\mathfrak{S}_d}
			\left[\mathrm{KL}\bigl(\mathbb{P}_{A}^{(\pi^{(-k)},k)},\mathbb{P}_{A}^{\pi'}\bigr)+
			\mathrm{KL}\bigl(\mathbb{P}_{A}^{(\pi^{(-k)},k)},\widetilde{\mathbb{P}}_{A}^{\pi'}\bigr)\right]\nonumber\\
			&\qquad\leqslant 8\log\!\left(\frac{4}{3}\right) \frac{\|M_{k,\cdot}\|^2}{d} \sum_{j=d+1}^{K} \mathbb{E}_{A}^{(\pi^{(-k)},k)} \left[N_{\{k,j\}}+N_{\{1,j\}}\right] \nonumber\\
			&\qquad\leqslant 8\log\!\left(\frac{4}{3}\right) \frac{\|M_{k,\cdot}\|^2}{d}\,T \enspace,
			\label{eq:H2_minimax_KL}
		\end{align}
		where $N_{\{i,j\}}$ denotes the number of duels of the unordered pair $\{i,j\}$.
		The last inequality holds because the two families of pairs are disjoint.
		By construction of $M$, we have $\|M_{k,\cdot}\|^2=s_k\Delta_k^2$.
		
		Averaging \eqref{eq:H2_minimax_KL} over $\Pi^{(-k)}$ and combining with \eqref{eq:H2_minimax_fano}, we get
		\[
		T \geqslant \frac{1}{8\log(4/3)} \frac{d}{s_k\Delta_k^2} \log\!\left(\frac{1}{4\epsilon_T}\right) = \frac{1}{16\log(4/3)} \frac{K}{s_k\Delta_k^2} \log\!\left(\frac{1}{4\epsilon_T}\right) \enspace.
		\]
		This holds for every $k\in\{2,\dots,d\}$. Taking the maximum over $k$ gives the desired bound~\eqref{eq:LB_FB_H2}.
	\end{proof}
	
	\begin{proof}[Proof of Lemma~\ref{lem:LB_FB_H0}]
		
		Assume additionally that there exist $\mu>0$ and $s\in[d]$ such that, for every $i\in\{2,\dots,d\}$, $\Delta_i=\mu$ and $s_i=s$.
		We want to prove Bound~\eqref{eq:LB_FB_H0}, that is
		\[
		\epsilon_T \geqslant \frac{1}{2}-\sqrt{128\log(4/3) \frac{s\mu^2}{K^2}T}\enspace.
		\]
		
		\noindent\underline{Step 1: reference and alternative instances.} 
		
		Consider $\bar M$ as the matrix defined by
		\begin{equation}\label{eq:def_bar_M}
			\bar M
			=
			\begin{pmatrix}
				\mathbf{0} & - \bar D \\
				\bar D^{\top} & \Lambda
			\end{pmatrix}~, 
		\end{equation}
		where $\Lambda$ is as in~\eqref{def:Lambda}, and $\bar D$ is the $d\times d$ matrix such that, for any $i=1,\dots,d$, 
		\[
		\bar D_{i,\cdot} = (\underbrace{\mu, \dots, \mu}_{s \text{ times}}, 0, \dots, 0) \in \mathbb{R}^d~.
		\]  
		
		Observe that the first $d$ rows of $\bar M$ are equal, and that $\bar{M}$ does not admit a Condorcet winner; in particular, $\bar{M}\not\in \mathbb{D}^{(3)}(\underline{\Delta}, \underline{s})$. 
		
		Again, for any $\pi\in \Pi$, define $\bar M^{\pi}$ as the matrix obtained by permuting the $d$ first rows of $\bar M$ according to $\pi_1,\dots,\pi_d$ as in~\eqref{eq:def_M_pi}. For this part of the proof, we denote by $\mathbb{P}_{A}^{\pi}$ the distribution of the observations induced by the interaction between algorithm $A$ and the environment with gap matrix $\bar M^{\pi}$.

		\noindent\underline{Construction of perturbed instances $\bar M^{(\pi,k)}$.}
		Consider any arm $k\in [d]$.  
		
		We construct $\bar M^{(\pi,k)}$ as the matrix obtained from $\bar M^{\pi}$ by setting to zero all entries in the $k$-th row and the $k$-th column. By construction, $\bar M^{(\pi,k)}\in\mathbb{D}^{(3)}(\underline{\Delta}, \underline{s})$ and $i^*(\bar M^{(\pi,k)})=k$. We denote by $\mathbb{P}_{A}^{(\pi,k)}$ the distribution of the observations induced by the interaction between algorithm $A$ and the environment with gap matrix $\bar M^{(\pi,k)}$.
		
		\noindent\underline{Step 2: bound on the total variation distance.}
		
		Consider the recommendation rule $\hat{i}$ and the budget $T$ of algorithm $A$, which is $\epsilon_T$-correct over $\mathbb{D}^{(3)}(\underline{\Delta}, \underline{s})$, by definition of the maximum error $\epsilon_T$. 
		
		Intuitively, under $\bar{M}$ there is no Condorcet winner among the first $d$ arms, so algorithm $A$ cannot systematically decide in favour of a specific subset of them, and it must make a large error on at least half of these arms. Indeed, it always holds that $\{\hat{i} \not\in [|1;d/2|]\}$ or $\{\hat{i} \not\in [|d/2+1;d|]\}$. Therefore,
		\[
		\frac{1}{|\Pi|} \sum_{\pi\in \Pi}\mathbb{P}_{A}^{\pi}(\hat i \not\in [1;d/2]) \geqslant \frac{1}{2}
		\quad\text{or}\quad
		\frac{1}{|\Pi|} \sum_{\pi\in \Pi}\mathbb{P}_{A}^{\pi}(\hat{i} \not\in [d/2+1;d]) \geqslant \frac{1}{2}.
		\]
		
		Without loss of generality\footnote{in the other case, we consider $k\in[d/2+1,d]$ and run the same arguments} we assume that 
		\begin{equation}\label{eq:H0_minimax_tv1}
			\frac{1}{|\Pi|} \sum_{\pi\in \Pi}\mathbb{P}_{A}^{\pi}(\hat{i} \not\in [1;d/2]) \geqslant \frac{1}{2} ~.
		\end{equation}
		
		For any $k\in[d/2]$, consider the event 
		\begin{equation*}
			B_k \coloneqq \bigl\{\hat{i}=k \bigr\}\cup\bigl\{N_{\{k,\cdot\}}>\tfrac{16T}{K}\bigr\}~, 
		\end{equation*}
		$N_{\{k,\cdot\}}$ denotes the number of duels involving arm $k$ and an adversary in $[d+1;K]$ between time $t=1$ and time $T$, that is, 
		\[
		N_{\{k,\cdot\}}
		= |\Bigl\{ t\in [T] \;:\; \exists j\in [d+1;K] \text{ with } \{I_t,J_t\}=\{k,j\}\Bigr\}| \enspace.
		\]
		
		Observe first that, for any fixed $\pi$, $\mathbb{P}_{A}^{\pi}$ does not depend on $k$, so that
		\[
		\frac{1}{d/2} \sum_{k=1}^{d/2}\mathbb{P}_{A}^{\pi}(\hat{i}=k )
		=\frac{1}{d/2}\,\mathbb{P}_{A}^{\pi}(\hat{i} \in [1;d/2])  ~.
		\]
		Averaging over $\pi\in\Pi$ and using \eqref{eq:H0_minimax_tv1}, we obtain
		\begin{equation}\label{eq:H0_minimax_tv2}
			\frac{1}{d/2} \sum_{k=1}^{d/2}\frac{1}{|\Pi|} \sum_{\pi\in \Pi}\mathbb{P}_{A}^{\pi}(\hat{i}=k)
			\leqslant \frac{1}{d} ~.
		\end{equation}
		
		Now, by definition, the family $\bigl(N_{\{k,\cdot\}}\bigr)_{k\in[d/2]}$ counts duels with pairwise disjoint sets of arms, for time-steps between $1$ and $T$, so that 
		\[
		\sum_{k=1}^{d/2} N_{\{k,\cdot\}} \leqslant T \enspace.
		\]
		From this upper bound, a simple counting argument implies that at most a fraction $1/4$ of the arms in $[d/2]$ can satisfy $N_{\{k,\cdot\}}>\frac{16T}{K} = \frac{4T}{d/2}$. Hence,
		\[
		\frac{1}{d/2} \sum_{k=1}^{d/2} \mathds{1}_{\{N_{\{k,\cdot\}}>\frac{16T}{K}\}} \leqslant \frac{1}{4} \enspace.
		\]
		Taking expectation with respect to the probability $\frac{1}{|\Pi|}\sum_{\pi\in \Pi} \mathbb{P}_A^{\pi}$, we obtain 
		\begin{equation}\label{eq:H0_minimax_tv3}
			\frac{1}{d/2} \sum_{k=1}^{d/2} \frac{1}{|\Pi|}\sum_{\pi\in \Pi} \mathbb{P}_A^{\pi}\Bigl(N_{\{k,\cdot\}}>\tfrac{16T}{K}\Bigr)
			\leqslant \frac{1}{4} \enspace.
		\end{equation}
		
		Combining \eqref{eq:H0_minimax_tv2} and \eqref{eq:H0_minimax_tv3}, and using $d\geqslant 4$, we get 
		\begin{equation}\label{eq:H0_minimax_tv4}
			\frac{1}{d/2} \sum_{k=1}^{d/2} \frac{1}{|\Pi|}\sum_{\pi\in \Pi} \mathbb{P}_A^{\pi}(B_k)
			\leqslant \frac{1}{4}+\frac{1}{d}
			\leqslant \frac{1}{2} \enspace.
		\end{equation}
		
		Now, consider $B_k$ under $\mathbb{P}_A^{(\pi,k)}$. Observe that $B_k^c\subset \{\hat{i} \ne k\}$. The environment $\bar M^{(\pi,k)}$ admits $k$ as Condorcet winner and belongs to $\mathbb{D}^{(3)}(\underline{\Delta}, \underline{s})$. Using that $A$ is $\epsilon_T$-correct over this class, by definition of $\epsilon_T$, we obtain, for any $\pi\in\Pi$, 
		\begin{equation}\label{eq:H0_minimax_tv5}
			\mathbb{P}_A^{(\pi,k)}(B_k^c)
			\leqslant  \mathbb{P}_A^{(\pi,k)}(\hat{i} \ne k)
			\leqslant \epsilon_T \,.
		\end{equation}
		The event $B_k$ can be computed by a procedure $\widetilde A_k$ using at most $16T/K$ duels between arm $k$ and opponents in $\{d+1,\dots,K\}$.
		For $t=1,\dots,T$, let $A$ choose its next duel. If this duel is between $k$ and an opponent in $\{d+1,\dots,K\}$ and would make the count $N_{\{k,\cdot\}}$ exceed $16T/K$, stop before querying that duel and return $\psi_k=1$. Otherwise, perform the duel and continue.
		If the procedure has not stopped after $T$ duels, compute $\hat{i}_T$ and output $\psi_k:=\mathbf{1}_{\{\hat{i}_T=k\}}$.

		By construction, the decision $\psi_k$ produced by $\tilde{A}_k$ satisfies $\psi_k = \mathds{1}_{B_k}$. Moreover, for any environment $\nu$, we have $\mathbb{P}_{\tilde A_k,\nu}(B_k)=\mathbb{P}_{A,\nu}(B_k)$. From these observations, we deduce
		\begin{align*}
			&\mathrm{TV}\left(\frac{1}{d/2}\sum_{k=1}^{d/2}\frac{1}{|\Pi|}\sum_{\pi\in \Pi} \mathbb{P}_{\tilde{A}_k}^{(\pi,k)},\,
			\frac{1}{d/2}\sum_{k=1}^{d/2}\frac{1}{|\Pi|}\sum_{\pi\in \Pi} \mathbb{P}_{\tilde{A}_k}^{\pi}\right) \\
			&\qquad \geqslant \frac{1}{d/2}\sum_{k=1}^{d/2}\frac{1}{|\Pi|}\sum_{\pi\in \Pi} \mathbb{P}_{\tilde{A}_k}^{(\pi,k)}(B_k)
			-\frac{1}{d/2}\sum_{k=1}^{d/2}\frac{1}{|\Pi|}\sum_{\pi\in \Pi} \mathbb{P}_{\tilde{A}_k}^{\pi}(B_k) \\
			&\qquad = \frac{1}{d/2}\sum_{k=1}^{d/2}\frac{1}{|\Pi|}\sum_{\pi\in \Pi} \mathbb{P}_{A}^{(\pi,k)}(B_k)
			-\frac{1}{d/2}\sum_{k=1}^{d/2}\frac{1}{|\Pi|}\sum_{\pi\in \Pi} \mathbb{P}_{A}^{\pi}(B_k)~. 
		\end{align*}
		Using \eqref{eq:H0_minimax_tv4} and \eqref{eq:H0_minimax_tv5}, we obtain
		\begin{align}
			\mathrm{TV}\left(\frac{1}{d/2}\sum_{k=1}^{d/2}\frac{1}{|\Pi|}\sum_{\pi\in \Pi} \mathbb{P}_{\tilde{A}}^{(\pi,k)},\,
			\frac{1}{d/2}\sum_{k=1}^{d/2}\frac{1}{|\Pi|}\sum_{\pi\in \Pi} \mathbb{P}_{\tilde{A}}^{\pi}\right)
			\geqslant 1-\epsilon_T-\frac{1}{2}
			= \frac{1}{2}-\epsilon_T  \label{eq:H0_minimax_TV6}
		\end{align}
		
		Finally, using the convexity of the total variation distance together with Pinsker's inequality and \eqref{eq:H0_minimax_TV6}, we get
		\begin{align}
			\frac{1}{2}-\epsilon_T
			&\leqslant \mathrm{TV}\left(\frac{1}{d/2}\sum_{k=1}^{d/2}\frac{1}{|\Pi|}\sum_{\pi\in \Pi} \mathbb{P}_{\tilde{A}_k}^{(\pi,k)},\,
			\frac{1}{d/2}\sum_{k=1}^{d/2}\frac{1}{|\Pi|}\sum_{\pi\in \Pi} \mathbb{P}_{\tilde{A}_k}^{\pi}\right) \nonumber\\
			&\leqslant \sqrt{\frac{1}{2}\frac{1}{d/2}\sum_{k=1}^{d/2}\frac{1}{|\Pi|}\sum_{\pi\in \Pi}
				\mathrm{KL}\bigl(\mathbb{P}_{\tilde{A}_k}^{(\pi,k)},\mathbb{P}_{\tilde{A}_k}^{\pi}\bigr)} \label{eq:H0_minimax_Pinsker}~.
		\end{align}
		
		\noindent\underline{Step 3: computing the KL divergence.}  
		
		An important property of procedure $\tilde A_k$ is that the budget spent on duels with arm $k$ is upper bounded by $16T/K$; precisely, for any environment $\nu$, 
		\begin{equation}\label{eq:H0_minimax_budget}
			\mathbb{E}_{\tilde A_k,\nu}\!\left[ \sum_{i=d+1}^K N_{\{k,i\}}\right]\leqslant \frac{16T}{K} \enspace.
		\end{equation}
		
		We now upper bound 
		\[
		\frac{1}{|\Pi|}\sum_{\pi\in \Pi}\mathrm{KL}\bigl(\mathbb{P}_{\tilde{A}_k}^{(\pi,k)},\mathbb{P}_{\tilde{A}_k}^{\pi}\bigr).
		\]
		As in previous proofs, we fix $\pi_1,\ldots, \pi_{k-1},\pi_{k+1},\ldots$ and we average over the $\pi_k$'s. Hence, for any $k\in[d/2]$, we get 
		\begin{align*}
			\frac{1}{|\mathfrak{S}_d|}\sum_{\pi_k \in \mathfrak{S}_d}
			\mathrm{KL}\bigl(\mathbb{P}_{\tilde{A}_k}^{(\pi,k)}, \mathbb{P}_{\tilde{A}_k}^{\pi}\bigr)
			&\leqslant 8\log\!\left(\frac{4}{3}\right)\,\frac{\|\bar M_{k,\cdot}\|^2}{d}\;
			\mathbb{E}_{\tilde A_k}^{(\pi,k)}\!\left[ \sum_{i=d+1}^K N_{\{k,i\}}\right] \\
			&\leqslant 8\log\!\left(\frac{4}{3}\right)\,\frac{s\mu^2}{d}\,\frac{16T}{K}
			\enspace,
		\end{align*}
		where we use $\|\bar M_{k,\cdot}\|^2=s\mu^2$ and Equation~\eqref{eq:H0_minimax_budget}. 
		
		Gathering Equation~\eqref{eq:H0_minimax_Pinsker} with the bound above, and rearranging (using $d=K/2$), we obtain 
		\[
		\epsilon_T \geqslant \frac{1}{2}-\sqrt{128\log(4/3) \frac{s\mu^2}{K^2}T}\enspace,
		\]
		which is exactly the desired bound \eqref{eq:LB_FB_H0}.
	\end{proof}
	
	\begin{proof}[Proof of Lemma~\ref{lem:LB_FB_H3}]
		
		Consider again the constant case where there exist $\mu>0$ and $s\in[d]$ such that, for every $i\in\{2,\dots,d\}$, $\Delta_i=\mu$ and $s_i=s$. We want to prove the following bound, equivalent to~\eqref{eq:LB_FB_H3}:
		\[
		T \geqslant \frac{1}{32\log(4/3)}\frac{K}{\mu^2}\log\!\left(\frac{1}{4\epsilon_T}\right)\,.
		\]
		We may assume that $\epsilon_T<1/4$, since otherwise the bound is immediate.
		
		\noindent\underline{Step 1.}
		Take the matrix $\bar{M}$ defined in \eqref{eq:def_bar_M}.
		
		Fix for now $k\in\{1,\dots,d\}$. Consider $\bar M^{(k)}$, the matrix obtained from $\bar{M}$ by setting to zero the $k$-th row and the $k$-th column of $\bar M$.
		
		Recall that $\bar{M}$ does not admit a Condorcet winner, while $\bar M^{(k)}\in \mathbb{D}^{(3)}(\underline{\Delta},\underline{s})$ with $i^*(\bar M^{(k)})=k$. We denote by $\mathbb{P}_{A}$ (resp.\ $\mathbb{P}^{(k)}_{A}$) the distribution of the observations induced by the interaction between algorithm $A$ and the environment with gap matrix $\bar M$ (resp.\ $\bar M^{(k)}$).
		
		\noindent\underline{Step 2: bound on the total variation distance.}
		
		For each $k\in[d]$, consider the event
		\[
		B_k\coloneqq \{\hat{i}=k\}~.
		\]
		Since these events are pairwise disjoint, we have
		\begin{equation}\label{eq:H3_minimax_tv1}
			\frac{1}{d}\sum_{k=1}^{d}\mathbb{P}_{A}(B_k) =
			\frac{1}{d}\mathbb{P}_{A}(\hat{i}\in[d]) \leqslant \frac{1}{d} \leqslant \frac{1}{2}
			~.
		\end{equation}
		
		Now, consider $B_k$ under $\mathbb{P}_A^{(k)}$. By definition of the maximum error $\epsilon_T$, $A$ is $\epsilon_T$-correct over $\mathbb{D}^{(3)}$, so that
		\begin{equation}\label{eq:H3_minimax_tv2}
			\mathbb{P}_A^{(k)}(B_k^c) = \mathbb{P}_A^{(k)}(\hat{i}\ne k) \leqslant \epsilon_T
			\,.
		\end{equation}
		
		By the Fano-type inequality from Lemma~\ref{lem:fano-type}, it holds that
		\[
		\mathbb{P}_{A}(B_k^c) \leqslant \frac{\mathrm{KL}(\mathbb{P}_{A},\mathbb{P}^{(k)}_{A})+\log(2)}{-\log(\mathbb{P}^{(k)}_{A}(B_k^c))}
		~.
		\]
		Averaging over $k\in[d]$ and using \eqref{eq:H3_minimax_tv1} and \eqref{eq:H3_minimax_tv2}, we obtain
		\begin{align}
			\frac{1}{2}
			&\leqslant \frac{1}{d}\sum_{k=1}^{d}\mathbb{P}_{A}(B_k^c) \nonumber\\
			&\leqslant \frac{\frac{1}{d}\sum_{k=1}^{d} \mathrm{KL}(\mathbb{P}_{A},\mathbb{P}^{(k)}_{A}) +\log(2)}
			{-\log(\epsilon_T)}~.
			\label{eq:H3_minimax_fano}
		\end{align}
		
		\noindent\underline{Step 3: computing the KL divergence.}
		
		We now upper bound $\mathrm{KL}\left(\mathbb{P}_{A},\mathbb{P}_{A}^{(k)}\right)$.
		
		From the decomposition of the KL divergence and the definition of $\bar M^{(k)}$, we have
		\begin{align}
			\mathrm{KL}\left(\mathbb{P}_{A},\mathbb{P}_{A}^{(k)}\right)
			&= \sum_{i=d+1}^{K}\mathbb{E}_{A}[N_{\{k,i\}}]\,
			\mathrm{kl}\Bigl(\tfrac{1}{2}+\bar{M}_{k,i},\tfrac{1}{2}\Bigr)
			\nonumber\\
			&\leqslant \sum_{i=d+1}^{K}\mathbb{E}_{A}[N_{\{k,i\}}]\, 8\log(4/3)\,\mu^2
			~,
			\label{eq:H3_minimax_KL}
		\end{align}
		where we use the fact that the row $\bar{M}_{k,\cdot}$ only takes values in $\{0,-\mu\}$, together with
		\[
		\mathrm{kl}\Bigl(\tfrac{1}{2}-\mu,\tfrac{1}{2}\Bigr) \leqslant
		-\frac{1}{2}\log(1-4\mu^2) \leqslant 8\log(4/3)\,\mu^2
		~.
		\]
		
		Combining \eqref{eq:H3_minimax_fano} with \eqref{eq:H3_minimax_KL} and rearranging, we obtain
		\begin{align*}
			\frac{1}{16\log(4/3)}\,\frac{d}{\mu^2}\, \log\!\left(\frac{1}{4\epsilon_T}\right)
			&\leqslant
			\frac{1}{8\log(4/3)}\,\frac{1}{\mu^2} \sum_{k=1}^{d} \mathrm{KL}\left(\mathbb{P}_{A},\mathbb{P}_{A}^{(k)}\right)
			\\
			&\leqslant
			\sum_{k=1}^{d}\sum_{i=d+1}^{K} \mathbb{E}_{A}[N_{\{k,i\}}] \leqslant T
			~,
		\end{align*}
		where the last inequality holds because each comparison is counted at most once.
		Since $d=K/2$, this gives
		\[
		T \geqslant \frac{1}{32\log(4/3)}\,\frac{K}{\mu^2}\, \log\!\left(\frac{1}{4\epsilon_T}\right)
		\,,
		\]
		which is the desired bound~\eqref{eq:LB_FB_H3}.
	\end{proof}

	\section{Technical Results}\label{sec:tech}

	\subsection{Deterministic bounds}
	
	\begin{lemma}[Section 6.1 of \cite{audibert2010best}]\label{lem:tech1}
		Let $x_1, \dots, x_K$ denote a decreasing sequence of positive numbers. We have:
		\[
		\max_{k\in \{1, \dots, K\}} kx_k^2 \le \sum_{i=1}^{K} x_i^2 \le \log(4K)\max_{k\in \{1, \dots, K\}} kx_k^2.
		\]
	\end{lemma}
	
	\begin{lemma}\label{lem:T}
		Let $x_1, \dots x_n$ denote a sequence of positive numbers such that $x_1 \le \dots \le x_n$.  
		Then we have for any $p \in (0,1)$
		\[
		\frac{pn}{x_{\left(\ceil*{pn} \right)}} \le \sum_{i=1}^{n}\frac{1}{x_i}~.
		\]
	\end{lemma}
	\begin{proof}\label{lem:card}
		We have	$\frac{1}{x_n} \le \dots \le \frac{1}{x_1}$. Therefore
		\begin{align*}
			\frac{pn}{x_{\ceil*{pn}}} &\le \frac{\ceil*{pn}}{x_{\ceil*{pn}}} \le \sum_{i=1}^{\ceil*{pn}} \frac{1}{x_i} \le \sum_{i=1}^{n} \frac{1}{x_i}~.
		\end{align*}
	\end{proof}

	\begin{lemma}\label{lem:sum_tech}
		Let $\mathfrak{S}_d$ denote the set of permutations of $\{1, \dots, d\}$. Let $a_1, \dots, a_d$ and $b_1, \dots, b_d$ be two sequences of numbers. We have
		\[
		\frac{1}{d!}\sum_{\sigma \in \mathfrak{S}_d} \sum_{i=1}^{d} a_i b_{\sigma(i)} = \frac{1}{d}\left(\sum_{i=1}^d a_i\right)\left(\sum_{i=1}^d b_i\right)~.
		\]
	\end{lemma}
	\begin{proof}
		The result is just a consequence of summation manipulation. We have
		\begin{align*}
			\sum_{\sigma}\sum_{i=1}^d a_i b_{\sigma(i)} &= \sum_{i=1}^d a_i\sum_{\sigma} b_{\sigma(i)}
			= \sum_{i=1}^d a_i\frac{d!}{d}\sum_{i=1}^d b_i
			= (d-1)! \left(\sum_{i=1}^d a_i\right)\left(\sum_{i=1}^d b_i\right)~. 
		\end{align*}
	\end{proof}
	
	\noindent Below we present a useful Fano-type inequality presented as Proposition 4 in \cite{gerchinovitz2020fano}
	\begin{lemma}\label{lem:fano-type}
		Let $\mathbb{P}$ and $\mathbb{Q}$ be two probability distributions, and let $A$ be an event such that $\mathbb{Q}(A) \in (0, 1)$. We have
		\[
		\mathbb{P}(A) \le \frac{\KL(\mathbb{P}, \mathbb{Q})+\log(2)}{-\log(\mathbb{Q}(A))}~.
		\]
		More generally, for all probability pairs $\mathbb{P}_i, \mathbb{Q}_i$ and all events $A_i$, where $i \in \{1, \dots, N\}$, with $0 < \frac{1}{N} \sum_{i=1}^N \mathbb{Q}_i(A_i) < 1$, we have
		\[
		\frac{1}{N}\sum_{i=1}^N \mathbb{P}_i(A_i) \le \frac{\frac{1}{N}\sum_{i=1}^N \KL\left(\mathbb{P}_i, \mathbb{Q}_i\right)+\log(2)}{-\log\left(\frac{1}{N}\sum_{i=1}^N \mathbb{Q}_i(A_i)\right)}~.
		\] 
	\end{lemma}

	\begin{lemma}\label{lem:pure_tech}
		For any $x>0$, we have
		\[
		0 < \frac{\frac{1}{2}-e^{-x}}{\log(1-e^{-x})-\log(e^{-x})} < \frac{1}{2x}\,,
		\]
		where the middle expression is defined by continuity to equal $1/4$ at $x= \log 2$.
	\end{lemma}
	
	\begin{proof}
		Let $y=e^{-x}\in(0,1)$ and define
		\[
		h(y):=\log\!\Big(\frac{1-y}{y}\Big)=\log(1-y)-\log y, \qquad R(y):=\frac{\frac12-y}{h(y)} \quad (y\neq\tfrac12)~.
		\]
		
		\smallskip
		\noindent Let us start with a proof of the positivity of the middle expression.
		The function $h$ is strictly decreasing on $(0,1)$ and satisfies $h(\tfrac12)=0$, hence
		$\mathrm{sign}(h(y))=\mathrm{sign}(\tfrac12-y)$. Therefore $R(y)>0$ for all $y\neq \tfrac12$.
		At $y=\tfrac12$, both numerator and denominator vanish, by l'Hôpital's rule,
		\[
		\lim_{y\to 1/2} R(y) =\frac{-1}{h'(1/2)} =\frac{-1}{-\big(\frac{1}{1-y}+\frac{1}{y}\big)\big|_{y=1/2}} =\frac14~.
		\]
		Thus the middle expression is well-defined by continuity and is strictly positive for all $x>0$.
		
		\smallskip
		\noindent Now let us prove the stated upper bound. Since $x=\log(1/y)>0$, we need to show
		\[
		R(y)<\frac{1}{2\log(1/y)}~.
		\]
		If $y<\tfrac12$ (so $h(y)>0$) we need to prove that
		\[
		2\log(1/y)\Big(\tfrac12-y\Big) < h(y)~.
		\]
		which is equivalent to $2y\log(1/y)+\log(1-y) > 0$, define $g(y) \coloneq 2y\log(1/y)+\log(1-y)$, therefore we need to show that
		\[
		\forall y \in (0, 1/2),\quad g(y)>0~.
		\]
		
		\medskip
		\noindent If $y>\tfrac12$ (so $h(y)<0$), the same manipulation yields the equivalent condition $g(y)<0$.
		Hence it suffices to prove $g(y)>0$ on $(0,\tfrac12)$ and $g(y)<0$ on $(\tfrac12,1)$.
		
		\smallskip
		\noindent A direct computation shows
		\[
		g''(y)=-\frac{2}{y}-\frac{1}{(1-y)^2}<0\qquad(y\in(0,1)),
		\]
		so $g$ is strictly concave. Moreover,
		\[
		\lim_{y\downarrow 0} g(y)=0
		\quad\text{and}\quad
		g(1/2)=2\cdot\frac12\log 2+\log(1/2)=0.
		\]
		By strict concavity, this implies $g(y)>0$ for all $y\in(0,1/2)$.
		Finally,
		\[
		g'(y)=2\log(1/y)-2-\frac{1}{1-y},
		\qquad
		g'(1/2)=2\log 2-4<0.
		\]
		Since $g$ is concave, $g'$ is nonincreasing, so $g'(y)\le g'(1/2)<0$ for all $y\ge 1/2$.
		Thus $g$ is strictly decreasing on $[1/2,1)$, and hence $g(y)<0$ for all $y\in(1/2,1)$.
		
		\noindent Therefore $R(y)<1/(2\log(1/y))=1/(2x)$ for all $x>0$, concluding the proof.
	\end{proof}

	\begin{lemma}\label{lem:pure_tech3}
		Let $d$ be an integer greater than $1$. Let $M \in \mathbb{R}^{d\times d}$ such that $M$ is skew symmetric (i.e. $\forall i,j \in [d]: M_{i,j} = -M_{j,i}$). Then the number of lines of $M$ with at least $\ceil*{d/4}$ non-positive off-diagonal entries is at least $\ceil*{d/4}$.
	\end{lemma}
	\begin{proof}
		For $i \in [d]$ define
		\[
		s_i := |{\{j \in [d]:~M_{i,j} \le 0\}}|~.
		\] 
		Since the matrix $M$ is skew symmetric (i.e. $M_{i,j} = -M_{j,i}$ for any $i,j$), for  every unordered pair $\{i,j\}$ where $i\neq j$, at least one of the two quantities $M_{i,j}$ or $M_{j,i}$ is nonpositive. Therefore, the number of off-diagonal non-positive entries is at least ${d \choose 2} $. We conclude by taking into account the diagonal entries that 
		\[
		\sum_{i=1}^d s_i \ge {d \choose 2}+d = \frac{d(d+1)}{2}~.
		\]
		The conclusion follows by a simple contradiction argument.
	\end{proof}
	
	\begin{lemma}\label{lem:39}
		Let $n \ge 3$ and $M$ be an $n \times n$ skew-symmetric matrix, let $E$ denote the set of rows such that the number of non-positive entries is at least $\ceil{n/4}+1$, then the intersection of $E$ with any subset of $\{1,\dots, n\}$ of cardinality $n-\ceil{n/8}$ is non-empty.
	\end{lemma}
	\begin{proof}
		Assume $n\ge3$ and set $a:=\lceil n/8\rceil$, $b:=\lceil n/4\rceil$.
		Suppose by contradiction that there exists $S\subseteq[n]$ with $|S|=n-a$ and $S\cap E=\varnothing$.
		Then each row $i\in S$ has at most $b$ non-positive entries, hence at least $n-b$ positive entries.
		Let $P:=|\{(i,j):M_{ij}>0\}|$ be the total number of positive entries. Summing over rows in $S$ gives
		\[
		P\ \ge\ |S|\,(n-b)=(n-a)(n-b).
		\]
		On the other hand, skew-symmetry implies that for each unordered pair $\{i,j\}$ with $i\neq j$,
		at most one of $M_{ij},M_{ji}$ is positive, hence
		\[
		P\ \le\ \binom{n}{2}=\frac{n(n-1)}{2}.
		\]
		Thus $(n-a)(n-b)\le \frac{n(n-1)}{2}$. But for $n\ge5$, using $\lceil x\rceil\le x+1$,
		\[
		n-a\ge \frac{7(n-1)}{8},\qquad n-b\ge \frac{3(n-1)}{4}
		\quad\Rightarrow\quad
		(n-a)(n-b)\ge \frac{21}{32}(n-1)^2>\frac{n(n-1)}{2},
		\]
		a contradiction; and the remaining cases $n=3,4$ are checked directly:
		$(n-a)(n-b)=4>3$ and $9>6$, respectively. Hence no such $S$ exists, i.e.\ every
		$S$ with $|S|=n-\lceil n/8\rceil$ intersects $E$.
	\end{proof}

	\begin{lemma}\label{lem:variations}
		Let $x>10^3$ and $y>8$. Then
		\[
		2\,\frac{\log^2(x)\,\log^2(y)}{\log\!\bigl(xy(\log y)^5\bigr)} \;\ge\; \log(xy).
		\]
	\end{lemma}
	
	\begin{proof}
		Since $x>10^3$ and $y>8$, we have $\log x>\log(10^3)>6$ and $\log y>\log 8>2$. Hence
		\[
		\log x\,\log y-(\log x+\log y) =(\log x-1)(\log y-1)-1 >0~,
		\]
		so
		\begin{equation}\label{eq:var_ab1}
			\log x\,\log y  \ge  \log(xy)~.
		\end{equation}
		Moreover,
		\[
		\log\!\bigl(xy(\log y)^5\bigr)=\log(xy)+5\log\log y.
		\]
		Since $y>8$ implies $\log y>2>1$, we have $\log\log y\le \log y$, and thus
		\[
		\log\!\bigl(xy(\log y)^5\bigr)\le \log(xy)+5\log y=\log x+6\log y.
		\]
		Therefore,
		\begin{align*}
			2\log x\,\log y-(\log x+6\log y)&=(2\log y-1)\log x-6\log y\\
			&\ge 6(2\log y-1)-6\log y\\
			&= 	6(\log y-1)>0~,
		\end{align*}
		so
		\begin{equation}\label{eq:var_ab2}
			2\log x\,\log y \;\ge\; \log\!\bigl(xy(\log y)^5\bigr).
		\end{equation}
		Multiplying \eqref{eq:var_ab1} and \eqref{eq:var_ab2} yields
		\[
		2\log^2(x)\,\log^2(y) \;\ge\; \log(xy)\,\log\!\bigl(xy(\log y)^5\bigr).
		\]
		Since $\log\!\bigl(xy(\log y)^5\bigr)>0$, dividing by it gives the claim.
	\end{proof}

	\begin{lemma}\label{lem:absorb}
		Let $K\ge 2$, $H\ge 4$, and $T\ge 8K\log_{8/7}(K)$. Define
		\[
		p_k \coloneqq \frac{1}{18}\wedge (\log T+K)\exp\!\Bigl(-c\,\frac{T}{\log^4(K)\log(T)\,H}\Bigr).
		\]
		Assume $T\ge c_0\,H\log^6(H)$ for some numerical constant $c_0$ large enough (depending only on $c$).
		Then there exists a numerical constant $c'>0$ (e.g.\ $c'=c/2$) such that
		\[
		p_k \le \exp\!\Bigl(-c'\,\frac{T}{\log^4(K)\log(T)\,H}\Bigr).
		\]
	\end{lemma}
	
	\begin{proof}
		Let us introduce the following notation
		\[
		x \coloneqq \frac{T}{\log^4(K)\log(T)\,H}, \qquad A \coloneqq \log T+K~.
		\]
		Then $p_k\le A e^{-cx}$. Since $T\ge 8K\log_{8/7}(K)$ and $K\ge 2$, we have $A=\log T+K\le T$, so $\log A\le \log T$ and
		\[
		p_k \le \exp(-cx+\log T).
		\]
		Thus it suffices to prove $\log T\le \frac{c}{2}x$, i.e.
		\begin{equation}\label{eq:abs_goal}
			\frac{T}{(\log T)^2} \;\ge\; \frac{2H\log^4(K)}{c}.
		\end{equation}
		Since $T\ge 8K\log_{8/7}(K)>K$, we have $\log^4(K)\le (\log T)^4$. Therefore \eqref{eq:abs_goal} follows from
		\[
		\frac{T}{(\log T)^2} \;\ge\; \frac{2H}{c}(\log T)^4~,
		\]
		which is equivalent to 
		\[
		\frac{T}{(\log T)^6} \;\ge\; \frac{2H}{c}.
		\]
		Let $g(t)\coloneqq t/(\log t)^6$, which is increasing for $t\ge e^6$.
		Choose $c_0$ large enough so that $T\ge c_0H\log^6(H)\ge e^6$ for all $H\ge 4$.
		Then, with $T_0\coloneqq c_0H\log^6(H)$, we have $g(T)\ge g(T_0)$ and
		\[
		g(T_0)=\frac{c_0H\log^6(H)}{\bigl(\log(c_0H\log^6(H))\bigr)^6}.
		\]
		For $H\ge 4$, we have $\log\log H\le \log H$, so
		\begin{align*}
			\log(c_0H\log^6(H)) &= \log c_0+\log H+6\log\log H \le \log c_0+7\log H\\
			&\le \Bigl(7+\frac{\log c_0}{\log 4}\Bigr)\log H \eqqcolon C_0\log H~.
		\end{align*}
		Hence $g(T_0)\ge \frac{c_0}{C_0^6}H$. Taking $c_0$ large enough so that $\frac{c_0}{C_0^6}\ge \frac{2}{c}$
		yields $g(T)\ge g(T_0)\ge \frac{2H}{c}$, proving \eqref{eq:abs_goal}. Therefore
		$p_k\le \exp(-(c/2)x)$, i.e.\ the claim with $c'=c/2$.
	\end{proof}

	\subsection{Concentration inequalities}
	
	\noindent Below is Hoeffding concentration inequality.
	\begin{lemma}{\label{lem:H}}
		Let $X_1, \dots, X_n$ be independent random variables such that $a_i \le X_i \le b_i$ almost surely. Let $S_n = \sum_{i=1}^n X_i$. Then we have for all $t>0$:
		\[
		\mathbb{P}\left( S_n -\mathbb{E}[S_n] \le -t\right) \le \exp\left(-\frac{2t^2}{\sum_{i=1}^{n} (b_i-a_i)^2}\right).
		\]
		The right-hand side is defined to be $0$ when $\sum_{i=1}^{n} (b_i-a_i)^2=0$.
	\end{lemma}

	Below we restate two results on the concentration of the sum of independent binary random variable from \cite{buldygin2013sub}. First, let us introduce some notation. Let $\xi$ denote a sub-Gaussian random variable, its sub-Gaussian standard is defined by:
	\[
	\tau(\xi) := \inf\left\lbrace a \ge 0: \mathbb{E}\left[\exp\left(\lambda \xi\right)\right] \le \exp\left( \frac{a^2\lambda^2}{2}\right), \lambda \in \mathbb{R}  \right\rbrace.
	\]
	\begin{lemma}[Theorem 2.1 in \cite{buldygin2013sub} ]\label{thm:tech1}
		Let $X$ denote a Bernoulli random variable with parameter $p \in [0,1]$. Then we have
		\[
		\tau^2\left(X-p\right) = \phi(p),
		\] 
		where $\phi(.)$ is the function defined by:
		\[
		\phi(p) =
		\begin{cases}
			0, & p \in \{0,1\}; \\
			\frac{1}{4}, & p = \frac{1}{2}; \\
			\frac{\frac{1}{2}-p}{ \log(1-p) - \log(p)}, & p \in (0,1) \setminus \left\{ \frac{1}{2} \right\}.
		\end{cases}
		\]
	\end{lemma}
	
	The lemma below gives a concentration bound on the binomial random variables.
	\begin{lemma}\label{thm:tech2}
		Let $X_j$ for $j\in \{1, \dots, n\}$ denote a sequence of independent Bernoulli random variables with parameter $p \in [0,1]$. Define $S_n = \sum_{j=1}^{n}\left(X_j - p\right)$. Then, we have for all $x>0$
		\[
		\mathbb{P}\left( S_n \ge x\right) \le \exp\left(\frac{-x^2}{2n\phi(p)}\right),
		\]
		where $\phi$ is defined in Lemma~\ref{thm:tech1}. We also have for all $x>0$
		\[
		\mathbb{P}\left( S_n \le -x\right) \le \exp\left(\frac{-x^2}{2n\phi(p)}\right)~.
		\]
		For $p \in \{0,1\}$, both right-hand sides are defined to be $0$.
	\end{lemma}
	\begin{proof}
		This is a direct consequence of Chernoff's bound with Lemma~\ref{thm:tech1}.
	\end{proof}
	
	\section{Additional experiments}
	\label{sec:appendix_sim_extra}
	
	We report two complementary experiments. The first is a Bradley--Terry
	difficulty sweep, which tests a classical structured regime with a total order.
	The second uses unstructured random sparse matrices, to evaluate robustness
	beyond the controlled circulant model used in the main text.
	
	For the Bradley--Terry experiment, arm $0$ is the Condorcet winner. We assign
	latent scores $\theta_i=K-i$ for $i=0,\ldots,K-1$ and set
	\[
	q_{i,j}
	=
	\frac{1}{1+\exp(-\lambda_{\rm BT}(\theta_i-\theta_j))},
	\qquad
	\Delta_{i,j}=q_{i,j}-\frac12 .
	\]
	The parameter $\lambda_{\rm BT}$ controls the difficulty: smaller values produce
	smaller gaps and therefore harder instances. We fix $K=200$, $\delta=0.1$, and
	run $5$ trials for
	\[
	\lambda_{\rm BT}\in\{0.04,0.05,0.06,0.07,0.08\},
	\]
	with a cap of $7\times 10^6$ pairwise queries.
	
	\begin{figure}[H]
		\centering
		\includegraphics[width=0.6\linewidth]{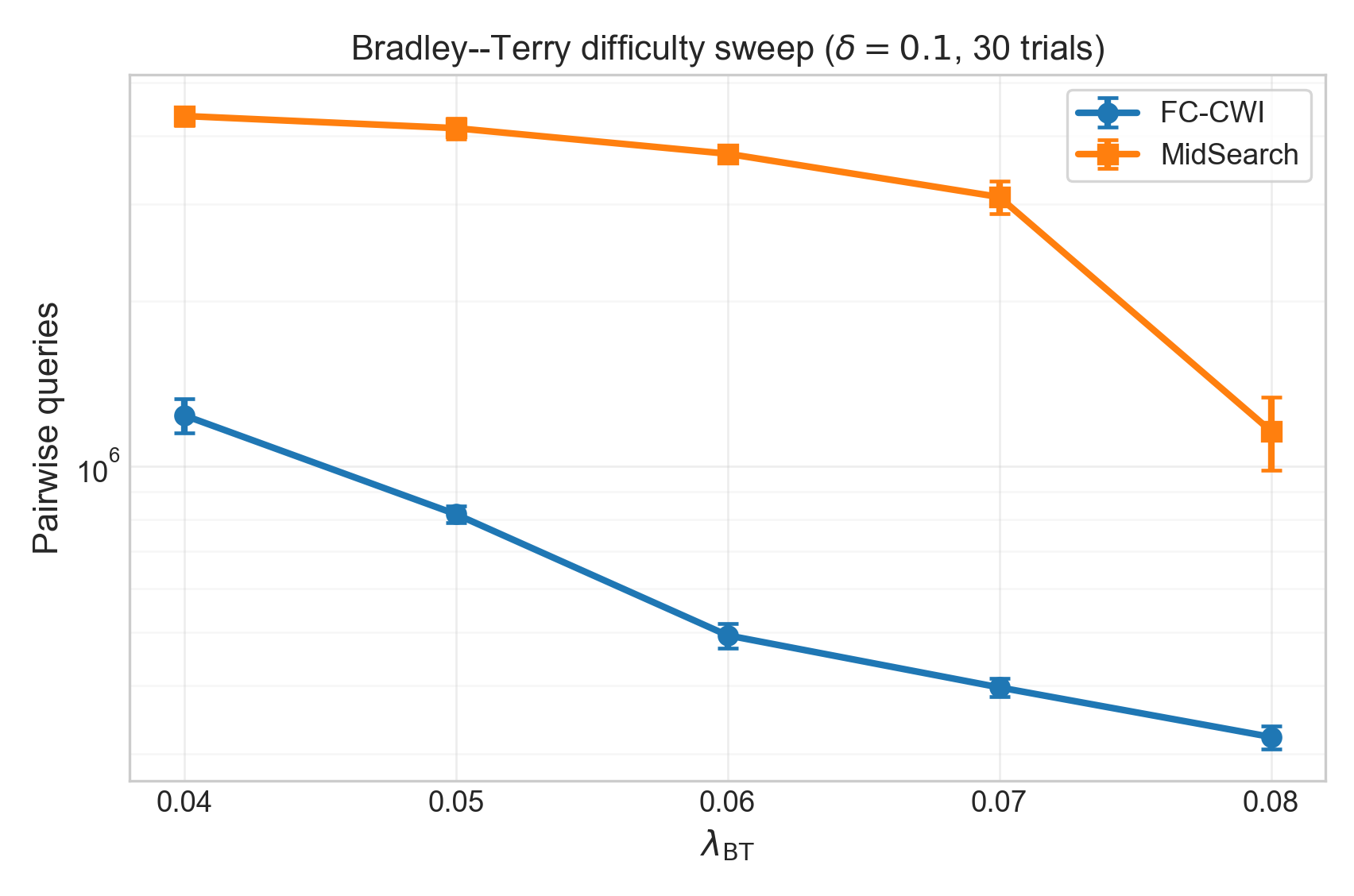}
		\caption{
			Query complexity on Bradley--Terry instances as the difficulty parameter
			$\lambda_{\rm BT}$ varies. FC-CWI remains consistently more sample-efficient
			than MidSearch across the sweep.
		}
		\label{fig:bt_appendix}
	\end{figure}
	
	Figure~\ref{fig:bt_appendix} shows that FC-CWI identifies the Condorcet winner
	in all runs and uses substantially fewer queries than MidSearch. The gap is most
	pronounced at intermediate difficulty; for instance, when
	$\lambda_{\rm BT}=0.05$, FC-CWI uses about $8.4\times 10^5$ queries on average,
	whereas MidSearch uses about $5.3\times 10^6$ queries and reaches the query cap
	once.
	
	We also evaluate random sparse gap matrices. Arm $0$ is again the Condorcet
	winner, with
	\[
	\Delta_{0,i}\sim\mathrm{Unif}(0.05,0.1),
	\qquad
	\Delta_{i,0}=-\Delta_{0,i}.
	\]
	For each pair of suboptimal arms, the entry is set to zero with probability
	$p_0$ and otherwise sampled from $\mathrm{Unif}[-\rho,\rho]$, with
	skew-symmetry imposed. We use $K=50$, $\delta=0.1$, and average over $10$ trials
	for $\rho\in\{0.1,0.2,0.3,0.4,0.499\}$.
	
	\begin{figure}[t]
		\centering
		\includegraphics[width=\linewidth]{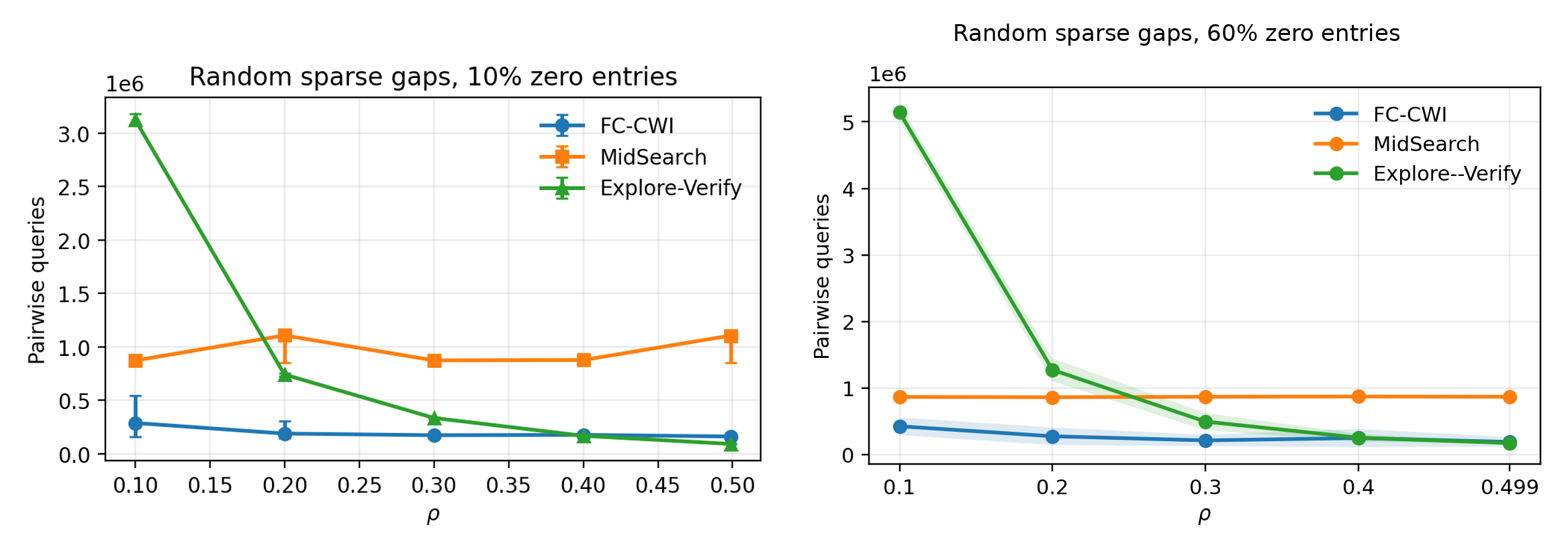}
		\caption{
			Query complexity on random sparse gap matrices. Left: $10\%$ zero entries.
			Right: $60\%$ zero entries. FC-CWI is most effective for small and
			intermediate $\rho$, while Explore--Verify improves as stronger random
			witnesses become easier to find and certify.
		}
		\label{fig:random_sparse_experiments}
	\end{figure}
	
	Overall, these experiments show that FC-CWI remains competitive in classical
	structured instances and is particularly effective when useful
	suboptimal-suboptimal comparisons exist but are sparse or heterogeneous.
	
	\paragraph{Calibration of numerical constants.}
	All three algorithms contain numerical constants that affect their empirical
	sampling budgets. Since the theoretical constants are conservative, we calibrate
	one scalar constant per method on a single representative instance and then keep
	these constants fixed in all subsequent experiments. The calibration instance is
	the sparse controlled model with
	\[
	K=128,\qquad
	\delta=0.1,\qquad
	\alpha=0.2,\qquad
	\gamma=0.05,\qquad
	\rho=0.3.
	\]
	For FC-CWI, the scalar \(c_{\mathrm{FC}}\) is the constant used in the
	elimination/certification threshold. For MidSearch, the scalar
	\(c_{\mathrm{MS}}\) rescales the sampling budgets used inside the MidSearch
	subroutines. For Explore--Verify, the scalar \(c_{\mathrm{EV}}\) rescales the
	initial budget in the doubling schedule. In all cases, the confidence radii and
	verification thresholds are left unchanged.
	
	To make the comparison insensitive to random fluctuations, we use paired trials:
	for a fixed method and trial index, all constants are evaluated on the same
	randomly permuted instance and with the same initial random seed. Each entry
	below is computed over \(10\) trials.
	
	\begin{table}[t]
		\centering
		\caption{Selected calibration constants on the sparse controlled instance.
			All selected constants identify the Condorcet winner in all \(10\) paired
			trials.}
		\label{tab:calibration-summary}
		\begin{tabular}{lcccc}
			\toprule
			Method & Selected constant & Errors / 10 & Median queries & 90\% quantile \\
			\midrule
			FC-CWI & \(0.5\) & \(0/10\) & \(1.95\!\times\!10^6\) & \(2.14\!\times\!10^6\) \\
			MidSearch & \(0.3\) & \(0/10\) & \(8.42\!\times\!10^6\) & \(8.75\!\times\!10^6\) \\
			Explore--Verify & \(4\) & \(0/10\) & \(4.00\!\times\!10^6\) & \(4.03\!\times\!10^6\) \\
			\bottomrule
		\end{tabular}
	\end{table}
	
	\begin{table}[t]
		\centering
		\caption{Robustness of the calibration. Query-complexity ranges are computed
			over local sweeps around the selected constants. All entries in the table
			have zero errors over \(10\) paired trials.}
		\label{tab:calibration-stability}
		\begin{tabular}{lccc}
			\toprule
			Method & Tested constants & Median query range & 90\% quantile range \\
			\midrule
			FC-CWI &
			\(\{0.5,0.7,0.9,1.1,1.3,1.5\}\) &
			\([1.95,2.11]\!\times\!10^6\) &
			\([2.14,2.14]\!\times\!10^6\) \\
			MidSearch &
			\(\{0.2,0.3,0.5\}\) &
			\([6.09,13.08]\!\times\!10^6\) &
			\([6.44,13.12]\!\times\!10^6\) \\
			Explore--Verify &
			\(\{1,2,4\}\) &
			\([4.00,4.02]\!\times\!10^6\) &
			\([4.03,4.05]\!\times\!10^6\) \\
			\bottomrule
		\end{tabular}
	\end{table}
	
	The stability table shows that the empirical conclusions are not driven by a
	finely tuned constant. FC-CWI remains essentially unchanged over the tested range
	of \(c_{\mathrm{FC}}\), while Explore--Verify is also stable near the selected
	constant. MidSearch is more sensitive because its scalar directly multiplies
	several internal sampling allocations, but the selected value \(c_{\mathrm{MS}}=0.3\)
	lies in a clean zero-error regime and the resulting query complexity remains
	well separated from FC-CWI on this calibration instance.

\end{document}